\pdfoutput=1
\PassOptionsToPackage{usenames,dvipsnames}{xcolor}

\documentclass[10pt]{article}
\usepackage[accepted]{tmlr}
\usepackage{tabularx}
\usepackage{hyperref}
\usepackage{url}
\hypersetup{
    colorlinks=true,
    linkcolor=blue,
    filecolor=magenta,
    urlcolor=cyan,
    citecolor=blue,
    pdftitle={To Retain or to Adapt? Generalizing Continual Learning},
    pdfauthor={Giulia Lanzillotta, Mandana Samiei, Doina Precup, Razvan Pascanu, Claire Vernade},
}
\usepackage{float}   %

\usepackage{dsfont}
\usepackage{amsmath,amsfonts,amssymb, amsthm}
\usepackage{accents}
\usepackage{bbm} %
\usepackage{bm} %
\usepackage{nicefrac}%
\usepackage{microtype} %
\usepackage{graphicx}
\graphicspath{{../../}}
\usepackage{wrapfig,lipsum,booktabs}
\usepackage{titlesec}
\usepackage[capitalize,noabbrev,nameinlink]{cleveref}
\usepackage[normalem]{ulem}
\usepackage{bbm}
\usepackage{comment}
\usepackage{subcaption} %
\usepackage{xfrac}
\usepackage{natbib}
\usepackage{wrapfig}
\usepackage{booktabs}
\usepackage[toc,page,header]{appendix}
\usepackage{minitoc}
\usepackage{todonotes}
\setuptodonotes{color=orange!30, bordercolor=gray!50, tickmarkheight=0.3cm,size=\small}
\usepackage{booktabs} %
\usepackage{etoc}
\usepackage{color, colortbl}
\usepackage{wrapfig}
\usepackage[most]{tcolorbox}
\usepackage{rotating}
\usepackage{tikz}
\usetikzlibrary{positioning, arrows.meta, shapes.geometric, calc, backgrounds, fit}

\definecolor{Gray}{gray}{0.95}

\usepackage{hyperref}

\usepackage{amsmath}
\usepackage{amssymb}
\usepackage{mathtools}
\usepackage{amsthm}

\newcommand{\rebuttalcolor}{black}   %
\newcommand{\rebuttal}[1]{{\color{\rebuttalcolor}#1}}
\newcommand{\rebuttalBcolor}{black}   %
\newcommand{\rebuttalB}[1]{{\color{\rebuttalBcolor}#1}}
\newcommand{\revcolor}{black}   %
\newcommand{\rev}[1]{{\color{\revcolor}#1}}
\newcommand{\revpcolor}{black}   %
\newcommand{\revp}[1]{{\color{\revpcolor}#1}}
\newcommand{\majorcolor}{black}   %
\newcommand{\major}[1]{{\color{\majorcolor}#1}}

\makeatletter
\newtheorem*{rep@theorem}{\rep@title}
\newcommand{\newreptheorem}[2]{%
\newenvironment{rep#1}[1]{%
 \def\rep@title{#2 \ref{##1}}%
 \begin{rep@theorem}}%
 {\end{rep@theorem}}}
\makeatother

\theoremstyle{plain}
\newtheorem{theorem}{Theorem}
\newreptheorem{theorem}{Theorem}
\newtheorem{proposition}[theorem]{Proposition}
\newtheorem{lemma}[theorem]{Lemma}
\newreptheorem{lemma}{Lemma}
\newtheorem{corollary}[theorem]{Corollary}
\newreptheorem{corollary}{Corollary}
\theoremstyle{definition}
\newtheorem{definition}[theorem]{Definition}
\newreptheorem{definition}{Definition}
\newtheorem{assumption}[theorem]{Assumption}
\newreptheorem{assumption}{Assumption}
\crefname{assumption}{Assumption}{Assumptions}
\Crefname{assumption}{Assumption}{Assumptions}
\newtheorem*{definition*}{Definition}
\newtheorem*{assumption*}{Assumption}
\theoremstyle{remark}
\newtheorem*{remark}{Remark}

\tcolorboxenvironment{definition}{
  colback=gray!10!white,
  width=\dimexpr\linewidth+10pt\relax,%
  enlarge left by=-5pt,%
  enlarge right by=-5pt,%
  boxsep=5pt,%
  left=0pt,%
  right=0pt,%
  top=0pt,%
  bottom=0pt,%
  arc=0pt,%
  boxrule=0pt,%
  colframe=white
}

\tcolorboxenvironment{definition*}{
  colback=gray!10!white,
  width=\dimexpr\linewidth+10pt\relax,%
  enlarge left by=-5pt,%
  enlarge right by=-5pt,%
  boxsep=5pt,%
  left=0pt,%
  right=0pt,%
  top=0pt,%
  bottom=0pt,%
  arc=0pt,%
  boxrule=0pt,%
  colframe=white
}

\tcolorboxenvironment{assumption}{
  colback=gray!10!white,
  width=\dimexpr\linewidth+10pt\relax,%
  enlarge left by=-5pt,%
  enlarge right by=-5pt,%
  boxsep=5pt,%
  left=0pt,%
  right=0pt,%
  top=0pt,%
  bottom=0pt,%
  arc=0pt,%
  boxrule=0pt,%
  colframe=white
}

\tcolorboxenvironment{assumption*}{
  colback=gray!10!white,
  width=\dimexpr\linewidth+10pt\relax,%
  enlarge left by=-5pt,%
  enlarge right by=-5pt,%
  boxsep=5pt,%
  left=0pt,%
  right=0pt,%
  top=0pt,%
  bottom=0pt,%
  arc=0pt,%
  boxrule=0pt,%
  colframe=white
}
\tcolorboxenvironment{remark}{
  colback=gray!10!white,
  width=\dimexpr\linewidth+10pt\relax,%
  enlarge left by=-5pt,%
  enlarge right by=-5pt,%
  boxsep=5pt,%
  left=0pt,%
  right=0pt,%
  top=0pt,%
  bottom=0pt,%
  arc=0pt,%
  boxrule=0pt,%
  colframe=white
}

\tcolorboxenvironment{theorem}{
  colback=blue!5!white,
  width=\dimexpr\linewidth+10pt\relax,%
  enlarge left by=-5pt,%
  enlarge right by=-5pt,%
  boxsep=5pt,%
  left=0pt,%
  right=0pt,%
  top=0pt,%
  bottom=0pt,%
  arc=0pt,%
  boxrule=0pt,%
  colframe=white
}

\tcolorboxenvironment{corollary}{
  colback=blue!5!white,
  width=\dimexpr\linewidth+10pt\relax,%
  enlarge left by=-5pt,%
  enlarge right by=-5pt,%
  boxsep=5pt,%
  left=0pt,%
  right=0pt,%
  top=0pt,%
  bottom=0pt,%
  arc=0pt,%
  boxrule=0pt,%
  colframe=white
}

\tcolorboxenvironment{lemma}{
  colback=blue!5!white,
  width=\dimexpr\linewidth+10pt\relax,%
  enlarge left by=-5pt,%
  enlarge right by=-5pt,%
  boxsep=5pt,%
  left=0pt,%
  right=0pt,%
  top=0pt,%
  bottom=0pt,%
  arc=0pt,%
  boxrule=0pt,%
  colframe=white
}

\title{To Retain or to Adapt? Generalizing Continual Learning}

\def\month{10}
\def\year{2026}
\def\openreview{\revp{\url{https://openreview.net/forum?id=C4zRmT9jYO}}}



\makeatletter
\renewcommand\@fnsymbol[1]{\ensuremath{\ifcase#1\or \dagger\or \ast\or
  \ddagger\or \mathsection\or \mathparagraph\or \|\else\@ctrerr\fi}}
\makeatother

\author{\name Giulia Lanzillotta\thanks{Equal contribution.}
  \email glanzillo@ethz.ch \\
  \addr ETH AI Center, Department of Computer Science, ETH Zürich
  \AND
  \name Mandana Samiei\footnotemark[1]
  \email mandana.samiei@mail.mcgill.ca \\
  \addr Mila - Quebec AI Institute, School of Computer Science, McGill University
  \AND
  \name Doina Precup \\
  \addr  Mila - Quebec AI Institute, School of Computer Science,  McGill University
  \AND
  \name Razvan Pascanu\thanks{Equal supervision.} \\
  \addr Mila - Quebec AI Institute
  \AND
  \name Claire Vernade\footnotemark[2] \\
  \addr University of Technology Nuremberg
}

\begin{document}
\doparttoc %
\faketableofcontents %

\maketitle

\begin{abstract}
The Continual Learning (CL) literature has long been driven by the goal of mitigating catastrophic forgetting. \rev{This objective rests on an often unstated assumption:} that a lifelong learner should approximate the Joint-Task Learning (JTL) solution and retain all previously acquired knowledge. We challenge this retention-centered premise, arguing that in non-stationary environments prioritizing retention can impede real-time adaptation. Shifting the focus to the \emph{Average Lifelong Error} (ALE), we formalize CL as an online optimization problem governed by the interaction between environmental and learning dynamics. We introduce \emph{Transfer Efficiency} as a quantitative measure of the tension between \emph{Instability}, the bias inherited from conflicting past experience, and \emph{Transient Error}, the optimization cost of learning new tasks from scratch. \rev{Under mild convergence conditions that hold across linear and neural network models, this decomposition yields a \emph{Critical Task Duration}: a threshold beyond which historical knowledge transitions from a warm-start advantage to an optimization liability whenever retention induces a positive stationary bias.} We validate these theoretical predictions on continual image classification and reinforcement learning benchmarks.
Finally, by connecting continual learning to the online learning framework of predictable sequences, \rev{we show that JTL is only one instance of a broader family of objectives and propose a new, general class of continual learning algorithms,} which we call \emph{Predictive Continual Learning}.
Predictive CL algorithms optimize expected future performance under an explicit, dynamically updated model of future tasks. As a proof of concept, we analyze a Window algorithm \rev{that interpolates between JTL and Independent-Task Learning (ITL) and outperforms both under controlled distributional drift.}\looseness=-1
\vspace{-0.5mm}
\end{abstract}

\section{Introduction}
\setlength{\parskip}{.5pc plus 3pt}

\rev{Real-world environments are non-stationary: they evolve over time and require systems that adapt continually and efficiently.}
\emph{Continual learning (CL)}~\citep[e.g.][]{Ring94,Thrun95,parisi2019continual,hadsell2020embracing}\rev{, also known as lifelong learning, addresses} this challenge by developing algorithms that can acquire, retain, and update knowledge over extended sequences of experience.
\rev{Unlike conventional machine learning, which assumes a fixed data distribution, continual learning operates under distributional drift, so performance must be maintained relative to a moving target.}

Progress in continual learning has historically been \rev{guided and measured by a diverse set of metrics that reflect the many desiderata of an adaptive agent:}
 \emph{forward transfer}, \emph{minimal forgetting}, \emph{preservation of plasticity}, and \emph{computational efficiency}~\citep{schwarz18a,hadsell2020embracing,mundt2023wholistic}. \rev{Having separate metrics has led to methods that focus on only one or a subset of these, at the expense of the other desiderata} \citep{wolczyk2021continual,dohareLossPlasticityDeep2024a}.
Recently, \citet{kumar2023continual} \rev{formally proved that maximizing real-time performance leads to a trade-off between these requirements, one set dynamically by the environment rather than fixed a priori.}
Their proposed \emph{Average Lifelong Error (ALE)} objective expresses continual learning as the minimization of the online expected error in a dynamic environment, subject to resource constraints.
\rev{This perspective frames forgetting as a \emph{functional necessity} for adapting with finite computation. As \citet{kumar2023continual} show, a learner with limited compute must \emph{retain only the information predictive of the future} and discard the rest to maintain adaptive capacity.}

In the real world, it is difficult to
know in advance which piece of past information will be useful in the future.
Consequently, any lifelong learning algorithm must rely on assumptions or an explicit model about the future.
We argue that algorithms designed to preserve performance on all previously encountered tasks, including most continual learning methods aimed at minimizing catastrophic forgetting, implicitly encode such a model.
\rev{Specifically, the \emph{Joint-Task Learning (JTL)} objective (the average loss over all tasks observed up to the current time) corresponds to a particular assumption about future structure: future tasks are drawn} from the same distribution as past tasks.
Equivalently, JTL is the special case of \rev{the} ALE objective under an ergodic environment, where \rev{the} statistical properties of the task stream remain stationary over time.

\rev{This assumption can be reasonable in many scenarios, e.g., when the agent gradually explores a larger but stable world. It breaks down, however, when} the environment changes \emph{irreversibly} and past data become progressively less predictive.
In such settings, optimizing the JTL objective can place excessive weight on outdated information, hindering adaptation and yielding suboptimal lifelong performance.
\vspace{-0pt}

\begin{figure}[htbp]
    \centering
    \begin{minipage}[t]{0.57\linewidth}
        \vspace{0pt}
        \centering
        \includegraphics[width=\linewidth]{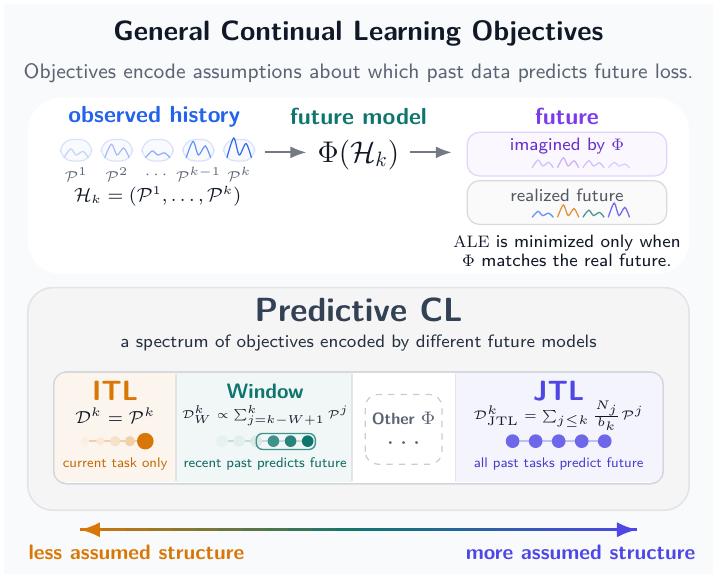}
    \end{minipage}
    \hfill
    \begin{minipage}[t]{0.39\linewidth}
        \vspace{0pt}
        \caption{\textbf{General Continual-learning objectives.} ITL and JTL are two limiting cases of a broader continuum. ITL makes the weakest structural assumption, optimizing only for the current task. JTL makes a strong predictable-structure assumption, using the accumulated past as evidence for future performance. Many existing continual-learning methods can be viewed as approximations to this JTL-like objective, because they introduce mechanisms for retaining or reconstructing past-task performance. Our generalization of CL treats the middle of the continuum as an explicit design space; we call algorithms that instantiate this generalized objective \emph{Predictive CL algorithms}.}
        \label{fig:general_cl_definition}
    \end{minipage}
    \vspace{-0.cm}
\end{figure}

To make this perspective concrete, we contrast the JTL formulation with the \emph{Independent-Task Learning (ITL)} objective, which optimizes only on the current task and \rev{retains} no information from past tasks.
This comparison highlights that neither objective is universally optimal. In rapidly changing environments, the ITL objective can achieve lower lifelong error by avoiding the implicit bias toward past distributions.
Conversely, when past experience remains predictive of the future, the JTL objective benefits from this bias to accelerate learning and improve generalization.

By analyzing these two extremes in both controlled quadratic models and deep neural networks, we characterize the conditions under which memory becomes a liability for learning. This analysis introduces three key concepts: \emph{Transfer Efficiency}, which quantifies this trade-off; \emph{Instability}, which measures the cost of bias inherited from previous tasks; and \emph{Transient Error}, which captures the cost of learning a new task from scratch.

\rev{Our framework connects several adjacent literatures (online learning, multi-task learning, and learning to learn)
and reveals a shared theoretical foundation.} It offers analytical tools to challenge and extend existing theories.
In particular, the connection to the online learning literature allows us to interpret the JTL and ITL objectives as part of a family of algorithms based on the concept of \emph{predictable sequences} \citep{rakhlinOnlineLearningPredictable2013}.

In light of this connection, we propose a \emph{generalization} of continual learning algorithms beyond the restrictive assumption of ergodicity. \rev{Instead of optimizing for perfect retention or complete adaptation, the algorithm minimizes the expected future error under a model of future tasks, which may itself be learned and updated online. We refer to algorithms that optimize this generalized objective as \emph{Predictive Continual Learning (Predictive CL)}.}
\rev{This formulation turns the space between JTL and ITL into a principled design space, in which continual learning algorithms can adapt} their retention strategy to the structure of the environment rather than to a fixed, a priori objective.

\rebuttalB{
Our results identify environments in which full retention helps and others in which it hurts. \rev{Retention is therefore not \textit{always} optimal, but this does not imply that it is \textit{never} optimal. Our claims are also specific} to purely performance-based evaluation. In cases where memory matters
beyond real-time performance, JTL rightly remains the standard.}
\vspace{3mm}%
\paragraph{Contributions.} \rev{Our main contributions are:}
\vspace{-3mm}%
\begingroup\setlength{\leftmargini}{1.5em}%
\begin{enumerate}
    \item We recast continual learning as an online optimization problem centered on Average Lifelong Error (ALE), showing that the widely adopted Joint-Task Learning (JTL) objective is not a universally optimal target in non-stationary environments, but rather corresponds to a particular assumption about future task structure.
    \par\vspace{-2mm}
    \item We formalize CL as the interaction between environment dynamics and learning dynamics,
    revealing that the utility of memory depends jointly on task distributions, optimization dynamics, and model properties.
    \par\vspace{-2mm}
    \item We introduce three analytical quantities: \emph{Transfer Efficiency}, \emph{Instability}, and \emph{Transient Error} that decompose the benefit of transferring past knowledge. This decomposition yields a \emph{Critical Task Duration} \revp{beyond} which retained knowledge transitions from accelerating learning to hindering adaptation.
    \par\vspace{-2mm}
    \item We derive closed-form expressions for these quantities in quadratic objectives, showing that the environment Growth Rate determines the asymptotic benefits of \rev{long-term} memory. We further demonstrate that the same mechanisms emerge empirically in deep linear models, continual image classification, and reinforcement learning benchmarks.
    \par\vspace{-2mm}
    \item We connect continual learning to the online learning literature, unifying JTL and ITL within a broader family of objectives. Building on this perspective, we propose \emph{Predictive CL}, a generalized framework in which algorithms optimize expected future performance under an explicit predictive model of future tasks. As a proof of concept, we analyze a Window algorithm that interpolates between ITL and JTL and outperforms both under controlled distributional drift.
\end{enumerate}
\endgroup

\pagebreak
\setlength{\parskip}{.5pc}
\section{Background \& Setup}
\label{sec:background}

\subsection{A General Objective for Continual Learning}
\rev{We first recall the continual learning objective introduced by \citet{kumar2023continual} and instantiate it} in the supervised learning setting.\footnote{We adopt a slightly different notation  from the original formulation to better align with the continual learning literature.}  We review the main concepts here and provide an in-depth exposition in the supplementary material (\cref{app:intro}).

\paragraph{General setting.}
During its lifetime, an agent accumulates a \emph{history} $H_T = \{(a_0, o_1), \dots, (a_{T-1}, o_T)\}$ of action-observation pairs in $A \times O$. At each time step $n$ the agent samples an action $a_n \sim \pi_n(\cdot \mid H_n)$ from a \emph{policy} $\pi_n$ conditioned on the history. The environment then produces a new observation $o_{n+1} \sim \mathcal{P}(\cdot \mid H_n, a_n)$, and the learner uses the new experience $(a_n, o_{n+1})$ together with the pre-existing history $H_n$ to produce a new policy estimate $\pi_{n+1}$.

In this work we model the learning step as a \emph{stochastic process} guided by an \emph{algorithm} $\mathcal{A}$, which updates the policy according to $\pi_n \sim \mathcal{A}(\cdot | H_n)$. Therefore, at each time step there are two interacting stochastic processes (depicted in \cref{fig:setting}): \emph{the learning process}, by which $\pi_{n}$ and then $a_n$ are sampled, and \emph{the environment process}, by which $o_{n+1}$ is sampled. Together, these coupled processes determine the agent's lifetime trajectory.

\begin{figure}[htbp]
    \centering
    \begin{minipage}{0.4\textwidth}
        \centering
        \includegraphics[width=\linewidth]{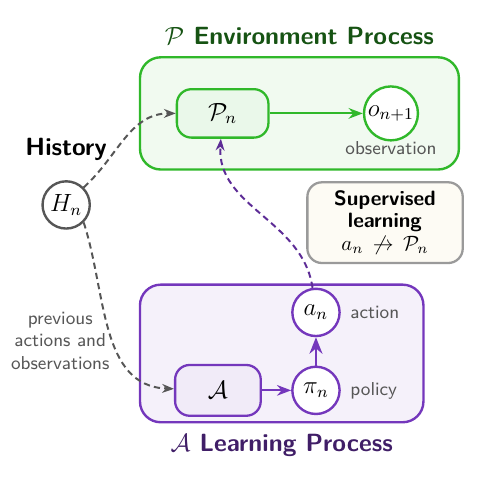}
    \end{minipage}
    \hfill
    \begin{minipage}{0.50\textwidth}
        \caption{\textbf{The environment and learning (or agent) stochastic processes}, whose interaction gives rise to the history of actions and observations $H_n = \{(a_0, o_1), \dots, (a_{n-1}, o_n)\}$. Dashed lines indicate conditioning, \rev{colored} solid lines indicate sampling. Samples of random variables are enclosed in circles, \rev{and stochastic functions in squares.} The supervised-learning inset marks the action-to-environment conditioning edge omitted in supervised learning, where the environment process is \rev{independent of the action process, but not the reverse.}\looseness=-1}
        \label{fig:setting}
    \end{minipage}
    \vspace{-1mm}
\end{figure}

The \textbf{supervised learning} problem is a special case of this general formulation in which \emph{the environment process is independent of the learning process}\rev{. The converse does not hold, since} the learner's policy still depends on past observations. Concretely, the new observation $o_{n+1}$ is independent of the action $a_n$, so $o_{n+1} \sim \mathcal{P}(\cdot \mid H_n)$.

\paragraph{Continual learning objective.}
Let $\ell(\pi, o)$ be a non-negative measure of error for a learner policy $\pi$ when given an observation $o$\footnote{\rev{Here we bypass the action sampling, which is incorporated in the definition of $\ell$.}} and denote by $\mathcal{P}_n$ the conditional distribution $\mathcal{P}(\cdot |H_n, a_n)$. Lower values of $\ell$ correspond to better performance, a convention we keep throughout the paper.

\vspace{-1mm}
\citet{kumar2023continual} define the general continual learning objective as the \emph{minimization of the real-time error}:
\begin{align}
    \label{eq:SLO}
    &\inf_\mathcal{A}\,\, \underbrace{\frac{1}{T} \sum_{n=1}^T
    \mathbb{E}_{\pi_n\sim\mathcal{A}(\cdot\mid H_{n}), \,o_{n+1}\sim\mathcal{P}_n}\,
    \left[\ell(\pi_n,o_{n+1})\right]}_{\mathrm{ALE}(\mathcal{A})} \qquad\text{s.t. computational constraints}
\end{align}
The computational constraints are intentionally left generic in the original formulation (Eq.~2 of \citet{kumar2023continual}), \rev{as their precise form depends on the specific algorithm $\mathcal{A}$.}

Two features of this formulation are central to what follows. First, the learner is evaluated on the next-step observation $o_{n+1}$\rev{. In stationary settings this is the conventional test error, but under non-stationarity the data distribution may change between successive time steps ($\mathcal{P}_n \neq \mathcal{P}_{n+1}$), so previously acquired knowledge \revp{loses} predictive value.} Second, the computational constraint on $\mathcal{A}$ \rev{induces the central trade-off of CL.} With unlimited computational resources, an agent could retain all past information  while remaining adaptive. Under realistic resource constraints, it may be better to forget past data once they cease to predict the future. As we show in the following sections, these trade-offs emerge even in simple quadratic learning problems. \looseness=-1
\vspace{-0pt}

\subsection{Analysis Setup and Formalism}

The general formulation above accommodates arbitrary forms of non-stationarity. Following the CL literature, we introduce a task-based abstraction, where the data-generating process (described by $\mathcal{P}_n$) is \emph{locally stationary}.
\vspace{-0.1cm}

\begin{definition}[Task-based data generating process.]
\label{def:task}
A given history $H_T$ is a \emph{task-based} sequence if the sampling distribution $\mathcal{P}_n$ is locally stationary. The timeline is partitioned into time intervals called \emph{tasks} $S_k = [a_k,b_k]$ such that at each time step $n \in S_k$ an observation $o_n$ is sampled from a task-specific distribution $\mathcal{P}^k$. \looseness=-1
\end{definition}
Unless otherwise stated, we assume the history is a task-based sequence with \emph{known task switch points}. We denote by $z(n)$ the task identity at time $n$ and by $N_k = |S_k|$ the duration of task $k$.

With a slight abuse of notation we use $\ell(\theta, o)$ to indicate $\ell(\pi_\theta, o)$, where $\theta$ are the parameters of the policy $\pi$. We denote the expected loss for task $k$ over the model parameters $\theta$ as
\[
\mathcal{L}^k(\theta) = \mathbb{E}_{o\sim \mathcal{P}^k}[\ell(\theta, o)]
\]
We refer to the corresponding oracle minimizer $\theta^k_\star \in \arg\min_\theta \mathcal{L}^k(\theta)$ as a \emph{task solution vector}.  Within this formalism, we define the \emph{Average Task Error} as the real-time error accumulated during a single task:
\[
    \mathrm{ALE}^k(\mathcal{A}) = \frac{1}{N_k} \sum_{n=a_k}^{b_k} \mathbb{E}_{\theta \sim \mathcal{A}(\cdot|H_n)}[\mathcal{L}^k(\theta)]
\]
It follows that the total lifelong error is a weighted sum of these task errors $\mathrm{ALE}(\mathcal{A}) = \sum_{k=1}^{z(T)} \bar{w}_k \cdot \mathrm{ALE}^k(\mathcal{A})$, where the \emph{weight of each task} $\bar{w}_k = N_k/T$ \rev{is its relative presence} over the agent's lifetime.\footnote{\revp{With $c^k = \mathcal{L}^k(\theta^k_\star)$, the optimal loss for task $k$ in the chosen model class, $\mathrm{ALE}^k(\mathcal{A}) - c^k$ is the average \emph{regret} of the algorithm. Regret is common in online learning \citep{cesa2006prediction, shalev2012online, zinkevich2003online, alquier2017regret} and measures the cumulative loss of a learner along its trajectory \emph{with respect to a given comparator}.}}

\subsection{The algorithm as a Markov Process}

We treat the algorithm as a stochastic process. \revp{Formally, we define} $\mathcal{A}$ to be a discrete-time stochastic process $\{\theta_n\}_{n\ge 0}$ taking values in the parameter space, and we denote by $Q_n$ the resulting parameter distribution at time $n$. The evolution of the parameters is governed by a stochastic update rule $U$. In this work we consider algorithms where, at step $n$, the update rule depends on the current parameters $\theta_n$ and a data batch $\xi_n$ \emph{drawn from an algorithm-specific sampling distribution} $\mathcal{D}_n$:
\[
\theta_{n+1} = U(\theta_n, \xi_n),  \quad \text{where } \xi_n \sim \mathcal{D}_n.
\]
\rev{This formula covers SGD, our running example. By augmenting $\theta_n$ with auxiliary variables, it also covers momentum-based methods, Adam, and more broadly} the entire family of \emph{stochastic approximation algorithms based on first-order Markov processes}.

\emph{\rev{The} sampling distribution $\mathcal{D}_n$ is distinct from the environmental distribution $\mathcal{P}_n$}. While $\mathcal{P}_n$ dictates the incoming data stream, $\mathcal{D}_n$ dictates the data used for optimization; \rev{for replay-based algorithms such as the JTL algorithm presented here, $\mathcal{D}_n$ may include data from previous tasks. We also assume that the sampling strategy is fixed within each task,} such that $\mathcal{D}_n = \mathcal{D}^k$ for all $n \in S_k$.

\revp{The \textbf{computational cost} of such algorithms depends on the complexity of the update function $U$; \cref{ss:JTL-ITL} makes this concrete for SGD.}

Throughout this paper we reason about $Q_n$ (the parameter distribution at step $n$) and its stationary limit $Q_\infty$\rev{. In practice, we estimate expectations under $Q_n$ from the empirical ensemble of training trajectories over multiple random seeds, and} the deep-network experiments treat each trained model as a sample from $Q_n$ at the corresponding step.

\subsection{JTL and ITL processes}
\label{ss:JTL-ITL}

Within these boundaries, many different learning algorithms can be implemented through choices of $U$ or $\mathcal{D}_n$. \rev{Here we introduce the two examples that this paper studies.}

\paragraph{The Joint-Task Learning algorithm ($\mathcal{A}_{\text{JTL}}$).} The JTL agent aims to optimize the average loss over all seen tasks. This is equivalent to experience replay with an infinitely large buffer\rev{, which is often used} in continual learning as an upper bound on performance. To maintain time-homogeneous dynamics within a task, we model the experience replay as a fixed mixture approximation. Specifically, the sampling distribution within task $k$, $\mathcal{D}^k_{JTL}$, is a weighted mixture of the current and all
preceding task distributions:
\[
\mathcal{D}^k_{JTL} = \sum_{j=1}^k w_j^k\cdot \mathcal{P}^j, \qquad \text{where }w_j^k = \frac{N_j}{b_k}\,.
\]

\paragraph{The Independent-Task Learning algorithm ($\mathcal{A}_{\text{ITL}}$).}
\rev{The ITL agent optimizes only on the current data stream, so its sampling distribution is identical} to the current task distribution,
\[
\mathcal{D}^k_{ITL} = \mathcal{P}^k\,.
\]
At the boundary between tasks, the ITL agent \emph{resets} the parameters to a fixed initialization distribution, \rev{discarding} prior knowledge to avoid negative transfer and loss of plasticity \citep{dohareLossPlasticityDeep2024a}. %
\revp{A related variant warm-starts each task from the previous parameters instead of resetting. It has the same Instability as ITL and differs only in its Transient Error (\cref{app:reset_itl}), so comparing the two isolates the benefit of carried-over parameters.}

\paragraph{Equivalent computational costs.}
\revp{By construction, ITL and JTL share the same optimization compute per iteration. For SGD, the cost of an update is set by the forward and backward passes over the batch, so it scales linearly with the batch size: $\text{compute}(\mathcal{A}) \in \Theta(B)$ for batch size $B$.}\footnote{Hardware optimizations might violate this linear relationship; for the present discussion the linear scaling suffices.} This is the cost appearing in \cref{eq:SLO}. JTL may require additional memory or data-access machinery to implement replay, but those are outside the optimization budget studied here. Our experiments match this budget by design, using the same number of parameter updates and the same per-update batch budget for both agents.

We study JTL and ITL not as competing heuristics but as the two limiting cases of a single design axis: both are instances of the broader family of \emph{Predictive CL algorithms} we formalize in \cref{sec:generalizing} (\cref{def:predictive-cl}), each corresponding to a particular assumption about how the future relates to the past.

\section{Transfer Efficiency}
\label{sec:transfer-efficiency}

\rev{To quantify when a JTL model is preferable to an ITL one, we define \emph{Transfer Efficiency}, the central quantity of our analysis.}

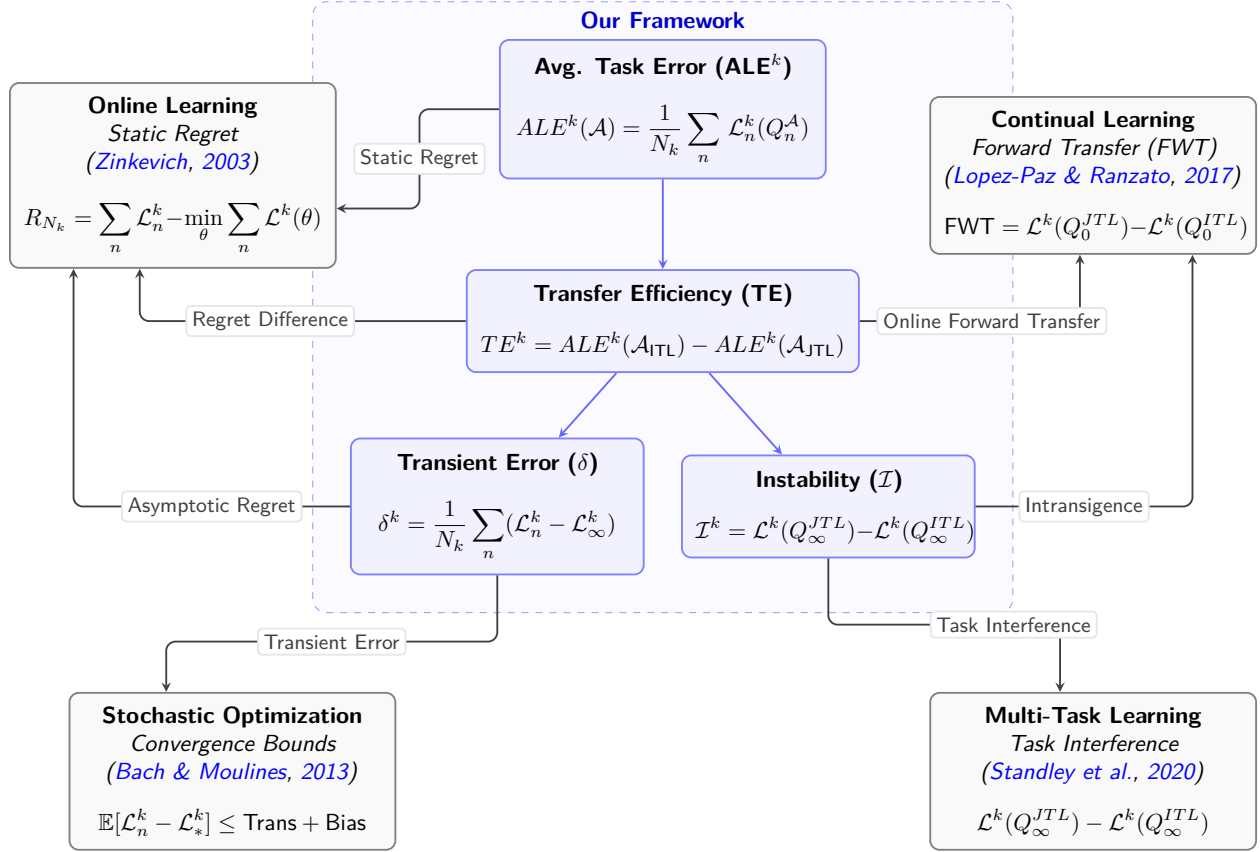
\begin{figure*}[t]
    \centering
    \vspace{-5mm}%
    \resizebox{0.9\linewidth}{!}{%
    \begin{tikzpicture}[
        >=stealth,
        corenode/.style={rectangle, draw=blue!50, fill=blue!5, thick, rounded corners=4pt, align=center, font=\sffamily\normalsize, inner sep=6pt},
        litnode/.style={rectangle, draw=black!50, fill=gray!5, thick, rounded corners=4pt, text width=4.5cm, align=center, font=\sffamily\normalsize, inner sep=6pt},
        edge-lbl/.style={fill=white, font=\sffamily\small, align=center, inner sep=2.5pt, draw=black!10, rounded corners=2pt},
        framework/.style={rectangle, draw=blue!40, fill=blue!2, dashed, rounded corners=6pt, inner sep=16pt}
    ]

    \node[corenode, text width=4.5cm] (ale) at (0, 0) {\textbf{Avg. Task Error (ALE$^k$)} $$ALE^k(\mathcal{A}) = \frac{1}{N_{k}}\sum_{n}\,\mathcal{L}^{k}_n(Q_n^\mathcal{A})$$};

    \node[corenode, text width=5.5cm] (te) at (0, -3.2) {\textbf{Transfer Efficiency (TE)} $$TE^k = ALE^k(\mathcal{A}_{\text{ITL}}) - ALE^k(\mathcal{A}_{\text{JTL}})$$};

    \node[corenode, text width=4cm] (transient) at (-2.5, -6.0) {\textbf{Transient Error ($\delta$)} $$\delta^k = \frac{1}{N_{k}}\sum_{n}(\mathcal{L}^k_{n} - \mathcal{L}^k_{\infty})$$};

    \node[corenode, text width=4cm] (instability) at (2.5, -6.0) {\textbf{Instability ($\mathcal{I}$)} $$\mathcal{I}^k = \mathcal{L}^k(Q_{\infty}^{JTL}) - \mathcal{L}^k(Q_{\infty}^{ITL})$$};

    \draw[->, thick, blue!60] (ale) -- (te);
    \draw[->, thick, blue!60] (te) -- (transient);
    \draw[->, thick, blue!60] (te) -- (instability);

    \begin{scope}[on background layer]
        \node[framework, fit=(ale) (transient) (instability), label={[font=\sffamily\bfseries\normalsize, text=blue!80!black, yshift=-1.5em]above:Our Framework}] (framework_box) {};
    \end{scope}

    \node[litnode] (ol) at (-7.4, -1.0) {\textbf{Online Learning} \\ \textit{Static Regret \citep{zinkevich2003online}} $$R_{N_k} = \sum_{n} \mathcal{L}^k_n - \min_{\theta} \sum_{n} \mathcal{L}^k(\theta)$$};

    \node[litnode] (so) at (-6.5, -10.0) {\textbf{Stochastic Optimization} \\ \textit{Convergence Bounds \citep{bach}} $$\mathbb{E}[\mathcal{L}^k_n - \mathcal{L}^k_*] \leq \text{Trans} + \text{Bias}$$};

    \node[litnode] (cl) at (6.5, -1.0) {\textbf{Continual Learning} \\ \textit{Forward Transfer (FWT) \citep{lopez2017gradient}} $$\text{FWT} =  \mathcal{L}^k(Q_{0}^{JTL}) - \mathcal{L}^k(Q_{0}^{ITL})$$};

    \node[litnode] (mtl) at (6.5, -10.0) {\textbf{Multi-Task Learning} \\ \textit{Task Interference \\ \citep{standley2020tasks}}
    $$\mathcal{L}^k(Q_{\infty}^{JTL}) - \mathcal{L}^k(Q_{\infty}^{ITL})$$};

    \draw[->, thick, darkgray, rounded corners] (ale.west) -- ++(-1.2,0) |- ([yshift=-0.5cm]ol.east) node[pos=0.25, edge-lbl] {Static Regret};

    \draw[->, thick, darkgray, rounded corners] (te.west) -| ([xshift=-0.5cm]ol.south) node[pos=0.3, edge-lbl] {Regret Difference};
    \draw[->, thick, darkgray, rounded corners] (te.east) -| ([xshift=-0.2cm]cl.south) node[pos=0.3, edge-lbl] {Online Forward Transfer};

    \draw[->, thick, darkgray, rounded corners] (transient.south) -- ++(0,-1.0) -| ([xshift=-1.0cm]so.north) node[pos=0.25, edge-lbl] {Transient Error};
    \draw[->, thick, darkgray, rounded corners] (transient.west) -| ([xshift=-1.5cm]ol.south) node[pos=0.25, edge-lbl] {Asymptotic Regret};

    \draw[->, thick, darkgray, rounded corners] (instability.south) -- ++(0,-1.0) -| ([xshift=-0.5cm]mtl.north) node[pos=0.4, edge-lbl] {Task Interference};
    \draw[->, thick, darkgray, rounded corners] (instability.east) -| ([xshift=1.5cm]cl.south) node[pos=0.25, edge-lbl] {Intransigence};

    \end{tikzpicture}%
    }
    \vspace{3mm}
    \caption{\textbf{Concept map relating our framework to adjacent learning paradigms.} The central blue spine shows our decomposition: the Average Task Error ($\mathrm{ALE}^k$) is decomposed into the Transfer Efficiency ($\mathrm{TE}^k$), which in turn splits into the Instability ($\mathcal{I}^k$) and the Transient Error ($\delta^k$). \rev{Each labeled arrow} names the formal equivalent of the originating concept in one of four adjacent literatures (Online Learning, Continual Learning, Stochastic Optimization, Multi-Task Learning). \rev{For compactness we use} the shorthands $\mathcal{L}^k_{n} = \mathcal{L}^k(Q_n^\mathcal{A})$ and $\mathcal{L}^k_{\infty} = \mathcal{L}^k(Q_\infty^\mathcal{A})$. \vspace{-3mm}}
    \label{fig:unifying_concept_map}
\end{figure*}
\begin{definition}[Transfer Efficiency]
The \textbf{Transfer Efficiency} for task $k$ is the performance gain of a Joint-Task Learning (JTL) agent over an Independent-Task Learning (ITL) baseline:
\[
    \mathrm{TE}^k
    := \mathrm{ALE}^k(\mathcal{A}_{\text{ITL}}) - \mathrm{ALE}^k(\mathcal{A}_{\text{JTL}}).
\]
\end{definition}
A positive value $\mathrm{TE}^k > 0$ indicates \emph{positive transfer}, where the JTL agent's past experience improves its performance on task $k$ relative to the ITL agent; conversely, $\mathrm{TE}^k < 0$ indicates \emph{negative transfer}. Because $\mathrm{ALE}^k(\mathcal{A})$ represents the regret of $\mathcal{A}$ on task $k$ (up to a constant), the Transfer Efficiency can equivalently be read as the regret reduction obtained by JTL relative to ITL.

To understand the mechanisms of transfer, we split the error into two components: a \emph{stationary component}, the error at equilibrium, and a \emph{transient component}, the error accumulated while the algorithm is actively moving toward that equilibrium.\footnote{We use \emph{equilibrium}, \emph{stationary distribution}, and \emph{convergence point} interchangeably throughout the paper. These terms originate in different communities (statistical physics, Markov chains, and optimization respectively), but for the iterate sequence $\{\theta_n\}$ they all refer to the same object: the limiting distribution $Q_\infty$ to which $Q_n$ converges as the algorithm runs on a fixed task. We avoid \emph{critical point}, which in optimization denotes a zero-gradient parameter value rather than a distribution over parameters.} Equilibrium is only well-defined if the optimization process \emph{converges}, which we henceforth assume.

We denote by $Q_\infty^\mathcal{A}$ the stationary parameter distribution to which $Q_n$ converges under algorithm $\mathcal{A}$, given a sufficiently long task duration. \rev{It depends on both the task data and the agent's initialization at the task boundary.} For the JTL and ITL algorithms we write $Q_\infty^{JTL}$ and $Q_\infty^{ITL}$ respectively. We also use $\mathcal{L}^k(Q) := \mathbb{E}_{\theta \sim Q}[\mathcal{L}^k(\theta)]$ to lift the expected loss to a distribution over parameters.

\begin{definition}[Instability]
    \label{def:instability}
    The \textbf{Instability} of a task $k$ ($\mathcal{I}^k$) is the performance gap \rev{between the JTL and ITL agents \emph{measured} at convergence}
    \[
    \mathcal{I}^k := \mathcal{L}^k(Q_\infty^{JTL}) - \mathcal{L}^k(Q_\infty^{ITL}).
    \]
\end{definition}
\begin{definition}[Transient Error]
    \label{def:inferential-error}
    The \textbf{Transient Error} of a task $k$ and algorithm $\mathcal{A}$ ($\delta^k(\mathcal{A})$) is the average deviation from the stationary risk over the task duration $S_k$:
    \[
        \delta^k(\mathcal{A}) = \frac{1}{N_k}\sum_{n\in S_k} \left(\mathcal{L}^k(Q^\mathcal{A}_n) - \mathcal{L}^k(Q^\mathcal{A}_\infty)\right).
    \]
\end{definition}

Combining these definitions yields a handy decomposition of Transfer Efficiency, separating the ``quality of the final solution'' from the ``speed of learning'' (visualized in \cref{fig:transfer_efficiency_final}, appendix).

\begin{lemma}[TE Decomposition]
\label{lemma:TE-decomposition}
For any task $k$, assuming the algorithms converge, the Transfer Efficiency decomposes exactly as:
\begin{equation}
\label{eq:TE-decomp}
    \mathrm{TE}^k = \underbrace{-\mathcal{I}^k}_{\text{Stationary Gap}} + \underbrace{(\delta^k(\mathcal{A}_{\text{ITL}}) - \delta^k(\mathcal{A}_{\text{JTL}}))}_{\text{Learning Speed Gap}}.
\end{equation}
\end{lemma}

\emph{Literature note.}
The Transient Error has been extensively studied in the stochastic optimization literature \citep{murataStatisticalStudyOnline1999}, \rev{whereas the Instability has been studied in the multi-task and transfer-learning literature} \citep{baxterModelInductiveBias2000,caruana1997multitask}. These two fields have mostly developed independently, except for work in \emph{online Meta-Learning} \citep{deneviLearningLearnCommon2018,nazariDynamicRegretAnalysis2021} where both online performance and between-task transfer are considered, and where joint-task and independent-task baselines are routinely compared. \rev{\cref{fig:unifying_concept_map} places our definitions within this literature, and \cref{sec:related-work-new} discusses these connections in depth.}
\subsection{The Critical Task Duration}
\label{ssec:critical-task-duration}

The decomposition in \cref{lemma:TE-decomposition} \rev{yields a first general result that needs only minimal assumptions:} when the stationary JTL solution is worse than the stationary ITL solution, there is a threshold task duration beyond which the initial advantage of transfer is lost.

\begin{lemma}[The Critical Task Duration]
\label{lem:critical_TD}
Assume a sequence of $K$ tasks of equal duration $N$. Let $\Delta^k(\mathcal{A}) := N \cdot \delta^k(\mathcal{A})$ denote the \emph{cumulative} Transient Error of algorithm $\mathcal{A}$ on task $k$, let $\bar{\Delta}(\mathcal{A}) = \frac{1}{K}\sum_{k=1}^K \Delta^k(\mathcal{A})$ be its average across tasks, and let \revp{$\bar{\mathcal{I}} = \frac{1}{K}\sum_{k=1}^K \mathcal{I}^k$} be the average Instability. If the cumulative Transient Error is sub-linear in $N$, i.e.\ $\Delta^k(\mathcal{A}) \in o(N)$, and $\bar{\mathcal{I}}>0$, then there exists a \major{Critical Task Duration} $N^\infty_{\max}$ beyond which transfer becomes detrimental ($\mathrm{TE} < 0$). The breakeven horizon is:
\begin{equation}
    N^\infty_{\max} = \frac{\bar{\Delta}(\mathcal{A}_{\text{ITL}}) - \bar{\Delta}(\mathcal{A}_{\text{JTL}})}{\bar{\mathcal{I}}}.
\end{equation}
\end{lemma}

In other words, if the task duration exceeds $N^\infty_{\max}$, \rev{the ITL agent has enough time to overcome its poor initialization and eventually outperforms the JTL agent.} The Critical Task Duration is inversely proportional to the average Instability in the agent's lifetime: the more `unpredictable' the world, the narrower the window for transfer. When $\bar{\mathcal{I}}\le 0$, the stationary JTL solution is not worse on average, and the decomposition predicts no finite horizon at which retention must become detrimental. \rebuttalB{What it means for a \emph{measured} Instability to be negative, and which factors determine its sign, is discussed in \cref{apssec:instability-proxies}.} The sub-linearity condition $\Delta^k(\mathcal{A}) \in o(N)$ requires only that the learning algorithm \rev{converges toward its stationary risk, a mild assumption that holds across standard optimization regimes, from convex optimization to neural networks (see \cref{lem:critical_TD_appd}).}

\subsection{Empirical Validation of the Framework}
\label{ssec:empirical-validation}

\vspace{-0.15cm}
\begin{figure}[h]
    \vspace{-0.3cm}
    \centering  \includegraphics[width=0.88\linewidth]{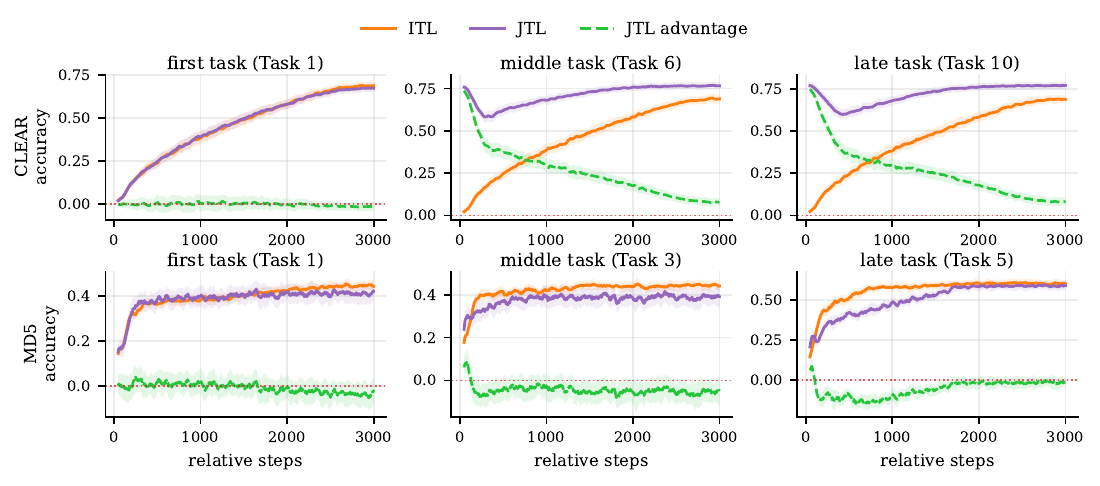}\\[-0.28cm]
    \includegraphics[width=0.88\linewidth]{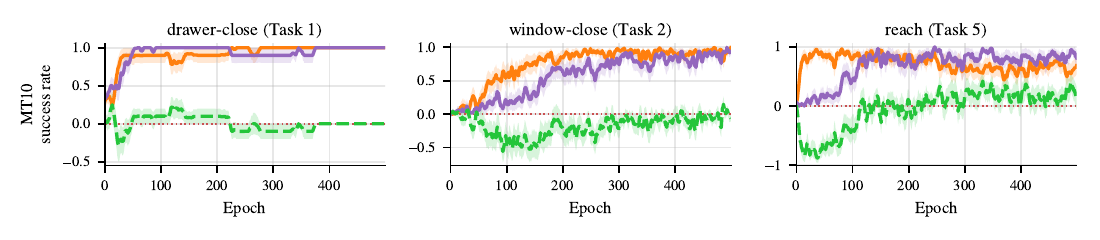}
\vspace{-0.2cm}
\caption{\textbf{Representative continual-learning trajectories for supervised benchmarks and MT10.} Top two rows: CLEAR and MD5 task trajectories, showing smooth shifts with positive JTL transfer versus sharp semantic shifts with negative transfer. Bottom row: selected MT10 per-task success-rate trajectories on Meta-World in an easy-to-hard ordering with batch size 5000 (drawer-close, window-close, reach). Both supervised benchmarks report test accuracy, while MT10 reports success rate. ITL is shown in \textcolor{orange}{orange}, JTL in \textcolor{violet}{violet}, and the \textcolor{green}{green} trace gives the JTL advantage. CLEAR shows a consistently positive advantage, while MD5 and MT10 show negative advantage after the initial transient, matching the Instability ordering.}
\label{fig:example_performance_traj}
\end{figure}

We use three benchmarks, chosen to span a range of Instabilities: \textbf{CLEAR} \citep{lin2021clear}, ten image classes evolving smoothly over 2004--2014; \textbf{MD5}, a sequence we construct from five unrelated public classification datasets \rev{with sharp semantic transitions;} \textbf{MT10} \citep{yu2020meta}, a 10-task suite of Meta-World robotic-manipulation environments previously used to highlight continual-learning interference \citep{wolczyk2021continual}. We use the same task duration $N$ across tasks and tune hyperparameters for the lowest ALE per algorithm; full setup details are in \cref{Apdx: Exps}.

Throughout the paper, classification performance is reported as \emph{accuracy in $[0,1]$} and reinforcement-learning performance as \emph{success rate in $[0,1]$}. We summarize each agent by its \emph{Average Lifelong Accuracy} ($\texttt{ALA}$) or \emph{Average Lifelong Success Rate} \rev{($\texttt{ALS}$), which are monotone (higher-is-better) re-readings of the ALE. We sign Transfer Efficiency} so that $\mathrm{TE}>0$ always means JTL outperforms ITL.

\rev{We run ITL and JTL with matched compute on each benchmark. The resulting learning curves} (\cref{fig:example_performance_traj}) and average lifelong performance (\cref{tab:results-wild}) already tell the qualitative story: where Instability\footnote{In the experiments, expectations under $Q_n$ and $Q_\infty$ are estimated empirically: we run multiple training seeds for each algorithm and treat the resulting ensemble of trained models as samples from the parameter distribution at step $n$. The expected loss $\mathcal{L}^k(Q_n)$ is then approximated by the average test loss across seeds, and the Instability $\mathcal{I}^k$ and Transient Error $\delta^k$ in the decomposition are estimated from this empirical ensemble after the algorithms have converged on each task. \rebuttalB{Finite-time and finite-sample effects can affect these estimates slightly; see \cref{apssec:instability-proxies}.}} is large (MD5, MT10), \rev{joint training confers no advantage and is \revp{at times actively} worse;} where Instability is small (CLEAR), it confers a large, compounding advantage. \rev{The sign of Transfer Efficiency thus tracks the sign of $-\mathcal{I}$, \revp{exactly} as \cref{eq:TE-decomp} predicts.}

\begin{table}[h]
    \centering
    \setlength{\tabcolsep}{8pt}
    \renewcommand{\arraystretch}{1.2}
    \caption{Average lifelong accuracy ($\texttt{ALA}$, in $[0,1]$) / success rate ($\texttt{ALS}$) in the wild. Higher is better; $\mathrm{TE}>0$ means JTL outperforms ITL. Empirical instabilities per benchmark are reported in \cref{tab:instability}. \rebuttalB{The MT10 row reports the \texttt{easy\_to\_hard} task ordering at SAC mini-batch size $5000$.}}
    \label{tab:results-wild}
    \begin{tabular}{lcccc}
        \toprule
         & $N$ & $\mathbf{\texttt{ALA}_{ITL}}$ & $\mathbf{\texttt{ALA}_{JTL}}$ & $\mathbf{{\mathrm{TE}}}$ \\
        \midrule
        \textbf{MD5}   & 3000 & 0.470 {\scriptsize $\pm 0.002$}  & 0.430 {\scriptsize $\pm 0.005$}  & -0.040 {\scriptsize $\pm 0.005$} \\
        \addlinespace[0.1cm]
        \textbf{CLEAR} & 3000 & 0.465 {\scriptsize $\pm 0.0004$} & 0.681 {\scriptsize $\pm 0.0005$} & +0.216 {\scriptsize $\pm 0.002$} \\
        \midrule
        \addlinespace[0.2cm]
         &      & $\mathbf{\texttt{ALS}_{ITL}}$ & $\mathbf{\texttt{ALS}_{JTL}}$ & $\mathbf{{\mathrm{TE}}}$ \\
        \cmidrule(lr){3-4} \cmidrule(lr){5-5}
        \textbf{MT10}  & 500  & \rebuttalB{0.514 {\scriptsize $\pm 0.318$}}   & \rebuttalB{0.274 {\scriptsize $\pm 0.353$}}   & \rebuttalB{-0.240 {\scriptsize $\pm 0.03$}}
        \\
        \bottomrule
    \end{tabular}
\end{table}

Beyond this qualitative pattern, the decomposition in \cref{eq:TE-decomp} and the Critical Task Duration of \cref{lem:critical_TD} make two sharper, quantitative predictions about real continual-learning systems:

\begin{itemize}
    \item \textbf{(P1)} \emph{Instability shifts Transfer Efficiency by an amount independent of task duration; Transfer Efficiency then decays with task duration at a $1/N$ rate.} Direct prediction of \cref{eq:TE-decomp}: $\mathcal{I}^k$ enters as an additive constant in $N$, while the transient gap is averaged over $N$ steps.
    \item \textbf{(P2)} \emph{For sequences with $\bar{\mathcal{I}}>0$, increasing $N$ drives Transfer Efficiency from positive to negative; and once $N$ is large enough for both agents to converge, $\mathrm{TE}$ is dominated by $\mathcal{I}$.} Direct prediction of \cref{lem:critical_TD}: the transient gap shrinks with $N$, leaving the stationary Instability as the residual.
\end{itemize}

\rev{Next, we test these predictions \revp{directly} on controlled benchmarks that let us vary Instability and task duration in isolation.}

\paragraph{(P1) Controlled Instability: Permuted- and Shuffled-CIFAR.}
We manipulate Instability by perturbing the input and target distribution in CIFAR-10, creating two separate benchmarks: \emph{Permuted-CIFAR} (PC-16/\major{24/}32) shuffles input pixels \rev{within a centered patch} (preserving targets, transforming only the input distribution), and \emph{Shuffled-CIFAR} (SC-10/30/50) randomly re-assigns labels to a fraction of training samples (transforming only the target distribution). We expect higher Instability in Shuffled- than in Permuted-CIFAR, as label shuffling forces the same input to map to inconsistent labels across tasks. \rev{Our measurements in \cref{tab:instability} confirm this.}

\begin{figure}[ht]
    \centering
    \includegraphics[width=0.25\linewidth]{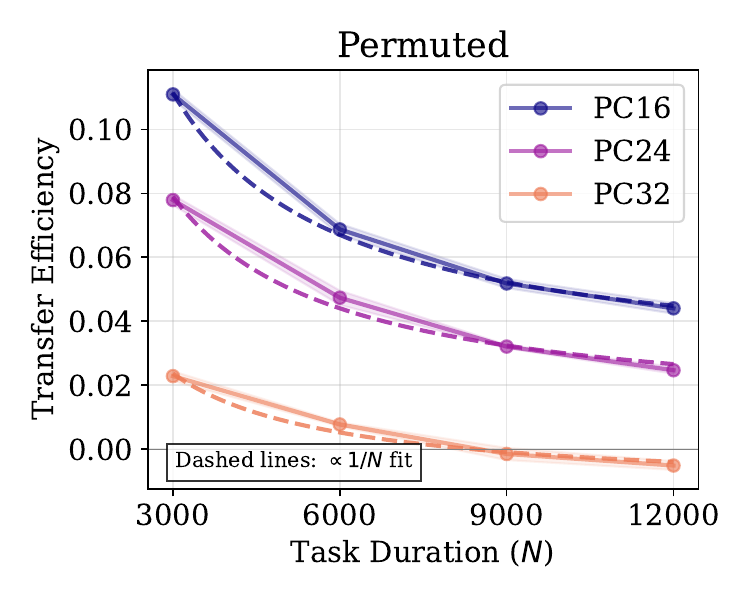}
    \includegraphics[width=0.25\linewidth]{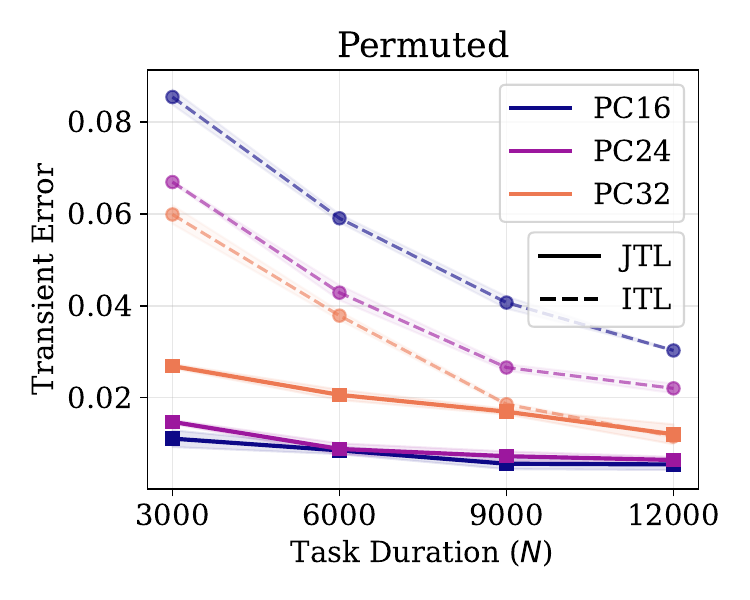}
    \includegraphics[width=0.3\linewidth]{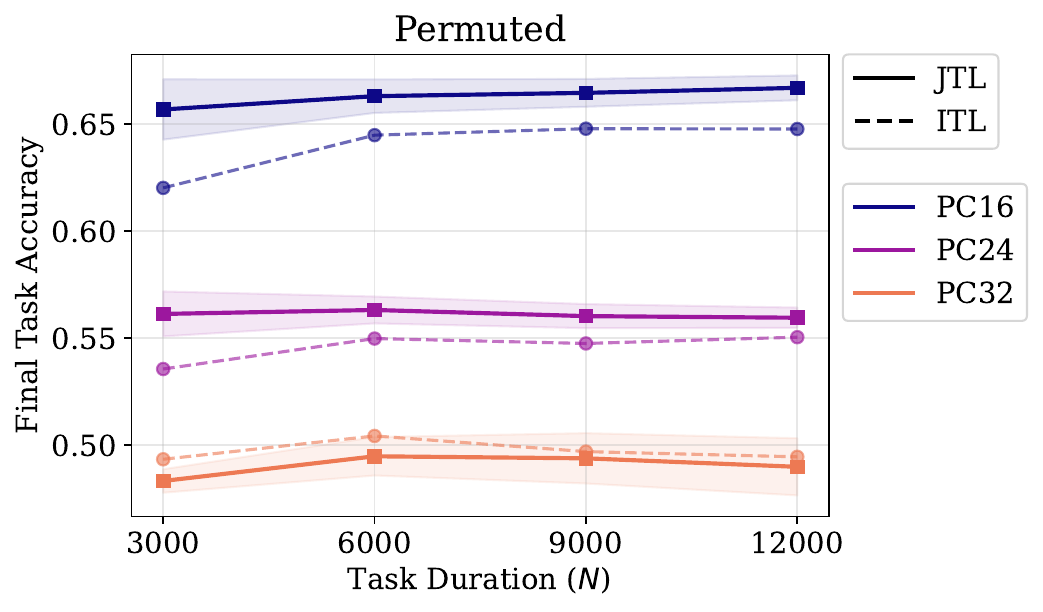}\\
    \includegraphics[width=0.25\linewidth]{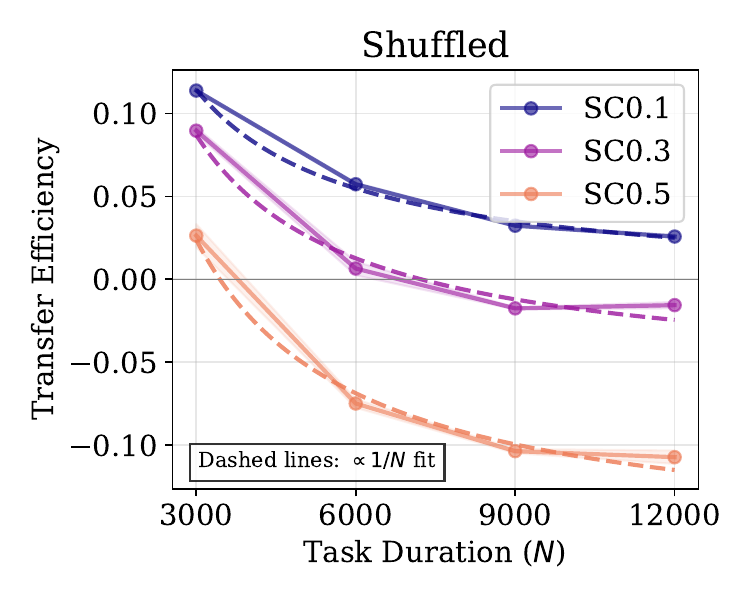}
    \includegraphics[width=0.25\linewidth]{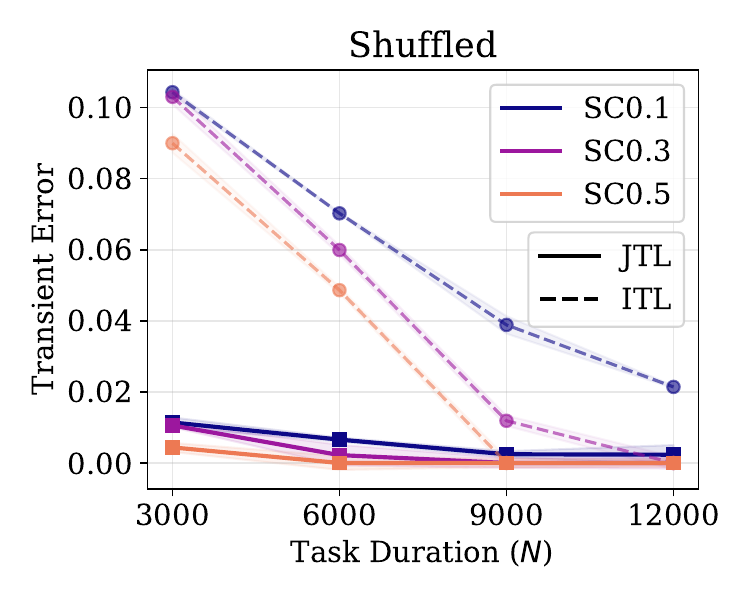}
    \includegraphics[width=0.3\linewidth]{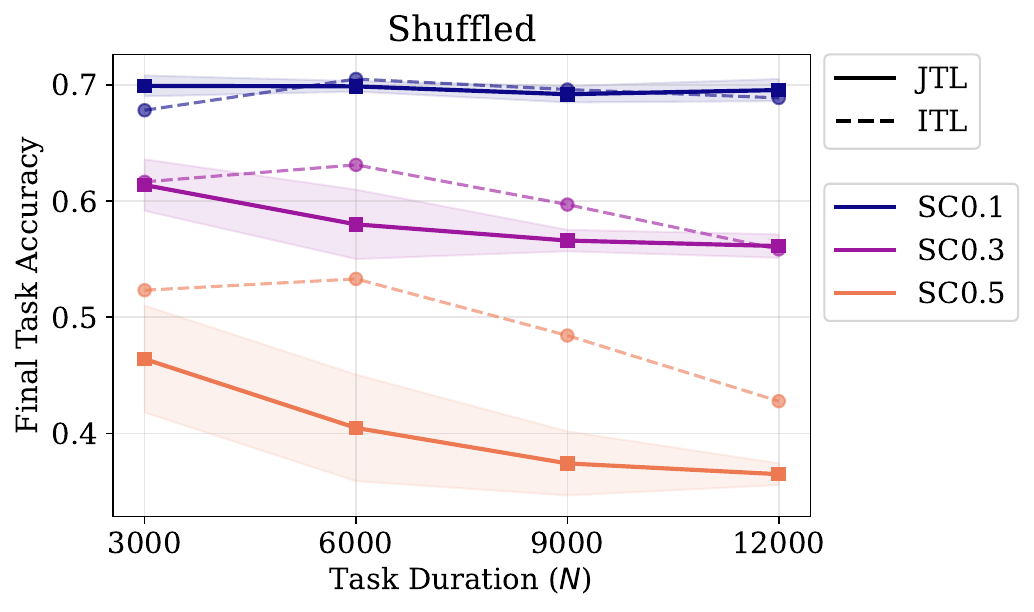}
    \caption{\textbf{Manipulating Instability on CIFAR-10.} Permuted-CIFAR (\emph{Top}) and Shuffled-CIFAR (\emph{Bottom}) across task durations $N$. \emph{Left:} Transfer Efficiency decays as $1/N$ (dashed fits); higher Instability translates the curves vertically downward. \emph{Middle:} JTL Transient (solid) is negligible relative to ITL Transient (dashed) at every $N$, matching the asymptotic prediction. \emph{Right:} \rev{Final task accuracy ceilings. For interpretability, all quantities are measured in} \emph{accuracy} (between 0 and 1).}
    \label{fig:input-distr-shift}
\end{figure}

\rev{\cref{fig:input-distr-shift} confirms both halves of P1.} First, increasing the shift magnitude (PC-16$\to$\major{24$\to$}32, SC-10$\to$30$\to$50) \rev{translates the TE curves vertically downward, \revp{exactly} as the additive} $-\mathcal{I}^k$ term in \cref{eq:TE-decomp} predicts. Second, the empirical TE is well fit by a $1/N$ curve across all settings, indicating that the per-step transient gap $\mathcal{L}^k(Q_n) - \mathcal{L}^k(Q_\infty)$ \rev{vanishes super-linearly, the regime in which} \cref{lem:critical_TD} applies. \rev{The middle panels also confirm} an asymptotic prediction we revisit in \cref{sec:quadratic-losses}: the JTL Transient is consistently several orders of magnitude smaller than the ITL Transient, so Transfer Efficiency is ultimately driven by Instability and the ITL Transient.

\paragraph{(P2) Critical Task Duration in CLEAR.}
\rebuttalB{For MD5 and MT10 the Critical Task Duration lies below the smallest tested duration: their Transfer Efficiency is already negative (\cref{tab:results-wild}). For CLEAR we measure a finite-time proxy of the Instability at each task duration, \revp{$\mathcal{I}_{conv}$,} the difference between the converged JTL and ITL \rev{test errors (\cref{apssec:instability-proxies}). Its magnitude shrinks with $N$; at $N = 12000$ it is} $-0.0006 \pm 0.0006$, statistically indistinguishable from zero. When $\bar{\mathcal{I}} = 0$, as in the i.i.d.\ case, \cref{lem:critical_TD} \rev{predicts that zero is an asymptote of the Transfer Efficiency, so this experiment tests the predicted decay of TE toward zero rather than a crossing.}} We sweep $N\in\{3000,6000,9000,12000\}$ training steps per task and re-measure each component of the decomposition.

\begin{table}[ht]
    \centering
    \caption{Increasing the task duration $N$ in CLEAR closes the gap in average lifelong performance (accuracy in $[0,1]$), as predicted by \cref{lem:critical_TD}. All quantities ($\texttt{ALA}$, $\mathrm{TE}$, $\delta$) are on the same $[0,1]$ scale. \revp{The true Instability is near zero; its finite-time estimate $\mathcal{I}_{conv}$ shrinks toward zero as the ITL agent converges.}}
    \label{tab:clear_benchmark}
    \setlength{\tabcolsep}{4.5pt}%
    \renewcommand{\arraystretch}{1.2}
    \begin{tabular}{lccccccc}
        \toprule
         & $N$ & $\texttt{ALA}_{ITL}$ & $\texttt{ALA}_{JTL}$ & ${\mathrm{TE}}$ & $\delta_{ITL}$ & $\delta_{JTL}$ & $\mathcal{I}_{conv}$\\
        \midrule
        \textbf{CLEAR}
            & 3000  & 0.465 {\scriptsize $\pm 0.0004$} & 0.681 {\scriptsize $\pm 0.0005$} & 0.216 {\scriptsize $\pm 0.002$} &  0.227 {\scriptsize $\pm 0.0013$} & 0.079 {\scriptsize $\pm 0.0014$} & $-0.066$ {\scriptsize $\pm 0.002$}\\
           & 6000  & 0.562 {\scriptsize $\pm 0.0001$} & 0.713 {\scriptsize $\pm 0.004$}  & 0.151 {\scriptsize $\pm 0.004$} &  0.195 {\scriptsize $\pm 0.0001$} & 0.063 {\scriptsize $\pm 0.0001$} & $-0.011$ {\scriptsize $\pm 0.002$}\\
           & 9000  & 0.610 {\scriptsize $\pm 0.0004$} & 0.720 {\scriptsize $\pm 0.001$}  & 0.11 {\scriptsize $\pm 0.001$}   &  0.166 {\scriptsize $\pm 0.0009$} & 0.057 {\scriptsize $\pm 0.0010$} & $-0.003$ {\scriptsize $\pm 0.002$}\\
           & 12000 & 0.641 {\scriptsize $\pm 0.0002$} & 0.732 {\scriptsize $\pm 0.0006$} & \textbf{0.10} {\scriptsize $\pm 0.001$} &  0.148 {\scriptsize $\pm 0.0004$} & 0.054 {\scriptsize $\pm 0.0008$} & $-0.0006$ {\scriptsize $\pm 0.0006$}\\
        \bottomrule
    \end{tabular}
\end{table}

\cref{tab:clear_benchmark} \rebuttalB{shows the predicted monotonic narrowing}: \rev{Transfer Efficiency falls \revp{monotonically} from} $+0.216$ at $N=3000$ to $+0.10$ at $N=12000$. \revp{Part of this drop comes from the transient gap $\delta_{\text{ITL}}-\delta_{\text{JTL}}$ ($0.148 \to 0.094$) and part from $\mathcal{I}_{conv}$ moving toward zero ($-0.066 \to -0.0006$). Both have the same cause: with more time per task, the ITL agent converges further, which shrinks its Transient and removes the finite-time bias of the estimate.}
 \rebuttalB{\rev{Within the tested range, Transfer Efficiency decays as $1/N$ while remaining positive, \revp{exactly what} \cref{eq:TE-decomp} \revp{prescribes} when} $\bar{\mathcal{I}} \approx 0$.}

\section{Quadratic-Loss Analysis}
\label{sec:quadratic-losses}

\rev{This section makes the transfer mechanisms analytically transparent in two steps. First,} for quadratic losses we derive \emph{exact} closed-form expressions for \rev{the Instability and the Transient Error (and hence for Transfer Efficiency) that expose} how each quantity depends on the design choices (step size $\eta$, batch size $B$, parametrization). Second, we pass to an \emph{asymptotic} (large-$k$) analysis that collapses these expressions into a single sufficient condition for positive transfer, governed by one environment-level descriptor: the rate at which the environment expands with new tasks. Different environments yield different balances between ITL and JTL, and this descriptor predicts which one wins.
\revp{Quadratic losses are well studied in stochastic optimization and multi-task learning, so we build on existing results; even this simple setting shows how learning dynamics and bias interact to shape transfer.}

\rebuttalB{\paragraph{Assumptions.}  The results below assume quadratic task losses, aligned (commuting) positive-definite Hessians with bounded positive curvature, small constant step sizes, and stationary gradient-noise covariances with a uniform observation-noise floor; \cref{app:setup} states them formally. \rev{These assumptions make the results easier to interpret:} in \cref{sec:beyond_alignment} we remove the alignment assumption, and in \cref{sec:ridgeless} the bounded-curvature assumption (allowing for over-parametrized models). We evaluate the stationary noise assumption empirically: simulations with sampled inputs, whose gradient-noise covariance is state-dependent, shift the predicted values by $O(\eta)$ but leave the predicted transfer boundary unchanged (\cref{apdx:assumption-studies}).}

\subsection{Setup}

For every task $j \in \{1, \dots, K\}$, the expected loss is
\begin{equation}
    \mathcal{L}^j(\theta) = \tfrac{1}{2}(\theta - \theta_\star^j)^\top H_j (\theta - \theta_\star^j) + c_j,
\end{equation}
with $H_j \succ 0$ (symmetric positive definite) and $c_j$ the irreducible loss (Bayes error) for task $j$. The Hessian $H_j$ is data-driven: in linear regression with inputs $x\sim\mathcal{P}^j_x$, \rev{it is} $H_j = \mathbb{E}_{x\sim\mathcal{P}^j_x}[xx^\top]$. The re-centered form above is convenient for the analysis, but the gradient $\nabla\mathcal{L}^j(\theta) = H_j(\theta - \theta_\star^j) = H_j\theta - \mathbb{E}[xy]$ is computable from data without knowing $\theta_\star^j$; we use the $(\theta_\star^j, c_j)$ parametrization only to expose the dependence of dynamics on the task minimizer.

In this setting, the expected risk for task $k$, $\mathbb{E}_\theta[\mathcal{L}^k(\theta_n)]$, is completely determined by the first two moments of the parameter distribution $Q_n$: the mean $\mu_n$ and covariance $\Sigma_n$. Under the gradient-noise assumptions in \cref{app:setup}, the stationary covariance $\Sigma_\infty$ \rev{decomposes into} a transient part and a stationary noise part whose trace scales as $1/B$ in the batch size.

\subsection{Closed-form Instability and Transient Error}

To express the Instability and Transient Error in closed form, we introduce the \emph{relative curvature} matrix $\mathcal{R}^k = H_{\text{dyn},k}^{-1} H_k$: the curvature of the current task measured against the curvature of the optimization objective (which governs the parameter dynamics). \rev{For ITL, $\mathcal{R}^k = I$.} For JTL it generally is not, and the misalignment between dynamics and evaluation \rev{is what makes transfer non-trivial.}

Throughout this section we use the Mahalanobis-style notation $\|v\|_M^2 := v^\top M v$ for any positive-(semi)definite matrix $M$. \rev{A central quantity is the \textbf{Task Diversity}:} the additional variance in the optimization process due to the aggregation of multiple tasks. Recalling the task weights $w_j^k = |S_j|/b_k$ from \cref{ss:JTL-ITL},
\[
    \mathcal{V}_{div}^k = \sum_{j=1}^k w_j^k\cdot  \left\|\nabla\mathcal{L}^j(\theta_\star^{JTL,k})\right\|_{\mathcal{R}^k}^2,
\]
where $\theta_\star^{JTL,k}$ is the minimizer of the JTL objective at task $k$. The norm is taken in the relative-curvature geometry, so tasks whose gradients are large along directions the optimizer treats as informative contribute more.\looseness=-1\vspace{-0cm}

\rev{Under these assumptions,} the Instability of task $k$ admits an exact decomposition into a bias term and a variance term proportional to the Task Diversity (full derivation in \cref{prop:exact_instability}):
\[
\mathcal{I}^k = \underbrace{\mathcal{L}^k(\theta_\star^{\text{JTL},k}) - \mathcal{L}^k(\theta_\star^{k})}_\text{JTL Bias} + \,\underbrace{\,\,\frac{\eta}{4\,B}\,\mathcal{V}_{div}^k\,\,}_\text{JTL Variance}.
\]
\rev{The two terms are distinct costs of joint training.} The \emph{JTL Bias} is the structural penalty: the JTL solution $\theta_\star^{\text{JTL},k}$ is a compromise across the task collection, so when evaluated on a single task it pays a fixed gap relative to that task's own optimum. The \emph{JTL Variance} is the dynamical penalty: aggregating tasks injects gradient noise proportional to the step size $\eta$ \rev{and inversely proportional to the batch size $B$. This noise is amplified only along directions that the relative-curvature geometry identifies as informative: directions where task $k$'s curvature is large relative to the pooled objective. These correspond to features} that are important for task $k$ but receive little support from the aggregated gradient of the other tasks.

The Transient Error admits an analogous decomposition (\cref{theo:transient_potential}), expressed in terms of the deviations $\tilde\mu_n := \mu_n - \mu_\infty$ and $\tilde\Sigma_n := \Sigma_n - \Sigma_\infty$ of the mean and covariance from their stationary values:
\[
\tau\,\delta^k = \frac{1}{N_k}\,\underbrace{\left(\|\tilde \mu_0\|^2_{\mathcal{R}^k} - \|\tilde \mu_{N_k}\|^2_{\mathcal{R}^k}\right)}_\text{Mean Potential Difference} + \frac{1}{N_k}\,\underbrace{\left(\text{Tr}(\mathcal{R}^k \tilde{\Sigma}_0) - \text{Tr}(\mathcal{R}^k \tilde{ \Sigma}_{N_k})\right)}_\text{Variance Potential Difference},
\]
where $\tau = 4\eta$ is the learning time scale. \rev{This expression has a physical reading: the cost of learning a task is the drop in potential energy between} where the agent started the task and where it ended, measured on a quadratic potential surface warped by the relative curvature $\mathcal{R}^k$. The transient cost is therefore high precisely when initialization is far from the task's stationary distribution, and decays as the agent drains that potential.\looseness=-1\vspace{-0cm}

\subsection{Asymptotic Conditions for Positive Transfer}

The exact formulas above hold at every $k$ but are opaque. To expose how transfer depends on the environment, we now take an asymptotic (large-$k$) view, in which the entire ITL/JTL balance is \rev{governed by a single environment-level descriptor, the rate at which the environment \emph{expands} with new tasks.} \revp{Different environments yield different descriptors, and hence different balances between ITL and JTL. We call this descriptor the \textbf{Growth Rate}.}

\begin{definition}[Growth Rate]
\label{ass:growth-rate}
    Let $\psi: \mathbb{N} \to \mathbb{R}^+$ be a monotonically increasing function such that the expected squared distance of the $k$-th task from the origin scales as
    \[
        \mathbb{E}[\|\theta_\star^k\|^2] \in \Theta(\psi(k)).
    \]
\end{definition}

The Growth Rate quantifies how fast the environment is expanding. If $\psi(k) = 1$ the environment is \textbf{bounded}, meaning a finite diameter $M$ contains all task-solution vectors. If $\psi(k) = k^2$ the environment has a \textbf{linear drift}: tasks move with constant velocity $\partial_k\theta_\star^k = v$. Between these extremes, a \textbf{diffusive} environment ($\psi(k) = k$) corresponds to tasks drifting via a random walk. \cref{fig:growth_rate_final} visualizes these three cases.

\rebuttal{For the quadratic environments, the Growth Rate and Instability are tightly connected, as we show in \cref{sec:mt_asymptotics} (\cref{lemma:centroid_growth,cor:growth_lower}).}
\rebuttalB{\rev{ The Growth Rate is an environment \emph{envelope} that can capture multiple trajectory shapes. }}

\begin{figure}[htbp]
    \centering
    \begin{minipage}{0.50\textwidth}
        \centering
        \includegraphics[width=\linewidth]{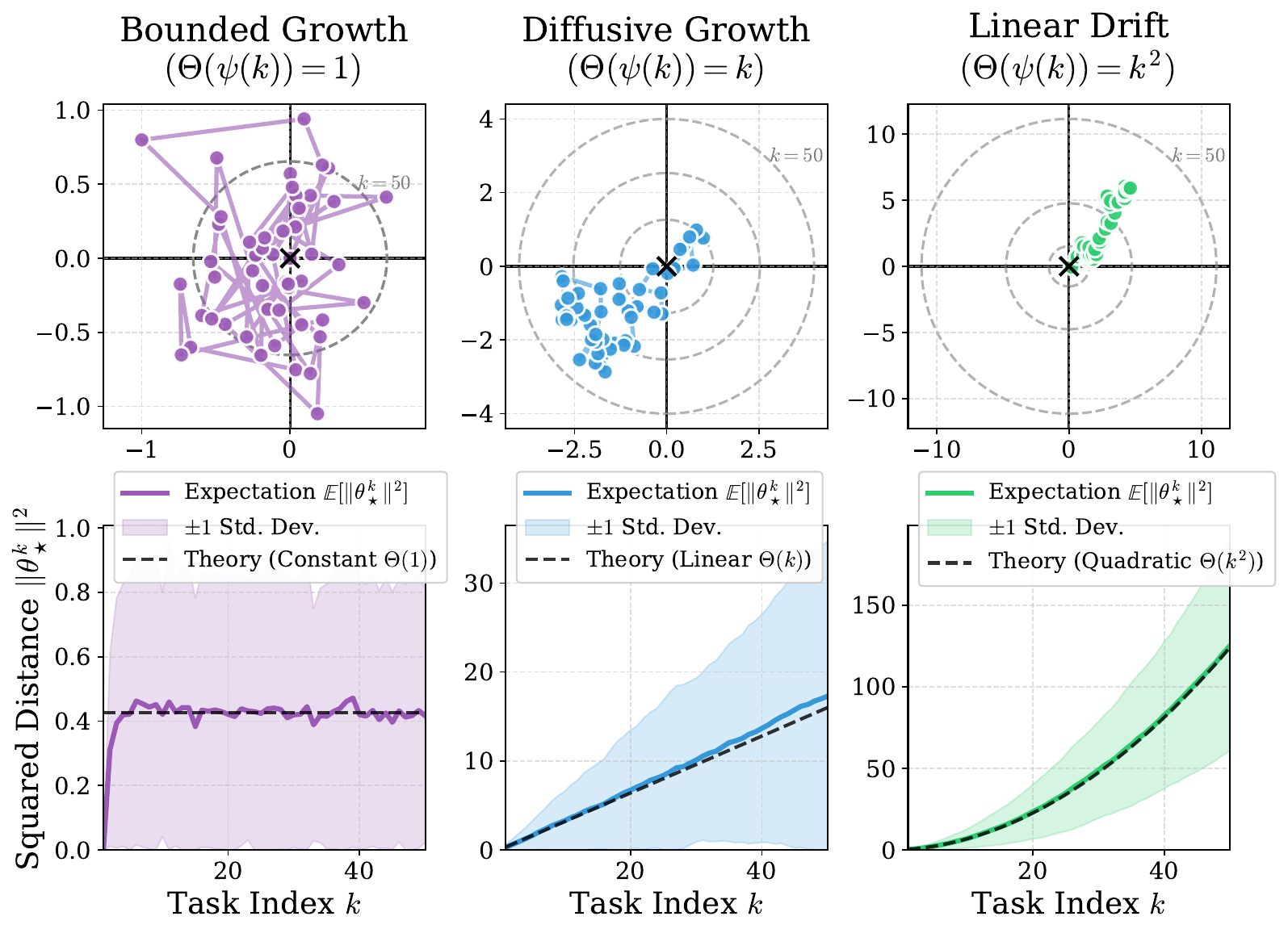}
    \end{minipage}
    \hspace{0.1cm}
    \begin{minipage}{0.46\textwidth}
    \caption{\textbf{Task Environment Growth Rates.}
    \emph{Top Row:} 2D spatial trajectories of the task optimal parameters $\theta_\star^k$ for a single random seed over $K=50$ tasks. Each panel uses its own spatial scale (note the differing axis ranges, roughly $\pm 1$, $\pm 4$, and $\pm 10$), so that the trajectory remains visible despite the very different expansion rates. The dashed gray circles represent the theoretical expected distance from the origin ($\sqrt{\mathbb{E}[\|\theta_\star^k\|^2]}$) at milestones $k \in \{5, 20, 50\}$, providing a common reference within each panel.
    \emph{Bottom Row:} The expected squared distance $\mathbb{E}[\|\theta_\star^k\|^2]$ as a function of the task index $k$. Solid lines: empirical expectation over 500 trajectories ($\pm 1$ standard deviation). Dashed black lines: theoretical growth rates. \vspace{-0.3cm}}
    \label{fig:growth_rate_final}
    \end{minipage}
    \vspace{-0.2cm}
\end{figure}

We can characterize the rate at which each component of the Transfer Efficiency scales, as a function of the Growth Rate. \revp{Putting everything together, we obtain} a sufficient asymptotic condition for positive transfer:

\begin{theorem}[\rev{Asymptotic behavior of} Transfer for Quadratic Objectives --- Informal]
\label{theo:asymptotic-transfer}
    In the setup described above, in an environment with Growth Rate $\Theta(\psi(k))$ we have
    \[
        \mathcal{I}^k \in \Theta(\psi(k)), \quad \delta^k_{ITL} \in \Theta\left(\tfrac{\psi(k)}{N_k}\right), \quad \delta^k_{JTL} \in \major{\mathcal{O}}\left(\tfrac{\psi(k)}{k^2} + \psi'(k)\right),
    \]
    where $\psi'(k)$ denotes the environment \emph{acceleration}. Consequently, for \emph{environments with polynomial growth rates} the expected Transfer Efficiency is positive for large $k$ if
    \[
        \mathcal{I}^k < C^k_{ITL}\,\frac{\mathbb{E}[\|\tilde{\mu}_0^{ITL,k}\|^2]}{N_k},
    \]
    where $C^k_{ITL}$ is a constant of the bound.
\end{theorem}
An extended version of this result and its proof can be found in \cref{sec:mt_asymptotics}.

The three rates have a clear interpretation. The Instability inherits the full $\psi(k)$ scaling: a faster-expanding environment makes the JTL stationary compromise progressively worse on each new task. The ITL Transient is the full $\psi(k)$ paid back through $N_k$ steps: when each task is long, the comparative cost of restarting learning from scratch is low. The JTL Transient, by contrast, is dominated by the much smaller terms $\psi(k)/k^2$ and $\psi'(k)$: as $k$ grows, the JTL agent only has to track how the environment is \emph{moving}, not where it currently \emph{is}, which is comparatively cheap.

For polynomial growth, Transfer Efficiency is therefore asymptotically dominated by two competing terms: the Instability $\mathcal{I}^k$ (the JTL stationary-bias penalty) and the ITL Transient $\delta^k_{ITL}$ (the ITL reset penalty); the JTL Transient is lower-order in the regimes we analyze. \rev{The sufficient condition in \cref{theo:asymptotic-transfer} expresses this trade-off:} transfer is positive when the JTL stationary penalty is paid back within the budget the ITL agent needs to warm up. \rev{This recovers the conventional wisdom that joint-task training is most valuable in the ``\emph{few-shot}'' regime (where $N_k$ is small and ITL has no time to recover from a cold start) and ceases} to help when the environment is so unpredictable that the bias cost dominates.

\rebuttal{\paragraph{Beyond quadratic tasks.} Outside the quadratic regime, parameter-space distances lose meaning, and the Growth Rate defined in \cref{ass:growth-rate} no longer captures the rate at which the environment departs from the learner's accumulated solution. Nevertheless, the general concept survives beyond the quadratic tasks under model-class-appropriate measures. For example, measures akin to the Growth Rate already exist in the broader literature: variation budgets in online learning \citep{besbes2015non}, distributional discrepancies between successive tasks \citep{mansour2009domain, bendavid2010theory}, curvature- or NTK-overlap similarity for neural networks \citep{doan2021theoretical}, and transition and reward variation budgets in non-stationary reinforcement learning \citep{cheung2020reinforcement}}.

\subsection{The Continuous Drift Model}
\label{ss:continuous-drift}

\begin{figure}[t]
    \centering
    \definecolor{pilltaskgreen}{RGB}{194,243,222}
    \begin{tikzpicture}
        \node[anchor=south west,inner sep=0] (sshallow) at (0,0)
            {\includegraphics[width=0.9\linewidth]{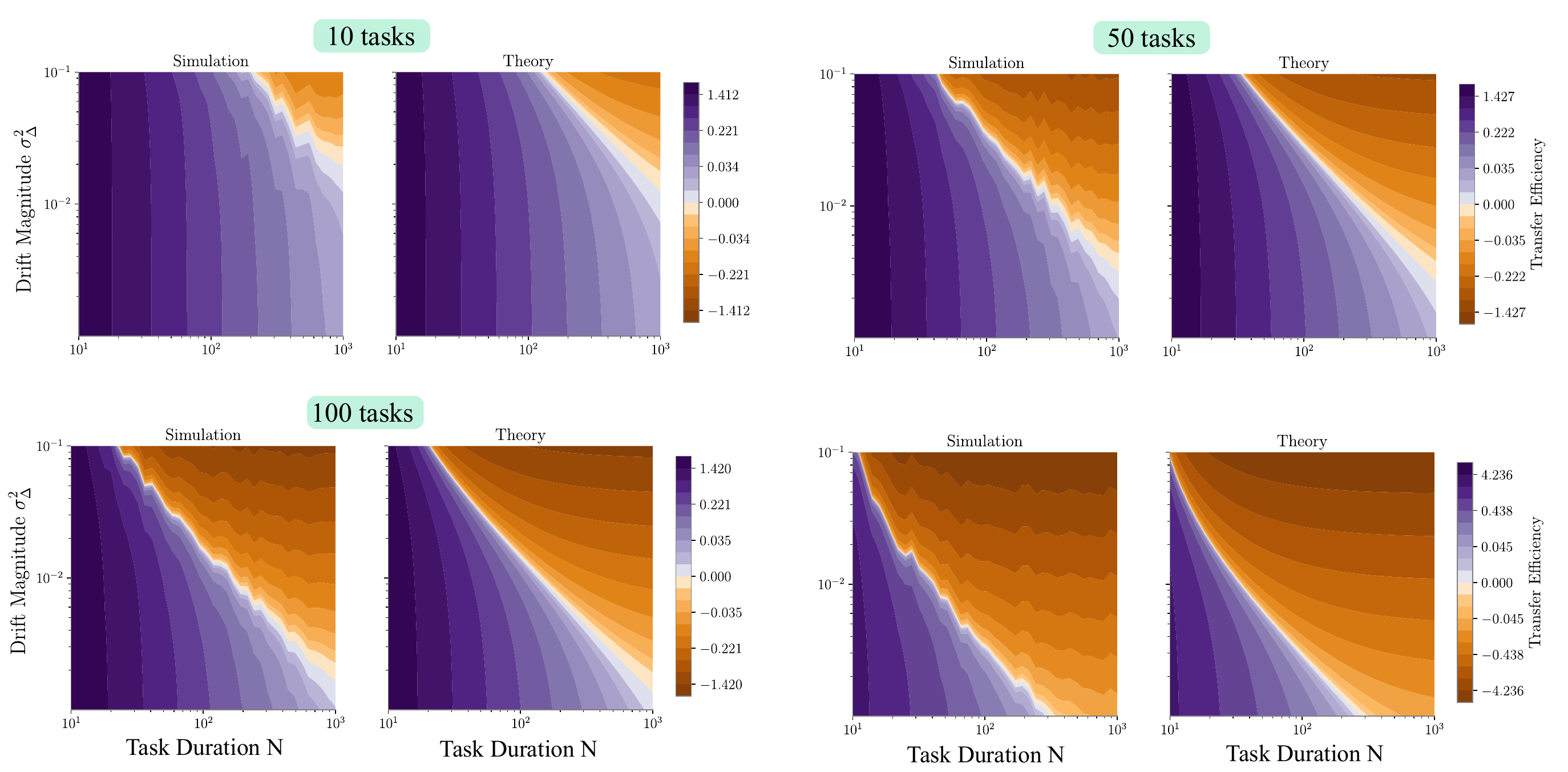}};
        \begin{scope}[x={(sshallow.south east)},y={(sshallow.north west)}]
            \node[anchor=center, rounded corners=2pt, fill=pilltaskgreen, inner sep=0,
                  minimum width=0.0665\linewidth, minimum height=0.0188\linewidth]
                  at (0.734,0.466) {\resizebox{!}{0.0106\linewidth}{500 tasks}};
        \end{scope}
    \end{tikzpicture}
    \caption{\textbf{Transfer Efficiency in the Continuous Drift Model.} Simulated vs.\ theoretical TE across task durations ($N$, x-axis) and drift magnitudes ($\sigma_\Delta^2$, y-axis), at $k\in\{10,50,100,500\}$. \textcolor{violet}{Violet}: positive transfer; \textcolor{orange}{orange}: negative transfer. As $k$ grows, the positive-transfer region shrinks, tracking the theoretical boundary from \cref{eq:master_inequality_app}.}
    \label{fig:sim-shallow}
\end{figure}
To make the asymptotic story concrete, we instantiate the framework on a \textbf{diffusive regime}: a Wiener-walk model where optimal parameters evolve as $\theta_\star^k = \theta_\star^{k-1} + \Delta_k$ with $\Delta_k \sim \mathcal{N}(0, \Sigma_\Delta)$, and drift magnitude $\sigma_\Delta^2 := \text{Tr}(\Sigma_\Delta)$ controlling both how far the process may travel and how fast the relevance of past data decays. This regime captures real-world continuous distribution shifts such as the slow evolution of user preferences or weather variation in autonomous navigation. The full setup (Gaussian inputs with covariance $\Sigma_x$, label noise $\sigma_\xi^2$) is in \cref{sec:drift_analysis}.

Applying the exact Instability and Transient Error formulas to this setup yields a \textbf{non-asymptotic sufficient condition for positive transfer}: a single inequality balancing an \emph{ITL Transient cost} (which grows linearly in $k$ but decays as $1/N$), a \emph{JTL Instability cost} (which grows linearly in $k$ with no compensating $1/N$ factor), and a smaller \emph{JTL Transient cost} (which decays in both $k$ and $N$). The three terms mirror \cref{theo:asymptotic-transfer}; the full expression is given in \cref{eq:master_inequality_app}.

Simulations of this model (\cref{fig:sim-shallow}; details in \cref{apsec:simulations}) confirm the boundary at finite $k$ and show the positive-transfer region shrinking as $k$ grows, since the JTL Transient is lower-order than the linear-in-$k$ terms. \rebuttalB{The same simulation protocol empirically validates the relaxations in \cref{sec:relaxing}.}
In the limit $k\to\infty$, the condition \rev{reduces to a closed form for} the Critical Task Duration:
\begin{equation}
\label{eq:critical-task-dur}
N_{\max}^\infty = \frac{3\gamma_{\min}\,\sigma^2_\Delta}{2\eta\,\text{Tr}(\Sigma_x \Sigma_\Delta)}.
\end{equation}
This asymptotic boundary is set by the geometry of the \emph{environment} together with the \emph{optimizer}: it depends on the input--drift coupling $\text{Tr}(\Sigma_x \Sigma_\Delta)$, the loss curvature $\gamma_{\min}$, and the learning rate $\eta$ \rev{(smaller $\eta$ favors biased} initializations such as JTL), while being invariant to the overall drift magnitude $\sigma^2_\Delta$ and independent of the batch size $B$. \rev{The full asymptotic picture and} a trajectory-level view of drift eroding the warm-start advantage are in \cref{fig:asymptotic-simulation,fig:MT_Drift_Analysis}.

\subsection{The Optimization Bias Matters}
\label{ss:continuous-drift-beyond}

The closed-form analysis above takes the function class as given. A complementary question is how much of the transfer profile is set by the function class versus the \emph{optimization bias} of the algorithm that fits it. We isolate this by extending the Continuous Drift Model to a deep-linear teacher-student setup (following \citealt{saxe}; details in \cref{apsec:simulations}) and comparing two parametrizations of the same linear function class: NTP (\emph{lazy regime}, \citealt{Jacot18NTK}) and $\mu$P (\emph{rich regime}).

\begin{figure}[h]
    \centering
    \includegraphics[width=0.8\linewidth]{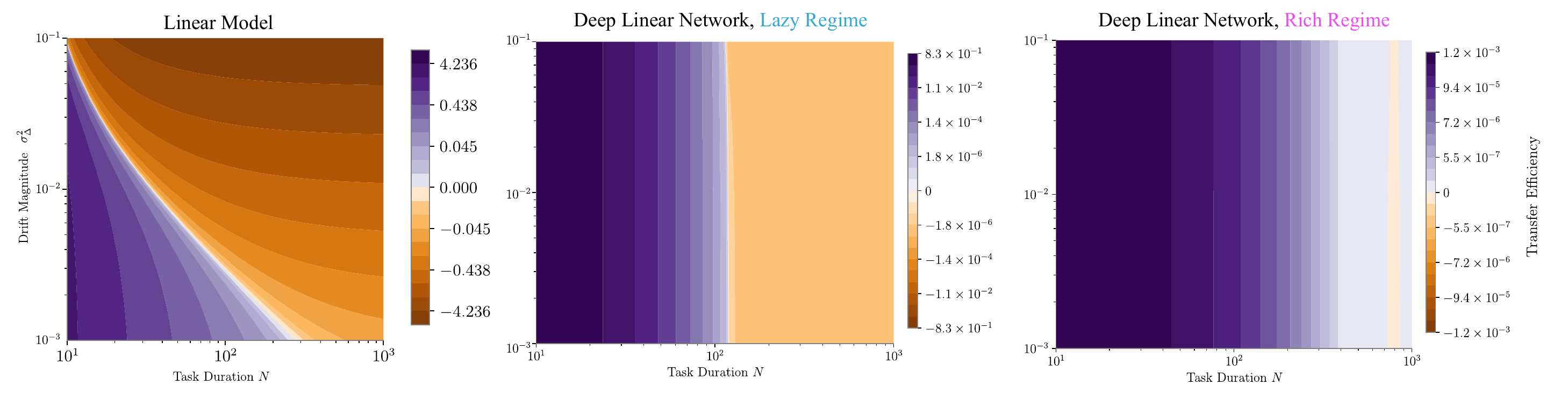}
    \caption{\textbf{Optimization bias shifts the Transfer Efficiency front, at fixed function class.} TE in the Continuous Drift Model versus $N$ and $\sigma_\Delta^2$. \emph{Left:} convex Linear Model. \emph{Middle:} Deep Linear Network, Lazy Regime (NTP). \emph{Right:} Deep Linear Network, Rich Regime ($\mu$P). The function class is identical across the three panels; the parametrization is not.}
    \label{fig:optimization_bias}
\end{figure}

\cref{fig:optimization_bias} shows that the three TE fronts differ markedly even though the function class is the same; per-width fronts are in \cref{fig:sim-deep-lazy,fig:sim-deep-rich}. \rev{The individual components (\cref{fig:insta-transient-scaling}) show two patterns consistent with} \cref{theo:asymptotic-transfer}: \emph{Instability is several orders of magnitude higher under NTP}, yielding substantially lower JTL performance; and \emph{the JTL Transient is several orders of magnitude below the ITL Transient and Instability at every width}. \rev{This has two implications. First, the framework still holds (Transfer Efficiency is governed by $-\mathcal{I}$ plus a vanishing transient gap), but the value of $\mathcal{I}$ depends on implementation-level choices as well as on the task sequence. Second, it shows the limit} of any purely \emph{statistical} account of transfer (cf.\ the early Learning-to-Learn literature): the design of adaptive algorithms must jointly consider the stochastic processes defined by the data and by the optimizer.

\section{Generalizing CL beyond the JTL/ITL dichotomy}
\label{sec:generalizing}
\subsection{From implicit to explicit future models}
\label{ssec:predictive-cl}

We have shown, empirically and theoretically, that neither the JTL nor the ITL algorithm is universally optimal: each minimizes the Average Lifelong Error (\cref{eq:SLO}) only under specific environmental regimes. \rev{This raises the question:} \emph{is there a single optimal objective for continual learning, one that provably minimizes the ALE?}

Foundational results in online convex optimization \citep{besbes2015non} established that, when the environment is unpredictable (i.e.\ adversarial), \emph{no first-order stochastic method} can outperform online gradient descent \citep{zinkevich2003online}. \rev{However, \citet{rakhlinOnlineLearningPredictable2013} show that much more is possible when the future is somewhat \emph{predictable}.} If $\mathcal{P}_{n+1}$ were known, we could choose $\mathcal{D}_n$ to match it and minimize the next-step loss directly\rev{. In general $\mathcal{P}_{n+1}$ is unknown, but one can always \emph{approximate the future} under some assumptions.}

This lets us reinterpret the sampling choice (the choice of $\mathcal{D}_n$), and more generally the choice of objective, as an implicit assumption about the future. For JTL, the assumption is that \emph{the environment is ergodic and the future will resemble the past}\rev{,} the classical i.i.d.\ premise underlying \emph{empirical risk minimization}. For ITL, the assumption is the opposite: that no structure can be reliably extrapolated, in which case online gradient descent on the current task is minimax-optimal \citep{zinkevich2003online,besbes2015non}.

Algorithms that maximize average accuracy or minimize forgetting can be viewed as implicit (or, for replay-based methods, explicit) approximations of the JTL objective. \rev{This reading opens a new way to analyze CL algorithms and to \emph{build} them.} \revp{The same framing also suggests a constructive path: rather than fix} an objective implicitly, one can make the future model an explicit design choice. We call the resulting family \emph{Predictive CL algorithms}.

\begin{definition}[Predictive CL Algorithm]
\label{def:predictive-cl}
    \rebuttalB{Let $\mathcal{H}_{\le k}$ denote the history observable during task $k$, the same conditioning as the ALE objective (\cref{eq:SLO}). A prediction mapping $\Phi$ maps this information to a predicted distribution for the current task, $\tilde{\mathcal{P}}^k = \Phi(\mathcal{H}_{\le k})$.} A \emph{Predictive CL algorithm} optimizes its parameters by minimizing the expected loss under the predicted distribution $\tilde{\mathcal{P}}^k$:
    \begin{align}
    \label{eq:GCLA}
    &\inf_\mathcal{A}\,\, \frac{1}{N_k} \sum_{n=a_k}^{b_k}
    \mathbb{E}_{\theta_n\sim\mathcal{A}(\cdot\mid H_{n}) \,}\mathbb{E}_{o\sim\tilde{\mathcal{P}}^k}\,
    \left[\ell(\theta_n,o)\right]\qquad\text{s.t. computational constraints}
\end{align}
\end{definition}

\cref{def:predictive-cl} generalizes existing CL objectives by making the future model explicit. JTL is recovered by choosing $\Phi(\mathcal{H}_{\rebuttalB{\le} k}) = \sum_{j\le k} w_j^k\,\mathcal{P}^j$ (the empirical past is the prediction); ITL is recovered by $\Phi(\mathcal{H}_{\rebuttalB{\le} k}) = \mathcal{P}^k$ (the current task is the prediction). The Window algorithm introduced below sits between these limits. \rev{The mapping $\Phi$ need not be fixed:} it can itself be a learned model of the future.

\subsection{Proof of Concept: the Window Predictive CL Algorithm}
\label{ssec:window-poc}

\begin{figure}[ht]
    \centering
    \includegraphics[width=0.24\linewidth]{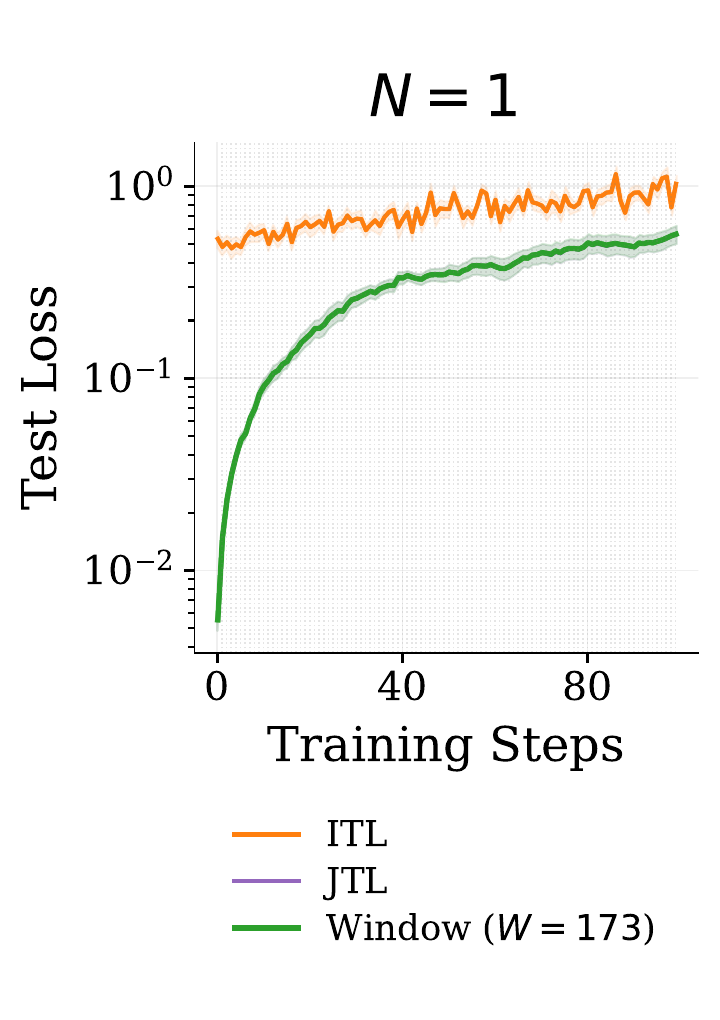}
    \includegraphics[width=0.24\linewidth]{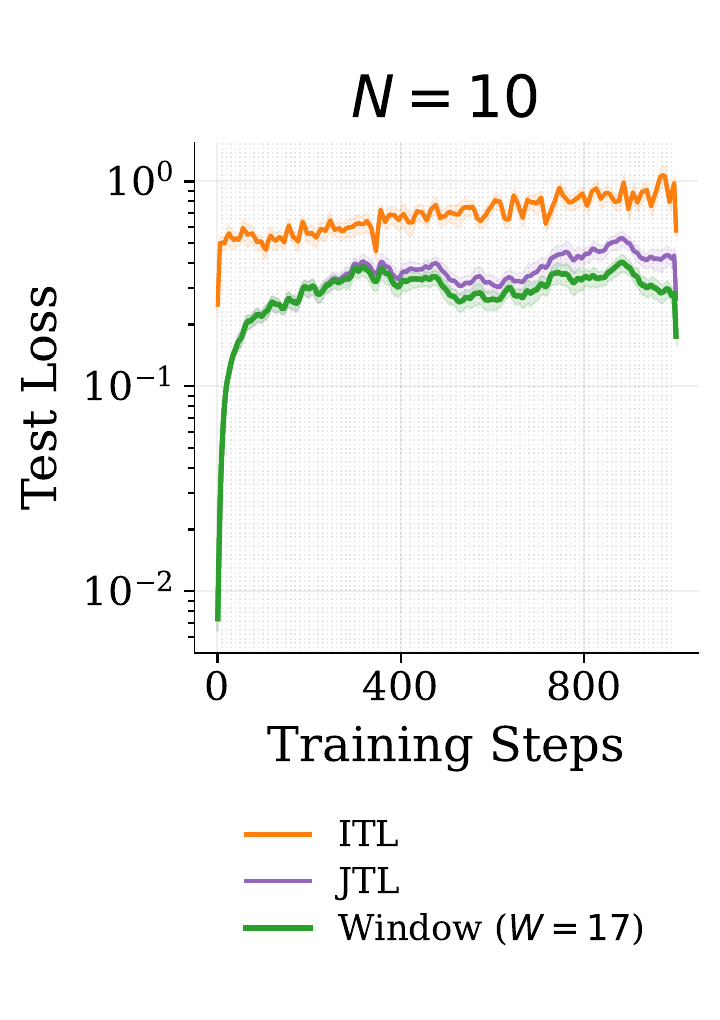}
    \includegraphics[width=0.24\linewidth]{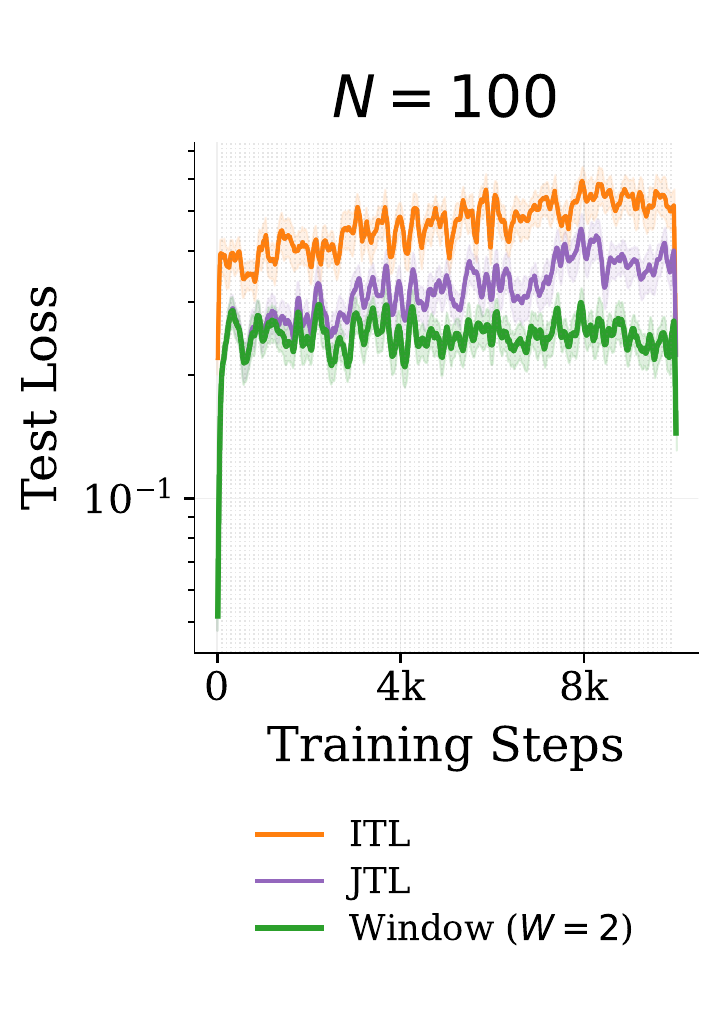}
    \includegraphics[width=0.24\linewidth]{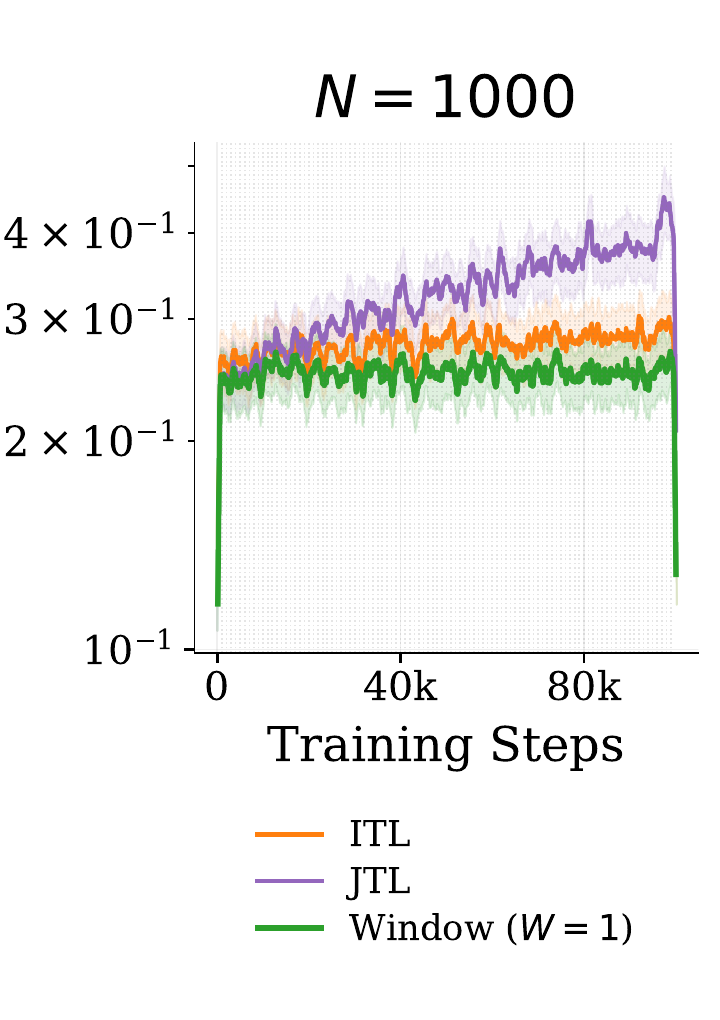}
    \caption{\textbf{Real-Time Performance of the Window Algorithm.} Test loss trajectories for the ITL (\textcolor{orange}{orange}), JTL (\textcolor{violet}{violet}), and Window (\textcolor{ForestGreen}{green}) agents across varying task durations $N$. The curves are smoothed by averaging over $N$ steps. By setting the window size to the theoretical optimum $W^\star$ derived in \cref{eq:optimal-window}, the Window agent consistently matches or outperforms both baselines. \rev{The optimal memory horizon} is inversely proportional to task duration: shorter tasks ($N=1$, left) call for a large memory window ($W^\star=173$, effectively recovering JTL), whereas longer tasks ($N=1000$, right) favor a nearly memoryless approach ($W^\star=1$, equivalent to ITL with warm starts). The intermediate panels follow the same trend ($W^\star=17$ at $N=10$ and $W^\star=2$ at $N=100$).} \vspace{-0.2cm}
    \label{fig:sim-W-ITL-JTL}
\end{figure}

\rebuttalB{\rev{This section gives a proof of concept: with an explicit, data-dependent model of the future (here a finite window over recent tasks), the resulting sampling distribution \major{matches or improves on} the Average Lifelong Error of both the ITL and JTL extremes} in the Continuous Drift Model (\cref{fig:sim-W-ITL-JTL}).}

\rev{Consider again the quadratic losses of \cref{sec:quadratic-losses}. As an example, we instantiate a Predictive CL algorithm that we call} the \emph{Window} (W) algorithm. The Window algorithm predicts the future using a finite recent history: it is defined by a \rev{hyperparameter} $W$, which we call \emph{window size}. For each task $k$ the SGD sampling distribution $\mathcal{D}^k_W$ is a mix of the past $W$ tasks\looseness=-1
\[
\mathcal{D}^k_W = \sum_{j = k-W+1}^k w_j^k \,\mathcal{P}^j.\]
\rev{Like the JTL agent, the Window agent does not reset its parameters at the beginning of a new task. As $W$ grows, the Window agent converges to the JTL agent (which can be seen as an \emph{infinite window} agent), whereas} for $W = 1$ it coincides with an \emph{ITL agent with warm starts}.

\rev{The bias of the JTL agent is set by the global growth of the task solutions relative to the origin, whereas the Window agent's bias depends on local variations of the environment over the horizon $W$. To capture this bias asymptotically, we introduce a new measure,} akin to the Growth Rate, which we call the \textbf{Mean Squared Displacement}.

\begin{definition}[Mean Squared Displacement (MSD)]
\label{def:msd}
    The Mean Squared Displacement $m: \mathbb{N} \times \mathbb{N} \to \mathbb{R}^+$ characterizes the expected squared Euclidean distance between two task solutions separated by a lag $W$ at time $k$:
    \begin{equation}
        \mathbb{E}[\|\theta_\star^k - \theta_\star^{k-W}\|^2] \in \Theta(m(k, W)).
    \end{equation}
\end{definition}
\rev{Using the MSD, we can compute the asymptotic Transfer Efficiency of the Window algorithm relative to ITL training (as in \cref{ss:continuous-drift}) and solve for the window size $W$ that maximizes the Transfer Efficiency.}
\begin{lemma}[Asymptotic Optimality of Window Size]
    \label{lemma:optimal_window_main}
    Let $m(k,W)$ be the Mean Squared Displacement of tasks within a window $W$. The optimal window size $W^*$ maximizing transfer depends on the environment's non-stationarity regime. In particular\rev{,} consider three examples:
    \begin{enumerate}
        \item \emph{Bounded Regime ($\psi(k)=1$):} If $m(k,W) \in \Theta(1)$,  transfer increases monotonically with $W$. The asymptotically optimal strategy is the JTL agent ($W^* \to \infty$).
        \item \emph{Linear Drift Regime ($\psi(k)=k^2$):} If $m(k,W) \in \Theta(W^2)$, transfer decreases monotonically with $W$. The asymptotically optimal strategy is the ITL agent ($W^* = 1$).
        \item \emph{Diffusive Regime ($\psi(k)=k$):} If $m(k,W) \in \Theta(W)$, transfer is strictly \major{concave} in $W$. The asymptotically optimal strategy is a Window agent with finite memory ($1 < W^* < \infty$).
    \end{enumerate}
\end{lemma}

This result formally establishes that \textbf{the best objective for real-time performance depends on the environment and the model} ($m(k,W)$ is model-dependent). The bounded and linear-drift cases recover the JTL and ITL extremes respectively; the most interesting case is the diffusive regime, where the optimal strategy is neither and depends on further environmental parameters.

For the Continuous Drift Model of \cref{ss:continuous-drift}, we derive a non-asymptotic closed form for $W^\star$ (\cref{eq:optimal-window}; full derivation in \cref{sec:optimal_window}). \rev{The optimum $W^\star$ depends on the \emph{learning rate} $\eta$ and not only on environment parameters, which confirms that} the optimal objective for fast adaptation is jointly determined by the environment and the learning algorithm.

\rev{We simulate the Window agent in the Continuous Drift Model} (details in \cref{apsec:simulations}). In \cref{fig:sim-W-ITL-JTL} we compare the loss of the ITL, JTL and Window agents with optimal window size: \rev{the Window agent \major{matches or outperforms} both baselines at every tested task duration.} The optimal window \emph{shrinks} as $N$ grows, recovering ITL-with-warm-starts for large $N$ and JTL for sufficiently small $N$. \rev{The Transfer Efficiency across all window sizes and drift magnitudes} (\cref{fig:sim-W-TE} in the appendix) confirms these theoretical predictions.

\paragraph{Beyond the Window: richer future models.} The Window algorithm is a deliberately simple instance of \cref{def:predictive-cl}: it encodes the prior that recent tasks are the most informative about the next one, summarized by a single scalar $W$. The same framework accommodates far richer prediction mappings \rev{$\Phi$: for instance,} a \emph{generative model} that produces samples resembling future tasks, a learned \emph{sufficient summary} of the history tailored to next-task prediction, \rebuttalB{or a learned, selective replay-buffer distribution that does not sample the past uniformly \citep{schaul2016prioritized}}. \rev{In each case, the design question becomes \emph{which statistics of the past predict the future}, and the bias-variance trade-off formalized by our framework gives the criterion for comparing competing $\Phi$.}  We leave the exploration of such mappings to future work.

\section{Interdisciplinary Literature Review}
\label{sec:related-work-new}

\rev{One aim of this work is to connect continual learning with adjacent literatures whose concepts and theoretical results bear on its core questions. The contribution goes both ways: it gives CL new analytical tools, and it offers these fields fresh problems on which to test and apply their results.}

\paragraph{Online Learning and Stochastic Optimization.}
Our analysis sits at the intersection of Online Convex Optimization (OCO) and Stochastic Optimization \citep{cesa2006prediction,hazan2016introduction,orabona2019modern}. The Average Lifelong Error metric is a finite-horizon analogue of the OCO \emph{static regret} \citep{zinkevich2003online}, and across multiple tasks the average ALE corresponds to a notion of \emph{dynamic regret} \citep{herbster1998tracking,besbes2015non}. \citet{besbes2015non} establish a minimax dynamic regret bound of $O(T^{2/3} V_T^{1/3})$ that grows with the comparator's path variation $V_T$, formalizing how non-stationarity inflates regret with time. \rev{In our task-indexed setting, the \emph{Growth Rate} $\psi(k)$ plays the role of $V_T$:} both quantify how non-stationarity accumulates as the environment expands. Our Critical Task Duration $N_{\max}^\infty$ \rev{then converts this quantity into a threshold on task duration, beyond which the environment's Growth Rate $\psi(k)$ makes accumulating past observations harmful relative to learning each task afresh.}

Our transient analysis of the JTL and ITL agents recovers and extends classical Lyapunov-type results on constant step-size SGD \citep{murataStatisticalStudyOnline1999,bach,mandt2017stochastic}; the Transient Potential Law (Theorem~\ref{theo:transient_potential}) makes the geometric role of the relative-curvature matrix $\mathcal{R}^k$ in those bounds explicit.

\rev{A central idea we import from this literature is} the framework of \emph{predictable sequences} \citep{rakhlinOnlineLearningPredictable2013}: an online learner can improve over the worst-case minimax dynamic regret of \citet{zinkevich2003online,besbes2015non} when given a useful hypothesis about future losses. Our generalization of CL beyond JTL (\cref{sec:generalizing}) \rev{is this idea applied to non-stationary task sequences. From this perspective, JTL (and with it most of the continual learning algorithms it embodies) is \revp{not a new algorithm but} a particular instance of} predictable-sequence learning under an ergodicity hypothesis.

\paragraph{Multi-Task Learning (MTL).}
The comparison of JTL and ITL aligns with a foundational question of Multi-Task Learning: does jointly training over multiple tasks benefit performance relative to learning each independently? JTL, as we use it, is best understood as an \emph{online, cumulative} version of MTL: the task pool grows over time, and joint training is performed online over the accumulated history rather than offline on a fixed pool. The classical MTL literature operates in a different regime, typically a fixed task pool sampled i.i.d.\ and \rev{evaluated post hoc, but the underlying tension between sharing and interference is the same.}

\rev{Statistical results show} that shared representations reduce sample complexity when the task family is self-contained \citep{baxterModelInductiveBias2000,maurerbenefit,tripuraneni2020theory,du2020few}, \rev{but they abstract away the optimization process. A separate line studies} optimization-time interference and gradient conflict \citep{standley2020tasks,sener2018multi,yu2020gradient,liu2023nonstationary}. Neither line jointly tracks environment and optimization processes as they co-evolve, nor formalizes the resulting trade-off through an ALE-like metric.

\paragraph{Learning To Learn and Online Meta-Learning.}
Our framework has its closest ties to Learning-to-Learn and Online Meta-Learning \citep{Kuzborskij,deneviLearningLearnCommon2018}. The \emph{Online-Within-Online} setting \citep{alquier2017regret,denevi2019learning,khodak2019adaptive,nazariDynamicRegretAnalysis2021} considers regret accumulated across a sequence of tasks \rev{adapted via inner-loop optimization, a quantity formally analogous to our ALE.} \citet{denevi2019learning,khodak2019adaptive} prove that the regret-optimal running prior is the centroid of past task minimizers; under our task-alignment assumption (\cref{ass:alignment}), \rev{this centroid coincides \revp{exactly} with} the JTL minimizer $\theta_\star^{JTL,k}$. \rev{The two algorithm families therefore provably share the same asymptotic target. The main difference is what they use:} meta-learning algorithms typically use past gradients or update statistics, whereas JTL uses the source data directly via replay, with consequences for variance and robustness beyond convex regimes.

A direct corollary of the MTL/MAML equivalence in \rev{over-parameterized} networks \citep{wang2021bridging} is that our Instability bound for JTL also bounds the \emph{asymptotic} interference cost of MAML; the transient trajectories of the two families, which are the focus of much of our analysis, remain distinct. Our Transfer Efficiency also reconciles the warm-start puzzle of \citet{ash2020warmstart}: warm-starting helps when Instability is small relative to the ITL Transient, and hurts when this inequality reverses.

\paragraph{Continual Learning.}
We finally situate our contribution within continual learning. The defining aim of modern CL is to mitigate \emph{catastrophic forgetting} \citep{parisi2019continual,hadsell2020embracing,delange2022survey}, pursued via mechanisms as diverse as regularization, replay, parameter isolation, expert mixtures, and prompt-based adaptation. Although these methods differ substantially in implementation, they implicitly share a common target: maximize average accuracy across all tasks seen so far. \rev{In our terminology, each is an approximation of the JTL objective, and the variety of CL methods reflects different practical strategies for approximating \revp{this common objective, not different objectives in themselves}.}

The closest theoretical neighbor of our work is the line of analysis on linear-model CL by \citet{evron2022catastrophic,evron2024joint}. \rev{Their analyses bound \emph{forgetting} for sequential gradient descent on quadratic losses. We share the quadratic regime and the gradient-descent algorithm but bound \emph{real-time error} (ALE) instead of retrospective forgetting, and obtain} the Critical Task Duration as the central object of study. Our analysis is most naturally interpreted under the \emph{domain-incremental} regime, where task identity shifts the input/output distribution while the label space is preserved; this matches both our experimental setting (CLEAR, MD5, Permuted/Shuffled CIFAR-10) and our quadratic-drift theory.

A complementary line of work, recently crystallized as the \emph{Big World Hypothesis} \citep{javed2024bigworld}, holds that the world is too large for any bounded learner to model exhaustively\rev{, so that continual learning is a problem of} \emph{tracking} rather than memorization. This view has technical roots in the early work of \citet{sutton2007tracking}, and is empirically supported by recent work on the \emph{loss of plasticity} \citep{dohareLossPlasticityDeep2024a,dohare2021continual}. The ALE metric is itself a tracking objective, and our framework can be read as the optimization-theoretic instantiation of this hypothesis.

Within Online Continual Learning \citep{cai2021online,lopez2017gradient,buzzega2020dark}, \rev{\citet{cai2021online} observed that learning efficacy and information retention conflict; our Transfer Efficiency decomposition turns this observation into a quantitative trade-off} mediated by the Critical Task Duration. \citet{kumar2023continual} provide the information-theoretic ALE framing under which we operate; our contribution complements theirs by characterizing which implementable algorithms approach the information-theoretic ceiling and under what conditions. The framework also clarifies the relationship to non-stationary RL theory \citep{lecarpentier2019non,lecarpentier2021lipschitz,khetarpal2022towards}, where similar adaptation--stability trade-offs appear in the context of value-function tracking.

\section{Discussion and Conclusion}
\label{sec:discussion}

The \rev{continual learning} (CL) literature has long been framed around mitigating catastrophic forgetting. This perspective implicitly assumes that the ideal lifelong learner is the Joint-Task Learning (JTL) solution\rev{, a model trained jointly on all data observed so far. We challenge this assumption and propose a more general} formulation of continual learning that shifts the focus from knowledge retention to real-time adaptation.

Our starting point is the Average Lifelong Error (ALE). \citet{kumar2023continual} formalized the trade-offs underlying ALE through an information-theoretic lens, establishing the performance ceiling that a compute-bounded learner can hope to approach. Our framework \emph{complements} this view by characterizing the \emph{optimization-time} realization of the same trade-off: which implementable algorithms can approach this ceiling, and under what conditions. We show that real-time performance is governed by the interaction of two stochastic processes\rev{, one describing the environment and the other the learner, and} that implementation-level choices (the optimizer, the model, its parametrization, e.g.\ NTK versus $\mu$P) directly shape the balance between positive and negative transfer.

We compared the JTL agent (which assumes an ergodic and stable past) with the Independent-Task Learning (ITL) agent (which assumes a fully volatile environment) through the lens of \emph{Transfer Efficiency} ($\mathrm{TE}$), a quantity that decomposes into two largely independent components: \emph{Instability} (the stationary bias from optimizing over conflicting past tasks) and \emph{Transient Error} (the cost of learning from a naive initialization). Under mild convergence conditions and whenever the average Instability is positive, this decomposition yields a \emph{Critical Task Duration} ($N_{\max}^\infty$): a threshold beyond which the warm-start advantage of JTL is overtaken by the ITL agent's ability to fit the current environment.

In the quadratic regime, this decomposition admits closed-form expressions that make the tension between stationary bias and the cost of learning fully transparent. Two takeaways carry over to general non-convex settings: (1) the environment's \rev{\emph{Growth Rate}, the speed at which it expands, is} the decisive parameter for the balance between positive and negative transfer; and (2) once the number of tasks is sufficiently large, transfer is primarily determined by the relative weight of the Instability and the ITL Transient, with the JTL Transient acting as a lower-order correction. We empirically verified that these mechanisms persist on continual image-classification benchmarks (CLEAR and MD5) and the reinforcement-learning benchmark MT10. \looseness=-1\vspace{-0cm}

\rev{Viewing the JTL algorithm as a specific instance of \emph{predictive} online learning, which rests on the assumption of predictable sequences, suggests} a natural generalization: a broader family of \emph{Predictive CL algorithms} that optimize \emph{explicit, environment-specific predictions of the future}. As a proof of concept, we introduced the Window \rev{algorithm}, which \revp{sets} its memory horizon \revp{from} the environment's Mean Squared Displacement (MSD), a local, interval-based analogue of the Growth Rate. In controlled drift settings, \rev{this simple approach \major{matches or outperforms} both JTL and ITL.

Because the prediction mapping can itself be learned, future algorithms need not commit to a fixed objective a priori; they could instead infer the most appropriate predictive target \revp{directly} from experience. We believe this is a promising direction for continual learning, one that adds prediction as an axis for designing adaptive learning algorithms.} \rebuttalB{\rev{Predictive CL is a design space that contains ITL, JTL, the Window algorithm, and many other algorithms that we do not instantiate here.}}

\major{\paragraph{What is proven and what is proposed.} Our formal results are the decomposition of Transfer Efficiency into Instability and Transient Error (\cref{lemma:TE-decomposition}), the Critical Task Duration (\cref{lem:critical_TD}), the closed forms and Growth-Rate conditions for quadratic losses (\cref{sec:quadratic-losses}) with their relaxations in \cref{sec:relaxing}, and the optimal window of \cref{lemma:optimal_window_main}. Each holds under the assumptions stated with it. Two things are proposed rather than proven: that the same mechanisms govern deep networks, which our experiments support but do not prove, and Predictive CL as a way to design algorithms.}
\rebuttalB{\rev{The message of this paper is not that full retention is harmful; we prove that it is not always optimal, which is a different claim. Our claims also take a performance-first perspective:} when memory matters more than real-time performance, the retention-centered standard rightly remains.}

\rev{The broader contribution of this work is conceptual. Our framework (\emph{Transfer Efficiency, Instability,} and \emph{Transient Error}), together with the Predictive CL perspective,} helps unify continual learning with closely related areas such as online learning, stochastic optimization, multi-task learning, and meta-learning. \rev{Beyond new metrics, it gives the continual learning community a common language and analytical toolkit for studying real-time adaptation in non-stationary environments.}

\paragraph{Limitations}
\emph{Choice of objective.} Our conclusions depend on the choice of evaluation objective. ALE captures a natural notion of real-time adaptation under computational constraints, but alternative desiderata such as robustness, coverage, fairness, or worst-case performance may favor different learning strategies and retention mechanisms. Accordingly, we do not claim that ALE is the universally correct objective for continual learning; quantities such as the Critical Task Duration and Transfer Efficiency should be interpreted as properties of the ALE objective rather than universal characteristics of continual learning.

\emph{Scope of the theory.} \rebuttal{The alignment and bounded-curvature assumptions are not load-bearing: \cref{sec:relaxing} derives exact counterparts of the Instability, the Transient Error, and the Transfer Efficiency for non-commuting task Hessians and for rank-deficient (over-parameterized) task covariances.} \rebuttalB{Nonetheless, our theoretical results hold in the limited quadratic-task regime, and significant work is necessary to extend them beyond it.

\revp{\emph{Regime changes.}} The Growth Rate analysis further assumes that the statistics of the drift increments are stable at the timescale being classified. \rev{This is weaker than it may appear. The setting is piecewise-stationary by construction, so a sudden change of the data distribution is a task boundary rather than a failure of the taxonomy. A change of the drift statistics themselves either} re-partitions the task sequence (if isolated and bounded, it is absorbed into the envelope's constants) or defines a growth envelope of its own at the coarser timescale of the change points (if recurring). \rev{Such regime changes do invalidate any static choice tuned to a single regime, in particular the static optimal window of \cref{lemma:optimal_window_main}; the remedy is to adapt the window, or the prediction mapping itself (see} the discussion closing \cref{ssec:window-poc}).}

\emph{Scope of the experiments.} \rebuttalB{Our experimental setup does not cover all realistic settings and \rev{architectures. Benchmarking practical CL methods under the ALE objective is left to future work; this includes the warm-start ITL baseline (simple fine-tuning across tasks), which separates the benefit of carried parameters from that of replayed data.}}

\emph{Predictive CL.} \rebuttalB{Predictive CL is presented as a design space rather than a specific algorithm: we leave to future work the design of effective predictive models of the future, whether learned from the stream (for instance, by estimating the drift statistics online) or derived from structural assumptions that hold across broad families of environments.}

\subsubsection*{\revp{Acknowledgments}}
\revp{This research was enabled in part by support provided by (Calcul Québec) (https://www.calculquebec.ca/en/), the Digital Research Alliance of Canada (https://alliancecan.ca/en), and CIFAR (Canada AI Chair; Learning in Machine and Brains Fellowship). The authors acknowledge the material support of NVIDIA in the form of computational resources. GL would like to acknowledge the support of the AI Center PhD fellowship. Additionally, MS was supported by  Fonds de recherche du Québec – Nature et Technologies (FRQNT) \#2023-2024-B2X-336394. CV is funded by the Deutsche Forschungsgemeinschaft (DFG) under both the project 468806714 of the Emmy Noether Programme and under Germany’s Excellence Strategy – EXC number 2064/1 – Project number 390727645. CV also thanks the international Max Planck Research School for Intelligent Systems (IMPRS-IS). We are grateful to David Abel for his generous and insightful feedback on earlier results of this work, which have made this paper better in many ways.}

\clearpage
\bibliography{main,version2_extra}

@article{khetarpal2022towards,
  title={Towards continual reinforcement learning: A review and perspectives},
  author={Khetarpal, Khimya and Riemer, Matthew and Rish, Irina and Precup, Doina},
  journal={Journal of Artificial Intelligence Research},
  volume={75},
  pages={1401--1476},
  year={2022}
}

@article{caruana1997multitask,
  title={Multitask learning},
  author={Caruana, Rich},
  journal={Machine learning},
  volume={28},
  pages={41--75},
  year={1997},
  publisher={Springer}
}

@inproceedings{gunasekar2018characterizing,
  title={Characterizing implicit bias in terms of optimization geometry},
  author={Gunasekar, Suriya and Lee, Jason and Soudry, Daniel and Srebro, Nathan},
  booktitle={International Conference on Machine Learning},
  pages={1832--1841},
  year={2018},
  organization={PMLR}
}

@article{yu2020gradient,
  title={Gradient surgery for multi-task learning},
  author={Yu, Tianhe and Kumar, Saurabh and Gupta, Abhishek and Levine, Sergey and Hausman, Karol and Finn, Chelsea},
  journal={Advances in Neural Information Processing Systems},
  volume={33},
  pages={5824--5836},
  year={2020}
}

@article{lecarpentier2019non,
  title={Non-stationary Markov decision processes, a worst-case approach using model-based reinforcement learning},
  author={Lecarpentier, Erwan and Rachelson, Emmanuel},
  journal={Advances in neural information processing systems},
  volume={32},
  year={2019}
}

@inproceedings{evron2024joint,
  title={The Joint Effect of Task Similarity and Overparameterization on Catastrophic Forgetting--An Analytical Model},
  author={Goldfarb, Daniel and Evron, Itay and Weinberger, Nir and Soudry, Daniel and Hand, Paul},
  booktitle={International Conference on Learning Representations (ICLR)},
  year={2024}
}

@article{wolczyk2021continual,
  title={Continual world: A robotic benchmark for continual reinforcement learning},
  author={Wo{\l}czyk, Maciej and Zaj{\k{a}}c, Micha{\l} and Pascanu, Razvan and Kuci{\'n}ski, {\L}ukasz and Mi{\l}o{\'s}, Piotr},
  journal={Advances in Neural Information Processing Systems},
  volume={34},
  pages={28496--28510},
  year={2021}
}

@phdthesis{Ring94,
author = {Ring, Mark Bishop},
title = {Continual learning in reinforcement environments},
year = {1994},
school = {University of Texas at Austin},
address = {USA},
}

@incollection{Thrun95,
  title={Lifelong robot learning},
  author={Thrun, Sebastian and Mitchell, Tom M},
  booktitle={The biology and technology of intelligent autonomous agents},
  pages={165--196},
  year={1995},
  publisher={Springer},
  url={http://citeseerx.ist.psu.edu/viewdoc/download?doi=10.1.1.71.3723&rep=rep1&type=pdf}
}

@inproceedings{Jacot18NTK,
 author = {Jacot, Arthur and Gabriel, Franck and Hongler, Cl{\'e}ment},
 booktitle = {Advances in Neural Information Processing Systems},
 editor = {S. Bengio and H. Wallach and H. Larochelle and K. Grauman and N. Cesa-Bianchi and R. Garnett},
 pages = {},
 publisher = {Curran Associates, Inc.},
 title = {Neural Tangent Kernel: Convergence and Generalization in Neural Networks},
 url = {https://proceedings.neurips.cc/paper_files/paper/2018/file/5a4be1fa34e62bb8a6ec6b91d2462f5a-Paper.pdf},
 volume = {31},
 year = {2018}
}

@inproceedings{yang2021feature,
  title = {Tensor Programs IV: Feature Learning in Infinite-Width Neural Networks},
  author = {Yang, Greg and Hu, Edward J.},
  booktitle = {International Conference on Machine Learning},
  pages = {11727--11737},
  year = {2021},
  organization = {PMLR},
  url = {https://arxiv.org/abs/2011.14522}
}

@inproceedings{yang2021tensorprogramsv,
  title = {Tensor Programs V: Tuning Large Neural Networks via Zero-Shot Hyperparameter Transfer},
  author = {Yang, Greg and Hu, Edward J. and Babuschkin, Igor and Sidor, Szymon and Liu, Xiaodong and Farhi, David and Ryder, Nick and Pachocki, Jakub and Chen, Weizhu and Gao, Jianfeng},
  booktitle = {Advances in Neural Information Processing Systems},
  volume = {34},
  pages = {17084--17097},
  year = {2021},
  url = {https://proceedings.neurips.cc/paper/2021/hash/8df7c2e3c3c3be098ef7b382bd2c37ba-Abstract.html}
}

@article{parisi2019continual,
  title={Continual lifelong learning with neural networks: A review},
  author={Parisi, German I and Kemker, Ronald and Part, Jose L and Kanan, Christopher and Wermter, Stefan},
  journal={Neural networks},
  volume={113},
  pages={54--71},
  year={2019},
  publisher={Elsevier}
}

@article{hazan2016introduction,
  title={Introduction to online convex optimization},
  author={Hazan, Elad},
  journal={Foundations and Trends{\textregistered} in Optimization},
  volume={2},
  number={3-4},
  pages={157--325},
  year={2016},
  publisher={Now Publishers, Inc.}
}

@inproceedings{standley2020tasks,
  title={Which tasks should be learned together in multi-task learning?},
  author={Standley, Trevor and Zamir, Amir and Chen, Dawn and Guibas, Leonidas and Malik, Jitendra and Savarese, Silvio},
  booktitle={International Conference on Machine Learning},
  pages={9120--9132},
  year={2020},
  organization={PMLR}
}

@article{sener2018multi,
  title={Multi-task learning as multi-objective optimization},
  author={Sener, Ozan and Koltun, Vladlen},
  journal={Advances in neural information processing systems},
  volume={31},
  year={2018}
}

@article{kumar2023continual,
  title={Continual Learning as Computationally Constrained Reinforcement Learning},
  author={Kumar, Saurabh and Marklund, Henrik and Rao, Ashish and Zhu, Yifan and Jeon, Hong Jun and Liu, Yueyang and Van Roy, Benjamin},
  journal={Foundations and Trends in Machine Learning},
  volume={18},
  number={5},
  pages={913--1053},
  year={2025},
  publisher={Now Publishers},
  doi={10.1561/2200000116}
}

@inproceedings{zinkevich2003online,
  title={Online convex programming and generalized infinitesimal gradient ascent},
  author={Zinkevich, Martin},
  booktitle={Proceedings of the 20th international conference on machine learning (icml-03)},
  pages={928--936},
  year={2003}
}

@article{orabona2019modern,
  title={A modern introduction to online learning},
  author={Orabona, Francesco},
  journal={arXiv preprint arXiv:1912.13213},
  year={2019}
}

@article{mundt2023wholistic,
  title={A wholistic view of continual learning with deep neural networks: Forgotten lessons and the bridge to active and open world learning},
  author={Mundt, Martin and Hong, Yongwon and Pliushch, Iuliia and Ramesh, Visvanathan},
  journal={Neural Networks},
  volume={160},
  pages={306--336},
  year={2023},
  publisher={Elsevier}
}

@InProceedings{schwarz18a,
  title = 	 {Progress \& Compress: A scalable framework for continual learning},
  author =       {Schwarz, Jonathan and Czarnecki, Wojciech and Luketina, Jelena and Grabska-Barwi{\'n}ska, Agnieszka and Teh, Yee Whye and Pascanu, Razvan and Hadsell, Raia},
  booktitle = 	 {Proceedings of the 35th International Conference on Machine Learning},
  pages = 	 {4528--4537},
  year = 	 {2018},
  editor = 	 {Dy, Jennifer and Krause, Andreas},
  volume = 	 {80},
  series = 	 {Proceedings of Machine Learning Research},
  month = 	 {10--15 Jul},
  publisher =    {PMLR},
}

@inproceedings{liu2023nonstationary,
  title={Nonstationary bandit learning via predictive sampling},
  author={Liu, Yueyang and Van Roy, Benjamin and Xu, Kuang},
  booktitle={International Conference on Artificial Intelligence and Statistics},
  pages={6215--6244},
  year={2023},
  organization={PMLR}
}

@article{besbes2015non,
  title={Non-stationary stochastic optimization},
  author={Besbes, Omar and Gur, Yonatan and Zeevi, Assaf},
  journal={Operations research},
  volume={63},
  number={5},
  pages={1227--1244},
  year={2015},
  publisher={INFORMS}
}

@article{dohare2021continual,
  title={Continual backprop: Stochastic gradient descent with persistent randomness},
  author={Dohare, Shibhansh and Sutton, Richard S and Mahmood, A Rupam},
  journal={arXiv preprint arXiv:2108.06325},
  year={2021}
}

@inproceedings{lecarpentier2021lipschitz,
  title={Lipschitz lifelong reinforcement learning},
  author={Lecarpentier, Erwan and Abel, David and Asadi, Kavosh and Jinnai, Yuu and Rachelson, Emmanuel and Littman, Michael L},
  booktitle={Proceedings of the AAAI Conference on Artificial Intelligence},
  volume={35},
  pages={8270--8278},
  year={2021}
}

@inproceedings{lin2023theory,
  title={Theory on forgetting and generalization of continual learning},
  author={Lin, Sen and Ju, Peizhong and Liang, Yingbin and Shroff, Ness},
  booktitle={International Conference on Machine Learning},
  pages={21078--21100},
  year={2023},
  organization={PMLR}
}

@inproceedings{evron2022catastrophic,
  title={How catastrophic can catastrophic forgetting be in linear regression?},
  author={Evron, Itay and Moroshko, Edward and Ward, Rachel and Srebro, Nathan and Soudry, Daniel},
  booktitle={Conference on Learning Theory},
  pages={4028--4079},
  year={2022},
  organization={PMLR}
}

@article{buzzega2020dark,
  title={Dark experience for general continual learning: a strong, simple baseline},
  author={Buzzega, Pietro and Boschini, Matteo and Porrello, Angelo and Abati, Davide and Calderara, Simone},
  journal={Advances in neural information processing systems},
  volume={33},
  pages={15920--15930},
  year={2020}
}

@inproceedings{cai2021online,
  title={Online continual learning with natural distribution shifts: An empirical study with visual data},
  author={Cai, Zhipeng and Sener, Ozan and Koltun, Vladlen},
  booktitle={Proceedings of the IEEE/CVF international conference on computer vision},
  pages={8281--8290},
  year={2021}
}

@article{lopez2017gradient,
  title={Gradient episodic memory for continual learning},
  author={Lopez-Paz, David and Ranzato, Marc'Aurelio},
  journal={Advances in neural information processing systems},
  volume={30},
  year={2017}
}

@article{delange2022survey,
  title={A continual learning survey: Defying forgetting in classification tasks},
  author={De Lange, Matthias and Aljundi, Rahaf and Masana, Marc and Parisot, Sarah and Jia, Xu and Leonardis, Ale{\v{s}} and Slabaugh, Gregory and Tuytelaars, Tinne},
  journal={IEEE Transactions on Pattern Analysis and Machine Intelligence},
  volume={44},
  number={7},
  pages={3366--3385},
  year={2022},
  publisher={IEEE},
  doi={10.1109/TPAMI.2021.3057446}
}

@book{cesa2006prediction,
  title={Prediction, learning, and games},
  author={Cesa-Bianchi, Nicol{\`o} and Lugosi, G{\'a}bor},
  year={2006},
  publisher={Cambridge university press}
}

@inproceedings{herbster1998tracking,
  title={Tracking the best regressor},
  author={Herbster, Mark and Warmuth, Manfred K},
  booktitle={Proceedings of the eleventh annual conference on Computational learning theory},
  pages={24--31},
  year={1998}
}

@article{hadsell2020embracing,
  title={Embracing Change: Continual Learning in Deep Neural Networks},
  author={Hadsell, Raia and Rao, Dushyant and Rusu, Andrei A and Pascanu, Razvan},
  journal={Trends in Cognitive Sciences},
  volume={24},
  number={12},
  pages={1028--1040},
  year={2020},
  publisher={Elsevier},
  doi={10.1016/j.tics.2020.09.004},
  url={https://doi.org/10.1016/j.tics.2020.09.004}
}

@inproceedings{yu2020meta,
  title={Meta-World: A Benchmark and Evaluation for Multi-Task and Meta Reinforcement Learning},
  author={Yu, Tianhe and Quillen, Deirdre and He, Zhanpeng and Julian, Ryan and Hausman, Karol and Finn, Chelsea and Levine, Sergey},
  booktitle={Proceedings of the Conference on Robot Learning (CoRL)},
  pages={1094--1100},
  year={2020},
  volume={100},
  series={Proceedings of Machine Learning Research},
  organization={PMLR},
  url={https://proceedings.mlr.press/v100/yu20a.html}
}

@inproceedings{lin2021clear,
  title={The CLEAR Benchmark: Continual LEArning on Real-World Imagery},
  author={Lin, Zhiqiu and Shi, Jia and Pathak, Deepak and Ramanan, Deva},
  booktitle={Proceedings of the Neural Information Processing Systems Track on Datasets and Benchmarks 1 (NeurIPS Datasets and Benchmarks 2021)},
  year={2021},
}

@inproceedings{krause2013,
  title     = {3D Object Representations for Fine-Grained Categorization},
  author    = {Jonathan Krause and Michael Stark and Jia Deng and Li Fei-Fei},
  booktitle = {Proceedings of the IEEE International Conference on Computer Vision Workshops (ICCVW)},
  year      = {2013},
  address   = {Sydney, Australia},
  month     = {Dec. 8}
}

@article{MajiRKBV13,
  author       = {Subhransu Maji and
                  Esa Rahtu and
                  Juho Kannala and
                  Matthew B. Blaschko and
                  Andrea Vedaldi},
  title        = {Fine-Grained Visual Classification of Aircraft},
  journal      = {CoRR},
  volume       = {abs/1306.5151},
  year         = {2013},
  url          = {http://arxiv.org/abs/1306.5151},
  eprinttype   = {arXiv},
  eprint       = {1306.5151},
  bibsource    = {dblp computer science bibliography, https://dblp.org}
}

@InProceedings{cimpoi14describing,
Author    = {M. Cimpoi and S. Maji and I. Kokkinos and S. Mohamed and A. Vedaldi},
Title     = {Describing Textures in the Wild},
Booktitle = {Proceedings of the {IEEE} Conf. on Computer Vision and Pattern Recognition ({CVPR})},
Year      = {2014}}

@inproceedings{bossard14,
  title = {Food-101 -- Mining Discriminative Components with Random Forests},
  author = {Bossard, Lukas and Guillaumin, Matthieu and Van Gool, Luc},
  booktitle = {European Conference on Computer Vision},
  year = {2014}
}

@inproceedings{parkhi2012cats,
  title={Cats and Dogs},
  author={Parkhi, Omkar M. and Vedaldi, Andrea and Zisserman, Andrew and Jawahar, C. V.},
  booktitle={IEEE Conference on Computer Vision and Pattern Recognition},
  pages={3498--3505},
  year={2012},
  organization={IEEE}
}

@misc{krizhevsky2009learning,
  title={Learning Multiple Layers of Features from Tiny Images},
  author={Krizhevsky, Alex and Hinton, Geoffrey},
  year={2009},
  howpublished={\url{https://www.cs.toronto.edu/~kriz/learning-features-2009-TR.pdf}},
  institution={University of Toronto}
}

@article{baxterModelInductiveBias2000,
  title = {A {{Model}} of {{Inductive Bias Learning}}},
  author = {Baxter, Jonathan},
  date = {2000-03-01},
  journal = {Journal of Artificial Intelligence Research},
  shortjournal = {jair},
  volume = {12},
  pages = {149--198},
  issn = {1076-9757},
  doi = {10.1613/jair.731},
  year = 2000,
  url = {https://jair.org/index.php/jair/article/view/10253},
  urldate = {2026-01-20}
}

@incollection{murataStatisticalStudyOnline1999,
  title = {A {{Statistical Study}} of {{On-line Learning}}},
  booktitle = {On-{{Line Learning}} in {{Neural Networks}}},
  author = {Murata, Noboru},
  editor = {Saad, David},
  date = {1999-01-28},
  edition = {1},
  pages = {63--92},
  publisher = {Cambridge University Press},
  year = 1999,
  doi = {10.1017/CBO9780511569920.005},
  url = {https://www.cambridge.org/core/product/identifier/CBO9780511569920A011/type/book_part},
  urldate = {2026-01-23},
  isbn = {978-0-521-65263-6 978-0-521-11791-3 978-0-511-56992-0},
  langid = {english}
}

@inproceedings{deneviLearningLearnCommon2018,
  title = {Learning {{To Learn Around A Common Mean}}},
  booktitle = {Advances in {{Neural Information Processing Systems}}},
  author = {Denevi, Giulia and Ciliberto, Carlo and Stamos, Dimitris and Pontil, Massimiliano},
  year = 2018,
  volume = {31},
  publisher = {Curran Associates, Inc.},
  urldate = {2026-02-17}
}

@misc{nazariDynamicRegretAnalysis2021,
  title = {Dynamic {{Regret Analysis}} for {{Online Meta-Learning}}},
  author = {Nazari, Parvin and Khorram, Esmaile},
  year = 2021,
  month = sep,
  number = {arXiv:2109.14375},
  eprint = {2109.14375},
  primaryclass = {cs},
  publisher = {arXiv},
  doi = {10.48550/arXiv.2109.14375},
  urldate = {2026-02-17},
  archiveprefix = {arXiv}
}

@inproceedings{rakhlinOnlineLearningPredictable2013,
  title = {Online {{Learning}} with {{Predictable Sequences}}},
  booktitle = {Proceedings of the 26th {{Annual Conference}} on {{Learning Theory}}},
  author = {Rakhlin, Alexander and Sridharan, Karthik},
  year = 2013,
  date = {2013-06-13},
  pages = {993--1019},
  publisher = {PMLR},
  issn = {1938-7228},
  url = {https://proceedings.mlr.press/v30/Rakhlin13.html},
  urldate = {2026-02-22},
  eventtitle = {Conference on {{Learning Theory}}},
  langid = {english}
}

@article{dohareLossPlasticityDeep2024a,
  title = {Loss of Plasticity in Deep Continual Learning},
  author = {Dohare, Shibhansh and {Hernandez-Garcia}, J. Fernando and Lan, Qingfeng and Rahman, Parash and Mahmood, A. Rupam and Sutton, Richard S.},
  year = 2024,
  month = aug,
  journal = {Nature},
  volume = {632},
  number = {8026},
  pages = {768--774},
  publisher = {Nature Publishing Group},
  issn = {1476-4687},
  doi = {10.1038/s41586-024-07711-7},
  urldate = {2025-11-28},
  copyright = {2024 The Author(s)},
  langid = {english}
}

@article{shalev2012online,
  title={Online learning and online convex optimization},
  author={Shalev-Shwartz, Shai},
  journal={Foundations and Trends{\textregistered} in Machine Learning},
  volume={4},
  number={2},
  pages={107--194},
  year={2012},
  publisher={Now Publishers, Inc.}
}

@inproceedings{alquier2017regret,
  title={Regret bounds for lifelong learning},
  author={Alquier, Pierre and Mai, The Tien and Pontil, Massimiliano},
  booktitle={International Conference on Artificial Intelligence and Statistics},
  pages={261--269},
  year={2017},
  organization={PMLR}
}

@article{mandt2017stochastic,
  title   = {Stochastic Gradient Descent as Approximate {B}ayesian Inference},
  author  = {Mandt, Stephan and Hoffman, Matthew D. and Blei, David M.},
  journal = {Journal of Machine Learning Research},
  volume  = {18},
  number  = {134},
  pages   = {1--35},
  year    = {2017}
}

@inproceedings{ash2020warmstart,
  title     = {On Warm-Starting Neural Network Training},
  author    = {Ash, Jordan T. and Adams, Ryan P.},
  booktitle = {Advances in Neural Information Processing Systems (NeurIPS)},
  volume    = {33},
  pages     = {3884--3894},
  year      = {2020}
}

@inproceedings{bach,
  title     = {Non-strongly-convex smooth stochastic approximation with convergence rate {$O(1/n)$}},
  author    = {Bach, Francis and Moulines, Eric},
  booktitle = {Advances in Neural Information Processing Systems (NeurIPS)},
  volume    = {26},
  year      = {2013}
}

@article{maurerbenefit,
  title   = {The Benefit of Multitask Representation Learning},
  author  = {Maurer, Andreas and Pontil, Massimiliano and Romera-Paredes, Bernardino},
  journal = {Journal of Machine Learning Research},
  volume  = {17},
  number  = {81},
  pages   = {1--32},
  year    = {2016}
}

@inproceedings{tripuraneni2020theory,
  title     = {On the Theory of Transfer Learning: The Importance of Task Diversity},
  author    = {Tripuraneni, Nilesh and Jordan, Michael I. and Jin, Chi},
  booktitle = {Advances in Neural Information Processing Systems (NeurIPS)},
  volume    = {33},
  pages     = {7852--7862},
  year      = {2020}
}

@inproceedings{du2020few,
  title     = {Few-Shot Learning via Learning the Representation, Provably},
  author    = {Du, Simon S. and Hu, Wei and Kakade, Sham M. and Lee, Jason D. and Lei, Qi},
  booktitle = {International Conference on Learning Representations (ICLR)},
  year      = {2021}
}

@article{Kuzborskij,
  title   = {Fast rates by transferring from auxiliary hypotheses},
  author  = {Kuzborskij, Ilja and Orabona, Francesco},
  journal = {Machine Learning},
  volume  = {106},
  number  = {2},
  pages   = {171--195},
  year    = {2017}
}

@inproceedings{denevi2019learning,
  title     = {Online-Within-Online Meta-Learning},
  author    = {Denevi, Giulia and Stamos, Dimitris and Ciliberto, Carlo and Pontil, Massimiliano},
  booktitle = {Advances in Neural Information Processing Systems (NeurIPS)},
  volume    = {32},
  year      = {2019}
}

@inproceedings{khodak2019adaptive,
  title     = {Adaptive Gradient-Based Meta-Learning Methods},
  author    = {Khodak, Mikhail and Balcan, Maria-Florina F. and Talwalkar, Ameet S.},
  booktitle = {Advances in Neural Information Processing Systems (NeurIPS)},
  volume    = {32},
  year      = {2019}
}

@inproceedings{wang2021bridging,
  title     = {Bridging Multi-Task Learning and Meta-Learning: Towards Efficient Training and Effective Adaptation},
  author    = {Wang, Haoxiang and Zhao, Han and Li, Bo},
  booktitle = {Proceedings of the 38th International Conference on Machine Learning (ICML)},
  volume    = {139},
  pages     = {10991--11002},
  year      = {2021},
  organization = {PMLR}
}

@inproceedings{saxe,
  title     = {Exact solutions to the nonlinear dynamics of learning in deep linear neural networks},
  author    = {Saxe, Andrew M. and McClelland, James L. and Ganguli, Surya},
  booktitle = {International Conference on Learning Representations (ICLR)},
  year      = {2014}
}

@inproceedings{javed2024bigworld,
  title     = {The Big World Hypothesis and its Ramifications for Artificial Intelligence},
  author    = {Javed, Khurram and Sutton, Richard S.},
  booktitle = {Finding the Frame Workshop, Reinforcement Learning Conference (RLC)},
  year      = {2024},
  url       = {https://openreview.net/forum?id=Sv7DazuCn8}
}

@inproceedings{sutton2007tracking,
  title     = {On the role of tracking in stationary environments},
  author    = {Sutton, Richard S. and Koop, Anna and Silver, David},
  booktitle = {Proceedings of the 24th International Conference on Machine Learning (ICML)},
  pages     = {871--878},
  year      = {2007},
  publisher = {ACM},
  doi       = {10.1145/1273496.1273606}
}

@inproceedings{haarnoja2018soft,
  title     = {Soft Actor-Critic: Off-Policy Maximum Entropy Deep Reinforcement Learning with a Stochastic Actor},
  author    = {Haarnoja, Tuomas and Zhou, Aurick and Abbeel, Pieter and Levine, Sergey},
  booktitle = {Proceedings of the 35th International Conference on Machine Learning (ICML)},
  pages     = {1861--1870},
  year      = {2018},
  publisher = {PMLR},
  volume    = {80},
}

@article{haarnoja2018softapps,
  title   = {Soft Actor-Critic Algorithms and Applications},
  author  = {Haarnoja, Tuomas and Zhou, Aurick and Hartikainen, Kristian and Tucker, George and Ha, Sehoon and Tan, Jie and Kumar, Vikash and Zhu, Henry and Gupta, Abhishek and Abbeel, Pieter and Levine, Sergey},
  journal = {arXiv preprint arXiv:1812.05905},
  year    = {2018},
}

@book{villani2009optimal,
  title={Optimal Transport: Old and New},
  author={Villani, C{\'e}dric},
  volume={338},
  series={Grundlehren der mathematischen Wissenschaften},
  year={2009},
  publisher={Springer},
  address={Berlin},
  doi={10.1007/978-3-540-71050-9}
}

@article{bubeck2015convex,
  title={Convex Optimization: Algorithms and Complexity},
  author={Bubeck, S{\'e}bastien},
  journal={Foundations and Trends in Machine Learning},
  volume={8},
  number={3--4},
  pages={231--357},
  year={2015},
  doi={10.1561/2200000050}
}

@inproceedings{karimi2016linear,
  title={Linear Convergence of Gradient and Proximal-Gradient Methods Under the {Polyak}--{{\L}ojasiewicz} Condition},
  author={Karimi, Hamed and Nutini, Julie and Schmidt, Mark},
  booktitle={Joint European Conference on Machine Learning and Knowledge Discovery in Databases (ECML-PKDD)},
  pages={795--811},
  year={2016},
  organization={Springer}
}

@inproceedings{zhang2017rethinking,
  title     = {Understanding deep learning requires rethinking generalization},
  author    = {Zhang, Chiyuan and Bengio, Samy and Hardt, Moritz and Recht, Benjamin and Vinyals, Oriol},
  booktitle = {International Conference on Learning Representations},
  year      = {2017}
}

@article{bartlett2020benign,
  title   = {Benign overfitting in linear regression},
  author  = {Bartlett, Peter L. and Long, Philip M. and Lugosi, G{\'a}bor and Tsigler, Alexander},
  journal = {Proceedings of the National Academy of Sciences},
  volume  = {117},
  number  = {48},
  pages   = {30063--30070},
  year    = {2020}
}

@article{advani2020high,
  title   = {High-dimensional dynamics of generalization error in neural networks},
  author  = {Advani, Madhu S. and Saxe, Andrew M. and Sompolinsky, Haim},
  journal = {Neural Networks},
  volume  = {132},
  pages   = {428--446},
  year    = {2020}
}

@inproceedings{mansour2009domain,
  title     = {Domain adaptation: Learning bounds and algorithms},
  author    = {Mansour, Yishay and Mohri, Mehryar and Rostamizadeh, Afshin},
  booktitle = {Conference on Learning Theory (COLT)},
  year      = {2009}
}

@article{bendavid2010theory,
  title   = {A theory of learning from different domains},
  author  = {Ben-David, Shai and Blitzer, John and Crammer, Koby and Kulesza, Alex and Pereira, Fernando and Vaughan, Jennifer Wortman},
  journal = {Machine Learning},
  volume  = {79},
  number  = {1},
  pages   = {151--175},
  year    = {2010}
}

@inproceedings{doan2021theoretical,
  title     = {A theoretical analysis of catastrophic forgetting through the {NTK} overlap matrix},
  author    = {Doan, Thang and Bennani, Mehdi Abbana and Mazoure, Bogdan and Rabusseau, Guillaume and Alquier, Pierre},
  booktitle = {International Conference on Artificial Intelligence and Statistics (AISTATS)},
  year      = {2021}
}

@inproceedings{cheung2020reinforcement,
  title     = {Reinforcement learning for non-stationary {M}arkov decision processes: The blessing of (more) optimism},
  author    = {Cheung, Wang Chi and Simchi-Levi, David and Zhu, Ruihao},
  booktitle = {International Conference on Machine Learning (ICML)},
  year      = {2020}
}

@inproceedings{lee2021continual,
  title     = {Continual learning in the teacher-student setup: Impact of task similarity},
  author    = {Lee, Sebastian and Goldt, Sebastian and Saxe, Andrew},
  booktitle = {International Conference on Machine Learning (ICML)},
  year      = {2021}
}

@inproceedings{schaul2016prioritized,
  author    = {Schaul, Tom and Quan, John and Antonoglou, Ioannis and Silver, David},
  title     = {Prioritized Experience Replay},
  booktitle = {International Conference on Learning Representations (ICLR)},
  year      = {2016}
}
\bibliographystyle{tmlr}

\newpage
\appendix
\clearpage
\addcontentsline{toc}{section}{Appendix} %
\part{Appendix} %
\parttoc %

\clearpage

\section{Theoretical proofs}
\label{Apdx:Proofs}

\centerline{\bf Notation}
\begin{table}[htbp]
    \centering
    \renewcommand{\arraystretch}{1.2}
    \caption{Summary of Notation Used in the Theoretical Analysis}
    \label{tab:notation_summary}
    \begin{tabular}{@{}lp{11cm}@{}}
        \toprule
        \textbf{Notation} & \textbf{Description} \\
        \midrule
        \multicolumn{2}{@{}l}{\textbf{Indices \& Setup}} \\
        \addlinespace
        $k$, $j$ & Task indices, where $k$ typically denotes the current active task. \\
        $W$ & Window size (number of past tasks considered by the Window agent). \\
        $N_k$ & Number of optimization steps (duration) of task $k$. \\
        $\mathcal{D}^k_W$ & Sampling distribution for the Window agent at task $k$. \\
        $\rho_j^k$ & Sampling weight of task $j$ at time $k$. \\
        $B$ & Mini-batch size. \\
        $\eta$ & Learning rate. \\
        \midrule
        \multicolumn{2}{@{}l}{\textbf{Parameters \& Objectives}} \\
        \addlinespace
        $\theta$ & Model parameters. \\
        $\theta_\star^k$ & True optimal parameters for task $k$. \\
        $\theta_\star^{W,k}$ & Optimal parameters (centroid) for the Window agent at task $k$. \\
        $\mathcal{L}^k(\theta)$ & Expected loss function for task $k$. \\
        $H_k$ & Hessian matrix of the loss for task $k$. \\
        $\bar{H}_k$ & Average Hessian over the active window $\mathcal{W}_k$. \\
        $\Sigma_x$ & Input feature covariance matrix. \\
        \midrule
        \multicolumn{2}{@{}l}{\textbf{Non-Stationarity \& Drift}} \\
        \addlinespace
        $m(k, W)$ & Mean Squared Displacement (MSD) of task optima separated by lag $W$. \\
        $\Delta_k$ & Random walk increment for task optima ($\theta_\star^k - \theta_\star^{k-1}$). \\
        $\Sigma_\Delta$ & Covariance matrix of the task drift/increments. \\
        $\Delta \mu^{W,k}$ & Geometric mean drift (shift in the JTL agent's centroid from step $k-1$ to $k$). \\
        $\Delta \Sigma^{W,k}$ & Geometric covariance shift (change in stationary covariance). \\
        $\psi(k)$ & Global Growth Rate of the task optima relative to the origin. \\
        \midrule
        \multicolumn{2}{@{}l}{\textbf{Error \& Risk Components}} \\
        \addlinespace
        $\mathcal{I}^W$ & Expected Instability cost (approximation bias + stationary variance). \\
        $\delta^W$ & Transient Error cost (due to initialization error at the start of a task). \\
        $\tilde{\mu}_0$ & Initialization error vector (centroid tracking lag) at the start of a task. \\
        $\Sigma_\infty$ & Stationary covariance matrix of the agent. \\
        $v_W$ & Irreducible centroid lag vector ($\theta_\star^{W, k} - \theta_\star^k$). \\
        $\kappa$ & Condition number bound for the loss landscape. \\
        $\mathcal{R_B}$ & Asymptotic bias cost coefficient in the optimal window derivation. \\
        $\mathcal{R_T}$ & Asymptotic transient cost coefficient in the optimal window derivation. \\
        \bottomrule
    \end{tabular}
\end{table}
\newpage

\major{\paragraph{Roadmap.} This appendix follows the order of the main text. \Cref{app:intro} recalls the framework of \citet{kumar2023continual} that underlies the ALE objective (\cref{sec:background}). \Cref{app:setup} states the formal setup and assumptions, including the reset and warm-start variants of ITL (\cref{ss:JTL-ITL}). \Cref{app:decomposition} proves the decomposition into Instability and Transient Error and the Critical Task Duration (\cref{sec:transfer-efficiency}). \Cref{sec:quadratic_analysis,sec:mt_asymptotics,sec:drift_analysis} give the quadratic analysis of \cref{sec:quadratic-losses}: the exact closed forms, the asymptotic Growth-Rate conditions, and the continuous drift model. \Cref{sec:optimal_window} analyzes the Window algorithm of \cref{ssec:window-poc}, and \cref{sec:relaxing} removes the alignment and bounded-curvature assumptions.}

\subsection{Introduction to the Framework of \texorpdfstring{\cite{kumar2023continual}}{Kumar et al., 2025}}
\label{app:intro}

\paragraph{The Environment.}
We define an \emph{environment} as $\mathcal{E} = (A, O, \mathcal{P})$, where $A$ and $O$ denote the action and observation spaces, and $\mathcal{P}(\cdot \mid H_t, a_t)$ is the distribution of the next observation $o_{t+1}$ given the latest action $a_t$ and the history $H_t = \{(a_0,o_1),\dots,(a_{t-1},o_t)\}$.
While $\mathcal{E}$ itself is stationary, the effective observation distribution $\mathcal{P}_t(o) = \mathcal{P}(o \mid H_t, a_t)$ is non-stationary due to its dependence on $H_t$. Given a reward function $r$, the \emph{response} of the environment to an action-observation pair is the reward $r(a_t,o_{t+1})$.

\paragraph{Agent and Learning Algorithm.}
The agent acts according to a stochastic policy $\pi_t$ sampled from a stochastic process,
\rev{which we call the \emph{learning algorithm}} $\mathcal{A}(\cdot \mid H_t)$\footnote{We assume that $\mathcal{A}$ has access to the reward function or the history of collected rewards. In supervised learning, this is equivalent to assuming access to the loss function.}. Each history $H_t$ thus arises sequentially by sampling $\pi_t \sim \mathcal{A}(\cdot \mid H_t)$, $a_t \sim \pi_t(\cdot \mid H_t)$, and finally $o_{t+1} \sim \mathcal{P}(\cdot \mid H_t, a_t)$.

\paragraph{Continual Learning Objective.}
Let $\ell(\pi, o)$ be a non-negative measure of error for a learner policy $\pi$ evaluated on an observation $o$\footnote{\rev{We omit explicit action sampling here, since it is part of the definition of $\ell$.}}. We denote by $\mathcal{P}_n$ the conditional distribution $\mathcal{P}(\cdot \mid H_n, a_n)$. \citet{kumar2023continual} define the continual learning objective as the minimization of the real-time error, or \emph{Average Lifelong Error} (ALE):
\begin{equation}
    \label{eq:SLO-app}
    \inf_\mathcal{A} \,\, \mathrm{ALE}(\mathcal{A}) \qquad \text{s.t. computational constraints},
\end{equation}
where the ALE is defined as:
\begin{equation}
    \mathrm{ALE}(\mathcal{A}) := \frac{1}{T} \sum_{n=1}^T \mathbb{E}_{\pi_n\sim\mathcal{A}(\cdot\mid H_{n}), \,o_{n+1}\sim\mathcal{P}_n} \left[\ell(\pi_n,o_{n+1})\right].
\end{equation}
\emph{Note:} \citet{kumar2023continual} consider an infinite horizon. While asymptotic averaging is common in reinforcement learning, it can obscure maladaptive behaviors that cause failures over a finite lifespan. In this work, we follow the online learning tradition and focus on the \textbf{finite-horizon} version of the original objective.

\paragraph{Supervised Specialization.}
In supervised settings, observations are pairs $(x,y) \in X \times Y$, and actions correspond to predictions $\hat{y}=f(x)$. The learning algorithm is conditioned solely on the observation sequence $o_{1:t}$.
\rev{For consistency with \citet{kumar2023continual}, we still write the conditioning on the entire history as $H_t$, although in the supervised case it does not depend on the agent's actions.} Because $\mathcal{P}_n$ does not depend on the prediction $\hat{y}$, the continual learning objective simplifies to:
\begin{equation}
    \mathrm{ALE}(\mathcal{A}) = \frac{1}{T}\sum_{n=1}^T \mathbb{E}_{f\sim\mathcal{A}(\cdot\mid H_n)} \, \mathbb{E}_{(x,y)\sim\mathcal{P}_n} \, [\ell(f,(x,y))].
\end{equation}

\subsection{Formal Setup and Assumptions}
\label{app:setup}

We consider a sequence of supervised learning tasks where an agent interacts with a piecewise stationary environment, assuming task boundaries are known.
\rev{We specialize the general lifelong learning framework to stochastic gradient-based optimization.}

\paragraph{Task-Based Formulation.}
We assume the data-generating process is \emph{locally stationary}, defined formally as follows:
\begin{definition*}[Task-Based Data Generating Process]
\label{def:task-app}
A history $H_T$ is a \emph{task-based} sequence if the sampling distribution $\mathcal{P}_n$ is locally stationary. The timeline is partitioned into intervals called \emph{tasks} $S_k = [a_k, b_k]$. At each time step $n \in S_k$, an observation $o_n$ is sampled from a fixed, task-specific distribution $\mathcal{P}^k$.
\end{definition*}

\paragraph{Batched Observations.}
For generality, we assume each observation $o_t$ is a \emph{batch} of $B \ge 1$ samples:
\[
    o_t = \{(x_b, y_b)\}_{b=1}^B.
\]
Accordingly, the full history is $H_t = \{o_1, \dots, o_t\}$. The model $f$ operates at the \emph{sample level}, producing predictions $f(x_b)$ for each input in the batch. The \emph{batch loss} is defined as the average of the per-sample losses.

\paragraph{The Data-Generating Process.}
Unless stated otherwise, we assume a task-based sequence with \emph{known task switch points}. We denote the task identity at time $n$ by $z(n)$ and the duration of task $k$ by $N_k = |S_k|$. At each step $n \in S_k$, a batch $o_n$ is sampled from $\mathcal{P}^{k}$. The expected loss for task $k$ over the model parameters $\theta$ is:
\begin{equation}
    \mathcal{L}^k(\theta) := \mathbb{E}_{(x,y) \sim \mathcal{P}^k}[\ell(\theta, (x,y))].
\end{equation}
The corresponding oracle minimizer is $\theta^{k}_\star = \arg\min_\theta \mathcal{L}^k(\theta)$.

\paragraph{Measures of Error and Performance.}
Within this formalism, we define the \textbf{Average Task Error} ($\mathrm{ALE}^k$) as the task-specific view of the lifelong ALE:
\begin{equation}
    \mathrm{ALE}^k(\mathcal{A}) := \frac{1}{N_k}\sum_{n=a_k}^{b_k} \mathbb{E}_{\theta \sim Q_n} [\mathcal{L}^{k}(\theta)].
\end{equation}
It follows that $\mathrm{ALE}(\mathcal{A}) = \sum_{k=1}^{K} \bar{w}_k \cdot \mathrm{ALE}^k(\mathcal{A})$ with task weights $\bar{w}_k = N_k/T$.

The \textbf{Transfer Efficiency (TE)} for task $k$ quantifies the performance gain of a Joint-Task Learning (JTL) agent over an Independent-Task Learning (ITL) baseline:
\begin{equation}
    \mathrm{TE}^k := \mathrm{ALE}^k(\mathcal{A}_{\text{ITL}}) - \mathrm{ALE}^k(\mathcal{A}_{\text{JTL}}).
\end{equation}

\subsubsection{The Algorithm as a Markov Process}
We view the learning algorithm $\mathcal{A}$ as a discrete-time stochastic process $\{\theta_n\}_{n \ge 0}$ taking values in the parameter space $\Theta \subseteq \mathbb{R}^d$. Parameter evolution is governed by a stochastic update rule $U$. At step $n$, the update depends on the current parameters $\theta_n$ and a data batch $\xi_n$ drawn from an \emph{algorithm-specific} sampling distribution $\mathcal{D}_n$:
\begin{equation}
    \theta_{n+1} = U(\theta_n, \xi_n), \quad \text{where } \xi_n \sim \mathcal{D}_n.
\end{equation}
\rev{The optimization distribution $\mathcal{D}_n$ differs from the environmental distribution $\mathcal{P}_n$.} While $\mathcal{P}_n$ dictates the incoming data stream, $\mathcal{D}_n$ determines the data used for the update. For replay-based algorithms (such as our JTL agent), $\mathcal{D}_n$ may incorporate data from previous tasks.

\rev{If the sampling distribution is fixed within each task ($\mathcal{D}_n = \mathcal{D}^k$ for all $n \in S_k$), the process} induces a time-homogeneous \textbf{Markov Transition Operator} $\mathcal{T}^k$. For any bounded measurable test function $\phi: \Theta \to \mathbb{R}$, the operator is defined as:
\begin{equation}
    (\mathcal{T}^k \phi)(\theta) := \mathbb{E}_{\xi \sim \mathcal{D}^k}[\phi(U(\theta, \xi))].
\end{equation}
The sequence of parameter distributions $\{Q_n\}_{n \in S_k}$ evolves according to this linear operator, $Q_{n+1} = Q_n \mathcal{T}^k$ (following the Markov chain convention where measures act as row vectors).

\paragraph{Convergence.}
\rev{Our analysis centers on the long-term behavior of the learning algorithm.} An algorithm is said to converge if its transition operator $\mathcal{T}^k$ admits a stationary distribution $Q_\infty$; if run indefinitely, the process will converge to this distribution.
\rev{This property lets us separate the effect of starting conditions} (the transient phase) from asymptotic performance (the stationary phase).

\paragraph{Stochastic Gradient Descent Instantiations.}
We analyze two canonical instances of SGD with a constant step-size $\eta$. Both employ the standard gradient step $\theta_{n+1} = \theta_n - \eta \nabla \ell(\theta_n, \xi_n)$, differing only in their sampling distributions $\mathcal{D}^k$ and task-boundary initialization strategies.

\paragraph{ITL SGD ($\mathcal{A}_{\text{ITL}}$).}
The ITL agent optimizes solely on the current data stream. Its sampling distribution matches the current task distribution ($\mathcal{D}^k_{\text{ITL}} = \mathcal{P}^k$). At task boundaries, it resets its parameters to a fixed initialization distribution $Q_{\text{init}}$, discarding prior knowledge to prevent negative transfer.

\paragraph{JTL SGD ($\mathcal{A}_{\text{JTL}}$).}
The JTL agent minimizes the average loss across all seen tasks. To maintain time-homogeneous dynamics within a task, we model experience replay as a \textbf{fixed mixture approximation}. The sampling distribution $\mathcal{D}^k_{\text{JTL}}$ is a weighted mixture of the current and preceding task distributions:
\begin{equation}
    \mathcal{D}^k_{\text{JTL}} := \sum_{j=1}^k w^k_j\, \mathcal{P}^j, \quad \text{where } w^k_j = \frac{N_j}{\sum_{i=1}^k N_i}.
\end{equation}
Unlike the ITL agent, the JTL agent inherits parameters from the previous task ($\theta_{a_k} = \theta_{b_{k-1}}$), enabling forward transfer.

\subsubsection{Why use reset ITL as the reference?}
\label{app:reset_itl}

The ITL baseline used throughout the main comparison is intentionally reset at task boundaries. This choice isolates the value of the JTL objective as a \emph{prediction of the future}: JTL assumes that the accumulated past remains predictive, whereas reset ITL assumes no such structure and starts each task from the same prior. The comparison therefore measures the benefit and cost of injecting historical bias into the optimization process.

A warm-start current-task learner is also natural: it would set $\mathcal{D}^k=\mathcal{P}^k$ as ITL does, but initialize task $k$ from the previous endpoint rather than from $Q_{\mathrm{init}}$. In our decomposition, this variant has the same stationary distribution as ITL whenever the current-task optimizer is task-wise convergent, because both optimize only $\mathcal{P}^k$.
\rev{Its stationary gap against JTL is therefore still the Instability term $\mathcal{I}^k$; only the transient term changes:}
\[
    \mathrm{TE}^k_{\mathrm{warm}} = -\mathcal{I}^k + \delta^k(\mathcal{A}_{\mathrm{warm}}) - \delta^k(\mathcal{A}_{\mathrm{JTL}}).
\]
\rev{Consequently, warm starts can shift the Critical Task Duration by} reducing the current-task transient cost, but they do not remove the central trade-off: retaining past information helps only when the transient advantage it provides exceeds the stationary bias induced by optimizing against outdated tasks. The Window agent in \cref{sec:generalizing} includes this case as $W=1$, making it a useful bridge between reset ITL, warm-start adaptation, and full JTL.

\subsection{Decomposition: Instability and Transient Error}
\label{app:decomposition}

\rev{To analyze the source of the performance gap between ITL and JTL agents, we decompose it using the asymptotic behavior of the learning process.}

\begin{definition}[Task-wise Convergence]
\label{def:convergence}
An algorithm $\mathcal{A}$ is \textbf{Task-wise Convergent} if, for any task $k$, the associated Markov operator $\mathcal{T}^k$ admits a stationary probability measure $Q_\infty^k$. Formally, if the task duration were infinite, the distribution of iterates $Q^k_n$ would converge to $Q_\infty^k$ regardless of initialization:
\begin{equation}
    \lim_{n \to \infty} W_2(Q^k_{n}, Q_\infty^k) = 0,
\end{equation}
where $W_2$ is the Wasserstein-2 distance \citep{villani2009optimal}, standard in the convergence analysis of stochastic optimization algorithms.
\end{definition}

If convergence holds, the Average Task Error ($\mathrm{ALE}^k$) can be split into two components: the \emph{stationary component} (error at equilibrium) and the \emph{transient component} (error accumulated while moving toward equilibrium).

\begin{definition}[Instability $\mathcal{I}^k$]
Let $\mathcal{A}_{\text{ITL}}$ and $\mathcal{A}_{\text{JTL}}$ be task-wise convergent algorithms with stationary distributions $Q_\infty^{\text{ITL}}$ and $Q_\infty^{\text{JTL}}$. The \textbf{Instability} measures the performance gap in the asymptotic limit:
\begin{equation}
    \mathcal{I}^k := \mathcal{L}^k(Q_\infty^{\text{JTL}}) - \mathcal{L}^k(Q_\infty^{\text{ITL}}).
\end{equation}
\end{definition}

\begin{definition}[Transient Error $\delta^k$]
\label{def:transient_error}
For a task-wise convergent algorithm $\mathcal{A}$, the \textbf{Transient Error} captures the excess loss accumulated during the mixing phase. It is defined as the average deviation from the stationary risk over the task duration $S_k$:
\begin{equation}
    \delta^k(\mathcal{A}) := \frac{1}{N_k} \sum_{n \in S_k} \left( \mathcal{L}^k(Q_n) - \mathcal{L}^k(Q_\infty^{\mathcal{A}}) \right).
\end{equation}
\end{definition}

\rev{These definitions give an exact decomposition of Transfer Efficiency, which lets us treat} the ``quality of the final solution'' and the ``speed of learning'' as distinct analytical objects.

\begin{lemma}[Exact TE Decomposition]
For any task $k$, assuming the algorithms are task-wise convergent, the Transfer Efficiency decomposes exactly as:
\begin{equation}
    \mathrm{TE}^k = \underbrace{-\mathcal{I}^k}_{\text{Stationary Gap}} + \underbrace{\left(\delta^k(\mathcal{A}_{\text{ITL}}) - \delta^k(\mathcal{A}_{\text{JTL}})\right)}_{\text{Learning Speed Gap}}.
\end{equation}
\end{lemma}

\begin{proof}
By definition, $\mathrm{ALE}^k(\mathcal{A}) = \mathcal{L}^k(Q_\infty^{\mathcal{A}}) + \delta^k(\mathcal{A})$. Substituting this into $\mathrm{TE}^k = \mathrm{ALE}^k(\mathcal{A}_{\text{ITL}}) - \mathrm{ALE}^k(\mathcal{A}_{\text{JTL}})$ and rearranging terms yields the result.
\end{proof}

\vspace{0.3cm}

When tasks have equal duration $N$, we can state a general condition for when transfer is beneficial. To do so, we distinguish the average (per-step) Transient Error $\delta^k$ from the \emph{cumulative} Transient Error, defined as $\Delta^k(\mathcal{A}) := N \cdot \delta^k(\mathcal{A}) = \sum_{n \in S_k} (\mathcal{L}^k(Q_n) - \mathcal{L}^k(Q_\infty^{\mathcal{A}}))$.

\begin{lemma}[The Critical Task Duration]
\label{lem:critical_TD_appd}
Assume a sequence of $K$ tasks of equal duration $N$. Let $\bar{\Delta}(\mathcal{A}) = \frac{1}{K}\sum_{k=1}^K \Delta^k(\mathcal{A})$ denote the expected cumulative Transient Error,
\rev{and $\bar{\mathcal{I}} = \frac{1}{K}\sum_{k=1}^K \mathcal{I}^k$ the average Instability.} Assuming the learning process converges such that $\bar{\Delta}(\mathcal{A}) \in o(N)$, and assuming $\bar{\mathcal{I}}>0$,
there exists a \major{Critical Task Duration} $N^\infty_{\max}$ beyond which transfer becomes detrimental ($\text{TE} < 0$). The breakeven horizon is given by:
\begin{equation}
    N^\infty_{\max} = \frac{\bar{\Delta}(\mathcal{A}_{\text{ITL}}) - \bar{\Delta}(\mathcal{A}_{\text{JTL}})}{\bar{\mathcal{I}}}.
\end{equation}
\end{lemma}

\begin{proof}
    For equal-duration tasks, the average Transfer Efficiency is positive if:
    \[
        \frac{1}{K}\sum_{k=1}^{K} \left(\delta^k(\mathcal{A}_{\text{ITL}}) - \delta^k(\mathcal{A}_{\text{JTL}})\right) >  \frac{1}{K}\sum_{k=1}^{K} \mathcal{I}^k.
    \]
    \rev{By definition, the Instability $\mathcal{I}^k$ depends only on the stationary distributions and is independent of $N$.} Substituting $\delta^k(\mathcal{A}) = \Delta^k(\mathcal{A})/N$, the condition for positive transfer becomes:
    \[
        \frac{\bar{\Delta}(\mathcal{A}_{\text{ITL}}) - \bar{\Delta}(\mathcal{A}_{\text{JTL}})}{N} >  \bar{\mathcal{I}}.
    \]
    If $\bar{\mathcal{I}}>0$ and the cumulative Transient Error is sub-linear ($\bar{\Delta}(\mathcal{A}) \in o(N)$), the left-hand transient advantage vanishes with $N$, while the stationary gap remains. Solving the breakeven equality yields the critical threshold $N_{\max}^\infty$. If $\bar{\mathcal{I}}\le 0$, the stationary component does not make JTL worse on average, so this argument does not imply an eventual detrimental-transfer horizon.
\end{proof}

\paragraph{On the Sub-Linear Transient Error Assumption.}
The condition that the cumulative Transient Error grows sub-linearly ($\Delta^k(\mathcal{A}) \in o(N)$)
\rev{only requires that the algorithm converges toward its stationary risk}, meaning the instantaneous error $\mathcal{L}^k(Q_n) - \mathcal{L}^k(Q_\infty)$ vanishes as $n \to \infty$. This is a mild assumption across standard optimization regimes:
\begin{itemize}
    \item \textbf{Convex Optimization (Linear Models):} In stochastic strongly convex optimization, sub-optimality decays at $\mathcal{O}(1/n)$.
\rev{Summing over $N$ steps yields} a cumulative error bounded by $\mathcal{O}(\log N) \in o(N)$. Even in strictly convex or non-smooth cases with a rate of $\mathcal{O}(1/\sqrt{n})$, the cumulative error is $\mathcal{O}(\sqrt{N}) \in o(N)$ \citep{bubeck2015convex}.
    \item \textbf{Non-Convex Optimization (Neural Networks):}
\rev{Deep networks optimize non-convex landscapes, but in practice they converge to local minima.} Theoretically, under structural assumptions like the Polyak-{\L}ojasiewicz (PL) condition \citep{karimi2016linear} or the Neural Tangent Kernel regime \citep{Jacot18NTK}, SGD can achieve linear convergence rates, decaying exponentially as $\mathcal{O}(e^{-\lambda n})$. In such regimes,
\rev{the cumulative Transient Error is bounded by a constant $\mathcal{O}(1)$, which satisfies the $o(N)$ condition.}
\end{itemize}

\rebuttalB{Since the cumulative Transients themselves depend on $N$, the horizon in \cref{lem:critical_TD_appd} is an \emph{implicit} relation in $N$, and each of its properties rests on a stated assumption.
\rev{Sub-linearity of the cumulative Transient makes the relation solvable. It guarantees that the Instability side of the comparison grows linearly in $N$ while the transient side grows slower, so a crossing exists whenever $\bar{\mathcal{I}} > 0$. In the convergent regimes, where the cumulative transient gap approaches a constant (or grows as $O(\log N)$ or $O(\sqrt{N})$), sub-linearity also lets the asymptotic horizon be written with an $N$-independent numerator.} Uniqueness of the crossing requires the further assumption that the average transient gap $\big(\bar{\Delta}(\mathcal{A}_{\text{ITL}}) - \bar{\Delta}(\mathcal{A}_{\text{JTL}})\big)/N$
\rev{is non-increasing in $N$, which holds in those same regimes; without it, $N^\infty_{\max}$ denotes the last crossing.}}

\subsection{Exact Analysis in the Quadratic Regime}
\label{sec:quadratic_analysis}

\rev{We now specialize our analysis to quadratic loss functions. This setting admits closed-form derivations of the learning dynamics, with explicit expressions} for both the asymptotic stability (Instability) and the convergence speed (Transient Error).

\begin{assumption}[Quadratic Tasks]
\label{ass:quadratic}
    For every task $j \in \{1, \dots, K\}$, the loss function is defined as:
    \begin{equation}
        \mathcal{L}^j(\theta) = \frac{1}{2}(\theta - \theta_\star^j)^\top H_j (\theta - \theta_\star^j) + c_j.
    \end{equation}
    We assume the Hessian $H_j \succ 0$ is symmetric positive definite. The scalar $c_j$ represents the irreducible loss (or Bayes error) for task $j$.
\end{assumption}

Under this assumption, the expected risk for task $k$, $\mathbb{E}_\theta[\mathcal{L}^k(\theta_n)]$, is completely determined by the first two moments of the parameter distribution: the mean $\mu_n = \mathbb{E}[\theta_n]$ and the covariance $\Sigma_n = \text{Cov}(\theta_n)$:
\begin{equation}
\label{eq:risk_decomp}
    \mathbb{E}[\mathcal{L}^k(\theta_n)] = \underbrace{\frac{1}{2} \| \mu_n - \theta_\star^k \|_{H_k}^2}_{\text{Bias Contribution}} + \underbrace{\frac{1}{2} \text{Tr}(H_k \Sigma_n)}_{\text{Variance Contribution}} + c_k.
\end{equation}

\subsubsection{SGD Dynamics on Quadratic Landscapes}

Before analyzing the JTL/ITL trade-offs,
\rev{we recall two standard results on the behavior of} constant step-size SGD on quadratic objectives.

\rebuttal{\rev{All derivations in this section use the same asymptotic regime, which we state once here.}

\begin{assumption}[Small step-size regime]
\label{ass:small_step}
The step size is constant and satisfies $\eta < 1/\lambda_{\max}$, and all stationary and transient quantities are computed to leading order in $\eta\lambda_{\max}$: we retain terms of order $\eta$ and neglect corrections of relative order $\mathcal{O}(\eta\lambda_{\max})$.
\end{assumption}

}

\begin{lemma}[Convergence of the Mean]
\label{lemma:mean_convergence}
    Consider the update $\theta_{n+1} = \theta_n - \eta \hat{g}_n$ where $\mathbb{E}[\hat{g}_n \mid \theta_n] = H(\theta_n - \theta_\star)$. The expected parameter $\mu_n$ converges linearly to the optimum:
    \begin{equation}
        \mu_{n} - \theta_\star = (I - \eta H)^n (\mu_0 - \theta_\star).
    \end{equation}
    Convergence is guaranteed if the step-size satisfies $\eta < 2/\lambda_{\max}(H)$, where $H$ is the Hessian matrix of the objective being optimized.
\end{lemma}
\begin{proof}
    Taking expectations of the update rule yields $\mu_{n+1} = \mu_n - \eta H(\mu_n - \theta_\star)$. Subtracting $\theta_\star$ from both sides gives the recurrence $\mu_{n+1} - \theta_\star = (I - \eta H)(\mu_n - \theta_\star)$, which unrolls to the stated result.
\end{proof}

\begin{lemma}[Parameter Covariance]
\label{lemma:PCV}
    Consider a learning process with constant step-size $\eta$. The parameter covariance $\Sigma_n$ evolves according to the linear recursion:
    \[
        \Sigma_{n+1} = (I - \eta H)\Sigma_n(I - \eta H)^\top + \eta^2 C_n,
    \]
    where $C_n = \text{Cov}(\hat{g}_n)$ is the covariance of the batch gradient. For small step-sizes \rebuttal{(\cref{ass:small_step})}, the stationary covariance matrix $\Sigma_\infty$ satisfies the continuous-time Lyapunov equation:
    \[
        H \Sigma_\infty + \Sigma_\infty H = \eta C_\infty.
    \]
\end{lemma}
\begin{proof}
    The SGD update is $\theta_{n+1} = \theta_n - \eta \bar{g}_n$. Let $\mu_n = \mathbb{E}[\theta_n]$ be the mean parameter and $\tilde{\theta}_n = \theta_n - \mu_n$ be the centered iterate. The mean evolves as $\mu_{n+1} = \mu_n - \eta \mathbb{E}[\bar{g}_n]$. For quadratic losses, $\mathbb{E}[\bar{g}_n] = H(\mu_n - \theta_\star)$.
    The evolution of the fluctuation is:
    \[
        \tilde{\theta}_{n+1} = \theta_{n+1} - \mu_{n+1} = \tilde{\theta}_n - \eta(\bar{g}_n - \mathbb{E}[\bar{g}_n]).
    \]
    Expanding $\bar{g}_n$ around the mean $\mu_n$:
    \[
        \bar{g}_n \approx H(\theta_n - \theta_\star) + \xi_n = H\tilde{\theta}_n + H(\mu_n - \theta_\star) + \xi_n,
    \]
    where $\xi_n$ is the zero-mean batch noise with covariance $C_n$. Substituting this into the fluctuation update:
    \[
        \tilde{\theta}_{n+1} = (I - \eta H)\tilde{\theta}_n - \eta \xi_n.
    \]
    The covariance $\Sigma_{n+1} = \mathbb{E}[\tilde{\theta}_{n+1}\tilde{\theta}_{n+1}^\top]$ is then:
    \[
        \Sigma_{n+1} = (I - \eta H)\mathbb{E}[\tilde{\theta}_n \tilde{\theta}_n^\top](I - \eta H)^\top + \eta^2 \mathbb{E}[\xi_n \xi_n^\top],
    \]
    where the cross-terms vanish because $\xi_n$ is independent of $\tilde{\theta}_n$. Substituting $C_n = \mathbb{E}[\xi_n \xi_n^\top]$ yields the first result.

    Solving for stationarity ($\Sigma_{n+1} = \Sigma_n = \Sigma_\infty$) gives:
    \[
        \Sigma_\infty = (I - \eta H)\Sigma_\infty(I - \eta H)^\top + \eta^2 C_\infty.
    \]
    Expanding and ignoring the higher-order $\eta^2 H \Sigma_\infty H^\top$ term (by \cref{ass:small_step}) yields the Lyapunov equation $\eta H\Sigma_\infty + \eta \Sigma_\infty H = \eta^2 C_\infty$.
\end{proof}

\begin{lemma}[ITL Gradient Noise and Parameter Covariance]
\label{lemma:GNC}
    Consider a learning process with batch size $B \ge 1$ and constant step-size $\eta$. If the target noise is i.i.d. sampled from a zero-mean Gaussian $\epsilon \sim \mathcal{N}(0, \sigma^2)$, the gradient noise covariance for a single task is locally constant along the optimization trajectory, taking the form:
    \[
        C_n \approx C_\infty = \frac{\sigma^2}{B} H.
    \]
    Consequently, by \cref{lemma:PCV}, for small step-sizes, the stationary parameter covariance matrix $\Sigma_\infty$ is isotropic:
    \[
        \Sigma_\infty = \frac{\eta\sigma^2}{2B} I.
    \]
\end{lemma}
\begin{proof}
    Consider the quadratic loss $\ell(\theta, (x,y)) = \frac{1}{2}(y - \theta^\top x)^2$ where $y = \theta_\star^\top x + \epsilon$. The gradient for a single sample is $\hat{g} = -x(x^\top \Delta\theta + \epsilon)$, where $\Delta\theta = \theta_\star - \theta$.
    The batch gradient is the average $\bar{g}_n = \frac{1}{B} \sum_{b=1}^B \hat{g}_{n,b}$. Since samples within a batch are i.i.d., the covariance of the batch gradient scales as:
    \[
        \mathrm{Cov}(\bar{g}_n) = \frac{1}{B} \mathrm{Cov}(\hat{g}_n) = \frac{1}{B} \left( \mathbb{E}[\hat{g}_n \hat{g}_n^\top \mid \theta_n] - \mathbb{E}[\hat{g}_n]\mathbb{E}[\hat{g}_n]^\top \right).
    \]
    Expanding the single-sample second moment as an outer product:
    \begin{align*}
        \mathbb{E}[\hat{g}_n \hat{g}_n^\top \mid \theta_n] &= \mathbb{E}\left[(xx^\top \Delta\theta + x\epsilon)(xx^\top \Delta\theta + x\epsilon)^\top\right] \\
        &= \mathbb{E}[xx^\top ( \Delta\theta \Delta\theta^\top) xx^\top ] + \mathbb{E}[x\epsilon^2 x^\top].
    \end{align*}
    Using $\mathbb{E}[\epsilon]=0$ and $\mathbb{E}[\epsilon^2]=\sigma^2$, the covariance becomes:
    \[
        \mathrm{Cov}(\hat{g}_n) = \sigma^2 H + \underbrace{\mathbb{E}[xx^\top ( \Delta\theta \Delta\theta^\top) xx^\top ] - H(\Delta\theta\Delta\theta^\top) H}_{\text{State-dependent noise}}.
    \]
    The state-dependent noise decays toward zero as the agent converges ($\Delta \theta \to 0$).
\rev{If this state-dependent noise is negligible compared to the irreducible noise floor at convergence, then $\mathrm{Cov}(\hat{g}_n) \approx \sigma^2 H$ and the batch covariance is $\frac{\sigma^2}{B}H$.} \rebuttal{\rev{This assumption is made only at stationarity:} away from it the state-dependent noise is not negligible, and it is retained as the source term in the transient analysis (Step 1 of the proof of \cref{theo:transient_potential}).}

    Plugging this into the Lyapunov equation from \cref{lemma:PCV}:
    \[
        H\Sigma_\infty + \Sigma_\infty H = \frac{\eta\sigma^2}{B}H.
    \]
    Since $H$ and $\Sigma_\infty$ must commute, we obtain $\Sigma_\infty = \frac{\eta\sigma^2}{2B} I$.
\end{proof}
\rebuttalB{The size of the neglected state-dependent term is measured in \cref{apdx:assumption-studies} (\cref{tab:assumption-noise}): simulations with sampled inputs deviate from the closed-form predictions by an additional $\mathcal{O}(\eta \operatorname{Tr}(H)/B)$, while the predicted transfer boundary is unchanged.}

\begin{lemma}[JTL Gradient Noise Covariance]
\label{lemma:MT_GNC}
    Consider a Joint-Task Learning (JTL) agent optimizing a mixture of tasks with weights $\{w_j^k\}_{j=1}^k$. Let $H_j$ be the Hessian of task $j$ and $\theta_\star^j$ its optimum. The stationary gradient noise covariance is:
    \[
        C^{\text{JTL}}_\infty = \frac{1}{B} \left( \mathbb{E}_{j} [ \sigma_j^2 H_j] + \mathbb{E}_j \left[ \nabla \mathcal{L}^j(\theta_\star^{\text{JTL},k}) \nabla \mathcal{L}^j(\theta_\star^{\text{JTL},k})^\top \right] \right).
    \]
\end{lemma}

\begin{proof}
    Recall that the mixture optimum is defined as:
    \begin{equation}
        \theta_\star^{\text{JTL},k} = \left(\sum_{j=1}^k w_j^k H_j\right)^{-1} \left(\sum_{j=1}^k w_j^k H_j \theta_\star^j\right).
    \end{equation}
    At the mixture optimum $\theta_\star^{\text{JTL},k}$, the expected gradient of a sampled task $j$ is $g_j = H_j (\theta_\star^{\text{JTL},k} - \theta_\star^j)$. By definition of the optimum, the weighted sum of these gradients is zero: $\mathbb{E}_j[g_j] = \sum w_j^k g_j = 0$. Using the Law of Total Variance for the batch gradient noise $C^{\text{JTL}}_n$:
    \begin{equation}
        C^{\text{JTL}}_n = \frac{1}{B} \left( \mathbb{E}_{j} [ \mathrm{Cov}_{\xi^j}(\hat{g}^j_n) ] + \mathrm{Cov}_j ( \mathbb{E}[\hat{g}^j_n] ) \right).
    \end{equation}
    By \cref{lemma:GNC}, for each task, $\mathrm{Cov}(\hat{g}_j \mid j) = \sigma_j^2 H_j$. Because the mean expected gradient is zero,
\rev{the variance of the conditional expectations is the second moment:}
    \begin{equation}
        \mathrm{Cov}_j ( g_j ) = \mathbb{E}_j[g_j g_j^\top] = \sum_{j=1}^k w_j^k \left[ H_j (\theta_\star^{\text{JTL},k} - \theta_\star^j)(\theta_\star^{\text{JTL},k} - \theta_\star^j)^\top H_j \right].
    \end{equation}
    Thus, noting that $g_j = \nabla \mathcal{L}^j(\theta_\star^{\text{JTL},k})$, the total noise covariance at convergence is:
    \[
        C^{\text{JTL}}_\infty = \frac{1}{B} \left( \mathbb{E}_{j} [ \sigma_j^2 H_j] + \mathbb{E}_j \left[ \nabla \mathcal{L}^j(\theta_\star^{\text{JTL},k}) \nabla \mathcal{L}^j(\theta_\star^{\text{JTL},k})^\top \right] \right).
    \]
\end{proof}

\subsubsection{JTL Geometry}

For JTL,
\rev{deriving the Instability and the transient cost requires handling two distinct geometries:} the curvature driving the optimization dynamics (the mixture Hessian, $H_{\text{dyn}}$) and the curvature determining the evaluation cost (the target task Hessian, $H_{\text{eval}}$).

To obtain a tractable closed-form solution that preserves geometric intuition, we rely on a task alignment assumption.

\begin{assumption}[Task Alignment]
\label{ass:alignment}
    We assume the Hessians of all tasks are simultaneously diagonalizable. This implies that while tasks may vary in difficulty (eigenvalues), their principal axes of curvature (eigenvectors) are aligned. Consequently, the driving and evaluation Hessians commute:
    \begin{equation}
        H_{\text{dyn}} H_{\text{eval}} = H_{\text{eval}} H_{\text{dyn}}.
    \end{equation}
\end{assumption}

To interpret the cost geometrically, we introduce the \textbf{Relative Curvature Matrix}, denoted by $\mathcal{R}$. This matrix measures how \emph{steep} the evaluation landscape is relative to the driving landscape:
\begin{equation}
    \mathcal{R} := H_{\text{dyn}}^{-1} H_{\text{eval}}.
\end{equation}
\rev{The eigenvalues of $\mathcal{R}$ scale the cost. If $\mathcal{R} \prec I$, the agent benefits from a pre-conditioning effect: it trains on a steeper surrogate than the test task.}

We also require the following assumption, standard in the analysis of constant step-size SGD \citep{mandt2017stochastic, bach}.

\begin{assumption}[Bounded Curvature]
\label{ass:bounded_curvature}
We assume the Hessian matrices for all tasks have eigenvalues bounded strictly away from zero and infinity. There exist constants $0 < \lambda_{\min} \le \lambda_{\max}$ such that for any task $k$:
\[
    \lambda_{\min} I \preceq H^k \preceq \lambda_{\max} I.
\]
Consequently, the mixture Hessian $H_{\text{dyn}} = \sum_{j=1}^k w_j^k H^j$ satisfies the same spectral bounds.
\end{assumption}

\rebuttal{Neither \cref{ass:alignment} nor \cref{ass:bounded_curvature} is essential to the results of this section:
\rev{both simplify the derivations, and we relax each of them without changing the main conclusions.} Task Alignment is removed in \cref{sec:beyond_alignment} and Bounded Curvature in \cref{sec:ridgeless}, both in \cref{sec:relaxing}.} \rebuttalB{\rev{Simulation studies (\cref{apdx:assumption-studies}) confirm both relaxations quantitatively. Starting from the setting of \cref{fig:sim-shallow}, they remove these assumptions and the constant-noise approximation of \cref{lemma:GNC} one at a time.}}

\subsubsection{Exact Instability}

We now quantify the \textbf{Instability} $\mathcal{I}^k$,
\rev{defined as the difference in asymptotic performance between the JTL and ITL agents on task $k$.}

\begin{theorem}[Exact Instability Decomposition]
\label{prop:exact_instability}
    Let the JTL agent optimize a mixture of tasks with weights $\{w_j^k\}_{j\le k}$, converging in mean to the mixture optimum $\theta_\star^{\text{JTL},k}$. Assuming the observation noise is uniform across tasks ($\sigma_j^2 \approx \sigma^2$) and tasks are aligned (\cref{ass:alignment}), the Instability gap on task $k$ decomposes exactly into:
    \begin{equation}
        \mathcal{I}^k \approx \underbrace{\mathcal{L}^k(\theta_\star^{\text{JTL},k}) - \mathcal{L}^k(\theta_\star^{k})}_{\text{Approximation Bias}} + \frac{\eta}{4B}\underbrace{ \sum_{j=1}^k w_j^k \, \|\nabla \mathcal{L}^j(\theta_\star^{\text{JTL},k})\|^2_{\mathcal{R}^k}}_{\text{Task Diversity}}.
    \end{equation}
\end{theorem}

\begin{proof}
    The instability $\mathcal{I}^k$ is the difference between the asymptotic risks of the JTL and ITL agents. Using the decomposition in \eqref{eq:risk_decomp}, we analyze the bias and variance components separately.

    First, consider the \textbf{bias}. The ITL agent is unbiased for task $k$, as its mean converges exactly to $\theta_\star^k$. In contrast, the JTL agent converges in mean to the stationary point of the mixture loss:
    \[
        \theta_\star^{\text{JTL},k} = \left(\sum_{j=1}^k w_j^k H_j\right)^{-1} \left(\sum_{j=1}^k w_j^k H_j \theta_\star^j\right).
    \]
    This results in a permanent bias penalty of $\frac{1}{2} \| \theta_\star^{\text{JTL},k} - \theta_\star^k \|_{H_k}^2 = \mathcal{L}^k(\theta_\star^{\text{JTL},k}) - \mathcal{L}^k(\theta_\star^{k})$.

    Second, consider the \textbf{variance}. The variance contribution to the expected loss of task $k$ is proportional to the trace of the Hessian and the stationary parameter covariance $\Sigma_\infty$, i.e., $\frac{1}{2} \text{Tr}(H_k \Sigma_\infty)$. By \cref{lemma:GNC},
\rev{for the ITL agent this is $\frac{\eta \sigma^2}{4B} \text{Tr}(H_k)$.}

    For the JTL agent, the stationary covariance $\Sigma^{\text{JTL},k}_\infty$ satisfies the Lyapunov equation:
    \begin{equation}
        H_{\text{dyn}} \Sigma^{\text{JTL},k}_\infty + \Sigma^{\text{JTL},k}_\infty H_{\text{dyn}} = \eta C^{\text{JTL},k}_\infty.
    \end{equation}
    Although the gradient covariance term $C^{\text{JTL},k}_\infty$
\rev{need not share an eigenbasis with $H_{\text{dyn}}$}, $H_k$ and $H_{\text{dyn}}$ commute under \cref{ass:alignment}.
\rev{This lets us apply the trace identity} for Lyapunov equations ($\text{Tr}(M \Sigma) = \frac{1}{2} \text{Tr}(M H^{-1} \eta C)$)
\rev{to compute the loss contribution directly:}
    \begin{equation}
         \frac{1}{2} \text{Tr}(H_k \Sigma^{\text{JTL},k}_\infty) = \frac{\eta}{4} \text{Tr}\left( H_k H_{\text{dyn}}^{-1} C^{\text{JTL},k}_\infty \right) = \frac{\eta}{4} \text{Tr}\left( \mathcal{R}^k C^{\text{JTL},k}_\infty \right).
    \end{equation}
    \rev{By definition, $\mathcal{R}^k = H_{\text{dyn}}^{-1} H_k$.} Because the task Hessians commute, $\mathcal{R}^k$ is symmetric and $H_{\text{dyn}}^{-1} H_k = H_k H_{\text{dyn}}^{-1} = \mathcal{R}^k$.
    Substituting the JTL gradient noise covariance from \cref{lemma:MT_GNC}:
    \begin{equation}
        \frac{\eta}{4B} \text{Tr}\left( \mathcal{R}^k \mathbb{E}_j [\sigma_j^2 H_j] \right) + \frac{\eta}{4B} \text{Tr}\left( \mathcal{R}^k \mathbb{E}_j [g_j g_j^\top] \right),
    \end{equation}
    where $g_j = \nabla \mathcal{L}^j(\theta_\star^{\text{JTL},k})$.
    Under the assumption of uniform observation noise ($\sigma_j^2 \approx \sigma^2$), the first term yields $\frac{\eta\sigma^2}{4B} \text{Tr}(\mathcal{R}^k H_{\text{dyn}}) = \frac{\eta\sigma^2}{4B} \text{Tr}(H_k)$, which exactly cancels the ITL variance floor in the instability gap.

    The remaining term is the \textbf{Task Diversity} penalty. Using the linearity of expectation and the trace property $\text{Tr}(A u u^\top) = u^\top A u$:
    \begin{equation}
        \frac{\eta}{4B} \mathbb{E}_j \left[ \text{Tr}(\mathcal{R}^k g_j g_j^\top) \right] = \frac{\eta}{4B} \sum_{j=1}^k w_j^k g_j^\top \mathcal{R}^k g_j = \frac{\eta}{4B} \sum_{j=1}^k w_j^k \|\nabla \mathcal{L}^j(\theta_\star^{\text{JTL},k})\|^2_{\mathcal{R}^k},
    \end{equation}
    completing the decomposition.
\end{proof}

\rebuttal{\Cref{ass:alignment} enters this proof only through the trace identity; it is removed in \cref{sec:beyond_alignment}, where the exact variance term is derived for arbitrary, non-commuting Hessians.}

\subsubsection{Exact Transient Error}
\label{sec:transient_exact}

We now derive the explicit form of the Transient Error $\delta_N$. This metric captures the cumulative cost of learning as the agent travels from its initialization to its state at step $N$.

\vspace{0.5cm}

\begin{theorem}[The Transient Potential Law]
\label{theo:transient_potential}
    Consider an SGD agent characterized by its deviation from stationarity: the mean displacement $\tilde{\mu}_n = \mu_n - \mu_\infty$ and the covariance deviation $\tilde{\Sigma}_n = \Sigma_n - \Sigma_\infty$. Let $\delta_N$ be the transient cost accumulated over a horizon $N$.
    Under \cref{ass:alignment}\rebuttal{ and \cref{ass:small_step}}, $\delta_N$ is proportional to the drop in the \emph{weighted potential energy} of the system from start to finish:
    \begin{equation}
        \delta_N \approx \frac{1}{N\tau} \left( \mathcal{E}(\tilde{\mu}_0, \tilde{\Sigma}_0) - \mathcal{E}(\tilde{\mu}_N, \tilde{\Sigma}_N) \right),
    \end{equation}
    where $\tau = 4\eta$ is the learning timescale.
    The potential energy function $\mathcal{E}$ is defined by the Relative Curvature $\mathcal{R}$ as:
    \begin{equation}
        \mathcal{E}(\tilde{\mu}, \tilde{\Sigma}) = \underbrace{\tilde{\mu}^\top \mathcal{R} \tilde{\mu}}_{\text{Mean Potential}} + \underbrace{\text{Tr}(\tilde{\Sigma} \mathcal{R})}_{\text{Covariance Potential}}.
    \end{equation}
\end{theorem}

\paragraph{Geometric Interpretation.}
\rev{This result gives a physical intuition for the cost of learning: it is the difference} between where the agent started and where it ended, measured on a quadratic potential surface warped by $\mathcal{R}$. This curvature takes different forms depending on the optimization strategy:
\begin{itemize}
    \item \textbf{ITL:} Here, $H_{\text{dyn}} = H_{\text{eval}}$, meaning $\mathcal{R} = I$. The potential is the standard Euclidean squared distance, and the transient cost is strictly proportional to the physical distance traveled.
    \item \textbf{JTL:} Here, $\mathcal{R} \neq I$. The potential surface is stretched according to the task mixture. Directions where the mixture is steeper than the target task (eigenvalues of $\mathcal{R} < 1$) correspond to discounted costs.
\rev{Training on a diverse mixture thus acts as an implicit pre-conditioner, allowing the agent to ``fall'' into} the target task's basin more rapidly along those axes.
\end{itemize}

\begin{proof}
    The proof proceeds in four steps: establishing the dynamics of the second moment, summing the excess risk, simplifying the resulting geometric series, and relating the sum back to the state at step $N$. First, recall that:
    \begin{equation}
        \delta_N := \frac{1}{N}\sum_{n=0}^{N-1} \mathbb{E}\left[ \mathcal{L}(\theta_n) - \mathcal{L}(\theta_\star) \right].
    \end{equation}

    \textbf{Step 1.} Let $\tilde{\theta}_n = \theta_n - \mu_\infty$ be the error vector relative to the stationary mean. The dynamics are governed by the driving Hessian $H_{\text{dyn}}$. The SGD update implies that the second moment matrix $S_n = \mathbb{E}[\tilde{\theta}_n \tilde{\theta}_n^\top]$ evolves according to the Lyapunov recursion (\cref{lemma:PCV}):
    \begin{equation}
        S_{n+1} = (I - \eta H_{\text{dyn}}) S_n (I - \eta H_{\text{dyn}}) + \eta^2 C_n,
    \end{equation}
    where $C_n$ is the gradient noise covariance.
    The stationary covariance $\Sigma_\infty$ is the fixed point of this equation. To isolate the transient component, we consider the deviation $\Delta_n = S_n - \Sigma_\infty$. Let $C_\infty$ denote the gradient noise covariance at stationarity. Then:
    \[
        \Delta_{n+1} = (I - \eta H_{\text{dyn}}) \Delta_n (I - \eta H_{\text{dyn}}) + \eta^2 (C_n - C_\infty).
    \]

    For the ITL agent, $C_n - C_\infty \approx 0$ throughout training. For the JTL agent, $C_n - C_\infty \approx \frac{1}{B}\mathbb{E}_j [H_j \bar{S}_n H_j]$, where $\bar{S}_n = (\mu_n - \theta_\star^{\text{JTL},k})(\mu_n -\theta_\star^{\text{JTL},k})^\top$.
    Because $\bar{S}_n$ is $\mathcal{O}(1)$ at initialization and decays to zero, the source term $\eta^2 (C_n - C_\infty)$ is of order $\mathcal{O}(\eta^2)$, whereas the homogeneous part of the deviation $\Delta_n$ is $\mathcal{O}(1)$. \rebuttal{(\Cref{ass:small_step} again: the neglected source term contributes a correction of relative order $\eta\lambda_{\max}$.)} Therefore, in the small step-size limit, the transient trajectory is dominated by the geometric decay:
    \begin{equation}
        \Delta_n \approx (I - \eta H_{\text{dyn}})^n \Delta_0 (I - \eta H_{\text{dyn}})^n.
    \end{equation}

    \textbf{Step 2.} Because the risk is determined by the evaluation Hessian $H_{\text{eval}}$, the expected excess risk at step $n$ is $\frac{1}{2}\text{Tr}(H_{\text{eval}} \Delta_n)$. Summing over the horizon yields $\delta_N$:
    \begin{equation}
        \delta_N = \frac{1}{2N} \sum_{n=0}^{N-1} \text{Tr}(H_{\text{eval}} \Delta_n).
    \end{equation}
    Substituting the propagator form and
\rev{using the commutativity of $H_{\text{eval}}$ and $H_{\text{dyn}}$} (\cref{ass:alignment}):
    \begin{equation}
        \delta_N \approx \frac{1}{2N} \text{Tr}\left( \Delta_0 H_{\text{eval}} \sum_{n=0}^{N-1} (I - \eta H_{\text{dyn}})^{2n} \right).
    \end{equation}

    \textbf{Step 3.} The matrix geometric series $\sum_{n=0}^{N-1} A^n$ with $A = (I - \eta H_{\text{dyn}})^2$ is $(I - A)^{-1}(I - A^N)$. Expanding for small $\eta$:
    \begin{equation}
        I - (I - \eta H_{\text{dyn}})^2 = 2\eta H_{\text{dyn}} - \eta^2 H_{\text{dyn}}^2 \approx 2\eta H_{\text{dyn}}.
    \end{equation}
    Thus, $(I - A)^{-1} \approx (2\eta H_{\text{dyn}})^{-1}$. Introducing the Relative Curvature $\mathcal{R} = H_{\text{dyn}}^{-1} H_{\text{eval}}$:
    \begin{equation}
        \delta_N \approx \frac{1}{4N\eta} \text{Tr}\left( \Delta_0 \mathcal{R} \left[ I - (I - \eta H_{\text{dyn}})^{2N} \right] \right).
    \end{equation}

    \textbf{Step 4.} Distributing the trace and using the commutativity of $\mathcal{R}$ with the propagator:
    \begin{equation}
        \text{Tr}\left( \Delta_0 \mathcal{R} (I - \eta H_{\text{dyn}})^{2N} \right) = \text{Tr}\left( \mathcal{R} (I - \eta H_{\text{dyn}})^N \Delta_0 (I - \eta H_{\text{dyn}})^N \right) = \text{Tr}(\Delta_N \mathcal{R}).
    \end{equation}
    This yields the potential energy conservation law:
    \begin{equation}
        \delta_N \approx \frac{1}{4N\eta} \left( \text{Tr}(\Delta_0 \mathcal{R}) - \text{Tr}(\Delta_N \mathcal{R}) \right).
    \end{equation}
    Finally, decomposing the deviation into mean and covariance components, $\Delta_n = \tilde{\mu}_n \tilde{\mu}_n^\top + \tilde{\Sigma}_n$, and using $\text{Tr}(x x^\top \mathcal{R}) = \|x\|^2_{\mathcal{R}}$, we obtain:
    \begin{equation}
        \delta_N \approx \frac{1}{4N\eta} \left[ \left( \|\tilde{\mu}_0\|^2_{\mathcal{R}} + \text{Tr}(\tilde{\Sigma}_0 \mathcal{R}) \right) - \left( \|\tilde{\mu}_N\|^2_{\mathcal{R}} + \text{Tr}(\tilde{\Sigma}_N \mathcal{R}) \right) \right].
    \end{equation}
\end{proof}

\rebuttal{The alignment assumption is likewise removed from the Transient analysis in \cref{sec:beyond_alignment} (\cref{prop:transient_no_alignment}).}

\subsection{Asymptotic Dynamics of the First and Second Order Moments}
\label{sec:mt_asymptotics}

In \cref{sec:quadratic_analysis}, we derived closed-form expressions for the instability and transient costs by taking expectations over the stochasticity of the agent's parameter trajectory. To analyze the asymptotic stability of the Joint-Task Learning (JTL) and Independent-Task Learning (ITL) agents as the number of tasks grows, we must now account for the randomness in the environment itself. Therefore, in this section, unless otherwise specified, the expectation operator $\mathbb{E}[\cdot]$ is taken over the \emph{task generation process} (i.e., the distribution of the task sequence).

To formalize how ``concentrated'' or ``spread out'' the tasks are, we introduce the notion of a \emph{Growth Rate} $\psi(k)$, which characterizes the asymptotic spread of the task optima as a function of the number of tasks $k$.

\begin{definition}[Growth Rate]
\label{ass:growth-rate-app}
    Let $\psi: \mathbb{N} \to \mathbb{R}^+$ be a monotonically increasing function such that the expected squared distance of the $k$-th task optimum from the origin scales as:
    \begin{equation}
        \mathbb{E}[\|\theta_\star^k\|^2] \in \Theta(\psi(k)).
    \end{equation}
\end{definition}

In the analysis that follows, we consider \rev{three regimes} of environment growth: \emph{Bounded} ($\psi(k) = 1$), \emph{Diffusive} ($\psi(k) = k$), and \emph{Linear} ($\psi(k) = k^2$).
\major{We call the \emph{canonical environment} of each regime its simplest generating process: i.i.d.\ tasks from a fixed distribution (Bounded), a random walk with i.i.d.\ centered increments (Diffusive), and constant-velocity drift (Linear); these are the processes of \cref{lemma:dispersion_examples} and of \cref{fig:growth_rate_final}.}

\rebuttal{The Growth Rate controls distances \emph{from the origin}, whereas the Instability and the drift quantities analyzed below depend on distances \emph{from the running JTL solution} $\theta_\star^{\text{JTL},k}$. These are not interchangeable in general, so we make the link explicit. One direction is provable for all Quadratic Tasks of the form \cref{ass:quadratic}: the JTL solution can grow at most as fast as the tasks themselves.

\begin{lemma}[Centroid growth is dominated by task growth]
\label{lemma:centroid_growth}
Under \cref{ass:bounded_curvature}, with $\kappa = \lambda_{\max}/\lambda_{\min}$,
\begin{equation}
    \mathbb{E}\big[\|\theta_\star^{\text{JTL},k}\|^2\big] \in \mathcal{O}(\psi(k)),
    \qquad\text{and consequently}\qquad
    \mathbb{E}\big[\|\theta_\star^{k} - \theta_\star^{\text{JTL},k-1}\|^2\big] \in \mathcal{O}(\psi(k)).
\end{equation}
\end{lemma}
\begin{proof}
From $\theta_\star^{\text{JTL},k} = \bar{H}_k^{-1}\sum_j w_j^k H_j \theta_\star^j$ with $\bar{H}_k = \sum_j w_j^k H_j \succeq \lambda_{\min} I$ and $\|H_j\|_2 \le \lambda_{\max}$,
\[
    \|\theta_\star^{\text{JTL},k}\| \;\le\; \frac{\lambda_{\max}}{\lambda_{\min}} \sum_{j=1}^k w_j^k \|\theta_\star^j\| \;=\; \kappa \sum_{j=1}^k w_j^k \|\theta_\star^j\|.
\]
Since the weights sum to one, Jensen's inequality gives $\big(\sum_j w_j^k \|\theta_\star^j\|\big)^2 \le \sum_j w_j^k \|\theta_\star^j\|^2$, so
\[
    \mathbb{E}\big[\|\theta_\star^{\text{JTL},k}\|^2\big] \;\le\; \kappa^2 \sum_{j=1}^k w_j^k\, \mathbb{E}\big[\|\theta_\star^j\|^2\big] \;\in\; \mathcal{O}(\psi(k)),
\]
using the monotonicity of $\psi$. The second claim follows from $\|a-b\|^2 \le 2\|a\|^2 + 2\|b\|^2$ and $\psi(k-1) \le \psi(k)$.
\end{proof}

In the canonical environment of each growth regime we can prove that the Growth Rate is also a \emph{lower} bound on the expected Instability (\cref{cor:growth_lower} below). All expectations are over the task-generation process, per the convention fixed at the beginning of this section.

\begin{lemma}[Growth realized as dispersion in the canonical environments]
\label{lemma:dispersion_examples}
Let task durations be equal ($w_j^k = 1/k$) and the Hessians identical across tasks, so that $\theta_\star^{\text{JTL},k} = c_k := \frac{1}{k}\sum_{j \le k} \theta_\star^j$. Then, in each regime, the growth envelope is realized as dispersion of the tasks around the JTL solution, with exact constants:
\begin{enumerate}
    \item \textbf{Bounded} ($\psi(k) = 1$): for i.i.d.\ tasks with covariance of trace $s > 0$,
    \begin{equation}
        \mathbb{E}\|\theta_\star^k - c_{k-1}\|^2 = s\Big(1 + \frac{1}{k-1}\Big),
        \qquad
        \frac{1}{k}\sum_{j=1}^{k} \mathbb{E}\|\theta_\star^j - c_k\|^2 = s\Big(1 - \frac{1}{k}\Big).
    \end{equation}
    \item \textbf{Diffusive} ($\psi(k) = k$): for the random walk $\theta_\star^k = \sum_{m \le k} \Delta_m$ with i.i.d.\ centered increments of covariance trace $s > 0$,
    \begin{equation}
        \mathbb{E}\|\theta_\star^k - c_{k-1}\|^2 = s\Big(\frac{(k-2)(2k-3)}{6(k-1)} + 1\Big) = \frac{s\,k}{3} + \mathcal{O}(1),
        \qquad
        \frac{1}{k}\sum_{j=1}^{k} \mathbb{E}\|\theta_\star^j - c_k\|^2 = \frac{s\,(k^2-1)}{6k} = \frac{s\,k}{6} + \mathcal{O}(1).
    \end{equation}
    \item \textbf{Linear} ($\psi(k) = k^2$): for the constant-velocity drift $\theta_\star^k = k\,v$,
    \begin{equation}
        \|\theta_\star^k - c_{k-1}\|^2 = \frac{\|v\|^2 k^2}{4},
        \qquad
        \frac{1}{k}\sum_{j=1}^{k} \|\theta_\star^j - c_k\|^2 = \frac{\|v\|^2 (k^2-1)}{12}.
    \end{equation}
\end{enumerate}
\end{lemma}

\begin{proof}
All three cases use the variance-decomposition (K\"onig--Huygens) identity: for any $\theta_1, \dots, \theta_k$ with average $c_k$,
\begin{equation}
\label{eq:huygens}
    \frac{1}{k}\sum_{j=1}^k \|\theta_j - c_k\|^2 \;=\; \frac{1}{k}\sum_{j=1}^k \|\theta_j\|^2 \;-\; \|c_k\|^2 .
\end{equation}

\textbf{(i) Bounded.} Center the tasks at their common mean, $\bar\theta_j := \theta_\star^j - \mathbb{E}[\theta_\star^j]$, so that $\mathbb{E}\|\bar\theta_j\|^2 = s$. By independence of $\bar\theta_k$ and $\bar{c}_{k-1} = \frac{1}{k-1}\sum_{j < k}\bar\theta_j$,
\begin{equation}
    \mathbb{E}\|\theta_\star^k - c_{k-1}\|^2 = \mathbb{E}\|\bar\theta_k\|^2 + \mathbb{E}\|\bar{c}_{k-1}\|^2 = s + \frac{s}{k-1}.
\end{equation}
For the second display, take the expectation of \eqref{eq:huygens} applied to the centered variables: $\frac{1}{k}\sum_j \mathbb{E}\|\bar\theta_j\|^2 - \mathbb{E}\|\bar{c}_k\|^2 = s - \frac{s}{k}$.

\textbf{(ii) Diffusive.} Expressing the trailing average in the increments: the increment $\Delta_m$ appears in exactly the $k-m$ partial sums $\theta_\star^i$ with $m \le i \le k-1$, so
\begin{equation}
    c_{k-1} = \frac{1}{k-1}\sum_{i=1}^{k-1} \sum_{m \le i} \Delta_m = \sum_{m=1}^{k-1} \frac{k-m}{k-1}\, \Delta_m ,
    \qquad\text{hence}\qquad
    \theta_\star^k - c_{k-1} = \sum_{m=1}^{k-1} \frac{m-1}{k-1}\, \Delta_m + \Delta_k .
\end{equation}
The increments are independent and centered, so their squared norms add:
\begin{equation}
    \mathbb{E}\|\theta_\star^k - c_{k-1}\|^2
    = s \left( \sum_{m=1}^{k-1} \Big(\frac{m-1}{k-1}\Big)^2 + 1 \right)
    = s \left( \frac{(k-2)(2k-3)}{6(k-1)} + 1 \right),
\end{equation}
using $\sum_{u=1}^{K} u^2 = \frac{K(K+1)(2K+1)}{6}$ with $K = k-2$. For the second display, use \eqref{eq:huygens} in expectation: from $\mathbb{E}\|\theta_\star^j\|^2 = j s$ we get $\frac{1}{k}\sum_j \mathbb{E}\|\theta_\star^j\|^2 = \frac{s(k+1)}{2}$, and from $\mathbb{E}\langle \theta_\star^i, \theta_\star^{i'} \rangle = s \min(i, i')$ together with $\sum_{i,i' \le k} \min(i,i') = \frac{k(k+1)(2k+1)}{6}$ we get $\mathbb{E}\|c_k\|^2 = \frac{s(k+1)(2k+1)}{6k}$. Subtracting,
\begin{equation}
    \frac{1}{k}\sum_{j=1}^k \mathbb{E}\|\theta_\star^j - c_k\|^2
    = \frac{s(k+1)}{2} - \frac{s(k+1)(2k+1)}{6k}
    = \frac{s(k+1)(k-1)}{6k}
    = \frac{s(k^2 - 1)}{6k} .
\end{equation}

\textbf{(iii) Linear.} $c_{k-1} = v\,\frac{1}{k-1}\sum_{j=1}^{k-1} j = v\,\frac{k}{2}$, so $\theta_\star^k - c_{k-1} = v\,\frac{k}{2}$ and the first display follows. For the second, $\frac{1}{k}\sum_{j=1}^{k} \|v\|^2\big(j - \frac{k+1}{2}\big)^2 = \|v\|^2\,\frac{k^2-1}{12}$, the variance of the uniform distribution on $\{1, \dots, k\}$.

Since $\psi(k) = 1, k, k^2$ in the three regimes, every display above is $\Theta(\psi(k))$.
\end{proof}

\begin{corollary}[The Growth Rate lower-bounds the Instability]
\label{cor:growth_lower}
In the environments of \cref{lemma:dispersion_examples}, the expected Instability inherits the growth envelope as a lower bound:
\begin{equation}
    \mathbb{E}\big[\mathcal{I}^k\big]
    \;\ge\; \frac{\lambda_{\min}}{2}\, \mathbb{E}\big\|\theta_\star^k - \theta_\star^{\text{JTL},k}\big\|^2
    \;=\; \frac{\lambda_{\min}}{2}\Big(\frac{k-1}{k}\Big)^{2} \mathbb{E}\big\|\theta_\star^k - c_{k-1}\big\|^2
    \;\in\; \Omega\big(\psi(k)\big) .
\end{equation}
\end{corollary}
\begin{proof}
Both terms of \cref{prop:exact_instability} are non-negative, and the Approximation Bias is bounded below through the curvature floor $H_k \succeq \lambda_{\min} I$:
\begin{equation}
    \mathcal{L}^k\big(\theta_\star^{\text{JTL},k}\big) - \mathcal{L}^k\big(\theta_\star^{k}\big)
    \;=\; \frac{1}{2}\,\big\|\theta_\star^{\text{JTL},k} - \theta_\star^k\big\|^2_{H_k}
    \;\ge\; \frac{\lambda_{\min}}{2}\,\big\|\theta_\star^{\text{JTL},k} - \theta_\star^k\big\|^2 .
\end{equation}
With equal weights, $\theta_\star^{\text{JTL},k} = c_k = \frac{(k-1)\,c_{k-1} + \theta_\star^k}{k}$ gives $\theta_\star^k - c_k = \frac{k-1}{k}\big(\theta_\star^k - c_{k-1}\big)$, whose expected squared norm is $\Theta(\psi(k))$ by the first displays of \cref{lemma:dispersion_examples}.
\end{proof}

\paragraph{Non-identical, aligned Hessians.} \Cref{lemma:dispersion_examples} assumes identical Hessians, so that the JTL solution is the plain average $c_k$. Under \cref{ass:alignment} and \cref{ass:bounded_curvature}, the JTL solution is instead a convex combination $\sum_{j \le k} a_j\, \theta_{\star}^j$, in each eigendirection of the common basis, with weights bounded as $a_j \in \big[\frac{1}{\kappa k}, \frac{\kappa}{k}\big]$. \emph{The rates of \cref{lemma:dispersion_examples} are unchanged; only the constants degrade by factors of $\kappa$.} We verify the first display in the diffusive case; the remaining computations are analogous. Writing $\theta_\star^k - \sum_{j < k} a_j \theta_\star^j = \sum_{m \le k} c_m \Delta_m$, the coefficient of increment $\Delta_m$ is $c_m = 1 - \sum_{j=m}^{k-1} a_j$ (with $c_k = 1$), so the weight bound gives $c_m \ge 1 - \kappa\,\frac{k-m}{k} \ge \frac{1}{2}$ for all $m > k\big(1 - \frac{1}{2\kappa}\big)$, and, by the independence of the increments,
\begin{equation}
    \mathbb{E}\big\|\theta_\star^k - \theta_\star^{\text{JTL},k-1}\big\|^2 \;=\; s \sum_{m=1}^{k} c_m^2 \;\ge\; \frac{s}{4} \left\lfloor \frac{k}{2\kappa} \right\rfloor \;\in\; \Omega\!\left(\frac{k}{\kappa}\right),
\end{equation}
preserving the diffusive rate; the bounded and linear regimes are analogous.

Throughout the remainder of this section, the $\mathcal{O}(\psi(k))$ upper bounds hold for arbitrary environments via \cref{lemma:centroid_growth}, while the matching $\Theta(\psi(k))$ statements\major{, and the difference-rate bound of \cref{lemma:cov_drift},} hold for the canonical environments of \cref{lemma:dispersion_examples}\rev{, which include every process simulated in this paper}.}

\subsubsection{Geometric Drift}

As the agent incorporates new tasks, the stationary targets of the mixture shift. We define the \emph{Geometric Drift} of the first and second-order moments as the step-wise difference in the asymptotic JTL targets:
\begin{align}
    \Delta \mu^k = \theta_\star^{\text{JTL}, k-1} - \theta_\star^{\text{JTL}, k}, \qquad
    \Delta \Sigma^k = \Sigma_\infty^{\text{JTL},k-1} - \Sigma_\infty^{\text{JTL},k}.
\end{align}
\rev{We now characterize the asymptotic behavior of this drift for each Growth Rate, and \cref{tab:drift_scaling} summarizes} the convergence properties for each regime.

\begin{table}[h]
    \centering
    \renewcommand{\arraystretch}{1.2}
    \caption{Asymptotic Scaling of Geometric Drift for Different Environment Growth Rates.}
    \label{tab:drift_scaling}
    \begin{tabular}{@{}lccccc@{}}
        \toprule
         & & \multicolumn{2}{c}{\textbf{Mean Drift} $\mathbb{E}[\|\Delta \mu^k\|^2]$} & \multicolumn{2}{c}{\textbf{Covariance Drift} $\mathbb{E}[|\mathrm{Tr}(\Delta \Sigma^k)|]$} \\
        \cmidrule(lr){3-4} \cmidrule(l){5-6}
        \textbf{Regime} & \textbf{Growth} $\psi(k)$ & Scaling & Vanishing? & Scaling & Vanishing? \\
        \midrule
        Bounded & $\Theta(1)$ & $\mathcal{O}(k^{-2})$ & \checkmark & $\mathcal{O}(k^{-1})$ & \checkmark \\
        Diffusive & $\Theta(k)$ & $\mathcal{O}(k^{-1})$ & \checkmark & $\mathcal{O}(1)$ & $\times$ \\
        Linear & $\Theta(k^2)$ & $\mathcal{O}(1)$ & $\times$ & $\mathcal{O}(k)$ & $\times$ \\
        \bottomrule
    \end{tabular}
\end{table}

\begin{lemma}[Geometric Mean Drift]
\label{lemma:mean_drift}
If all tasks have finite training durations, the expected squared magnitude of the geometric mean \rev{drift} scales as:
\begin{equation}
    \mathbb{E}[\|\Delta{\mu}^k\|^2] \in \mathcal{O}\left( \frac{\psi(k)}{k^2} \right).
\end{equation}
\end{lemma}
\emph{Implication:} The mean transient vanishes asymptotically as long as the environment grows strictly slower than a linear trajectory ($\psi(k) \ll k^2$).

\begin{proof}
    Let $N_j$ denote the number of training steps for task $j$, and $b_k = \sum_{j=1}^k N_j$ be the total training duration up to task $k$. The effective weight of task $j$ in the mixture at stage $k$ is $w_j^{k} = N_j/b_k$.
    Let $\alpha_{k} = w_{k}^{k} := N_{k}/b_{k}$ be the relative contribution of the newest task. If task durations are finite and bounded, $\alpha_{k} \approx 1/k$ asymptotically.

    Recall the mixture optimum $\theta_\star^{\text{JTL},k} = (\bar{H}_k)^{-1}c_k$, where $c_k = \sum_{j=1}^k w_j^k H_j \theta_\star^j$.
    The recursive updates for the cumulative Hessian $\bar{H}_k$ and target vector $c_k$ are:
    \begin{align}
        \bar{H}_{k} &= (1 - \alpha_{k}) \bar{H}_{k-1} + \alpha_{k} H_{k}, \\
        c_{k} &= (1 - \alpha_{k}) c_{k-1} + \alpha_{k} H_{k} \theta_\star^{k}.
    \end{align}
    Substituting $c_{k-1} = \bar{H}_{k-1} \theta_\star^{\text{JTL}, k-1}$ into the second equation:
    \begin{equation}
        \bar{H}_k \theta_\star^{\text{JTL},k} = (1 - \alpha_{k}) \bar{H}_{k-1} \theta_\star^{\text{JTL}, k-1} + \alpha_{k} H_{k} \theta_\star^{k}.
    \end{equation}
    From the Hessian update, we write $(1 - \alpha_{k}) \bar{H}_{k-1} = \bar{H}_k - \alpha_{k} H_{k}$. Substituting this back yields:
    \begin{equation}
        \bar{H}_k \theta_\star^{\text{JTL},k} = (\bar{H}_k - \alpha_{k} H_{k}) \theta_\star^{\text{JTL}, k-1} + \alpha_{k} H_{k} \theta_\star^{k}.
    \end{equation}
    Multiplying by $\bar{H}_k^{-1}$ (which exists and commutes with $H_k$ by \cref{ass:alignment}):
    \begin{equation}
        \theta_\star^{\text{JTL},k} = \theta_\star^{\text{JTL}, k-1} + \alpha_{k} \bar{H}_k^{-1} H_{k} (\theta_\star^{k} - \theta_\star^{\text{JTL}, k-1}).
    \end{equation}
    The geometric drift is therefore:
    \begin{equation}
        \Delta \mu^k = \theta_\star^{\text{JTL}, k-1} - \theta_\star^{\text{JTL},k} = - \alpha_{k} (\bar{H}_k^{-1} H_{k}) (\theta_\star^{k} - \theta_\star^{\text{JTL}, k-1}).
    \end{equation}
    Taking the squared norm and the expectation over the task sequence:
    \begin{equation}
        \mathbb{E}[\|\Delta \mu^k\|^2] = \alpha_{k}^2 \, \mathbb{E}\left[ \| (\bar{H}_k^{-1} H_{k}) (\theta_\star^{k} - \theta_\star^{\text{JTL}, k-1}) \|^2 \right].
    \end{equation}
    Under \cref{ass:alignment}, $\bar{H}_k^{-1} H_k$ is a diagonal matrix of eigenvalue ratios. Let $\lambda_{\max}$ and $\lambda_{\min}$ bound the eigenvalues of all tasks. The spectral norm $\|\bar{H}_k^{-1} H_k\|_2$ is bounded by a constant $\kappa \approx \lambda_{\max}/\lambda_{\min}$. Thus:
    \begin{equation}
        \mathbb{E}[\|\Delta \mu^k\|^2] \le \alpha_{k}^2 \kappa^2 \mathbb{E}\left[ \| \theta_\star^{k} - \theta_\star^{\text{JTL}, k-1} \|^2 \right].
    \end{equation}
    The difference $\theta_\star^{k} - \theta_\star^{\text{JTL}, k-1}$ represents the deviation of the new task from the historical centroid. \rebuttal{By \cref{lemma:centroid_growth}, its expected squared norm is $\mathcal{O}(\psi(k))$ (only the upper bound is needed here).} Because $\alpha_k \in \Theta(1/k)$, we conclude:
    \begin{equation}
        \mathbb{E}[\|\Delta{\mu}^k\|^2] \in \mathcal{O}\left( \frac{\psi(k)}{k^2} \right).
    \end{equation}
\end{proof}

\begin{lemma}[Geometric Covariance Drift]
\label{lemma:cov_drift}
\major{In the canonical environments (equivalently, under stationary task increments), the} expected magnitude of the geometric covariance \rev{drift} scales with the finite difference (or derivative) of the growth function:
\begin{equation}
    \mathbb{E}[|\mathrm{Tr}(\Delta \Sigma^k)|] \in \mathcal{O}(\psi(k) - \psi(k-1)) \approx \mathcal{O}(\psi'(k)).
\end{equation}
\end{lemma}
\emph{Implication:} The covariance transient vanishes only if the environment grows sub-linearly.

\begin{proof}
    By \cref{lemma:PCV}, the stationary covariance $\Sigma_\infty^k$ scales linearly with the effective noise covariance $C_\infty^k$:
    \[
        \bar{H}_k \Sigma^k_\infty + \Sigma^k_\infty \bar{H}_k = \frac{\eta}{B} C^k_\infty.
    \]
    For a JTL objective, the noise covariance is dominated by the task diversity term (\cref{lemma:MT_GNC}):
    \[
        \mathbb{E}[\mathcal{V}_{\text{div}}^k] = \sum_{j=1}^k w_j^k \,\,\mathbb{E}[\|\nabla\mathcal{L}^j(\theta_\star^{\text{JTL},k})\|^2] = \sum_{j=1}^k w_j^k\,\,\mathbb{E}[\|\theta_\star^j - \theta_\star^{\text{JTL},k}\|_{H_j^2}^2].
    \]
    Because the Hessian spectrum is bounded (\cref{ass:bounded_curvature}), the expected task diversity \rev{scales \revp{directly} with} the spatial variance of the tasks\rebuttal{, which is $\Theta(\psi(k))$: the upper bound follows from \cref{lemma:centroid_growth}, and the lower bound from the second displays of \cref{lemma:dispersion_examples}}: $\mathbb{E}[\mathcal{V}_{\text{div}}^k] \in \Theta(\psi(k))$.

    The covariance drift is the difference $\Delta \Sigma^k = \Sigma_\infty^{k-1} - \Sigma_\infty^{k}$. Because $\Sigma_\infty^k$ is \rev{\revp{strictly} proportional} to $C_\infty^k$, the drift grows asymptotically with the derivative of the growth function:
    \[
        \mathbb{E}[|\mathrm{Tr}(\Delta \Sigma^k)|] \propto \psi(k) - \psi(k-1).
    \]
    \rev{The sign of $\mathrm{Tr}(\Delta \Sigma^k)$ might be negative, which means the variance \emph{grows} when a new task is added. We take the absolute value to analyze its asymptotic magnitude.}

    For bounded growth ($\psi(k) = 1$), the variance converges, so the difference is $\mathcal{O}(1/k)$. For diffusive sequences ($\psi(k) = k$), the variance grows linearly, making the difference constant ($\mathcal{O}(1)$). For linear growth ($\psi(k) = k^2$), the variance grows quadratically, so the difference grows linearly ($\mathcal{O}(k)$).
\end{proof}

\begin{lemma}[Expected Asymptotic Instability]
    \label{lemma:asymptotic_instability}
    The expected Instability cost $\mathbb{E}[\mathcal{I}^k]$, which comprises the approximation bias and the stationary gradient noise penalty, scales linearly with the environment Growth Rate:
    \begin{equation}
        \mathbb{E}[\mathcal{I}^k] \in \Theta(\psi(k)).
    \end{equation}
\end{lemma}
\emph{Implication:} In expectation over the task sequence, the JTL agent suffers from unbounded performance degradation if the environment grows indefinitely (i.e., Diffusive or Linear growth), even if it converges perfectly to its stationary distribution.

\begin{proof}
    Taking the expectation over the task sequence of the exact Instability cost derived in \cref{prop:exact_instability} yields:
    \[
        \mathbb{E}[\mathcal{I}^k] \approx \underbrace{\mathbb{E}\left[\mathcal{L}^k(\theta_\star^{\text{JTL},k}) - \mathcal{L}^k(\theta_\star^{k})\right]}_{\text{Bias}} + \underbrace{\frac{\eta}{4B}\sum_{j=1}^k w_j^k \, \mathbb{E}\left[\|\nabla \mathcal{L}^j(\theta_\star^{\text{JTL},k})\|_{\mathcal{R}^k}^2\right]}_{\text{Variance}}.
    \]

    \textbf{Bias Term.}
    For a quadratic loss, the bias term simplifies to the expected weighted squared distance between the target task optimum and the JTL centroid:
    \[
        \mathbb{E}\left[\mathcal{L}^k(\theta_\star^{\text{JTL},k}) - \mathcal{L}^k(\theta_\star^{k})\right] = \frac{1}{2} \mathbb{E}\left[\|\theta_\star^{\text{JTL},k} - \theta_\star^k\|_{H_k}^2\right].
    \]
    Because $\theta_\star^{\text{JTL},k}$ is a convex combination of tasks $1 \dots k$, its expected squared distance to the new task $\theta_\star^k$ scales with the variance of the task distribution\rebuttal{: the upper bound $\mathcal{O}(\psi(k))$ is \cref{lemma:centroid_growth}, and the matching lower bound holds in the canonical environments by \cref{lemma:dispersion_examples} (cf.\ \cref{cor:growth_lower})}.

    \textbf{Variance Term.}
    The expected gradient noise contribution depends on task diversity. Evaluating the gradients at the centroid yields $\nabla \mathcal{L}^j(\theta_\star^{\text{JTL},k}) = H_j (\theta_\star^{\text{JTL},k} - \theta_\star^j)$.
    Substituting this into the variance term:
    \[
        \text{Variance} \propto \sum_{j=1}^k w_j^k\,\, \mathbb{E}\left[\|H_j (\theta_\star^{\text{JTL},k} - \theta_\star^j)\|_{\mathcal{R}^k}^2\right].
    \]
    Assuming bounded Hessians (\cref{ass:bounded_curvature}), this term is proportional to the expected variance of the task optima around their mean. As established in \cref{lemma:cov_drift}, this spatial variance \rev{scales \revp{exactly} with the environment Growth Rate as} $\Theta(\psi(k))$.

    Since both the expected bias and expected variance scale with $\psi(k)$, the total expected \rev{Instability} scales as $\Theta(\psi(k))$.
\end{proof}

\subsubsection{Residual Errors}

The transient cost of learning task $k$ is governed by the displacement of the agent's moments relative to the \emph{new} stationary target.
Let $\mu_{N_{k-1}}^{k-1}$ and $\Sigma_{N_{k-1}}^{k-1}$ denote the agent's actual mean and covariance at the end of task $k-1$ (where the expectations inside $\mu$ and $\Sigma$ are \rev{taken only} over the parameter trajectory). The initialization errors for task $k$ are defined as:
\begin{align}
    \tilde{\mu}_0^{\text{JTL}, k} &:= \mu_{N_{k-1}}^{k-1} - \theta_\star^{\text{JTL}, k}, \qquad \tilde{\Sigma}_0^{\text{JTL}, k} := \Sigma_{N_{k-1}}^{k-1} - \Sigma_\infty^{\text{JTL}, k}.
\end{align}
\rev{We decompose these errors into the \major{Geometric Drift} defined earlier and a} \textbf{Residual Error} stemming from the incomplete convergence of the previous task:
\begin{align}
    \tilde{\mu}_0^{\text{JTL}, k} &= \Delta \mu^k + \underbrace{(\mu_{N_{k-1}}^{k-1} - \theta_\star^{\text{JTL}, k-1})}_{\text{Residual } r^{k-1}_\mu}, \qquad
    \tilde{\Sigma}_0^{\text{JTL}, k} = \Delta \Sigma^k + \underbrace{(\Sigma_{N_{k-1}}^{k-1} - \Sigma_\infty^{\text{JTL}, k-1})}_{\text{Residual } r^{k-1}_\Sigma}.
\end{align}
The residual is effectively a \emph{contracted} version of the previous task's initialization error. Since SGD on a quadratic objective is a contraction mapping for appropriate step sizes, let $\rho_\mu^k$ and $\rho_\Sigma^k$ denote the contraction rates for the mean and covariance errors during task $k$. Then:
\[
    \|r^{k-1}_\mu\| \le \rho_\mu^{k-1} \cdot \|\tilde{\mu}_0^{\text{JTL}, k-1}\|, \qquad
    \|r^{k-1}_\Sigma\| \le \rho_\Sigma^{k-1} \cdot \|\tilde{\Sigma}_0^{\text{JTL}, k-1}\|.
\]
The following lemma establishes that these contraction factors are strictly less than 1.

\begin{lemma}[Contraction of Residuals]
    \label{lemma:contraction}
    For a quadratic task with Hessian $H$ and step-size $\eta < 2/\lambda_{\max}(H)$, the parameter distribution errors contract at each step. After $N_k$ steps:
    \begin{equation}
        \rho_\mu^k = (1 - \eta \lambda_{\min}(H))^{N_k} < 1, \qquad \rho_\Sigma^k = (1 - \eta \lambda_{\min}(H))^{2 N_k} < 1.
    \end{equation}
\end{lemma}

\begin{proof}
    The dynamics of the mean error vector $\tilde{\mu}_n = \mu_n - \theta_\star$ under SGD with a quadratic objective follow the recursion $\tilde{\mu}_{n+1} = (I - \eta H)\tilde{\mu}_n$. Unrolling this for $N_k$ steps yields:
    \begin{equation}
        \tilde{\mu}_{N_k} = (I - \eta H)^{N_k} \tilde{\mu}_0.
    \end{equation}
    Taking the Euclidean norm, we have $\|\tilde{\mu}_{N_k}\| \le \|(I - \eta H)^{N_k}\|_2 \|\tilde{\mu}_0\|$. The spectral norm of the propagator is determined by the eigenvalues of $H$. For $\eta < 2/\lambda_{\max}(H)$, the eigenvalues of $(I - \eta H)$ lie strictly within $(-1, 1)$. Assuming a non-oscillatory step size ($\eta < 1/\lambda_{\max}(H)$), the contraction is bottlenecked by the slowest decaying mode, corresponding to $\lambda_{\min}(H)$. Thus, the residual mean error is bounded by:
    \begin{equation}
        \|r_\mu^k\| \le (1 - \eta \lambda_{\min}(H))^{N_k} \|\tilde{\mu}_0^{\text{JTL}, k}\|.
    \end{equation}

    Similarly, the covariance deviation $\tilde{\Sigma}_n = \Sigma_n - \Sigma_\infty$ evolves as $\tilde{\Sigma}_{n+1} = (I - \eta H) \tilde{\Sigma}_n (I - \eta H)$. After $N_k$ steps:
    \begin{equation}
        \tilde{\Sigma}_{N_k} = (I - \eta H)^{N_k} \tilde{\Sigma}_0 (I - \eta H)^{N_k}.
    \end{equation}
    Taking the spectral norm (or the trace) leads to a contraction factor governed by the square of the propagator's eigenvalues:
    \begin{equation}
        \|r_\Sigma^k\| \le \|(I - \eta H)^{N_k}\|_2^2 \|\tilde{\Sigma}_0\| = (1 - \eta \lambda_{\min}(H))^{2 N_k} \|\tilde{\Sigma}_0\|.
    \end{equation}
    This confirms that both first and second-moment errors contract exponentially with the task duration $N_k$.
\end{proof}

\vspace{0.5cm}
\rev{The decomposition into geometric drift and residuals defines the initialization error $\tilde{\mu}_0^k$ recursively. The following lemma bounds this recursion.}

\begin{lemma}[Asymptotic Dominance of Stable Recursions]
\label{lemma:recursion_dominance}
Let $\{x_k\}_{k \ge 0}$ be a sequence of non-negative real numbers satisfying the linear recurrence inequality:
\begin{equation}
    x_k \le \rho x_{k-1} + u_k, \quad \text{with } 0 \le \rho < 1.
\end{equation}
If the input sequence $u_k$ scales asymptotically as a polynomial $u_k \in \Theta(k^p)$ for any $p \in \mathbb{R}$, then the sequence $x_k$ is asymptotically dominated by the input:
\begin{equation}
    x_k \in \major{\mathcal{O}(k^p)}.
\end{equation}
More precisely, $\limsup_{k \to \infty} \frac{x_k}{u_k} \le \frac{1}{1-\rho}$. \major{The same conclusion holds if the input is only bounded as $u_k \in \mathcal{O}(k^p)$, by applying the lemma to the polynomial upper envelope of $u_k$.}
\end{lemma}

\begin{proof}
    Unrolling the recurrence yields the summation (variation of constants) formula:
    \begin{equation}
        x_k \le \rho^k x_0 + \sum_{j=0}^{k-1} \rho^j u_{k-j}.
    \end{equation}
    The term $\rho^k x_0$ decays exponentially and is negligible relative to the polynomial $u_k$. We focus on the convolution sum. Because $u_k \in \Theta(k^p)$, there exists a constant $C$ such that $u_k \approx C k^p$ for large $k$. We examine the ratio of the sum to the input $u_k$:
    \begin{equation}
        \frac{\sum_{j=0}^{k-1} \rho^j u_{k-j}}{u_k} = \sum_{j=0}^{k-1} \rho^j \left( \frac{u_{k-j}}{u_k} \right).
    \end{equation}
    For any fixed $j$, as $k \to \infty$, the ratio of the polynomial terms $(k-j)^p / k^p$ approaches $1$. Because $\rho < 1$, the weights $\rho^j$ decay exponentially, allowing us to apply the Dominated Convergence Theorem for series and bring the limit inside the sum:
    \begin{equation}
        \lim_{k \to \infty} \sum_{j=0}^{k-1} \rho^j \left( \frac{u_{k-j}}{u_k} \right) = \sum_{j=0}^{\infty} \rho^j \cdot 1 = \frac{1}{1-\rho}.
    \end{equation}
    \major{Thus $x_k$ is asymptotically bounded by $\frac{1}{1-\rho} u_k$, which preserves the polynomial order: $x_k \in \mathcal{O}(k^p)$. Because the recurrence only upper-bounds $x_k$, no matching lower bound follows.}
\end{proof}

\vspace{0.5cm}

\subsubsection{Asymptotic Transfer}
\label{sec:exact_transfer}

\begin{theorem}[Asymptotic Condition for Positive Transfer Efficiency]
    \label{theo:positive_transfer}
    The Expected Transfer Efficiency over the task sequence is positive ($\mathbb{E}[\mathrm{TE}^k] > 0$) if the transient cost avoided by the JTL agent's proximity to the target centroid dominates the expected instability penalties introduced by task conflict and bias. Specifically, for large $k$ and \emph{polynomial growth rate} $\psi(k)$, a sufficient condition is:
    \[
        \tau N \cdot \mathbb{E}[\mathcal{I}^k] < C^k_{\text{ITL}} \cdot \mathbb{E}[\|\tilde \mu_0^{\text{ITL}}\|^2]
    \]
\end{theorem}

\begin{proof}
    \rev{We decompose the expected Transfer Efficiency and analyze the asymptotic behavior of the transient costs.}

    Recall the definition of Transfer Efficiency:
    \[
        \text{TE}^k = (\delta^k_{\text{ITL}} - \delta^k_{\text{JTL}}) - \mathcal{I}^k.
    \]
    The condition for expected positive transfer over tasks is $\mathbb{E}[\delta^k_{\text{ITL}}] - \mathbb{E}[\delta^k_{\text{JTL}}] > \mathbb{E}[\mathcal{I}^k]$.

    We use the Transient Potential Law (\cref{theo:transient_potential}). Define the \emph{Spectral Dissipation Matrix} as $\xi_N(H) := I - (I - \eta H)^{2N}$. This matrix captures the fraction of the initial error energy dissipated after $N$ steps.
    Using Task Alignment (\cref{ass:alignment}), the transient costs simplify to:
    \begin{align}
        &\delta^k_{\text{ITL}} \approx \frac{1}{\tau N} \text{Tr}\left( \xi_N(H_k)(\tilde \mu^{\text{ITL}}_0{\tilde \mu^{\text{ITL}}_0}^\top + \tilde{\Sigma}_0^{\text{ITL}}) \right), \\
        &\delta^k_{\text{JTL}} \approx \frac{1}{\tau N} \text{Tr}\left( \xi_N(\bar{H}_k)\mathcal{R}^k\,(\tilde \mu^{\text{JTL},k}_0{\tilde \mu^{\text{JTL},k}_0}^\top + \tilde{\Sigma}_0^{\text{JTL},k}) \right).
    \end{align}
    Let $\gamma^k_{\min} := \lambda_{\min}(\xi_N(H_k))$ and $\gamma^k_{\max} := \lambda_{\max}(\xi_N(\bar{H}_k))$. We lower bound the ITL cost and upper bound the JTL cost:
    \begin{align}
        &\delta^k_{\text{ITL}} \ge \frac{\gamma^k_{\min}}{\tau N} \left( \|\tilde \mu_0^{\text{ITL}}\|^2 + \text{Tr}(\tilde{\Sigma}_0^{\text{ITL}}) \right), \\
        &\delta^k_{\text{JTL}} \le \lambda_{\max}(\mathcal{R}^k) \cdot \frac{\gamma^k_{\max}}{\tau N}  \left( \|\tilde \mu_0^{\text{JTL},k}\|^2 + \text{Tr}(\tilde{\Sigma}_0^{\text{JTL},k}) \right).
    \end{align}
    Denote the bounding constants by $C^k_{\text{ITL}} = \gamma^k_{\min}$ and $C^k_{\text{JTL}}=\lambda_{\max}(\mathcal{R}^k)\,\gamma^k_{\max}$. Taking the expectation over the task sequence on both sides, the condition for positive transfer translates to:
    \[
        \tau N \cdot \mathbb{E}[\mathcal{I}^k] < C^k_{\text{ITL}} \left( \mathbb{E}[\|\tilde \mu_0^{\text{ITL}}\|^2] + \mathbb{E}[\text{Tr}(\tilde{\Sigma}_0^{\text{ITL}})] \right) - C^k_{\text{JTL}} \left( \mathbb{E}[\|\tilde \mu_0^{\text{JTL},k}\|^2] + \mathbb{E}[|\text{Tr}(\tilde{\Sigma}_0^{\text{JTL},k})|] \right).
    \]
    We now analyze the asymptotic scaling of each expectation with respect to $k$:
    \begin{itemize}
        \item $\mathbb{E}[\mathcal{I}^k]$: By \cref{lemma:asymptotic_instability}, expected \rev{Instability} grows as $\Theta(\psi(k))$.
        \item $\mathbb{E}[\|\tilde \mu_0^{\text{ITL}}\|^2]$: Because the ITL agent resets to a fixed origin prior, its error is the raw distance to the new task optimum. By \cref{ass:growth-rate-app}, this scales as $\Theta(\psi(k))$.
        \item $\mathbb{E}[\text{Tr}(\tilde{\Sigma}_0^{\text{ITL}})]$: The ITL covariance resets to a fixed prior variance, remaining constant $\Theta(1)$ relative to the task growth.
        \item $\mathbb{E}[\|\tilde \mu_0^{\text{JTL},k}\|^2]$: We analyze the root-mean-square error $e_k := \sqrt{\mathbb{E}[\|\tilde \mu_0^{\text{JTL},k}\|^2]}$. By Minkowski's inequality applied to the decomposition $\tilde{\mu}_0 = \Delta \mu + r_\mu$, we obtain the recursive bound:
        \[
            e_k \le \sqrt{\mathbb{E}[\|\Delta \mu^k\|^2]} + \sqrt{\mathbb{E}[\|r_\mu^{k-1}\|^2]} \le d_k + \rho_\mu e_{k-1},
        \]
        where $d_k := \sqrt{\mathbb{E}[\|\Delta \mu^k\|^2]}$ is the geometric drift.
        From \cref{lemma:mean_drift}, the \rev{expected squared} drift scales as $\major{\mathcal{O}}(\psi(k)/k^2)$, meaning $d_k \in \major{\mathcal{O}}(\sqrt{\psi(k)}/k)$.
        Applying \cref{lemma:recursion_dominance}, the error sequence $e_k$ is dominated by $d_k$:
        \[
             \mathbb{E}[\|\tilde \mu_0^{\text{JTL},k}\|^2] = e_k^2 \in \major{\mathcal{O}}(d_k^2) = \major{\mathcal{O}}\left(\frac{\psi(k)}{k^2}\right).
        \]
        \item $\mathbb{E}[|\text{Tr}(\tilde{\Sigma}_0^{\text{JTL},k})|]$: Similarly, by \cref{lemma:cov_drift} and \cref{lemma:recursion_dominance}, the covariance initialization error is dominated by the covariance drift\major{: it is $\mathcal{O}(\psi'(k))$}.
    \end{itemize}

    Substituting these scalings back into the inequality for sufficiently large $k$:
    \[
    \tau N \cdot \Theta(\psi(k)) < C^k_{\text{ITL}} \left( \Theta(\psi(k)) + \Theta(1) \right) - C^k_{\text{JTL}} \left( \major{\mathcal{O}}\left(\frac{\psi(k)}{k^2}\right) + \major{\mathcal{O}}(\psi'(k)) \right).
    \]
    The JTL transient term (the negative contribution) becomes asymptotically negligible compared to the ITL transient term (the positive contribution) because $\frac{\psi(k)}{k^2} \ll \psi(k)$ and $\psi'(k) \ll \psi(k)$ for any polynomial growth.

    \rev{The asymptotic condition therefore reduces to a competition between the ITL transient energy and the JTL Instability:}
    \[
         \tau N \cdot \mathbb{E}[\mathcal{I}^k] < C^k_{\text{ITL}} \cdot \mathbb{E}[\|\tilde \mu_0^{\text{ITL}}\|^2].
    \]
    \rev{Solving for $N$ shows that expected transfer is positive only if the task duration is sufficiently short relative to the Instability:}
    \[
        N < \frac{C^k_{\text{ITL}}}{\tau} \frac{\mathbb{E}[\|\tilde \mu_0^{\text{ITL}}\|^2]}{\mathbb{E}[\mathcal{I}^k]}.
    \]
    \rev{Expected transfer is therefore positive only in the ``transient regime,'' where the benefit of initializing near the task centroid (avoiding re-learning) outweighs the cost of asymptotic bias and task conflict.}
\end{proof}

\subsection{(Non-asymptotic) Continuous Drift Analysis in the Quadratic Model}
\label{sec:drift_analysis}

\rev{We now instantiate the general Transfer Efficiency framework for drifting linear tasks, which yields explicit, non-asymptotic conditions for} the stability-plasticity phase transition observed in our experiments.

\subsubsection{Setup and Definitions}

\rev{We assume a sequence of tasks $k=1, \dots, K$ whose optima follow a random walk.} The data inputs $x \in \mathbb{R}^d$ are drawn from a fixed distribution with covariance matrix $\Sigma_x \succ 0$. The target for task $k$ \rev{is} $y = (\theta_\star^k)^\top x + \xi$, with label noise $\xi \sim \mathcal{N}(0, \sigma_{\xi}^2)$. The expected loss function for task $k$ (over the data distribution) is:
\begin{equation}
    \mathcal{L}^k(\theta) = \frac{1}{2} \| \theta - \theta_\star^k \|_{\Sigma_x}^2 + \frac{1}{2}\sigma_{\xi}^2,
\end{equation}
using the Mahalanobis-style norm $\|v\|_M^2 := v^\top M v$ introduced in the main text (\cref{sec:quadratic-losses}). \rev{This expression instantiates the quadratic loss of \cref{sec:quadratic-losses} with $H_j = \Sigma_x$ (a data Hessian shared across tasks) and irreducible loss} $c_j = \tfrac{1}{2}\sigma_\xi^2$. The optimal parameters evolve according to a Wiener process across tasks:
\begin{equation}
    \theta_\star^k = \theta_\star^{k-1} + \Delta_k, \quad \Delta_k \sim \mathcal{N}(0, \Sigma_\Delta),
\end{equation}
where $\Sigma_\Delta$ is the drift covariance matrix. Both the Independent-Task Learning (ITL) and Joint-Task Learning (JTL) agents optimize these objectives using SGD with a constant step-size $\eta$. For simplicity, we assume all tasks have equal duration $N$.

\rev{As the derivations below show, this type of non-stationarity has a \emph{Diffusive} Growth Rate} ($\psi(k) = k$), consistent with the regimes defined in \cref{tab:drift_scaling}. Unless otherwise noted, expectations $\mathbb{E}[\cdot]$ in this section are taken over the random walk process defining the task sequence.

\subsubsection{Derivation of Expected Instability}

The JTL agent optimizes the uniform mixture of all observed tasks, converging to the cumulative centroid $\theta_\star^{\text{JTL},k} = \frac{1}{k}\sum_{j=1}^k \theta_\star^j$. We apply the Exact Instability Decomposition from \cref{prop:exact_instability} to task $k$. Since the Hessian $H_k = \Sigma_x$ is constant across tasks, the instability for a given realization decomposes into:
\begin{equation}
    \mathcal{I}^k \approx \frac{1}{2} \| \theta_\star^{\text{JTL},k} - \theta_\star^k \|_{\Sigma_x}^2 + \frac{\eta}{4kB} \sum_{j=1}^k \| \theta_\star^{\text{JTL},k} - \theta_\star^j \|_{\Sigma_x^2}^2.
\end{equation}

We first express the lag vector $v_k := \theta_\star^{\text{JTL},k} - \theta_\star^k$ in terms of the sequence of drift increments $\Delta_m$. Using the backward recursion $\theta_\star^j = \theta_\star^k - \sum_{m=j+1}^k \Delta_m$, we substitute into the centroid definition:
\begin{equation}
    v_k = \left( \frac{1}{k} \sum_{j=1}^k \theta_\star^j \right) - \theta_\star^k
    = \frac{1}{k} \sum_{j=1}^k \left( \theta_\star^k - \sum_{m=j+1}^k \Delta_m \right) - \theta_\star^k
    = - \frac{1}{k} \sum_{j=1}^k \sum_{m=j+1}^k \Delta_m.
\end{equation}
\rev{Grouping the terms of the double sum by the drift increment $\Delta_m$ (a given $\Delta_m$ appears for every $j < m$, that is, $m-1$ times):}
\begin{equation}
    v_k = - \frac{1}{k} \sum_{m=2}^k (m-1) \Delta_m.
\end{equation}

\rev{The second term of the Instability involves the sum of squared deviations from the centroid (the task scatter), which can be rewritten with pairwise squared differences:}
\begin{equation}
    \sum_{j=1}^k (\theta_\star^{\text{JTL},k} - \theta_\star^j)(\theta_\star^{\text{JTL},k} - \theta_\star^j)^\top = \frac{1}{2k} \sum_{i=1}^k \sum_{j=1}^k (\theta_\star^i - \theta_\star^j)(\theta_\star^i - \theta_\star^j)^\top.
\end{equation}
\rev{The difference between any two optima is the sum} of the intermediate drift steps: $\theta_\star^i - \theta_\star^j = \sum_{m=j+1}^i \Delta_m$ (for $i>j$).

We now compute the \textbf{expected instability} $\mathbb{E}[\mathcal{I}^k]$ over the distribution of random walks. The covariance of the lag vector $v_k$ is the weighted sum of the increment covariances:
\begin{equation}
    \mathbb{E}[v_k v_k^\top] = \frac{1}{k^2} \sum_{m=2}^k (m-1)^2 \Sigma_\Delta = \frac{1}{k^2} \left( \frac{(k-1)k(2k-1)}{6} \right) \Sigma_\Delta.
\end{equation}
Using the continuous approximation $\sum_{p=1}^{k} p^2 \approx k^3/3$ for large $k$:
\begin{equation}
    \mathbb{E}\left[ \frac{1}{2} \| v_k \|_{\Sigma_x}^2 \right] \approx \frac{1}{2} \text{Tr}\left( \Sigma_x \cdot \frac{k}{3} \Sigma_\Delta \right) = \frac{k}{6} \text{Tr}(\Sigma_x \Sigma_\Delta).
\end{equation}

For the task scatter, since $\mathbb{E}[(\theta_\star^i - \theta_\star^j)(\theta_\star^i - \theta_\star^j)^\top] = |i-j|\Sigma_\Delta$, the expected sum of pairwise distances is:
\begin{equation}
    \mathbb{E}[\text{Scatter}] = \frac{1}{2k} \sum_{i,j} |i-j| \Sigma_\Delta \approx \frac{1}{2k} \left( \frac{k^3}{3} \right) \Sigma_\Delta = \frac{k^2}{6} \Sigma_\Delta.
\end{equation}
Substituting this into the variance trace term:
\begin{equation}
    \mathbb{E}\left[ \frac{\eta}{4kB} \sum_{j=1}^k \| \theta_\star^{\text{JTL},k} - \theta_\star^j \|_{\Sigma_x^2}^2 \right] = \frac{\eta}{4kB} \text{Tr}\left( \Sigma_x^2 \cdot \frac{k^2}{6} \Sigma_\Delta \right) = \frac{\eta k}{24B} \text{Tr}(\Sigma_x^2 \Sigma_\Delta).
\end{equation}

Combining the expected bias and variance components, the \textbf{final expected instability} is:
\begin{equation}
    \label{eq:drift_instability}
    \mathbb{E}[\mathcal{I}^k] \approx \frac{k}{6} \text{Tr}(\Sigma_x \Sigma_\Delta) + \frac{\eta k}{24B} \text{Tr}(\Sigma_x^2 \Sigma_\Delta).
\end{equation}

\subsubsection{Derivation of Expected Transient Error}
\label{sss:TB-CD}

The Transient Error $\delta^k$ is proportional to the energy of the initial parameter error relative to the current task's stationary distribution. For a specific realization, this is approximated by:
\begin{equation}
    \delta^k \approx \frac{1}{N\tau} \text{Tr}\left( \xi_N(\Sigma_x) \left( \tilde{\mu}_0 \tilde{\mu}_0^\top + \tilde{\Sigma}_0 \right) \right).
\end{equation}

\paragraph{ITL Agent.}
The ITL agent resets to a fixed prior $\theta_{0} \sim Q_0$ with covariance $\Sigma_0^{\text{ITL}}$. The initial error vector is $\tilde{\mu}_0^{\text{ITL}} = \theta_{0}-\theta_\star^k$. The expected energy of this error depends entirely on how far the random walk has drifted from the origin:
\begin{equation}
    \mathbb{E}[\| \theta_{0} -\theta_\star^k \|^2] = \mathbb{E}[\| \theta_{0}\|^2] + \text{Tr}\left( \sum_{j=1}^k \mathbb{E}[\Delta_j \Delta_j^\top] \right) = \text{Tr}\left(\Sigma_0^{\text{ITL}}\right) + k \text{Tr}(\Sigma_\Delta).
\end{equation}
Let $\gamma_{\min} := \lambda_{\min}(\xi_N(\Sigma_x))$. Because $\tilde{\Sigma}_0^{\text{ITL}} = \Sigma_0^{\text{ITL}} - \Sigma_\infty^{\text{ITL}}$, the expected ITL Transient Error is bounded by:
\begin{align}
    \mathbb{E}[\delta^k_{\text{ITL}}] &\ge \frac{\gamma_{\min}}{\tau N} \left( \mathbb{E}[\|\tilde{\mu}_0^{\text{ITL}}\|^2] + \text{Tr}(\tilde{\Sigma}_0^{\text{ITL}} )\right) \nonumber \\
    &= \frac{\gamma_{\min}}{\tau N} \left( k \text{Tr}(\Sigma_\Delta) + 2\text{Tr}\left(\Sigma_0^{\text{ITL}}\right) - \text{Tr}({\Sigma}^{\text{ITL}}_\infty )\right).
\end{align}
By \cref{lemma:GNC}, $\text{Tr}({\Sigma}^{\text{ITL}}_\infty ) = \frac{\eta \sigma_\xi^2}{2B}$. The trace of the prior $\text{Tr}\left(\Sigma_0^{\text{ITL}}\right)$ is a constant, which we denote by $\sigma_0^2$.

\paragraph{JTL Agent.}
The JTL agent initializes at the previous task's converged centroid $\theta_\star^{\text{JTL},k-1}$. The expected initialization error decomposes into the geometric drift and the residual contraction:
\[
    \tilde{\mu}_0^{\text{JTL},k} = \Delta\mu_k + r_{\mu}^{k-1}.
\]
First, we relate the geometric drift $\Delta \mu_k$ to the lag vector $v_k$. Recall $\theta_\star^{\text{JTL},k} = \frac{1}{k}\theta_\star^k + \frac{k-1}{k}\theta_\star^{\text{JTL},k-1}$. Rearranging for the increment:
\[
    \Delta \mu_k = \theta_\star^{\text{JTL},k} - \theta_\star^{\text{JTL},k-1} = \theta_\star^{\text{JTL},k} - \frac{k}{k-1}\left(\theta_\star^{\text{JTL},k} - \frac{1}{k}\theta_\star^k\right) = \frac{1}{k-1}(\theta_\star^k - \theta_\star^{\text{JTL},k}) = -\frac{1}{k-1} v_k.
\]
The expected squared norm of the geometric drift is therefore:
\begin{equation}
    \mathbb{E}[\|\Delta \mu_k\|^2] = \frac{1}{(k-1)^2} \text{Tr}(\mathbb{E}[v_k v_k^\top]) = \frac{2k-1}{6k(k-1)} \text{Tr}(\Sigma_\Delta).
\end{equation}
Let $d_k := \sqrt{\mathbb{E}[\|\Delta \mu_k\|^2]}$ be the RMS drift intensity. For $k \ge 2$, define the RMS initialization error $e_k := \sqrt{\mathbb{E}[\|\tilde{\mu}_0^{\text{JTL},k}\|^2]}$. By Minkowski's inequality, it satisfies the stable recursion $e_k \le d_k + \rho_\mu e_{k-1}$.
Because the drift $d_k$ changes slowly compared to the exponential decay rate $\rho_\mu$ (the quasi-static approximation), the error tracks the input scaled by the effective time constant $\frac{1}{1-\rho_\mu}$. For $k \gg 1$, we approximate $d_k \approx \sqrt{\frac{1}{3k}\text{Tr}(\Sigma_\Delta)}$. This yields the explicit expected transient energy curve:
\begin{equation}
    \mathbb{E}[\|\tilde{\mu}_0^{\text{JTL},k}\|^2] \approx \left( \frac{1}{1-\rho_\mu} \right)^2 \frac{\text{Tr}(\Sigma_\Delta)}{3k}.
\end{equation}
\rev{The JTL expected initialization error therefore vanishes as $1/k$, although the conditioning of the optimization problem inflates its magnitude through the $(1-\rho_\mu)^{-2}$ factor.}

\rev{Meanwhile}, the JTL agent's covariance tracks the equilibrium covariance of the centroid. Because the gradient noise increases as tasks spread out, this target $\Sigma_\infty^{\text{JTL},k}$ grows over time. The noise covariance $C^{\text{JTL},k}$ at the centroid is dominated by the task scatter (\cref{lemma:MT_GNC}):
\[
    C^{\text{JTL},k} = \frac{1}{B}\sigma_\xi^2 H + \frac{1}{kB} \sum_{j=1}^k H \mathbb{E}[(\theta_\star^{\text{JTL},k} - \theta_\star^j)(\theta_\star^{\text{JTL},k} - \theta_\star^j)^\top] H.
\]
Using the scatter result from the instability derivation, \rev{the sum is approximately} $\frac{k^2}{6}\Sigma_\Delta$. Thus:
\[
    C^{\text{JTL},k} \approx \frac{1}{B}\sigma_\xi^2 H + \frac{k}{6B} H \Sigma_\Delta H.
\]
The equilibrium covariance satisfies the Lyapunov equation $H\Sigma_\infty^{\text{JTL},k} + \Sigma_\infty^{\text{JTL},k} H = \eta C^{\text{JTL},k}$. Taking the trace (and \rev{using} the cyclic property $\text{Tr}(H \Sigma_\Delta H) = \text{Tr}(H^2 \Sigma_\Delta)$):
\[
    2 \text{Tr}(H \Sigma_\infty^{\text{JTL},k}) = \frac{\eta}{B}\sigma_\xi^2\text{Tr}(H) + \frac{\eta k}{6B} \text{Tr}(H^2 \Sigma_\Delta).
\]
\rev{Only the second term depends on $k$.} The \emph{covariance drift} is the difference between consecutive targets:
\[
    \text{Tr}\left( H(\Sigma_\infty^{\text{JTL},k-1} - \Sigma_\infty^{\text{JTL},k}) \right) = -\frac{\eta}{12B} \text{Tr}(H^2 \Sigma_\Delta).
\]
This drift is constant with respect to $k$. The initialization error $\tilde{\Sigma}_0^{\text{JTL},k}$ tracks this drift. By the contraction property (\cref{lemma:contraction}) with rate $\rho_\Sigma < 1$:
\[
    \text{Tr}(H \tilde{\Sigma}_0^{\text{JTL},k}) \le -\text{Tr}(H (\Sigma_\infty^{\text{JTL},k-1} - \Sigma_\infty^{\text{JTL},k})) + \rho_{\Sigma} \text{Tr}(H \tilde{\Sigma}_0^{\text{JTL},{k-1}}).
\]
Because the target variance is strictly increasing, \rev{the agent is always ``catching up'': $\text{Tr}(H \tilde{\Sigma}_0^{\text{JTL},{k-1}})$ stays negative and pulls the trace further down. We can therefore bound:}
\[
     \text{Tr}(H \tilde{\Sigma}_0^{\text{JTL},k}) \le -\frac{\eta}{12B} \text{Tr}(H^2 \Sigma_\Delta).
\]
Given the boundedness of $H$ (\cref{ass:bounded_curvature}), we approximate $\text{Tr}(\tilde{\Sigma}_0^{\text{JTL},k}) \le  - \frac{\eta}{12B} \text{Tr}(H \Sigma_\Delta)$.

\subsubsection{Non-asymptotic Condition for Positive Transfer}

\rev{We now ask when JTL outperforms ITL.} We define the \emph{Positive Transfer Regime} as the set of environment parameters for which the total \emph{expected} error of the JTL agent is strictly lower than that of the ITL agent:
\begin{equation}
    \mathbb{E}[\mathcal{I}^k_{\text{JTL}}] + \mathbb{E}[\delta^k_{\text{JTL}}] < \mathbb{E}[\delta^k_{\text{ITL}}].
\end{equation}
Denote by $\sigma^2_\Delta = \text{Tr}(\Sigma_\Delta)$ the \emph{drift magnitude} parameter. \rev{Let $c_\mu = \frac{1}{1-\rho_\mu}$ denote the effective time constant. Substituting the expected bounds derived above gives the master inequality for positive transfer under continuous drift:}
\begin{equation}
    \label{eq:master_inequality_app}
    \underbrace{\frac{\gamma_{\min}\,(k\sigma^2_\Delta + 2\sigma_0^2 - \eta\sigma_\xi^2/2B)}{\tau N}}_{\text{ITL Transient Cost}} >
    \underbrace{\frac{k}{6} \text{Tr}(\Sigma_x \Sigma_\Delta) + \frac{\eta k}{24B} \text{Tr}(\Sigma_x^2 \Sigma_\Delta)}_{\text{JTL Instability Cost}} +
    \underbrace{\frac{\gamma_{\max}}{\tau N} \left(c_\mu^2 \frac{\sigma^2_\Delta}{3k} - \frac{\eta}{12B} \text{Tr}(\Sigma_x \Sigma_\Delta)\right)}_{\text{JTL Transient Cost}}.
\end{equation}

\rev{This inequality recovers the more general asymptotic result of \cref{theo:positive_transfer}: as $k \to \infty$, the ITL Transient cost and the JTL Instability cost both grow linearly in $k$, while the JTL Transient cost is of lower order in $k$. In the limit of infinitely many tasks, these two linear terms therefore decide whether transfer is positive or negative.}

\begin{figure}[htbp]
    \centering
    \begin{minipage}{0.28\textwidth}
        \centering
        \includegraphics[width=\linewidth]{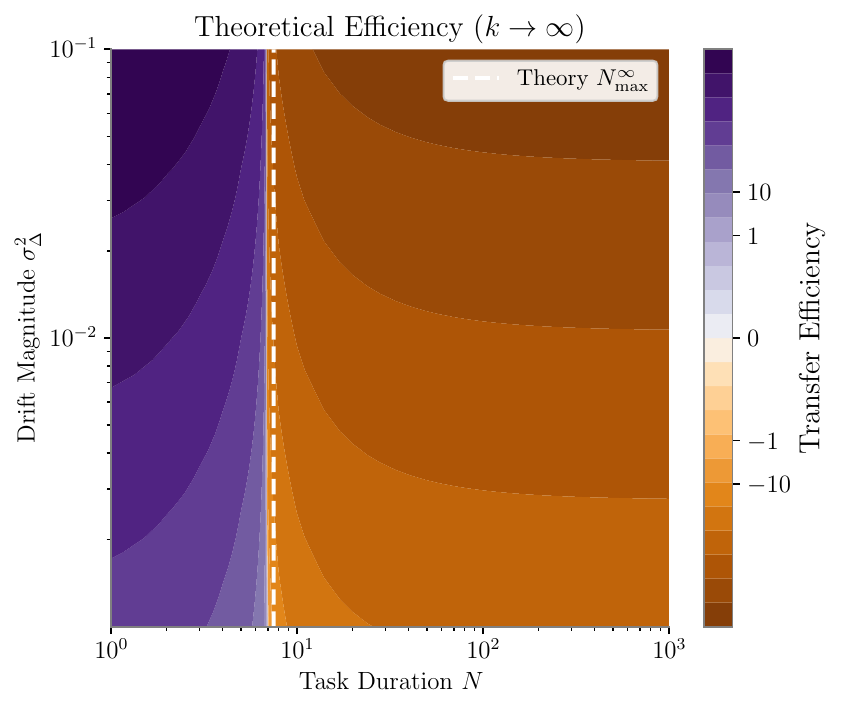}
    \end{minipage}
    \hspace{0.1cm}
    \begin{minipage}{0.6\textwidth}
    \caption{\textbf{Asymptotic Transfer Efficiency ($k \to \infty$).} The theoretical landscape of Transfer Efficiency in the limit of infinite tasks. The white dashed line marks the exact Critical Task Duration $N_{\max}^\infty$ derived in \cref{eq:critical-task-dur}. In this asymptotic regime, the boundary between positive and negative transfer becomes a vertical line, showing that the \major{Critical Task Duration} depends on spatial trace properties and learning rate, but is invariant to the drift magnitude $\sigma_\Delta^2$.}
    \label{fig:asymptotic-simulation}
    \end{minipage}
    \vspace{-0.2cm}
\end{figure}

\subsection{Window Algorithm as an Instance of Predictive CL}
\label{sec:optimal_window}

We introduce the Window algorithm as a concrete instance of our generalized CL framework:
\rev{a Predictive CL algorithm that bridges the ITL and JTL approaches through a \emph{window of tasks}. At any task $k$, the \emph{Window agent} predicts the relevant future distribution from the $W$ most recent tasks and optimizes the average loss over this window; $W$ is the \emph{window size}. The sampling distribution $\mathcal{D}^k_W$ is:}
\[
    \mathcal{D}^k_W = \sum_{j = k-W+1}^k \rho_j^k \,\mathcal{P}^j \qquad \rho_j^k = \frac{|N_j|}{\sum_{i\in [k-W,k]} N_i}.
\]
Like the JTL agent, the Window agent's parameters are not reset at the beginning of a new task.
\rev{As $W$ grows, the Window agent converges to the JTL agent (which can be viewed as an \emph{infinite window} agent), whereas $W = 1$ recovers a variant of the ITL agent with \emph{warm starts}.}

\subsubsection{Asymptotic Analysis in the Quadratic Loss Setting}

\rev{Unlike the JTL agent, whose stability is set by the global growth of the task solutions relative to the origin, the Window agent's stability depends only on the local variations of the environment over the horizon $W$. We characterize these variations with the \emph{Mean Squared Displacement} (MSD).}

\begin{definition}[Mean Squared Displacement]
\label{def:msd-appx}
    Let $m: \mathbb{N} \times \mathbb{N} \to \mathbb{R}^+$ be a function characterizing the expected squared Euclidean distance between two task solutions separated by a lag $W$ at time $k$:
    \begin{equation}
        \mathbb{E}[\|\theta_\star^k - \theta_\star^{k-W}\|^2] \in \Theta(m(k, W)).
    \end{equation}
\end{definition}

We analyze the three canonical regimes
\rev{assuming stationary increments.} The Mean Squared Displacement scales as $m(W) \in \Theta(1)$ for Bounded, $\Theta(W)$ for Diffusive, and $\Theta(W^2)$ for Linear environments.

\begin{remark}[Relationship between MSD and Growth Rate]
    \rev{The Mean Squared Displacement is always upper-bounded by the global Growth Rate:} $m(k, W) \in \mathcal{O}(\psi(k))$. This follows from the triangle inequality relative to the origin.
    \rev{However, in the standard regimes of interest (Diffusive and Linear), the process has \emph{stationary increments}. In these cases, $m(k, W)$ depends only on the lag $W$, not on the absolute time $k$.}
\end{remark}

\subsubsection{Window Geometric Drift}
\rev{The Geometric Drift of the Window agent consists of the update of the sliding centroid and the change in the stationary covariance matrix:}
\begin{align}
    \Delta \mu^{W,k} &= (\theta_\star^{W, k-1} - \theta_\star^{W, k}), \\
    \Delta \Sigma^{W,k} &= (\Sigma_\infty^{W,k-1} - \Sigma_\infty^{W,k}).
\end{align}
Table \ref{tab:window_drift_scaling} summarizes the asymptotic scaling of these drift terms. Unlike the JTL agent,
\rev{whose drift is suppressed by the ever-growing history $k$,} the Window agent's drift is governed by the ratio of the Mean Squared Displacement to the constant window inertia $W$.

\begin{table}[h]
    \centering
    \renewcommand{\arraystretch}{1.2}
    \caption{Scaling of Window Agent Geometric Drift with Window Size $W$.}
    \label{tab:window_drift_scaling}
    \begin{tabular}{@{}lccccc@{}}
        \toprule
         & & \multicolumn{2}{c}{\textbf{Mean Drift} $\mathbb{E}[\|\Delta \mu^{W,k}\|^2]$} & \multicolumn{2}{c}{\textbf{Covariance Drift} $\mathbb{E}[\mathrm{Tr}(\Delta \Sigma^{W,k})]$} \\
        \cmidrule(lr){3-4} \cmidrule(l){5-6}
        \textbf{Regime} & \textbf{MSD} $m(W)$ & Scaling & Vanishing? & Scaling & Vanishing? \\
        \midrule
        Bounded & $\Theta(1)$ & $\mathcal{O}(W^{-2})$ & \checkmark & $\mathcal{O}(W^{-1})$ & \checkmark \\
        Diffusive & $\Theta(W)$ & $\mathcal{O}(W^{-1})$ & \checkmark & $\mathcal{O}(1)$ & $\times$ \\
        Linear & $\Theta(W^2)$ & $\mathcal{O}(1)$ & $\times$ & $\mathcal{O}(W)$ & $\times$ \\
        \bottomrule
    \end{tabular}
\end{table}

\rev{We now analyze, for a quadratic loss, the Transfer Efficiency of the Window agent relative to the ITL agent. The analysis follows the steps of the JTL derivation and reuses several earlier results.}

As shown previously,
\rev{the transient cost is dominated by the Geometric Drift of the agent's optimum:} $\Delta \mu^{W, k} = \theta_\star^{W, k} - \theta_\star^{W, k-1}$.
\rev{At step $k$, the minimizer of the Window agent's objective is:}
\[
    \theta_\star^{W,k} = \left(\sum_{j=k-W}^k \rho_j^k H_j\right)^{-1} \left(\sum_{j=k-W}^k \rho_j^k H_j \theta_\star^j\right).
\]

\begin{lemma}[Window Geometric Drift]
\label{lemma:window_drift}
    For a Window agent of size $W$,
    \rev{the squared magnitude of the geometric mean drift scales} with the growth rate of the window interval:
    \begin{equation}
        \mathbb{E}[\|\Delta \mu^{W,k}\|^2] \in \major{\mathcal{O}}\left( \frac{m(k, W)}{W^2} \right).
    \end{equation}
\end{lemma}

\begin{proof}
    Let $\bar{H}_k = \frac{1}{W} \sum_{j \in \mathcal{W}_k} H_j$ be the average Hessian. The optimum satisfies $\bar{H}_k \theta_\star^{W,k} = \frac{1}{W} \sum_{j \in \mathcal{W}_k} H_j \theta_\star^j$.

    Comparing steps $k$ and $k-1$, the window $\mathcal{W}_k$ is obtained from $\mathcal{W}_{k-1}$ by removing task $k-W$ and adding task $k$. The average Hessian updates as:
    \[
        \bar{H}_k = \bar{H}_{k-1} + \frac{1}{W}(H_k - H_{k-W}).
    \]
    The weighted target vector $v_k = \bar{H}_k \theta_\star^{W,k}$ updates as:
    \[
        v_k = v_{k-1} + \frac{1}{W}(H_k \theta_\star^k - H_{k-W} \theta_\star^{k-W}).
    \]
    Substituting $v_{k-1} = \bar{H}_{k-1} \theta_\star^{W, k-1}$ and solving for the difference:
    \[
        \bar{H}_k \theta_\star^{W,k} = \left( \bar{H}_k - \frac{1}{W}(H_k - H_{k-W}) \right) \theta_\star^{W, k-1} + \frac{1}{W}(H_k \theta_\star^k - H_{k-W} \theta_\star^{k-W}).
    \]
    Rearranging terms to isolate the drift $\Delta \mu^{W,k} = \theta_\star^{W,k} - \theta_\star^{W, k-1}$:
    \[
        \bar{H}_k (\theta_\star^{W,k} - \theta_\star^{W, k-1}) = \frac{1}{W} \left[ H_k (\theta_\star^k - \theta_\star^{W, k-1}) - H_{k-W} (\theta_\star^{k-W} - \theta_\star^{W, k-1}) \right].
    \]
    Multiplying by $\bar{H}_k^{-1}$:
    \[
        \Delta \mu^{W,k} = \frac{1}{W} \bar{H}_k^{-1} \left[ H_k (\theta_\star^k - \theta_\star^{W, k-1}) + H_{k-W} (\theta_\star^{W, k-1} - \theta_\star^{k-W}) \right].
    \]
    \rev{This expression captures the pull of the new task $k$ and the ``release'' of the old task $k-W$ on the centroid.} \major{Taking the expected squared norm, bounding $\|\bar{H}_k^{-1}\|_2 \le 1/\lambda_{\min}$ and $\|H_j\|_2 \le \lambda_{\max}$ (\cref{ass:bounded_curvature}), and using the triangle inequality with $(x+y)^2 \le 2x^2 + 2y^2$:
    \[
        \mathbb{E}[\|\Delta \mu^{W,k}\|^2] \le \frac{2\kappa^2}{W^2} \left( \mathbb{E}\left\| \theta_\star^k - \theta_\star^{W, k-1} \right\|^2 + \mathbb{E}\left\| \theta_\star^{W, k-1} - \theta_\star^{k-W} \right\|^2 \right).
    \]
    Both terms are expected squared distances between a task and a convex combination of tasks at lag at most $W$, so by Jensen's inequality and the definition of the Mean Squared Displacement (\cref{def:msd-appx}) each is $\mathcal{O}(m(k, W))$. Thus:
    \[
        \mathbb{E}[\|\Delta \mu^{W,k}\|^2] \in \mathcal{O}\left( \frac{m(k, W)}{W^2} \right).
    \]
    When the task Hessians are identical, $\bar{H}_k^{-1}$ cancels exactly, the two centroid terms telescope, and $\Delta \mu^{W,k} = \frac{1}{W}\left(\theta_\star^k - \theta_\star^{k-W}\right)$; the bound is then attained, $\mathbb{E}[\|\Delta \mu^{W,k}\|^2] \in \Theta(m(k,W)/W^2)$, directly by \cref{def:msd-appx}.}
\end{proof}

\begin{lemma}[Window Covariance Drift]
\label{lemma:window_cov_drift}
    \major{Under stationary task increments, the geometric covariance drift of the Window agent is bounded by the finite difference of the Mean Squared Displacement:}
    \begin{equation}
        |\mathbb{E}[\mathrm{Tr}(\Delta \Sigma^{W,k})]| \in  \major{\mathcal{O}}(\Delta m(k,W)).
    \end{equation}
\end{lemma}
\emph{Implication:} Unlike the JTL agent, the Window agent achieves a stable second moment (constant uncertainty size)
\rev{because the task diversity within the active window does not grow without bound with time $k$.}

\begin{proof}
    \rev{As in \cref{lemma:cov_drift}, the stationary} covariance $\Sigma_\infty$ scales with the task diversity $\mathcal{V}_{\text{div}}$. For the Window agent, this diversity is measured over the set $\mathcal{W}_k$:
    \[
        \mathcal{V}_{\text{div}}^{W,k} = \frac{1}{W} \sum_{j \in \mathcal{W}_k} \|\nabla \mathcal{L}^j(\theta_\star^{W,k})\|^2 = \frac{1}{W} \sum_{j \in \mathcal{W}_k} \|\theta_\star^j - \theta_\star^{W,k}\|_{H_j^2}^2.
    \]
    \rev{This is the spatial variance of $W$ consecutive tasks. The covariance drift is:}
    \begin{align}
        \mathbb{E}[\mathrm{Tr}(\Delta \Sigma^{W,k})]
        &= \mathbb{E}[\mathrm{Tr}(C^{W,k-1}_\infty - C^{W,k}_\infty)]\\
        &= \frac{1}{W}\left(\sum_{j \in \mathcal{W}_{k-1}} \mathbb{E}[\|\nabla \mathcal{L}^j(\theta_\star^{W,k-1})\|^2] - \sum_{j \in \mathcal{W}_k} \mathbb{E}[\|\nabla \mathcal{L}^j(\theta_\star^{W,k})\|^2]\right).
    \end{align}

    \rev{The term $\mathbb{E}[\|\nabla \mathcal{L}^j(\theta_\star^{W,k-1})\|^2] = \mathbb{E}[\|\theta_\star^j - \theta_\star^{W,k-1}\|_{H_j^2}^2]$ is the expected norm of the average difference among the task optima in $\mathcal{W}_{k-1}$, so by \cref{def:msd-appx} it scales as $\Theta(m(k-1,W))$.} Similarly, $\mathbb{E}[\|\nabla \mathcal{L}^j(\theta_\star^{W,k})\|^2] \in \Theta(m(k,W))$.
    \major{Define $\Delta m(k,W) := |m(k,W) - m(k-1,W)|$, the change of the envelope between the two windows. Each sum is $\Theta$ of its envelope, so under stationary increments, where the per-lag expectations are the same in both windows, their difference is controlled by the change of the envelope (and vanishes exactly when $m(k,W) = m(W)$ is lag-only, consistent with $\Delta m = 0$). Taking the absolute value of the trace gives:
    \[
        |\mathbb{E}[\mathrm{Tr}(\Delta \Sigma^{W,k})]| \in \mathcal{O}(\Delta m(k,W)).
    \]}
\end{proof}

\begin{lemma}[Asymptotic Instability of the Window Agent]
    \label{lemma:asymptotic_instability_window}
    The expected Instability cost $\mathcal{I}^W$ for the Window agent, which comprises the approximation bias and the stationary gradient noise, scales \major{at most} linearly with the Mean Squared Displacement:
    \begin{equation}
        \mathbb{E}[\mathcal{I}^W] \in \major{\mathcal{O}}(m(k,W)).
    \end{equation}
\end{lemma}

\begin{proof}
    The proof follows the same structure as \cref{lemma:asymptotic_instability}.

    The instability for the Window agent is:
    \[
        \mathcal{I}^W \approx \underbrace{\left(\mathcal{L}^k(\theta_\star^{W,k}) - \mathcal{L}^k(\theta_\star^{k})\right)}_{\text{Bias}} + \underbrace{\frac{\eta}{4B}\sum_{j=k-W}^k \rho_j^k \, \|\nabla \mathcal{L}^j(\theta_\star^{W,k})\|_{\mathcal{R}^k}^2}_{\text{Variance}}.
    \]

    For a quadratic loss, the bias term simplifies to the weighted squared distance between the task optimum and the Window centroid:
    \[
        \mathbb{E}[\mathcal{L}^k(\theta_\star^{W,k}) - \mathcal{L}^k(\theta_\star^{k})] = \frac{1}{2} \mathbb{E}[\|\theta_\star^{W,k} - \theta_\star^k\|_{H_k}^2].
    \]
    Since $\theta_\star^{W,k}$ is a convex combination of all tasks in the window $k-W \dots k$, its position relative to the current task $\theta_\star^k$ depends on the spread of the task distribution. \major{By Jensen's inequality and the definition of the Mean Squared Displacement, $\mathbb{E}[\|\theta_\star^{W,k} - \theta_\star^k\|^2]$ is bounded by the spread of the window, which is $\mathcal{O}(m(k,W))$.}

    The variance term reads:
    \[
        \mathbb{E}[\text{Variance}] \propto \sum_{j=k-W}^k \rho_j^k\,\, \mathbb{E}[\|H_j (\theta_\star^{W,k} - \theta_\star^j)\|_{\mathcal{R}^k}^2].
    \]
    Assuming bounded Hessians (\cref{ass:bounded_curvature}), this term is proportional to the variance of the task optima distribution around their mean $\theta_\star^{W,k}$. As established in the proof of \cref{lemma:window_cov_drift},
    \major{this spatial variance is at most of the order of the local environment growth rate, $\mathcal{O}(m(k,W))$.}

    \major{Both non-negative terms are $\mathcal{O}(m(k,W))$, so the total Instability is $\mathcal{O}(m(k,W))$. The matching lower bound in the canonical environments is established in the window-dispersion paragraph below.}
\end{proof}

\major{\paragraph{Window dispersion in the canonical environments.} The lower bounds matching the lemmas above hold in the canonical environments, by the argument used for the JTL agent in \cref{sec:mt_asymptotics}. Under \cref{ass:alignment,ass:bounded_curvature}, in each eigendirection of the common basis the Window solution is a convex combination $\sum_{j \in \mathcal{W}_k} a_j\, \theta_\star^j$ with weights $a_j \in \big[\frac{1}{\kappa W}, \frac{\kappa}{W}\big]$. In the diffusive environment, writing $\theta_\star^k - \theta_\star^{W,k-1} = \sum_m c_m \Delta_m$, the coefficient of each of the last $\lfloor W/(2\kappa) \rfloor$ window increments satisfies $c_m \ge \frac{1}{2}$, so by the independence of the increments (with $s$ the increment covariance trace, as in \cref{lemma:dispersion_examples})
\[
    \mathbb{E}\big\|\theta_\star^k - \theta_\star^{W,k-1}\big\|^2 \;\ge\; \frac{s}{4}\left\lfloor \frac{W}{2\kappa} \right\rfloor \;\in\; \Omega\!\left(\frac{m(k,W)}{\kappa}\right),
\]
and the centroid including the current task differs only by a factor $(1 - a_k)^2 \ge (1 - \kappa/W)^2$, as in \cref{cor:growth_lower}. Through the bias term ($\mathcal{I}^W \ge \frac{\lambda_{\min}}{2}\|\theta_\star^{W,k} - \theta_\star^k\|^2$, the variance term being non-negative), the Instability of \cref{lemma:asymptotic_instability_window} is therefore $\Omega(m(k,W))$ up to factors of $\kappa$. The same increment expansion applied to $\Delta \mu^{W,k}$ gives every window increment a same-sign coefficient bounded in $[\lambda_{\min}, \lambda_{\max}]$ before the $1/(W \bar{h})$ factor, so $\mathbb{E}\|\Delta \mu^{W,k}\|^2 \in \Omega(m(k,W)/W^2)$ up to factors of $\kappa^2$, matching \cref{lemma:window_drift}. In the bounded environment the fresh task is independent of the history, so $\mathbb{E}\|\theta_\star^k - c\|^2 \ge s$ for \emph{any} history-measurable centroid $c$; in the linear environment the weight cap $a_j \le \kappa/W$ keeps the window centroid at least $\|v\| W/(2\kappa)$ behind the newest task, giving $\Omega(m(k,W))$. Throughout this subsection, the $\mathcal{O}$ upper bounds hold for arbitrary environments of each regime, while the matching $\Theta$ statements hold in the canonical environments, with constants degraded by powers of $\kappa$.}

\subsubsection{Asymptotically Optimal Window}

\rev{We now evaluate the asymptotic Transfer Efficiency of the Window agent and determine the optimal window. The following lemma mirrors the structure of \cref{theo:positive_transfer}, since the comparison with the ITL baseline is unchanged.}

\begin{lemma}[Asymptotic Optimality of Window Size]
    \label{lemma:optimal_window}
    Let the asymptotic excess risk of the Window agent be dominated by the trade-off between task diversity and tracking lag, given by $J(W) \asymp m(k,W) + \frac{m(k,W)}{W^2}$, where $m(k,W)$ is the Mean Squared Displacement (MSD) of tasks within the window. The optimal window size $W^*$ minimizing this risk depends on the environment's non-stationarity regime:
    \begin{enumerate}
        \item \emph{Bounded Regime:} If $m(k,W) \in \Theta(1)$, the risk decreases monotonically with $W$. The asymptotically optimal strategy is the JTL agent ($W^* \to \infty$).
        \item \emph{Linear Drift Regime:} If $m(k,W) \in \Theta(W^2)$, the risk increases monotonically with $W$. The asymptotically optimal strategy is the ITL agent ($W^* = 1$).
        \item \emph{Diffusive Regime:} If $m(k,W) \in \Theta(W)$, the risk is strictly convex. The asymptotically optimal strategy is a Window agent with finite memory ($1 < W^* < \infty$).
    \end{enumerate}
\end{lemma}

\begin{proof}
    \textbf{Step 1: The Asymptotic Transfer Condition.}
    We start from the general condition for positive transfer derived in \cref{theo:positive_transfer}.
    \rev{For the Window agent to outperform the ITL baseline, its expected excess risk must be lower than that of the baseline.} Let $C^k_{\text{ITL}}$ and $C^k_{W}$ denote the condition numbers scaling the respective bounds. The condition is:
    \[
        {\tau N \cdot \mathbb{E}[\mathcal{I}^W]} < {C^k_{\text{ITL}} \left( \mathbb{E}[\|\tilde \mu_0^{\text{ITL},k}\|^2] + \mathbb{E}[\text{Tr}(\tilde{\Sigma}_0^{\text{ITL},k})] \right)} - {C^k_{W} \left( \mathbb{E}[\|\tilde \mu_0^{W,k}\|^2] + \mathbb{E}[|\text{Tr}(\tilde{\Sigma}_0^{W,k})|] \right)}.
    \]

    \major{By \cref{lemma:asymptotic_instability_window,lemma:window_drift,lemma:window_cov_drift} the scalings below hold as upper bounds for any environment of each regime; by the window-dispersion paragraph above they are attained, up to factors of $\kappa$, in the canonical environments. The sufficient condition of Step 1 uses only the upper bounds; the classification of $W^*$ in Step 2 uses the two-sided scalings and is therefore stated for the canonical environments.}
    We substitute the asymptotic scaling laws derived in \cref{lemma:asymptotic_instability_window,lemma:window_drift,lemma:window_cov_drift}:
    \begin{itemize}
        \item The expected instability $\mathbb{E}[\mathcal{I}^W]$ scales with the local task diversity $\Theta(m(k,W))$.
        \item The expected centroid tracking lag $\mathbb{E}[\|\tilde \mu_0^{W,k}\|^2]$ scales as $\Theta(m(k,W)/W^2)$.
        \item The expected covariance drift $\mathbb{E}[|\text{Tr}(\tilde{\Sigma}_0^{W,k})|]$ scales with the change in diversity $\Delta m(k,W) = |\Theta(m(k,W)) - \Theta(m(k-1,W))|$.
    \end{itemize}
    Substituting these into the inequality:
    \begin{equation} \label{eq:full_scaling}
        \tau N \cdot \Theta(m(k,W)) < C^k_{\text{ITL}} \Theta(\psi(k)) - C^k_{W} \left( \Theta\left(\frac{m(k,W)}{W^2}\right) + \Theta(\Delta m(k,W)) \right).
    \end{equation}

    \textbf{Step 2: Optimal Window Size.}
    To find the optimal window size $W^*$,
    \rev{we isolate the terms that depend on the window.} We define the asymptotic cost function $J(W)$ as the sum of the diversity penalty (scaled by the task duration $N$) and the tracking lag (independent of $N$).

    \textit{Remark on Time Dependence ($k$):}
    \rev{For the canonical environments considered (Bounded, Linear Drift, Diffusive), the statistics of the task increments are stationary, so the local diversity $m(k,W)$ within a window depends only on the window width $W$, not on the absolute time index $k$. We therefore write $m(k,W) \to m(W)$.}

    The cost function is defined as:
    \[
        J(W) \asymp c_1 N \cdot m(W) + c_2 \frac{m(W)}{W^2} + c_3 \Delta m(W).
    \]
    Here, $c_1$ captures the interference cost per task, while $c_2$ and $c_3$
    \rev{capture the lag of the tracker.} We now minimize $J(W)$ for the three canonical regimes.

    \paragraph{Case 1: Bounded Regime.}
    Tasks are drawn from a fixed distribution with finite covariance: $m(W) \sim \Theta(1)$ and $\Delta m(W) \approx 0$. Thus, $J(W) \sim c_1 N + \frac{c_2}{W^2}$.
    \rev{The term $c_2/W^2$ is strictly decreasing while the diversity cost remains constant, so $J(W)$ is minimized as $W \to \infty$.}
    \[ W^* = \infty \quad \text{(JTL agent is optimal)}. \]

    \paragraph{Case 2: Linear Drift Regime.}
    Tasks drift with a constant velocity $v$, i.e., $\mathbb{E}[\|\theta_k - \theta_{k-1}\|] = v$. The variance of a line segment of length $W$ is $m(W) \sim \Theta(W^2)$. Because $\frac{m(W)}{W^2} \sim \frac{W^2}{W^2} \sim \Theta(1)$, we have $J(W) \sim c_1 N W^2 + c_2$.
    The diversity penalty $N W^2$
    \rev{grows quadratically and dominates the constant lag.} The function is strictly increasing for $W \ge 1$.
    \[ W^* = 1 \quad \text{(ITL agent is optimal)}. \]

    \paragraph{Case 3: Diffusive Regime.}
    Tasks follow a Gaussian random walk with variance $\sigma^2$ per step. The variance grows linearly with the window size, $m(W) \sim \Theta(W)$. The expected variance is stationary, so $\Delta m(W) = 0$. Thus, $J(W) \sim c_1 N W + c_2 \frac{W}{W^2} = c_1 N W + \frac{c_2}{W}$.
    This is a strictly convex function. We find the minimum by setting the derivative to zero:
    \[
        \frac{dJ}{dW} = c_1 N - \frac{c_2}{W^2} = 0 \implies W^2 = \frac{c_2}{c_1 N}.
    \]
    Asymptotically, the optimal window size follows the inverse square-root law:
    \[ W^* = \sqrt{\frac{c_2}{c_1 N}} \propto \frac{1}{\sqrt{N}}. \]
    \rev{In this case, the optimal strategy requires a finite window, whose size shrinks as the task duration $N$ increases.}
\end{proof}

\subsubsection{Window Agent for Continuously Drifting Tasks}

\rev{We now instantiate this analysis} for the Window agent in the continuous drift setting. Recall that the expected loss function for task $k$ is:
\begin{equation}
    \mathcal{L}^k(\theta) = \frac{1}{2} \| \theta - \theta_\star^k \|_{\Sigma_x}^2 + \frac{1}{2}\sigma_{\xi}^2,
\end{equation}
where $\|v\|_{\Sigma_x}^2 := v^\top \Sigma_x v$. The optimal parameters evolve according to a Wiener process:
\begin{equation}
    \theta_\star^k = \theta_\star^{k-1} + \Delta_k, \quad \Delta_k \sim \mathcal{N}(0, \Sigma_\Delta),
\end{equation}
where $\Sigma_\Delta$ is the drift covariance matrix. At each stage $k$,
\rev{the minimizer of the Window agent is the centroid of the tasks in the active window} $\mathcal{W}_k = \{k-W+1, \dots, k\}$:
\begin{equation}
    \theta_\star^{W, k} = \frac{1}{W} \sum_{j=k-W+1}^k \theta_\star^j.
\end{equation}

\subsubsection{Instability of the Window Agent}

\rev{The Instability $\mathcal{I}^W$ combines the discrepancy between the agent's target (the window centroid) and the true task optimum $\theta_\star^k$ with the internal diversity of the window. By the Exact Instability Decomposition (\cref{prop:exact_instability}), it splits into the centroid lag and the task scatter.}

\paragraph{Centroid Lag.}
We define the lag vector $v_W := \theta_\star^{W, k} - \theta_\star^k$.
\rev{We substitute the backward recursion $\theta_\star^j = \theta_\star^k - \sum_{m=j+1}^k \Delta_m$ into the centroid definition:}
\begin{equation}
    v_W = \frac{1}{W} \sum_{j=k-W+1}^k \left( \theta_\star^k - \sum_{m=j+1}^k \Delta_m \right) - \theta_\star^k = - \frac{1}{W} \sum_{j=k-W+1}^k \sum_{m=j+1}^k \Delta_m.
\end{equation}
Rearranging the double sum to group by the individual increment $\Delta_p$ (where each $\Delta_p$ appears for all tasks $j < p$):
\begin{equation}
    v_W = - \frac{1}{W} \sum_{p=1}^{W-1} p \Delta_{k-p+1}.
\end{equation}
The expected energy of the lag is the weighted sum of the increment covariances. Using the approximation $\sum_{p=1}^{W-1} p^2 \approx W^3/3$ for large $W$:
\begin{equation}
    \mathbb{E}[\| v_W \|_{\Sigma_x}^2] = \text{Tr}\left( \Sigma_x \frac{1}{W^2} \sum_{p=1}^{W-1} p^2 \Sigma_\Delta \right) \approx \frac{W}{3}\text{Tr}(\Sigma_x \Sigma_\Delta).
\end{equation}
\rev{This is the irreducible tracking error: the centroid lags behind} the random walk by a distance proportional to the window size.

\paragraph{Diversity Noise.}
\rev{The second component of the Instability is} the scatter of task optima within the window. Using the identity $\sum (x_i - \bar{x})^2 = \frac{1}{2W}\sum_{i,j}(x_i - x_j)^2$, we write the sum of squared errors as:
\begin{equation}
    \sum_{j \in \mathcal{W}_k} \|\theta_\star^{W,k} - \theta_\star^j\|_{\Sigma_x^2}^2 = \frac{1}{2W} \sum_{i,j \in \mathcal{W}_k} \|\theta_\star^i - \theta_\star^j\|_{\Sigma_x^2}^2.
\end{equation}
For a random walk, the expected squared distance scales with the time lag $|i-j|$:
\[ \mathbb{E}[\|\theta_\star^i - \theta_\star^j\|_{\Sigma_x^2}^2] = |i-j|\text{Tr}(\Sigma_x^2 \Sigma_\Delta). \]
Summing this over the window indices (using $\sum_{i,j=1}^W |i-j| \approx W^3/3$):
\begin{equation}
    \mathbb{E}\left[ \sum_{j \in \mathcal{W}_k} \|\theta_\star^{W,k} - \theta_\star^j\|_{\Sigma_x^2}^2 \right] \approx \frac{1}{2W} \left( \frac{W^3}{3} \right) \text{Tr}(\Sigma_x^2 \Sigma_\Delta) = \frac{W^2}{6} \text{Tr}(\Sigma_x^2 \Sigma_\Delta).
\end{equation}
\rev{Finally, we substitute this into the Instability decomposition, which averages this sum over the $W$ tasks and so introduces a $1/W$ coefficient:}
\begin{align}
    \mathbb{E}[\mathcal{I}^W] &\approx \frac{1}{2}\mathbb{E}[\| v_W \|_{\Sigma_x}^2] + \frac{\eta}{4 B W} \mathbb{E}\left[ \sum_{j \in \mathcal{W}_k} \|\theta_\star^{W,k} - \theta_\star^j\|_{\Sigma_x^2}^2 \right] \\
    &= \frac{1}{2}\left(\frac{W}{3}\text{Tr}(\Sigma_x \Sigma_\Delta)\right) + \frac{\eta}{4 B W} \left(\frac{W^2}{6} \text{Tr}(\Sigma_x^2 \Sigma_\Delta)\right) \\
    &= \frac{W}{6} \text{Tr}(\Sigma_x \Sigma_\Delta) + \frac{\eta W}{24B} \text{Tr}(\Sigma_x^2 \Sigma_\Delta).
    \label{eq:window_instability}
\end{align}

\subsubsection{Transient Error of the Window Agent}

The Transient Error $\delta^W$ arises from the initialization error at the start of each task.
\rev{The Window agent starts each task at the previous centroid $\theta_\star^{W, k-1}$, so the expected error vector $\tilde{\mu}_0^{W,k}$ is driven only by the drift of the centroid:}
\begin{equation}
    \Delta \mu_W = \theta_\star^{W, k} - \theta_\star^{W, k-1} = \frac{1}{W} \left( \sum_{j=k-W+1}^k \theta_\star^j - \sum_{j=k-W}^{k-1} \theta_\star^j \right) = \frac{1}{W} (\theta_\star^k - \theta_\star^{k-W}).
\end{equation}
\rev{This drift reduces to the average of the $W$ increments that separate the incoming task from the one leaving the window:}
\begin{equation}
    \Delta \mu_W = \frac{1}{W} \sum_{m=k-W+1}^k \Delta_m.
\end{equation}
The expected squared norm of this drift is:
\begin{equation}
    \mathbb{E}[\|\Delta \mu_W\|^2] = \frac{1}{W^2} \sum_{m=1}^W \text{Tr}(\Sigma_\Delta) = \frac{1}{W} \text{Tr}(\Sigma_\Delta).
\end{equation}
\rev{As in \cref{sss:TB-CD},} the expected initialization error $\tilde{\mu}_0$ tracks this input drift with an amplification factor $c_\mu = \frac{1}{1-\rho_\mu}$.
\rev{Unlike that of the JTL agent, the Window agent's target covariance is stationary in expectation (zero expected covariance drift), so the centroid drift alone dominates the Transient Error. Using the bound derived in \cref{theo:positive_transfer}:}
\begin{equation}
    \mathbb{E}[\delta^W] \approx \frac{\gamma_{\max}}{\tau N} c_\mu^2 \mathbb{E}[\|\Delta \mu_W\|^2] = \frac{\gamma_{\max}}{\tau N} \frac{c_\mu^2}{W} \text{Tr}(\Sigma_\Delta).
\end{equation}

\subsubsection{The Optimal Window \texorpdfstring{$W^\star$}{W*}}

To determine the optimal window size, we minimize the total asymptotic risk contributed by the Window agent.
\rev{This risk consists of the diversity cost} (the Instability), which grows with $W$, and the Transient Error cost, which shrinks with $W$:
\begin{equation}
    W^\star = \arg \min_W \left[ \underbrace{ \mathbb{E}[\mathcal{I}^W]}_{\text{Instability Cost}} + \underbrace{\mathbb{E}[\delta^W]}_{\text{Transient Cost}} \right].
\end{equation}
\rev{Substituting the derived expressions and making the dependence on $W$ explicit:}
\begin{equation}
    W^\star = \arg \min_W \left[ W \cdot \underbrace{ \left( \frac{\text{Tr}(\Sigma_x \Sigma_\Delta)}{6} + \frac{\eta \text{Tr}(\Sigma_x^2 \Sigma_\Delta)}{24B} \right)}_{\mathcal{R_B}} + \frac{1}{W} \cdot \underbrace{\frac{\gamma_{\max} c_\mu^2 \sigma^2_\Delta}{\tau N}}_{\mathcal{R_T}} \right].
\end{equation}
The objective function is of the form:
\[
    f(W) = \mathcal{R_B} W + \mathcal{R_T} W^{-1},
\]
which is strictly convex. The minimum is achieved at the square-root balance:
\[
    W^\star = \sqrt{\mathcal{R_T} /  \mathcal{R_B}}.
\]

Written explicitly, the optimal window length is:
\begin{align}
\label{eq:optimal-window}
    W^\star
    &= \sqrt{\frac{\gamma_{\max} c_\mu^2 \sigma^2_\Delta / (\tau N)}{\frac{\text{Tr}(\Sigma_x \Sigma_\Delta)}{6} + \frac{\eta \text{Tr}(\Sigma_x^2 \Sigma_\Delta)}{24B} }} \\
    &= \frac{c_\mu \,\sqrt{{\gamma_{\max}}}}{\sqrt{\tau N}} \left(\frac{\text{Tr}(\Sigma_x \Sigma_\Delta)}{6\,\sigma^2_\Delta} + \frac{\eta \text{Tr}(\Sigma_x^2 \Sigma_\Delta)}{24B\,\sigma^2_\Delta}\right)^{-1/2}.
\end{align}
\rev{This exact, non-asymptotic result agrees with \major{the asymptotic scaling derived in \cref{lemma:optimal_window}}, confirming that $W^\star \propto N^{-1/2}$.}

\rebuttal{\subsection{Relaxing the Assumptions}
\label{sec:relaxing}

\rev{The derivations above rest on two simplifying assumptions,} fixed in \cref{sec:quadratic_analysis}: Task Alignment (\cref{ass:alignment}) and Bounded Curvature (\cref{ass:bounded_curvature}).
\rev{This final subsection relaxes each of them and shows that neither is essential to the picture. First, we remove Task Alignment: the exact Instability and Transient hold for arbitrary, non-commuting Hessians, and we treat the special case of jointly realizable tasks (\cref{sec:beyond_alignment}).} Second, we drop Bounded Curvature and treat degenerate (rank-deficient) curvature exactly, under ridge regularization and in the ridgeless limit (\cref{sec:ridgeless}).}

\rebuttal{\subsubsection{Beyond Task Alignment}
\label{sec:beyond_alignment}
}

\rebuttal{The proof of \cref{prop:exact_instability}
\rev{invokes \cref{ass:alignment} at a single point, the trace identity used for the variance term. The result therefore generalizes to arbitrary, non-commuting Hessians: the Approximation Bias remains exactly as stated,} and only the variance term needs adjusting. The next proposition gives the exact form of the variance term in general; the identity used above is recovered as its commuting special case.
\begin{proposition}[Instability variance term for arbitrary Hessians]
\label{prop:instability_no_alignment}
Let $H_{\text{dyn}} \succ 0$ have the eigendecomposition $H_{\text{dyn}} = U \Gamma U^\top$ with $\Gamma = \mathrm{diag}(\gamma_1, \dots, \gamma_d)$, and let $C_\infty \succeq 0$ be the stationary gradient-noise covariance. The stationary parameter covariance solving the Lyapunov equation $H_{\text{dyn}}\Sigma_\infty + \Sigma_\infty H_{\text{dyn}} = \eta C_\infty$ is, in the eigenbasis of $H_{\text{dyn}}$,
\begin{equation}
    \tilde{\Sigma}_{ij} = \frac{\eta \, \tilde{C}_{ij}}{\gamma_i + \gamma_j}, \qquad \tilde{\Sigma} := U^\top \Sigma_\infty U, \quad \tilde{C} := U^\top C_\infty U,
\end{equation}
and for any symmetric evaluation Hessian $H_k$ the variance contribution to the risk is exactly
\begin{equation}
\label{eq:general_variance}
    \frac{1}{2}\mathrm{Tr}(H_k \Sigma_\infty) \;=\; \frac{\eta}{2} \sum_{i,j=1}^{d} \frac{(\tilde{H}_k)_{ij}\, \tilde{C}_{ij}}{\gamma_i + \gamma_j},
    \qquad \tilde{H}_k := U^\top H_k U .
\end{equation}
No commutativity between $H_k$, $H_{\text{dyn}}$, and $C_\infty$ is required; only $H_{\text{dyn}} \succ 0$ is used (\cref{ass:bounded_curvature}; relaxed in \cref{sec:ridgeless}).
\end{proposition}

\begin{proof}
The proof has three steps: we rotate the Lyapunov equation into the eigenbasis of the driving Hessian, solve it entrywise there, and evaluate the trace in the same basis.

Substituting $H_{\text{dyn}} = U\Gamma U^\top$ into the Lyapunov equation and conjugating with $U$ (multiplying by $U^\top$ on the left and $U$ on the right, and using $U^\top U = I$):
\begin{equation}
    U^\top \big( U\Gamma U^\top \Sigma_\infty + \Sigma_\infty U\Gamma U^\top \big) U \;=\; \eta\, U^\top C_\infty U
    \qquad\Longleftrightarrow\qquad
    \Gamma \tilde\Sigma + \tilde\Sigma \Gamma \;=\; \eta\, \tilde C .
\end{equation}

Because $\Gamma$ is diagonal, the rotated equation decouples entry by entry: the $(i,j)$ entry reads
\begin{equation}
    (\Gamma\tilde\Sigma)_{ij} + (\tilde\Sigma\Gamma)_{ij}
    \;=\; \gamma_i\,\tilde\Sigma_{ij} + \tilde\Sigma_{ij}\,\gamma_j
    \;=\; (\gamma_i + \gamma_j)\,\tilde\Sigma_{ij}
    \;=\; \eta\,\tilde C_{ij} .
\end{equation}
Since $H_{\text{dyn}} \succ 0$ gives $\gamma_i + \gamma_j > 0$ for every pair $(i,j)$, each entry is uniquely determined: $\tilde\Sigma_{ij} = \eta\,\tilde C_{ij}/(\gamma_i + \gamma_j)$.

Finally, the trace is invariant under the orthogonal change of basis, so
\begin{equation}
    \mathrm{Tr}(H_k \Sigma_\infty)
    \;=\; \mathrm{Tr}\big(U^\top H_k U \; U^\top \Sigma_\infty U\big)
    \;=\; \mathrm{Tr}\big(\tilde H_k \tilde\Sigma\big)
    \;=\; \sum_{i,j=1}^d (\tilde H_k)_{ij}\, \tilde\Sigma_{ji}
    \;=\; \eta \sum_{i,j=1}^d \frac{(\tilde H_k)_{ij}\, \tilde C_{ij}}{\gamma_i + \gamma_j} ,
\end{equation}
where the last equality uses the symmetry $\tilde\Sigma_{ji} = \tilde\Sigma_{ij}$ inherited from $\Sigma_\infty$. Halving both sides yields \cref{eq:general_variance}.
\end{proof}

\rev{Two immediate consequences follow from \cref{prop:instability_no_alignment}.}

\paragraph{The aligned case is the diagonal case.} Under \cref{ass:alignment}, $U$ can be chosen to diagonalize $H_k$ as well, so $\tilde H_k = \mathrm{diag}(h^{(k)}_1, \dots, h^{(k)}_d)$ and only the resonant terms $i=j$ survive in \cref{eq:general_variance}:
\begin{equation}
    \frac{1}{2}\mathrm{Tr}(H_k\Sigma_\infty) = \frac{\eta}{4}\sum_{i=1}^d \frac{h^{(k)}_i}{\gamma_i}\,\tilde C_{ii} = \frac{\eta}{4}\,\mathrm{Tr}\!\left(\mathcal{R}^k C_\infty\right),
\end{equation}
recovering exactly the identity used in the proof of \cref{prop:exact_instability}.
\rev{Task misalignment appears as the off-diagonal entries} of $\tilde H_k$, which couple to the off-diagonal noise correlations $\tilde C_{ij}$ with weights $1/(\gamma_i + \gamma_j)$.

\paragraph{The error of the aligned formula is controlled by the misalignment.} Split $\tilde H_k = D + O$ into its diagonal and off-diagonal parts in the eigenbasis of $H_{\text{dyn}}$. The exact variance term \eqref{eq:general_variance} deviates from its aligned (diagonal) surrogate by
\begin{equation}
\label{eq:misalignment_bound}
    \left| \frac{1}{2}\mathrm{Tr}(H_k\Sigma_\infty) - \frac{\eta}{4}\,\mathrm{Tr}\!\left(D\,\Gamma^{-1} \tilde C\right) \right|
    \;=\; \frac{\eta}{2}\left|\sum_{i\ne j}\frac{O_{ij}\,\tilde C_{ij}}{\gamma_i+\gamma_j}\right|
    \;\le\; \frac{\eta}{4\,\gamma_{\min}}\, \|O\|_F\, \|C_\infty\|_F ,
\end{equation}
using $\gamma_i + \gamma_j \ge 2\gamma_{\min}$ and Cauchy--Schwarz. The off-diagonal mass $\|O\|_F$ is the natural scalar measure of task misalignment.
\rev{If $U$ diagonalizes $\tilde H_k$ within each eigenspace of $H_{\text{dyn}}$ (always possible, since both matrices are symmetric), $O$ is supported} on index pairs with $\gamma_i \ne \gamma_j$ and satisfies $\|O\|_F \le \|[H_k, H_{\text{dyn}}]\|_F \,/\, g$, where $[A, B] := AB - BA$ denotes the commutator and $g$ is the smallest gap between
\rev{distinct eigenvalues of $H_{\text{dyn}}$. In particular, $O = 0$ precisely when the two Hessians commute.} The aligned formulas of \cref{prop:exact_instability} therefore degrade gracefully, with an error linear in the commutator of the evaluation and driving Hessians.
}

\rebuttal{As for the Instability, \cref{ass:alignment} can be removed from the Transient analysis by the same eigenbasis-coupling method; the Transient Potential Law is recovered as the diagonal case.

\begin{proposition}[Transient Error without Task Alignment]
\label{prop:transient_no_alignment}
In the notation of \cref{prop:instability_no_alignment}, let
\begin{equation}
    S_n := \mathbb{E}\big[(\theta_n - \mu_\infty)(\theta_n - \mu_\infty)^\top\big],
    \qquad
    \Delta_n := S_n - \Sigma_\infty ,
    \qquad
    \tilde\Delta_n := U^\top \Delta_n\, U ,
\end{equation}
and let $H_{\text{eval}}$ be the evaluation Hessian, $\tilde H_{\text{eval}} := U^\top H_{\text{eval}} U$. Under \cref{ass:small_step},
\begin{equation}
\label{eq:general_transient}
    \delta_N \;\approx\; \frac{1}{2N\eta} \sum_{i,j=1}^{d} \frac{(\tilde H_{\text{eval}})_{ij}\, (\tilde\Delta_0)_{ij}\,\big(1 - \rho_{ij}^{\,N}\big)}{\gamma_i + \gamma_j},
    \qquad
    \rho_{ij} := (1 - \eta\gamma_i)(1 - \eta\gamma_j) .
\end{equation}
Under \cref{ass:alignment}, \cref{eq:general_transient} reduces to the Transient Potential Law (\cref{theo:transient_potential}).
\end{proposition}
\begin{proof}
The proof mirrors the three-step structure of \cref{prop:instability_no_alignment}: propagate, rotate and decouple, then sum.

As established in the proof of \cref{theo:transient_potential}, under \cref{ass:small_step} the deviation propagates homogeneously:
\begin{equation}
    \Delta_n \;\approx\; (I - \eta H_{\text{dyn}})^n\, \Delta_0\, (I - \eta H_{\text{dyn}})^n .
\end{equation}

In the eigenbasis of $H_{\text{dyn}}$ the propagator is diagonal, $U^\top (I - \eta H_{\text{dyn}})\, U = I - \eta\Gamma$, so each entry of the rotated deviation evolves independently, as a scalar geometric decay:
\begin{equation}
    (\tilde\Delta_n)_{ij} \;=\; (1 - \eta\gamma_i)^n\, (1 - \eta\gamma_j)^n\, (\tilde\Delta_0)_{ij} \;=\; \rho_{ij}^{\,n}\, (\tilde\Delta_0)_{ij} .
\end{equation}

The transient cost is the average excess risk, $\delta_N = \frac{1}{2N}\sum_{n=0}^{N-1} \mathrm{Tr}(H_{\text{eval}}\Delta_n)$. Evaluating each trace in the eigenbasis (using the symmetry of $\tilde\Delta_n$) and exchanging the sums:
\begin{align}
    \delta_N
    &= \frac{1}{2N} \sum_{i,j=1}^{d} (\tilde H_{\text{eval}})_{ij}\, (\tilde\Delta_0)_{ij} \sum_{n=0}^{N-1} \rho_{ij}^{\,n}
    \;=\; \frac{1}{2N} \sum_{i,j=1}^{d} (\tilde H_{\text{eval}})_{ij}\, (\tilde\Delta_0)_{ij}\, \frac{1 - \rho_{ij}^{\,N}}{1 - \rho_{ij}} .
\end{align}
The denominator expands as
\begin{equation}
    1 - \rho_{ij} \;=\; \eta(\gamma_i + \gamma_j) - \eta^2\gamma_i\gamma_j
    \;=\; \eta(\gamma_i + \gamma_j)\Big(1 - \frac{\eta\,\gamma_i\gamma_j}{\gamma_i + \gamma_j}\Big)
    \;\approx\; \eta\,(\gamma_i + \gamma_j)
\end{equation}
by \cref{ass:small_step}, which yields \cref{eq:general_transient}.

\textbf{Reduction to the aligned case.} Under \cref{ass:alignment}, $\tilde H_{\text{eval}} = \mathrm{diag}(h_1, \dots, h_d)$, so only the resonant terms $i = j$ survive, with $\gamma_i + \gamma_i = 2\gamma_i$ and $\rho_{ii} = (1 - \eta\gamma_i)^2$:
\begin{equation}
    \delta_N \;\approx\; \frac{1}{4N\eta} \sum_{i=1}^d \frac{h_i}{\gamma_i}\, (\tilde\Delta_0)_{ii} \Big(1 - (1 - \eta\gamma_i)^{2N}\Big)
    \;=\; \frac{1}{4N\eta}\Big(\mathrm{Tr}(\Delta_0\, \mathcal{R}) - \mathrm{Tr}(\Delta_N\, \mathcal{R})\Big) ,
\end{equation}
which is the Transient Potential Law of \cref{theo:transient_potential}.
\end{proof}}

\rebuttal{The two results above
\rev{make one classical special case transparent: \emph{shared task solutions}, the jointly realizable setting of \citet{evron2022catastrophic}, in which tasks differ only on the input side} \citep{lee2021continual, lin2023theory, evron2024joint}.

\begin{corollary}[Shared task solutions]
\label{cor:shared_solutions}
Let all tasks share the minimizer $\theta_\star^j = \theta_\star$, with arbitrary Hessians $H_j \succ 0$, arbitrary mixture weights, and uniform observation noise $\sigma^2$. Then, for every $k$,
\begin{equation}
    \theta_\star^{\text{JTL},k} = \theta_\star ,
    \qquad
    \nabla\mathcal{L}^j\big(\theta_\star^{\text{JTL},k}\big) = 0 \;\;\forall j ,
    \qquad
    \Sigma^{\text{JTL},k}_\infty = \frac{\eta\sigma^2}{2B}\, I = \Sigma^{\text{ITL}}_\infty ,
    \qquad
    \mathcal{I}^k = 0 ,
\end{equation}
at every degree of curvature misalignment. Moreover, under the continual protocol of \cref{ss:JTL-ITL} (the ITL agent restarts at the deterministic $\theta_0$; the JTL agent continues from its stage-$(k-1)$ stationary law), the JTL Transient vanishes and, with $v := \theta_0 - \theta_\star$ and $\{(\gamma^{(k)}_i, u_i)\}_{i \le d}$ the eigenpairs of $H_k$, $v_i := u_i^\top v$,
\begin{equation}
\label{eq:TE_shared}
    \mathrm{TE}^k \;=\; \delta^{\text{ITL}}_N
    \;\approx\; \frac{1}{4N\eta} \sum_{i=1}^{d} \Big( v_i^2 - \frac{\eta\sigma^2}{2B} \Big) \Big( 1 - (1 - \eta\gamma^{(k)}_i)^{2N} \Big) .
\end{equation}
\end{corollary}
\begin{proof}
\rev{The mixture optimum equals the shared solution for any weights,}
\begin{equation}
    \theta_\star^{\text{JTL},k}
    = \Big(\sum_{j=1}^k w_j^k H_j\Big)^{-1} \Big(\sum_{j=1}^k w_j^k H_j\Big)\, \theta_\star
    = \theta_\star ,
\end{equation}
so every task gradient vanishes there, and by \cref{lemma:MT_GNC} only the observation-noise floor survives: $C^{\text{JTL},k}_\infty = \frac{\sigma^2}{B} H_{\text{dyn}}$. A noise covariance proportional to the driving Hessian makes the Lyapunov solution isotropic for any $H_{\text{dyn}} \succ 0$,
\begin{equation}
    H_{\text{dyn}}\,\Sigma_\infty + \Sigma_\infty\, H_{\text{dyn}} = \frac{\eta\sigma^2}{B}\, H_{\text{dyn}}
    \qquad\Longrightarrow\qquad
    \Sigma_\infty = \frac{\eta\sigma^2}{2B}\, I ,
\end{equation}
equal to the ITL covariance of \cref{lemma:GNC}: stationary means and covariances agree, so $\mathcal{I}^k = 0$.
\rev{For the protocol, every stage has the same stationary moments, so the JTL agent enters task $k$ at the fixed point of the stage-$k$ moment recursions. The quadratic risk depends} on the parameter law only through these moments, so its risk trajectory is constant and $\delta^{\text{JTL}}_N = 0$ to the order of \cref{ass:small_step}. The TE decomposition (\cref{lemma:TE-decomposition}) leaves $\mathrm{TE}^k = \delta^{\text{ITL}}_N$; for the ITL agent $H_{\text{dyn}} = H_{\text{eval}} = H_k$, with initial deviation $\Delta_0 = v v^\top - \frac{\eta\sigma^2}{2B} I$, and \cref{prop:transient_no_alignment} in its diagonal form gives \cref{eq:TE_shared}.
\end{proof}
}

\rebuttalB{\rev{We validate the results of this subsection in simulation in \cref{apdx:assumption-studies}, with the protocol of \cref{fig:sim-shallow}:} the exact expressions track SGD at every misalignment angle, and the transfer boundary of the phase diagram is insensitive to the misalignment (\cref{fig:assumption-misalignment}).}

\rebuttal{\subsubsection{The Ridgeless Over-Parameterized Regime}
\label{sec:ridgeless}

We now drop \cref{ass:bounded_curvature} and allow the task covariances themselves to be rank-deficient. Three questions organize the analysis. \emph{First}: can the full-rank theory be recovered by regularization?
\rev{Yes, with any weight decay $\lambda > 0$ it applies verbatim.} \emph{Second}: what are the stationary law and risk at $\lambda = 0$ exactly? The dynamics decompose into a full-rank process on the explored subspace and a frozen component, and the structural relation $\mathrm{range}(H_k) \subseteq \mathrm{range}(H_{\text{dyn}})$
\rev{provides all the invariances we need.} \emph{Third}: what do the Instability, Transient Error, and Transfer Efficiency become? We derive their exact forms (\cref{theo:ridgeless_instability_exact,theo:ridgeless_transients_exact}): the full-rank formulas survive with the matrix inverse replaced by the Moore--Penrose pseudoinverse, the current task is blind to the frozen component, and the benefit of memory moves entirely into the Transient.

\paragraph{The ridge lift ($\lambda > 0$).} Adding an isotropic regularizer to every task objective,
\begin{equation}
    \mathcal{L}^j_\lambda(\theta) \;=\; \mathcal{L}^j(\theta) + \frac{\lambda}{2}\|\theta\|^2 ,
    \qquad
    \nabla^2 \mathcal{L}^j_\lambda \;=\; H_j + \lambda I \;\succeq\; \lambda I ,
    \qquad
    \theta_\star^j(\lambda) \;=\; (H_j + \lambda I)^{-1} H_j\, \theta_\star^j ,
\end{equation}
restores \cref{ass:bounded_curvature} with spectral bounds $\lambda \le \mathrm{eig}(H_j + \lambda I) \le \lambda_{\max} + \lambda$. Every result of \cref{sec:quadratic_analysis} then applies verbatim under the substitution $H \mapsto H + \lambda I$, with contraction rate at least $\lambda$ in every direction and regularized minimizers $\theta_\star^j(\lambda)$, which converge to the min-norm minimizers as $\lambda \to 0$.
\rev{The ridge lift also matches practice: all our supervised experiments train} with non-zero weight decay (\cref{tab:tuned_sl_config}).

\paragraph{The structural low-rank model ($\lambda = 0$).} Hereafter we drop the regularizer and focus on the pure low-rank model.

\begin{assumption}[Low-rank task covariances]
\label{ass:low_rank}
Each task's input covariance $H_j = \mathbb{E}[xx^\top] \succeq 0$ may be rank-deficient. Writing $S_j := \mathrm{range}(H_j)$ and $r_j := \mathrm{rank}(H_j) \le d$, bounded curvature holds on the task subspace, and no curvature exists outside it:
\begin{equation}
    \lambda_{\min}\, P_{S_j} \;\preceq\; H_j \;\preceq\; \lambda_{\max}\, P_{S_j} .
\end{equation}
\end{assumption}

Under \cref{ass:alignment}, the $H_j$
\rev{remain simultaneously diagonalizable (possibly with zero eigenvalues), so every} $S_j$, and every sum and complement of them, is spanned by the common eigenbasis. Each agent updates with samples from its within-task sampling distribution $\mathcal{D}^k$ (\cref{ss:JTL-ITL}); we write
\begin{equation}
    H_{\text{dyn}} := \mathbb{E}_{x \sim \mathcal{D}^k}\big[xx^\top\big] ,
    \qquad
    R := \mathrm{range}(H_{\text{dyn}}) ,
\end{equation}
and call $R$ the agent's \emph{driving range} at task $k$: $R^{\text{ITL}}_k = S_k$ for the ITL agent, and $R^{\text{JTL}}_k = S_1 + \dots + S_k$ (the subspace sum, i.e., the span of $S_1 \cup \dots \cup S_k$) for the JTL agent, whose mixture covariance is $\bar{H}_k := \sum_{j \le k} w_j^k H_j$ with $\mathrm{range}(\bar{H}_k) = R^{\text{JTL}}_k$. The corresponding \emph{frozen sets} are the orthogonal complements
\begin{equation}
    N^{\text{ITL}}_k := S_k^{\perp},
    \qquad\qquad
    N^{\text{JTL}}_k := \big( S_1 + \dots + S_k \big)^{\perp} = \bigcap_{j \le k} S_j^{\perp} ,
\end{equation}
of dimensions $d - r_k$ and $d - \dim(S_1 + \dots + S_k)$ respectively: the ITL frozen set is redrawn at every task and never shrinks, while the JTL frozen set is non-increasing and generically empty after $k^\star = \lceil d/r \rceil$ tasks (equal ranks $r$).

\begin{lemma}[Automatic invariance]
\label{lemma:auto_invariance}
Let $R \supseteq S_k$ be the driving range of either agent at task $k$, with $N = R^{\perp}$. Then
\begin{equation}
    H_k\, P_R \;=\; P_R\, H_k \;=\; H_k ,
    \qquad\qquad
    H_k\, P_N \;=\; P_N\, H_k \;=\; 0 .
\end{equation}
In particular, the $H_k$-seminorm ignores frozen components: $\|P_R\,u + P_N\,w\|^2_{H_k} = \|P_R\,u\|^2_{H_k}$ for all $u, w$.
\end{lemma}
\begin{proof}
$N = R^{\perp} \subseteq S_k^{\perp} = \mathrm{null}(H_k)$, so $H_k P_N = 0$; transposing, $P_N H_k = 0$. Then $H_k P_R = H_k (I - P_N) = H_k$, and likewise $P_R H_k = H_k$. The seminorm identity follows by expanding the square: the cross and frozen terms carry a factor $H_k P_N = 0$.
\end{proof}

\begin{lemma}[Invariance of the null-space component]
\label{lemma:null_invariance}
Let the driving inputs satisfy $\mathbb{E}[xx^\top] = H_{\text{dyn}}$, and let $P_R, P_N$ project onto $\mathrm{range}(H_{\text{dyn}})$ and its complement. Then $x \in \mathrm{range}(H_{\text{dyn}})$ almost surely, every per-sample (hence every batch) gradient of the square loss lies in $\mathrm{range}(H_{\text{dyn}})$ almost surely, and consequently
\begin{equation}
    P_N\, \theta_n \;=\; P_N\, \theta_0 \qquad \text{for all } n \ge 0, \text{ almost surely.}
\end{equation}
\end{lemma}
\begin{proof}
For any $v \in \mathrm{null}(H_{\text{dyn}})$ we have $\mathbb{E}\big[\langle x, v\rangle^2\big] = v^\top H_{\text{dyn}}\, v = 0$, hence $\langle x, v \rangle = 0$ almost surely: every input lies in $\mathrm{range}(H_{\text{dyn}})$. The per-sample gradient of the square loss is
\begin{equation}
    \hat{g} \;=\; x\,(x^\top \theta - y) \;\in\; \mathrm{span}\{x\} \;\subseteq\; \mathrm{range}(H_{\text{dyn}}) \quad \text{a.s.},
\end{equation}
for any label $y$, and batch gradients are averages of per-sample gradients. Applying $P_N$ to the SGD update,
\begin{equation}
    P_N\,\theta_{n+1} \;=\; P_N\,\theta_n - \eta\, P_N\, \hat{g}_n \;=\; P_N\,\theta_n ,
\end{equation}
and the claim follows by induction.
\end{proof}

\begin{definition}[Product of a law and a point mass]
\label{def:product_law}
Under the orthogonal decomposition $\mathbb{R}^d = \mathrm{range}(H_{\text{dyn}}) \oplus \mathrm{null}(H_{\text{dyn}})$, for a probability measure $\mu$ supported on the range and a point $x$ in the null space, the \emph{product law} $\mu \otimes \delta_x$ is the law of $\theta = \theta^R + x$ with $\theta^R \sim \mu$; equivalently, $\big(\mu \otimes \delta_x\big)(A) = \mu\big(\{\theta^R : \theta^R + x \in A\}\big)$ for every Borel set $A \subseteq \mathbb{R}^d$.
\end{definition}

\begin{theorem}[Ridgeless stationary law and risk]
\label{theo:ridgeless}
Under the conditions of \cref{lemma:null_invariance}, the parameter process decomposes exactly as $\theta_n = P_R\,\theta_n + P_N\,\theta_0$. On $\mathrm{range}(H_{\text{dyn}})$ the restriction $H_{\text{dyn}}\big|_R \succ 0$ is positive definite with smallest eigenvalue $\gamma^{+}_{\min} > 0$, and every result of \cref{sec:quadratic_analysis} (\cref{lemma:mean_convergence,lemma:PCV,prop:exact_instability,theo:transient_potential,prop:transient_no_alignment,prop:instability_no_alignment}) applies to the restricted process with $\gamma_{\min}$ replaced by $\gamma^{+}_{\min}$. The stationary law is the product
\begin{equation}
    Q_\infty \;=\; Q^{R}_\infty \,\otimes\, \delta_{P_N \theta_0} ,
\end{equation}
where $Q^R_\infty$ is the range-space stationary distribution, with mean $\mu^R_\infty$ and covariance $\Sigma^R_\infty$ independent of the initialization. For any evaluation task with $S_k \subseteq \mathrm{range}(H_{\text{dyn}})$, the stationary risk carries no contribution from the frozen component:
\begin{equation}
\label{eq:ridgeless_risk}
    \mathcal{L}^k(Q_\infty)
    \;=\; c_k
    \;+\; \underbrace{\frac{1}{2}\,\big\| \mu^R_\infty - \theta_\star^k \big\|^2_{H_k}}_{\text{Bias}}
    \;+\; \underbrace{\frac{1}{2}\,\mathrm{Tr}\big(H_k\, \Sigma^R_\infty\big)}_{\text{Variance}} .
\end{equation}
In particular, no term depends on $\theta_0$: converged performance on the current task is initialization-independent, and history can affect the current task only through the Transient.
\end{theorem}
\begin{proof}
The proof proceeds in three steps: we separate the dynamics along the range and the null space, apply the full-rank theory on the range, and assemble the stationary risk.

\textbf{Step 1.} By \cref{lemma:null_invariance}, every gradient lies in $\mathrm{range}(H_{\text{dyn}})$ almost surely. Applying $P_N$ and $P_R$ to the SGD update therefore splits it into two independent components:
\begin{equation}
    P_N\,\theta_{n+1} = P_N\,\theta_n = P_N\,\theta_0 ,
    \qquad
    P_R\,\theta_{n+1} = P_R\,\theta_n - \eta\, \hat{g}_n .
\end{equation}
Taking means and covariances of the two components: the mean satisfies $P_N\,\mu_n = P_N\,\theta_0$ for all $n$, and the fluctuations $\tilde\theta_n = \theta_n - \mu_n$ are supported on the range, so
\begin{equation}
    \mu_n = P_R\,\mu_n + P_N\,\theta_0 ,
    \qquad
    \Sigma_n = P_R\, \Sigma_n\, P_R .
\end{equation}

\textbf{Step 2.} The restriction $H_{\text{dyn}}\big|_R$ is symmetric positive definite with smallest eigenvalue $\gamma^{+}_{\min} > 0$, so the range component is exactly the process analyzed in \cref{sec:quadratic_analysis}, on a lower-dimensional space. By \cref{lemma:mean_convergence,lemma:PCV},
\begin{equation}
    P_R\,\mu_n \;\xrightarrow[n \to \infty]{}\; \mu^R_\infty ,
    \qquad
    \Sigma_n \;\xrightarrow[n \to \infty]{}\; \Sigma^R_\infty ,
\end{equation}
both limits independent of $\theta_0$, since the contraction forgets the range component of the initialization. Combining with Step 1, $Q_\infty = Q^R_\infty \otimes \delta_{P_N\theta_0}$, with stationary mean $\mu_\infty = \mu^R_\infty + P_N\,\theta_0$.

\textbf{Step 3.} Substituting into the risk decomposition \eqref{eq:risk_decomp} and applying the seminorm identity of \cref{lemma:auto_invariance},
\begin{equation}
    \|\mu_\infty - \theta_\star^k\|^2_{H_k}
    \;=\; \big\| P_R(\mu^R_\infty - \theta_\star^k) + P_N(\theta_0 - \theta_\star^k) \big\|^2_{H_k}
    \;=\; \big\| \mu^R_\infty - \theta_\star^k \big\|^2_{H_k} ,
\end{equation}
which yields \eqref{eq:ridgeless_risk}; the last equality uses $P_R \mu^R_\infty = \mu^R_\infty$ and $H_k P_N = 0$ once more.
\end{proof}

\begin{theorem}[The Transient Error in the ridgeless regime]
\label{theo:ridgeless_transient}
Let $\{(\gamma_i, u_i)\}_{i \le r}$ be the positive eigenpairs of $H_{\text{dyn}}$, and for a matrix $A$ write $\tilde A_{ij} := u_i^\top A\, u_j$ for its coordinates in this basis. Let
\begin{equation}
    \delta_N := \frac{1}{N}\sum_{n=0}^{N-1} \Big( \mathcal{L}^k(Q_n) - \mathcal{L}^k(Q_\infty) \Big) ,
    \qquad
    \Delta_0 := \mathbb{E}\big[(\theta_0 - \mu_\infty)(\theta_0 - \mu_\infty)^\top\big] - \Sigma^R_\infty .
\end{equation}
Then the frozen component contributes nothing to $\delta_N$, and
\begin{equation}
\label{eq:ridgeless_transient_formula}
    \delta_N \;\approx\; \frac{1}{2N\eta} \sum_{i,j=1}^{r} \frac{(\tilde H_{k})_{ij}\, (\tilde\Delta_0)_{ij}\,\big(1 - \rho_{ij}^{\,N}\big)}{\gamma_i + \gamma_j} ,
    \qquad
    \rho_{ij} := (1 - \eta\gamma_i)(1 - \eta\gamma_j) :
\end{equation}
the Transient Error of the ridgeless process is that of the full-rank process on the input range.
\end{theorem}
\begin{proof}
By \cref{theo:ridgeless}, the mean error at step $n$ has a constant null component, annihilated by the evaluation Hessian (\cref{lemma:auto_invariance}):
\begin{equation}
\label{eq:ridgeless_transient_cancel}
    \mathcal{L}^k(Q_n) - \mathcal{L}^k(Q_\infty)
    \;=\; \frac{1}{2}\Big( \big\|\mu^R_n - \theta_\star^k\big\|^2_{H_k} - \big\|\mu^R_\infty - \theta_\star^k\big\|^2_{H_k} \Big)
    + \frac{1}{2}\,\mathrm{Tr}\big( H_k (\Sigma_n - \Sigma^R_\infty) \big) .
\end{equation}

Every quantity in \eqref{eq:ridgeless_transient_cancel} belongs to the range-restricted process, which by \cref{theo:ridgeless} is the full-rank process of \cref{sec:quadratic_analysis} on an $r$-dimensional space. Averaging over the horizon, as in the proof of \cref{theo:transient_potential},
\begin{equation}
    \delta_N \;=\; \frac{1}{2N} \sum_{n=0}^{N-1} \mathrm{Tr}\big( H_k\, \Delta_n \big) ,
    \qquad
    \Delta_n := \mathbb{E}\big[(\theta_n - \mu_\infty)(\theta_n - \mu_\infty)^\top\big] - \Sigma^R_\infty :
\end{equation}
exactly the Transient Error of the restricted process.

Repeating the proof of \cref{prop:transient_no_alignment} on the restricted process: the deviation propagates homogeneously and decouples entrywise in the basis $\{u_i\}$,
\begin{equation}
    \Delta_n \;\approx\; (I - \eta H_{\text{dyn}})^n\, \Delta_0\, (I - \eta H_{\text{dyn}})^n ,
    \qquad
    (\tilde\Delta_n)_{ij} \;=\; \rho_{ij}^{\,n}\, (\tilde\Delta_0)_{ij} ,
\end{equation}
and the geometric summation with the small-step expansion of the denominator,
\begin{equation}
    \sum_{n=0}^{N-1} \rho_{ij}^{\,n} \;=\; \frac{1 - \rho_{ij}^{\,N}}{1 - \rho_{ij}} ,
    \qquad
    1 - \rho_{ij} \;\approx\; \eta\,(\gamma_i + \gamma_j) ,
\end{equation}
yields \eqref{eq:ridgeless_transient_formula}.
\end{proof}

\rev{We now instantiate the two building blocks exactly for the JTL and ITL agents of \cref{ss:JTL-ITL}, following the full-rank derivation of \cref{prop:exact_instability} term by term:} first the Instability, then the Transient Errors. Both agents share the lifetime initialization $\theta_0$; the ITL agent restarts at $\theta_0$ at every task, while the JTL agent continues from its stage-$(k-1)$ stationary law. We write $v_k := \theta_0 - \theta_\star^k$.

\begin{theorem}[Exact Instability in the ridgeless regime]
\label{theo:ridgeless_instability_exact}
Assume \cref{ass:quadratic}, \cref{ass:low_rank}, \cref{ass:alignment}, and \cref{ass:small_step}, with uniform observation noise $\sigma^2$. Define the restricted mixture optimum, its task gradients, and the pseudo relative curvature,
\begin{equation}
\label{eq:pseudo_optimum}
    \tilde\theta^{\,k}_\star \;:=\; \bar{H}_k^{+} \Big( \sum_{j=1}^k w_j^k H_j\, \theta_\star^j \Big) \;+\; P_{N^{\text{JTL}}_k}\,\theta_0 ,
    \qquad
    \tilde{g}_j \;:=\; \nabla\mathcal{L}^j\big(\tilde\theta^{\,k}_\star\big) ,
    \qquad
    \tilde{\mathcal{R}}^k \;:=\; \bar{H}_k^{+} H_k ,
\end{equation}
where $\bar{H}_k^{+}$ denotes the Moore--Penrose pseudoinverse. Then the full-rank decomposition of \cref{prop:exact_instability} survives with the inverse replaced by the pseudoinverse:
    \begin{equation}
    \label{eq:ridgeless_instability_exact}
        \mathcal{I}^k \;\approx\;
        \underbrace{\frac{1}{2}\,\big\| \tilde\theta^{\,k}_\star - \theta_\star^k \big\|^2_{H_k}}_{\text{Approximation Bias}}
        \;+\; \frac{\eta}{4B} \underbrace{\sum_{j=1}^k w_j^k\, \big\| \tilde{g}_j \big\|^2_{\tilde{\mathcal{R}}^k}}_{\text{Task Diversity}}
        \;\;\ge\; 0 .
    \end{equation}
\end{theorem}
\begin{proof}
\emph{Step 1 (ITL stationary law).} The ITL driving covariance is $H_k$ itself. Its stationary mean minimizes the task risk over $\theta_0 + S_k$: writing the normal equations on the range and solving with the pseudoinverse,
\begin{equation}
    P_{S_k}\big( H_k\,\mu - H_k\,\theta_\star^k \big) = 0
    \qquad\Longleftrightarrow\qquad
    \mu^{\text{ITL}}_{k,\infty} \;=\; H_k^{+} H_k\, \theta_\star^k + P_{N^{\text{ITL}}_k}\theta_0
    \;=\; P_{S_k}\,\theta_\star^k + P_{N^{\text{ITL}}_k}\theta_0 .
\end{equation}
Its bias on task $k$ vanishes ($\mu^{\text{ITL}}_{k,\infty} - \theta_\star^k = P_{N^{\text{ITL}}_k} v_k$, annihilated by $H_k$; \cref{lemma:auto_invariance}), and by \cref{lemma:GNC} restricted to $S_k$, its stationary covariance and variance floor are
\begin{equation}
    \Sigma^{\text{ITL}}_\infty = \frac{\eta\sigma^2}{2B}\, P_{S_k} ,
    \qquad
    \frac{1}{2}\,\mathrm{Tr}\big( H_k\, \Sigma^{\text{ITL}}_\infty \big) = \frac{\eta\sigma^2}{4B}\,\mathrm{Tr}(H_k) .
\end{equation}

\emph{Step 2 (JTL stationary law and Instability).} The JTL driving covariance is $\bar{H}_k$, with range $R^{\text{JTL}}_k$. The restricted mixture optimum solves the normal equations on the range:
\begin{equation}
    \bar{H}_k\,\theta \;=\; \sum_{j=1}^k w_j^k H_j\,\theta_\star^j
    \;\;\text{on } R^{\text{JTL}}_k
    \qquad\Longleftrightarrow\qquad
    P_{R^{\text{JTL}}_k}\theta \;=\; \bar{H}_k^{+} \Big(\sum_{j=1}^k w_j^k H_j\,\theta_\star^j\Big) ,
\end{equation}
which together with the frozen component (\cref{theo:ridgeless}) gives \cref{eq:pseudo_optimum} with $\bar{H}_k^{-1} \mapsto \bar{H}_k^{+}$. Its task gradients sum to zero,
\begin{equation}
    \sum_{j=1}^k w_j^k\, \tilde{g}_j
    \;=\; \bar{H}_k\,\tilde\theta^{\,k}_\star - \sum_{j=1}^k w_j^k H_j\,\theta_\star^j
    \;=\; 0 ,
\end{equation}
using $H_j P_{N^{\text{JTL}}_k} = 0$ for every $j \le k$.

Next we compute the stationary covariance. \Cref{lemma:MT_GNC} expresses the JTL gradient-noise covariance at an optimum as the observation-noise floor plus the second moment of the per-task gradients; it requires only that the mean task gradient vanish, so it applies at $\tilde\theta^{\,k}_\star$:
\begin{equation}
    C^{\text{JTL}}_\infty \;=\; \frac{1}{B}\Big( \sigma^2\, \bar{H}_k \;+\; \sum_{j=1}^k w_j^k\, \tilde{g}_j \tilde{g}_j^\top \Big) .
\end{equation}
The stationary covariance solves the restricted Lyapunov equation $\bar{H}_k \Sigma^{\text{JTL}}_\infty + \Sigma^{\text{JTL}}_\infty \bar{H}_k = \eta\, C^{\text{JTL}}_\infty$ on $R^{\text{JTL}}_k$, and the trace identity used in the proof of \cref{prop:exact_instability} applies on the range, with $\bar{H}_k^{-1} \mapsto \bar{H}_k^{+}$ (the Hessians commute by \cref{ass:alignment}):
\begin{equation}
    \frac{1}{2}\,\mathrm{Tr}\big( H_k\, \Sigma^{\text{JTL}}_\infty \big)
    \;=\; \frac{\eta}{4}\,\mathrm{Tr}\big( \tilde{\mathcal{R}}^k C^{\text{JTL}}_\infty \big)
    \;=\; \frac{\eta\sigma^2}{4B}\,\mathrm{Tr}\big( \tilde{\mathcal{R}}^k \bar{H}_k \big)
    \;+\; \frac{\eta}{4B} \sum_{j=1}^k w_j^k\, \big\|\tilde{g}_j\big\|^2_{\tilde{\mathcal{R}}^k} .
\end{equation}
For the first term, \cref{lemma:auto_invariance} gives
\begin{equation}
    \mathrm{Tr}\big( \tilde{\mathcal{R}}^k \bar{H}_k \big) \;=\; \mathrm{Tr}\big( H_k\, \bar{H}_k^{+} \bar{H}_k \big) \;=\; \mathrm{Tr}\big( H_k\, P_{R^{\text{JTL}}_k} \big) \;=\; \mathrm{Tr}( H_k ) :
\end{equation}
the noise floors of the two agents cancel exactly, as in the full-rank case. The second term is the restricted Task Diversity. Since the ITL bias vanishes (Step 1), the bias difference is the JTL bias alone, $\frac{1}{2}\|\tilde\theta^{\,k}_\star - \theta_\star^k\|^2_{H_k}$,
\rev{and together these terms give \cref{eq:ridgeless_instability_exact}.}
\end{proof}

\Cref{eq:ridgeless_instability_exact} shows that $\mathcal{I}^k \ge 0$: it is the Approximation Bias and Task Diversity of \cref{prop:exact_instability}, evaluated at the pseudoinverse optimum.
\rev{Over-parameterization alone therefore does not cause negative Instability.}

\begin{theorem}[Exact Transient comparison in the ridgeless regime]
\label{theo:ridgeless_transients_exact}
In the setting of \cref{theo:ridgeless_instability_exact}, split the current task's directions into the covered part and the newly unlocked part,
\begin{equation}
    C_k := S_k \cap R^{\text{JTL}}_{k-1} ,
    \qquad
    V_k := R^{\text{JTL}}_k \cap \big(R^{\text{JTL}}_{k-1}\big)^{\perp} ,
    \qquad
    S_k = C_k \oplus V_k .
\end{equation}
In the common eigenbasis of \cref{ass:alignment}, write $h^{(k)}_i$ and $\bar\gamma_i$ for the eigenvalues of $H_k$ and $\bar{H}_k$, and define the \emph{rate response}
\begin{equation}
    \chi_N(\gamma) \;:=\; \frac{1 - (1 - \eta\gamma)^{2N}}{\gamma} ,
\end{equation}
a strictly decreasing function of $\gamma$; on the newly unlocked directions $\bar\gamma_i = w^k_k\, h^{(k)}_i < h^{(k)}_i$, hence $\chi_N(\bar\gamma_i) > \chi_N(h^{(k)}_i)$ there. The difference of the two Transient Errors decomposes into three groups:

\begin{align}
    \Delta\delta_N &\;:=\; \delta^{\text{ITL}}_N - \delta^{\text{JTL}}_N \notag \\
    &\;\approx\; \frac{1}{4N\eta}\, \underbrace{\sum_{i \,\in\, C_k} h^{(k)}_i \Big[ \chi_N\big(h^{(k)}_i\big) \big(\theta_0 - \theta_\star^k\big)_i^2 - \chi_N\big(\bar\gamma_i\big) \big(\tilde\theta^{\,k-1}_\star - \tilde\theta^{\,k}_\star\big)_i^2 \Big]}_{\text{coverage gain}} \label{eq:transient_comparison} \\
    &\;\quad - \frac{1}{4N\eta}\, \underbrace{\sum_{i \,\in\, V_k} h^{(k)}_i\, \big(\theta_0 - \theta_\star^k\big)_i^2 \Big( \chi_N\big(\bar\gamma_i\big) - \chi_N\big(h^{(k)}_i\big) \Big)}_{\text{dilution penalty}} \notag \\
    &\;\quad - \frac{1}{4N\eta}\, \underbrace{\sum_{i \,\in\, S_k} h^{(k)}_i \Big( \chi_N\big(h^{(k)}_i\big)\, \tfrac{\eta\sigma^2}{2B} + \chi_N\big(\bar\gamma_i\big)\, \big(\Delta\Sigma_k\big)_{ii} \Big)}_{\text{variance difference}} ,
\end{align}
where $\Delta\Sigma_k := \Sigma^{\text{JTL},k-1}_\infty - \Sigma^{\text{JTL},k}_\infty$ is the difference of consecutive JTL stationary covariances. The dilution penalty is bounded exactly,
\begin{equation}
\label{eq:dilution_sandwich}
    1 \;\le\; \frac{\chi_N\big(\bar\gamma_i\big)}{\chi_N\big(h^{(k)}_i\big)} \;\le\; \frac{1}{w^k_k}
    \qquad (i \in V_k) ,
\end{equation}
with the lower limit attained as $N \to 0$ and the upper limit as $N \to \infty$.
\end{theorem}

\begin{proof}
\textbf{Step 1.} The ITL agent restarts at the deterministic $\theta_0$, so its initial mean deviation is the whole in-task error and its initial covariance is zero:
\begin{equation}
    \mu_0 - \mu^{\text{ITL}}_{k,\infty} \;=\; v_k - P_{N^{\text{ITL}}_k} v_k \;=\; P_{S_k}\, v_k ,
    \qquad
    \Sigma_0 = 0
    \qquad\Longrightarrow\qquad
    \Delta_0^{\text{ITL}} \;=\; P_{S_k}\, v_k\, v_k^\top P_{S_k} \;-\; \frac{\eta\sigma^2}{2B}\, P_{S_k} .
\end{equation}
The JTL agent enters task $k$ at its stage-$(k-1)$ stationary law, so its initial deviations are differences of consecutive stationary quantities. Write $m_k := \mu^{\text{JTL}}_{k-1,\infty} - \mu^{\text{JTL}}_{k,\infty}$ for the stationary mean shift; using $P_{R^{\text{JTL}}_{k}} = P_{R^{\text{JTL}}_{k-1}} + P_{V_k}$ and $P_{N^{\text{JTL}}_{k-1}} - P_{N^{\text{JTL}}_{k}} = P_{V_k}$,
\begin{equation}
    m_k \;=\; P_{R^{\text{JTL}}_{k-1}}\big( \tilde\theta^{\,k-1}_\star - \tilde\theta^{\,k}_\star \big) \;+\; P_{V_k}\big( \theta_0 - \tilde\theta^{\,k}_\star \big) ,
    \qquad
    \Delta_0^{\text{JTL}} \;=\; m_k\, m_k^\top + \Delta\Sigma_k .
\end{equation}
On $V_k$ only task $k$ carries curvature, so the normal equations of the restricted mixture optimum reduce to the current task's:
\begin{equation}
    \bar\gamma_i \;=\; w^k_k\, h^{(k)}_i ,
    \qquad
    \big( \tilde\theta^{\,k}_\star \big)_i
    \;=\; \frac{w^k_k\, h^{(k)}_i\, \big(\theta_\star^k\big)_i}{w^k_k\, h^{(k)}_i}
    \;=\; \big(\theta_\star^k\big)_i
    \qquad (i \in V_k) ,
\end{equation}
whence $\big(\theta_0 - \tilde\theta^{\,k}_\star\big)_i = (v_k)_i$ on $V_k$: the cold-start component.

\textbf{Step 2.} Both $\Delta_0$'s feed \cref{eq:ridgeless_transient_formula}. In the common eigenbasis $\tilde{H}_k$ is diagonal (zero outside $S_k$), so only diagonal terms survive, and each retains the factor $\frac{h^{(k)}_i}{\gamma_i}\big(1 - (1 - \eta\gamma_i)^{2N}\big) = h^{(k)}_i\, \chi_N(\gamma_i)$ with the respective driving spectrum:
\begin{equation}
\label{eq:transients_chi}
    \delta^{\text{ITL}}_N \;\approx\; \frac{1}{4N\eta} \sum_{i \,\in\, S_k} h^{(k)}_i\, \chi_N\big(h^{(k)}_i\big) \Big( (v_k)_i^2 - \tfrac{\eta\sigma^2}{2B} \Big) ,
    \qquad
    \delta^{\text{JTL}}_N \;\approx\; \frac{1}{4N\eta} \sum_{i \,\in\, S_k} h^{(k)}_i\, \chi_N\big(\bar\gamma_i\big) \Big( (m_k)_i^2 + \big(\Delta\Sigma_k\big)_{ii} \Big) .
\end{equation}

\textbf{Step 3.} Split each bracket of \cref{eq:transients_chi} into its mean part, $(v_k)_i^2$ for the ITL agent and $(m_k)_i^2$ for the JTL agent, and its variance part, $-\tfrac{\eta\sigma^2}{2B}$ and $(\Delta\Sigma_k)_{ii}$ respectively. By Step 1 the mean parts coincide on $V_k$, $(m_k)_i = (v_k)_i$, and on $C_k$ they are the two displacements appearing in the coverage gain. Grouping the mean parts over $C_k$ and $V_k$, and the variance parts over $S_k$, yields \cref{eq:transient_comparison}.

\textbf{Step 4.} The map $x \mapsto 1 - (1-x)^{2N}$ is concave and vanishes at $0$, so its secant slope through the origin decreases: $\chi_N$ is strictly decreasing, which proves the claim $\chi_N(\bar\gamma_i) > \chi_N(h^{(k)}_i)$ on $V_k$ and the lower bound in \cref{eq:dilution_sandwich}. For the upper bound, the numerator of $\chi_N$ is increasing in $\gamma$, so
\begin{equation}
    \chi_N\big(\bar\gamma_i\big) \;=\; \frac{1 - (1 - \eta\bar\gamma_i)^{2N}}{\bar\gamma_i}
    \;\le\; \frac{1 - (1 - \eta h^{(k)}_i)^{2N}}{w^k_k\, h^{(k)}_i}
    \;=\; \frac{\chi_N\big(h^{(k)}_i\big)}{w^k_k} .
\end{equation}
As $N \to 0$, $\chi_N(\gamma) \to 2N\eta$ for every $\gamma$, so the ratio tends to $1$; as $N \to \infty$, $\chi_N(\gamma) \to 1/\gamma$, so the ratio tends to $h^{(k)}_i / \bar\gamma_i = 1/w^k_k$.
\end{proof}

Each group of \cref{eq:transient_comparison}
\rev{has a direct interpretation.} The coverage gain compares what the two agents must learn on the covered directions: the ITL agent relearns the full displacement from $\theta_0$, while the JTL agent only tracks the drift of the restricted optimum. The dilution penalty compares how fast they learn on the new directions: the signal is the same, but the JTL agent learns it at the diluted rate, and by \cref{eq:dilution_sandwich} it pays up to $1/w^k_k = b_k / N_k$ times more. The variance difference collects the covariance potentials.

Two consequences follow. Relative to the full-rank pipeline, where $V_k = \emptyset$ and only the coverage gain survives, memory's transient advantage is strictly smaller, by exactly the dilution penalty.
\rev{Both Transients are also supported on the current task's directions alone, so the comparison is local to $S_k$.}

\rev{Combining the two theorems through the TE decomposition (\cref{lemma:TE-decomposition}) gives the exact Transfer Efficiency, $\mathrm{TE}^k \approx -\,\mathcal{I}^k + \Delta\delta_N$, with the two terms given by \cref{eq:ridgeless_instability_exact} and \cref{eq:transient_comparison}. The decomposition requires only task-wise convergence, which holds for each restricted process (\cref{theo:ridgeless}).}

\paragraph{Comparison with the full-rank pipeline.}
\begin{center}
\footnotesize
\setlength{\tabcolsep}{4pt}
\renewcommand{\arraystretch}{1.25}
\begin{tabular}{@{}lll@{}}
\toprule
 & \textbf{Full rank (\cref{ass:bounded_curvature})} & \textbf{Ridgeless (\cref{ass:low_rank})} \\
\midrule
Mixture optimum & $\bar{H}_k^{-1} \sum_j w_j^k H_j\, \theta_\star^j$ & $\bar{H}_k^{+} \sum_j w_j^k H_j\, \theta_\star^j + P_{N^{\text{JTL}}_k}\theta_0$ \\
Stationary law & unique, initialization-independent & product law, frozen component $P_N \theta_0$ \\
Relative curvature & $\mathcal{R}^k = \bar{H}_k^{-1} H_k$ & $\tilde{\mathcal{R}}^k = \bar{H}_k^{+} H_k$ \\
Instability & Bias $+$ Diversity, $\mathcal{I}^k \ge 0$ & same form and sign, at the pseudoinverse optimum \\
JTL Transient & tracking term only ($V_k = \emptyset$ for $k \ge 2$) & tracking $+$ cold start on $V_k$, rates $\bar\gamma_i = w_k^k h^{(k)}_i$ \\
Contraction floor & $\gamma_{\min} \ge \lambda_{\min}$ & $\gamma^{+}_{\min}$ on the range; vanishing on late-unlocked directions \\
\bottomrule
\end{tabular}
\end{center}

\rev{The table also shows which of the paper's general statements survive without} \cref{ass:bounded_curvature}. \Cref{lemma:TE-decomposition,lem:critical_TD} are untouched: they require only task-wise convergence, which \cref{theo:ridgeless} provides on the explored subspace. The exact formulas survive under the single substitution $\bar{H}_k^{-1} \mapsto \bar{H}_k^{+}$. The asymptotic conclusions of \cref{sec:mt_asymptotics} persist qualitatively, because the evaluation Hessian weighs only the explored subspace, on which the full-rank machinery applies; their constants change ($\gamma_{\min} \mapsto \gamma^{+}_{\min}$), and the dilution of newly unlocked directions adds a lower-order transient correction.
\rev{Two phenomena are genuinely new: the frozen memory,} which makes the stationary law history-dependent without affecting the current task, and the dilution of the learning rates on newly unlocked directions. At full rank the JTL Transient
\rev{consists of the tracking term alone, which is the regime computed} in the Continuous Drift analysis (\cref{ss:continuous-drift}); the ridgeless case adds the cold-start term on top of it.

\paragraph{Takeaway.} Dropping \cref{ass:bounded_curvature} does not change the conclusions. Over-parameterization by itself does not produce negative Instability, and it does not break the retention trade-off.
\rev{It changes the location and the currency of the trade-off:} the benefit of memory moves into the Transient, where it equals the coverage of the current task, and its price is the dilution of the learning rates on the directions that remain to be learned. \rebuttalB{\rev{We validate these predictions in simulation in \cref{apdx:assumption-studies}: as the shared support shrinks, the transfer boundary of the phase diagram deforms from duration-limited to drift-limited, \revp{exactly} as the exact expressions predict} (\cref{fig:assumption-lowrank}).}

\paragraph{Relation to the literature.} An in-depth study of the learning dynamics in the ridgeless over-parameterized regime
\rev{is beyond the scope of this work; such a study can build directly on a rich existing literature.} The implicit-bias and benign-overfitting lines characterize the minimum-norm interpolants that cold-started agents converge to, and their risk \citep{zhang2017rethinking, gunasekar2018characterizing, bartlett2020benign, advani2020high}; the sequential, task-by-task picture is developed in the alternating-projections analysis of \citet{evron2022catastrophic}; and the role of task similarity, including the subspace parametrizations natural to this regime, is studied by \citet{lee2021continual, lin2023theory, evron2024joint} and, through the lens of NTK overlap, by \citet{doan2021theoretical}.
}

\newpage

\section{Empirical Setup}
\label{Apdx: Exps}

\major{\paragraph{Roadmap.} \Cref{apsec:simulations} describes the simulations behind \cref{sec:quadratic-losses,ssec:window-poc}. \Cref{Apdx:datasets-networks-conf,apsec:synthetic} describe the real-world (CLEAR, MD5, Meta-World) and synthetic (permuted and shuffled CIFAR-10) benchmarks used in \cref{ssec:empirical-validation}, and \cref{Apdx:training} gives the training procedures, networks, and hyperparameters. \Cref{Apx:add-results} collects additional results: pairwise transfer in MT10, the Instability proxies (including why a measured Instability can be negative), simulation studies that remove one assumption at a time, deep linear scaling experiments, and real-time training trajectories.}

\subsection{Simulations}
\label{apsec:simulations}

\rev{All simulations are implemented in JAX, which lets us vectorize training over multiple random seeds and hyperparameter grids.}

\subsubsection{Shallow Linear Experiments}

To validate our exact analytical derivations for the continuous drift setting, we simulate a sequence of linear tasks tracking a random walk. The input dimension is $d=10$. Data inputs are sampled as $x \sim \mathcal{N}(0, I)$, and task labels are generated by a teacher $\theta_\star^k$ as $y = x^\top \theta_\star^k + \xi$, where $\xi \sim \mathcal{N}(0, \sigma_\xi^2)$. The teacher undergoes a random walk such that $\theta_\star^k = \theta_\star^{k-1} + \Delta_k$, with $\Delta_k \sim \mathcal{N}(0, \frac{\sigma_\Delta^2}{d} I)$.

We sweep over the two main axes of our phase diagram:
\begin{itemize}
    \item \textbf{Task Duration ($N$):} Swept geometrically.
    \item \textbf{Drift Magnitude ($\sigma_\Delta^2$):} Swept logarithmically.
\end{itemize}
The models are trained using stochastic gradient descent (SGD) with a constant learning rate (e.g., $\eta=0.1$). \rev{We compute the exact Transfer Efficiency between the Independent-Task Learning (ITL) and Joint-Task Learning (JTL) agents, averaged over $K$ tasks and $40$ random seeds. To verify our non-asymptotic bounds, we overlay the theoretical breakeven boundary on the empirical phase transitions.}

\subsubsection{Deep Linear Experiments}

To analyze the scaling dynamics and the impact of the loss landscape, we extend the experiments to deep linear networks. We use a teacher-student setup with a depth-3 linear network, an input dimension of $d=10$, and an output dimension of $1$. The hidden width is swept across $m \in \{32, 64, 128, 256, 512, 1024\}$.

The teacher network's weights are initialized randomly and undergo a continuous random walk to induce task drift. We fix the total drift magnitude to $\sigma_\Delta^2 = 0.05$. \rev{We train only three tasks, each lasting $N=500$ steps, to isolate the initial transient and instability dynamics.} We use a batch size of $B=32$ and a label noise standard deviation of $\sigma_\xi=0.01$.

\rev{The ITL agent continues training on the incoming data stream, while the JTL agent uses exact experience replay and samples its batches uniformly from the current and all previously seen tasks.}

\rev{We evaluate two initialization and scaling regimes. Let the hidden width be $m$, and let the three trainable matrices be} $U \in \mathbb{R}^{d \times m}$, $V \in \mathbb{R}^{m \times m}$, and $a \in \mathbb{R}^{m \times 1}$, with all entries initialized i.i.d. as $\mathcal{N}(0,1)$. For an input $x \in \mathbb{R}^d$, the depth-3 linear network is implemented as
\[
    h_1(x) = \frac{xU}{\sqrt{d}}, \qquad
    h_2(x) = \frac{h_1(x)V}{\sqrt{m}}, \qquad
    f_\Theta(x) = \frac{h_2(x)a}{\sqrt{m}\,\gamma_m}.
\]
\rev{The input and hidden matrix multiplications therefore use fan-in normalization,} while the scalar factor $\gamma_m$ controls the output scale of the parametrization. The same parametrization is used for the teacher and the student networks.

The two regimes differ only in the output multiplier $\gamma_m$ and the learning-rate multiplier $s_\eta(m)$ applied to the base learning rate:
\begin{itemize}
    \item \textbf{NTK parametrization (NTP).} We set $\gamma_m = 1$ and $s_\eta(m)=1$, so the effective SGD learning rate is independent of width. This is the neural-tangent scaling used to approach lazy kernel dynamics in the wide-network limit: as $m$ grows, the network function changes mainly through the linearization around initialization, and the hidden representation moves by a vanishing amount \citep{Jacot18NTK,yang2021feature}.
    \item \textbf{Maximal Update parametrization ($\mu$P).} We set $\gamma_m=\sqrt{m}$ and multiply the base learning rate by $s_\eta(m)=m/m_0$, where the base width is $m_0=32$. \rev{Equivalently, compared with NTP, the output is scaled down by an additional factor $\sqrt{m}$ while the optimizer step size grows linearly with width. Under width scaling, hidden-layer updates then remain large enough to produce feature learning, and learning-rate choices transfer more stably across widths} \citep{yang2021feature,yang2021tensorprogramsv}.
\end{itemize}

We sweep the base learning rate logarithmically from $10^{-3}$ to $10^1$ across $15$ values. \rev{As expected, we observe hyperparameter transfer in $\mu$P: the optimal learning rate remains stable across widths, whereas it shifts significantly in NTP. For a fair comparison, we choose the ITL learning rate by the first task's performance and the JTL learning rate by the last task's performance (lowest loss).} All grid points are averaged over $5$ random seeds.

\rev{The figures below show the loss curves with the selected learning rate.}

\begin{figure}[h!]
    \centering
    \includegraphics[width=0.67\linewidth]{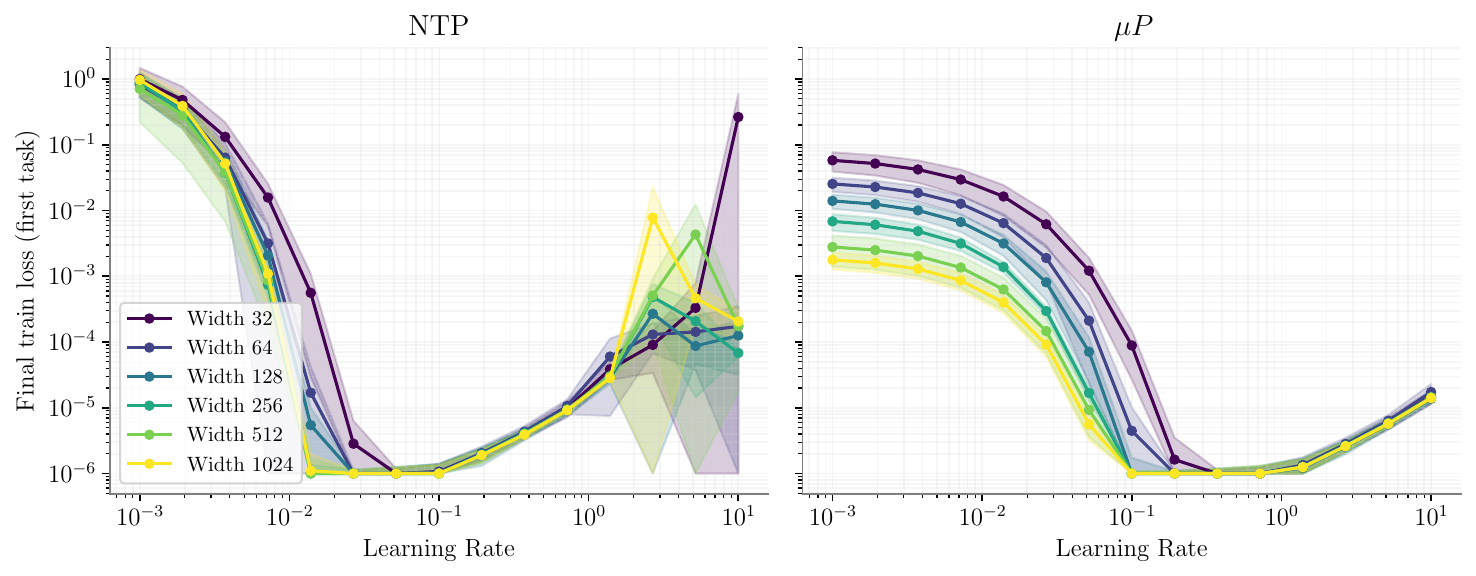}
    \caption{Learning rate sweeps across widths for ITL training in NTP and $\mu$P.}
    \label{fig:lr-sweeps-st}
\end{figure}

\begin{figure}[h!]
    \centering
    \includegraphics[width=0.67\linewidth]{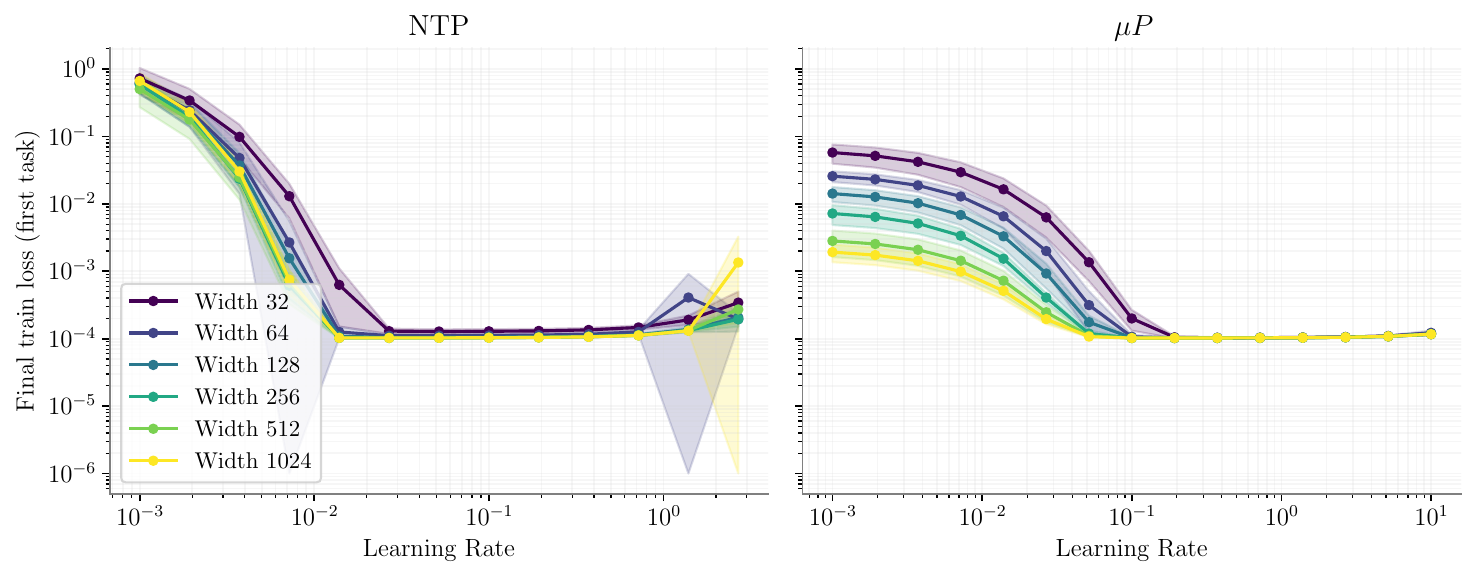}
    \caption{Learning rate sweeps for JTL training in NTP and $\mu$P.}
    \label{fig:lr-sweeps-mt}
    \end{figure}

\begin{figure}[h!]
    \centering
    \includegraphics[width=0.67\linewidth]{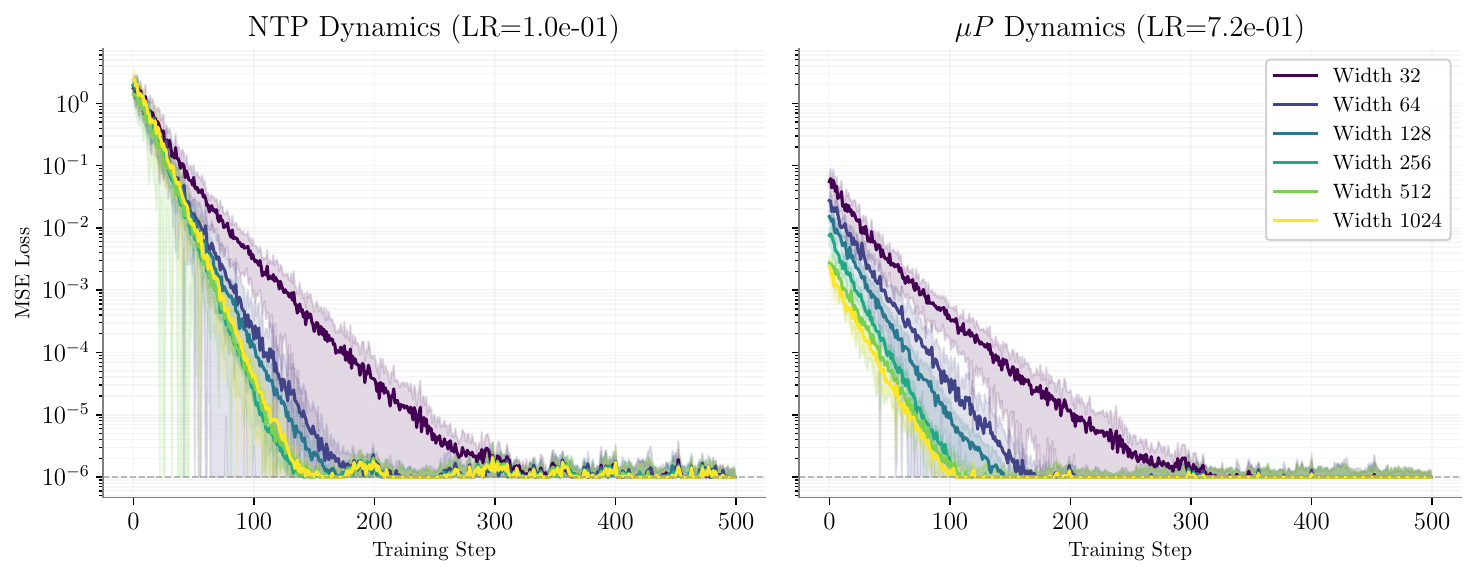}
    \caption{Loss curve for ITL training (first task) in NTP and $\mu$P across widths.}
    \label{fig:loss-curve-st}
\end{figure}

\begin{figure}[h!]
    \centering
    \includegraphics[width=0.67\linewidth]{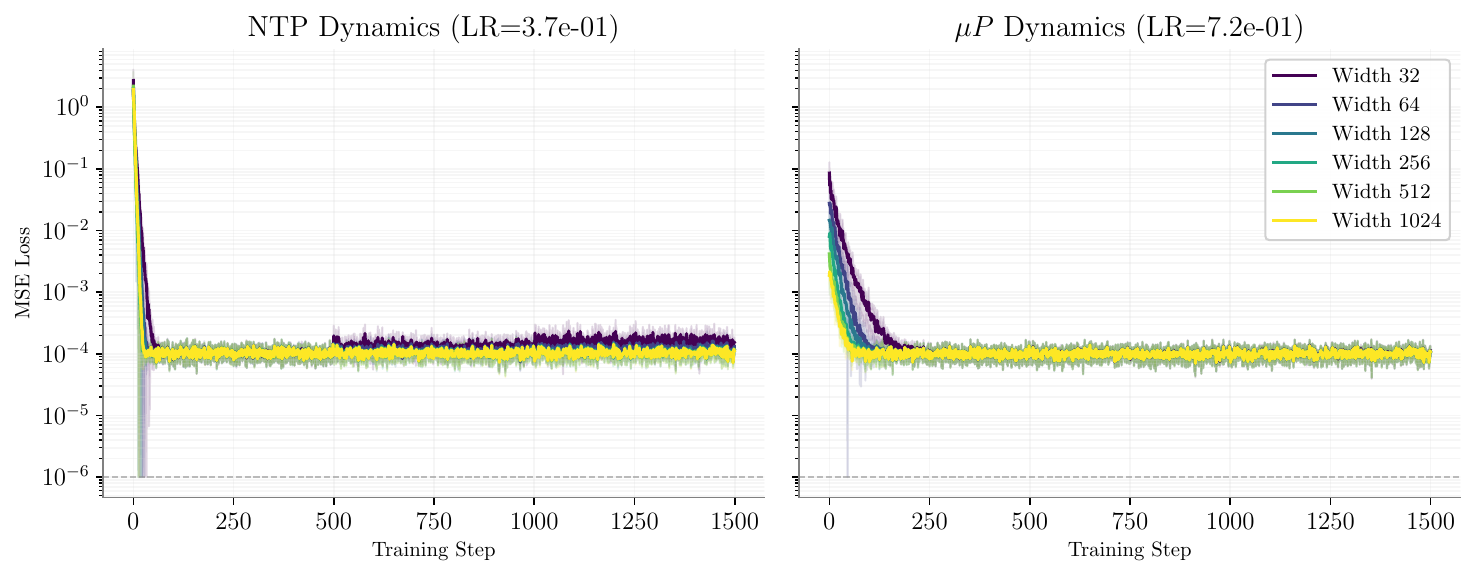}
    \caption{Loss curve for JTL training (three tasks) in NTP and $\mu$P across widths.}
    \label{fig:loss-curve-mt}
\end{figure}

\vspace{0.3cm}
\subsubsection{Window Predictive CL Algorithm Experiments}

To empirically validate our non-asymptotic derivations for the Window Predictive CL algorithm (\cref{sec:optimal_window}), we implement the Window agent in the shallow linear continuous drift environment. The Window agent optimizes a uniformly weighted mixture of the past $W$ tasks, $\mathcal{D}^k_W = \frac{1}{W} \sum_{j = k-W+1}^k \mathcal{P}^j$.

\rev{The Window agent experiments consist of two evaluations:}

\paragraph{Real-Time Trajectory Evaluation with $W^\star$:}
First, we test the predictive power of our theoretical optimal window formula (\cref{eq:optimal-window}). For a given task duration $N$, we compute the theoretical optimum $W^\star$ using the environment and optimizer parameters (e.g., drift magnitude $\sigma_\Delta^2$, learning rate $\eta$, batch size $B$, and input covariance $\Sigma_x$). \rev{We then simulate the full training trajectories of the ITL, JTL, and Window agents, with $W = W^\star$ for the Window agent. We average the test loss over multiple seeds and smooth the curves by averaging over the task duration $N$ to show the transient and stationary phases of adaptation.}

\paragraph{Transfer Efficiency Landscapes:}
\rev{Second, we map the region of positive transfer to show how the environment's non-stationarity limits the optimal memory horizon. We fix the task duration at several scales} (e.g., $N \in \{100, 500, 1000\}$) and perform a 2D grid search over:
\begin{itemize}
    \item \textbf{Window Size ($W$):} Swept across a discrete range of integer values.
    \item \textbf{Drift Magnitude ($\sigma_\Delta^2$):} Swept logarithmically.
\end{itemize}
For every point in this grid, we compute the exact Average Task Error ($\mathrm{ALE}^k$) over a long sequence of tasks (e.g., $K=100$) to bypass initial task transient effects. We then calculate the Transfer Efficiency (TE) relative to the ITL baseline. \rev{The resulting 2D heatmaps (\cref{fig:sim-W-TE}) empirically verify our claims that} (1) the area of positive transfer shrinks as task duration and drift increase, and (2) overly large window sizes introduce detrimental bias in highly non-stationary regimes, pulling the optimal $W$ toward $1$.

\begin{figure}[ht]
    \centering
    \includegraphics[width=0.32\linewidth]{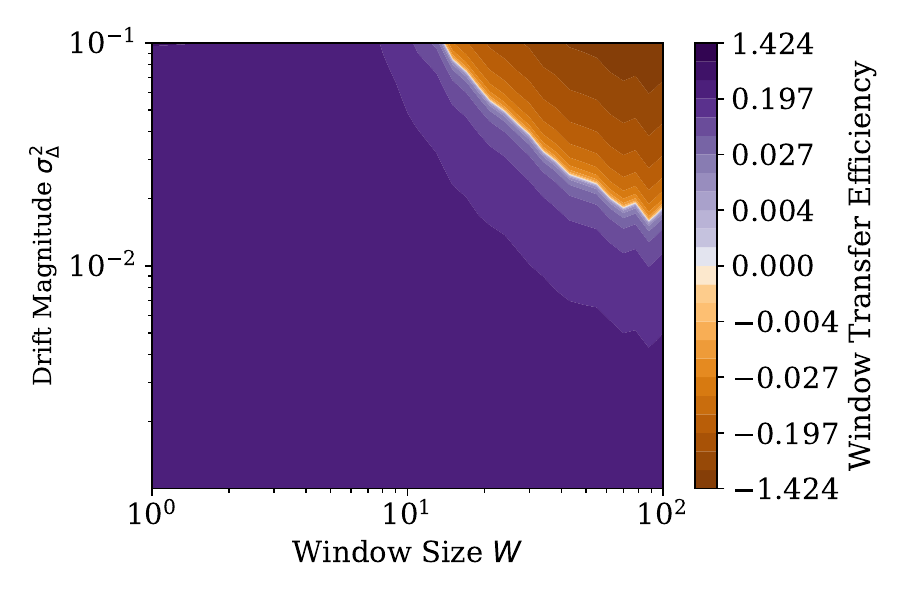}
    \includegraphics[width=0.32\linewidth]{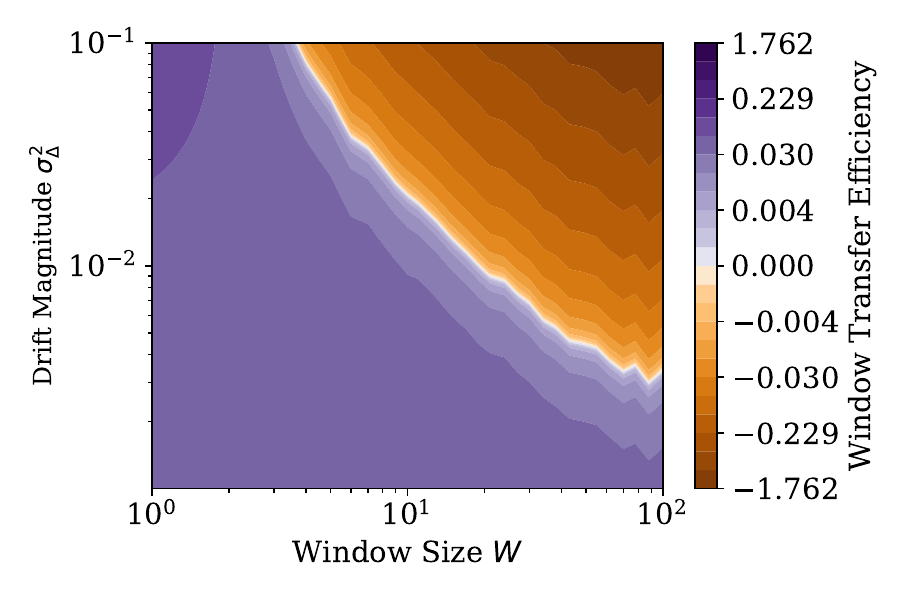}
    \includegraphics[width=0.32\linewidth]{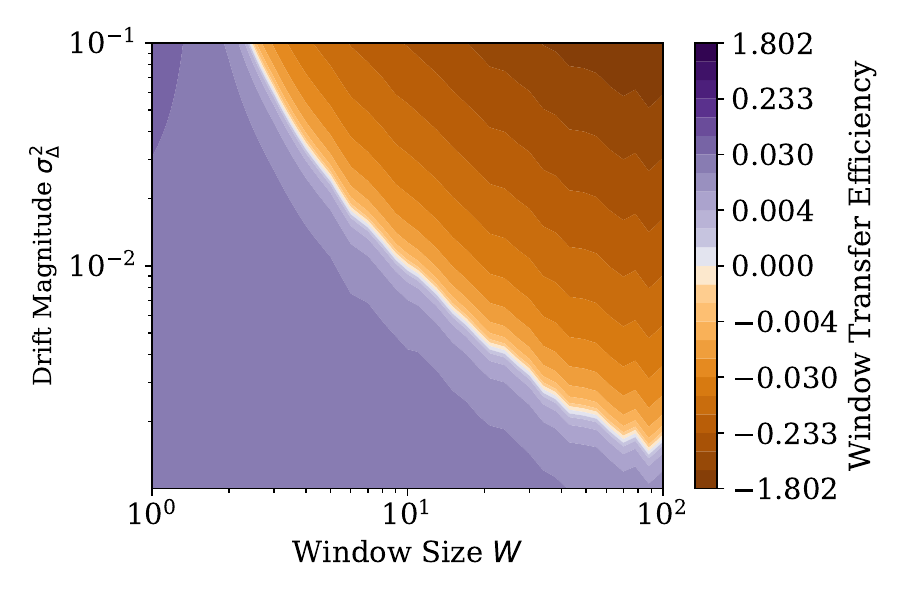}
    \caption{\textbf{Transfer Efficiency Landscape for the Window Agent.} The grids display the Transfer Efficiency of the Window agent as a function of Window Size $W$ (x-axis) and Drift Magnitude $\sigma_\Delta^2$ (y-axis). From left to right, the task duration increases ($N = 100, 500, 1000$). The area of positive transfer (white to orange) shrinks significantly as both task duration and drift magnitude increase. This visually confirms the theoretical analysis: overly large window sizes introduce a detrimental bias in highly non-stationary or long-duration environments, pulling the optimal $W^\star$ toward $1$.}
    \label{fig:sim-W-TE}
\end{figure}

\newpage

\subsection{Real-World Benchmarks}
\label{Apdx:datasets-networks-conf}

\begin{table}[ht]
    \centering
    \renewcommand{\arraystretch}{1.2}
    \caption{Supervised Learning benchmark statistics.}
    \begin{tabular}{@{}lcccc@{}}
        \toprule
        \textbf{Benchmark} & \textbf{Tasks ($K$)} & \textbf{Input Size} & \textbf{Classes} & \textbf{Samples per Task ($N_k$)} \\
        \midrule
        CLEAR & 10 & 224$\times$224 & 100 & $\approx 109$K  \\
        MD5 & 5 & 224$\times$224 & 30 & $\in [1480, 27750]$ \\
        P/S CIFAR10 & 10 & 32$\times$32 & 10 & $10$K \\
        \bottomrule
    \end{tabular}
    \label{tab:benchmark_statistics}
\end{table}

\begin{table}[ht]
    \centering
    \renewcommand{\arraystretch}{1.2}
    \caption{MULTIDATASET (MD5) individual dataset statistics prior to standardization.}
    \begin{tabular}{@{}lcc@{}}
        \toprule
         \textbf{Dataset} & \textbf{Original Classes} & \textbf{Total Samples} \\
        \midrule
        StanfordCars & 196 & 1523  \\
        FGVCAircraft & 100 & 2467 \\
        DTD & 47 & 1480 \\
        Food101 & 101 & 27750 \\
        OxfordPet & 37 & 3680 \\
        \bottomrule
    \end{tabular}
    \label{tab:MD5_statistics}
\end{table}

\subsubsection{CLEAR}
The CLEAR dataset \citep{lin2021clear} is a collection of images spanning the years 2004--2014, designed to evaluate models under natural temporal distribution shifts. \rev{We use the 100-class variant and split the collection into $K=10$ tasks, one per year, kept in their natural temporal order.} All input images are resized to 224$\times$224 pixels and normalized by subtracting the ImageNet mean $\mu = [0.485, 0.456, 0.406]$ and dividing by the standard deviation $\sigma = [0.229, 0.224, 0.225]$.

\begin{figure}[h]
    \centering
    \includegraphics[width=\linewidth]{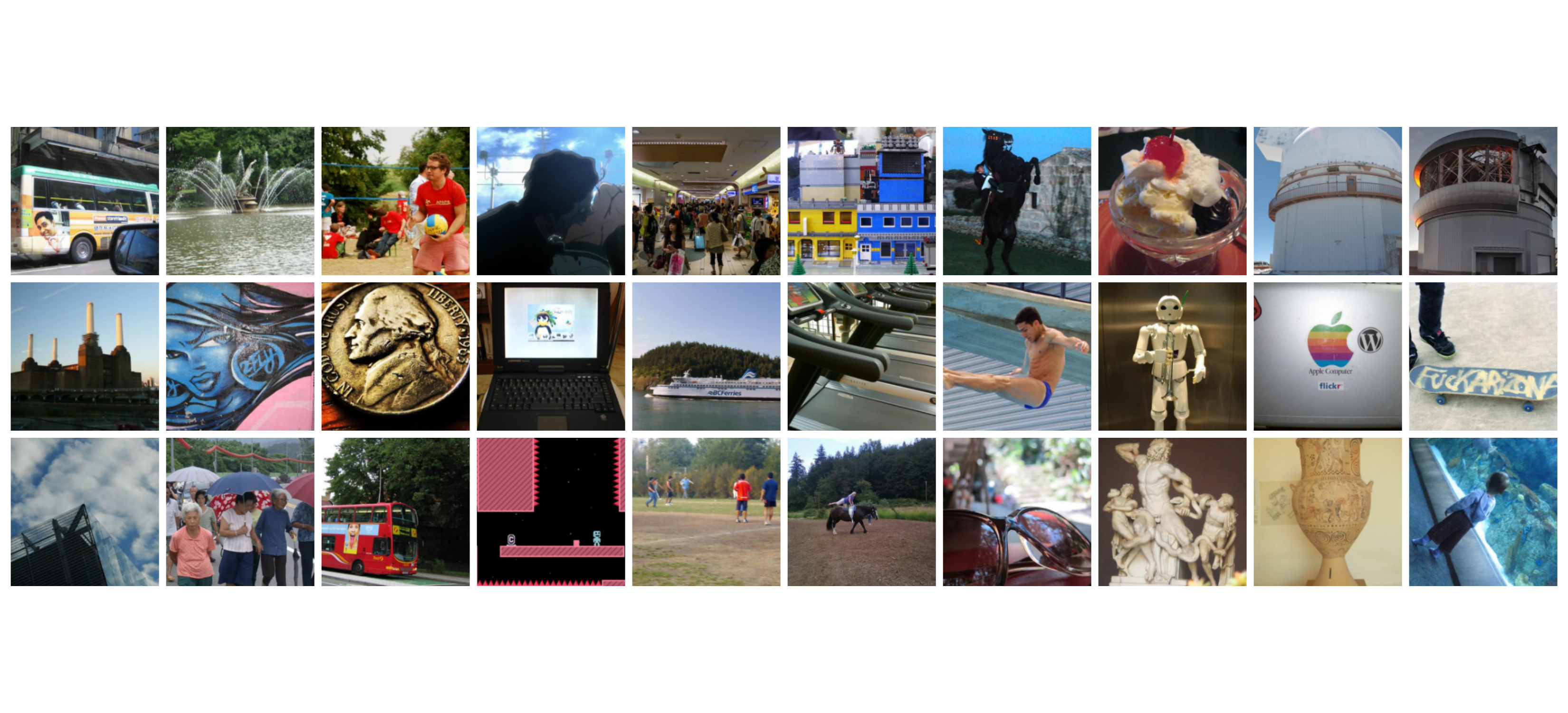}
    \caption{Samples from the CLEAR benchmark. Each column corresponds to a different temporally evolving task (year).}
    \label{fig:CLEAR-benchmark}
\end{figure}

\subsubsection{MULTIDATASET (MD5)}
The MULTIDATASET benchmark consists of a sequence of $K=5$ distinct open-source classification datasets with no semantic overlap. \rev{The tasks are to classify automobile models} \citep{krause2013}, aircraft models \citep{MajiRKBV13}, textures \citep{cimpoi14describing}, dishes \citep{bossard14}, and pets \citep{parkhi2012cats}. \rev{Each dataset originally has a different number of classes,} samples, and input resolutions (detailed in \cref{tab:MD5_statistics}).

\rev{To avoid architectural biases or capacity imbalances, we standardize all tasks to 30 classes. We also use identical batch sizes and total update steps across tasks, regardless of the original dataset's size.} All input images are resized to 224$\times$224 pixels and normalized using the standard ImageNet statistics ($\mu = [0.485, 0.456, 0.406]$, $\sigma = [0.229, 0.224, 0.225]$). During training, the data is augmented with random crops, random horizontal flips, and random rotations of up to 15 degrees.

\begin{figure}[h]
    \centering
    \includegraphics[width=0.6\linewidth]{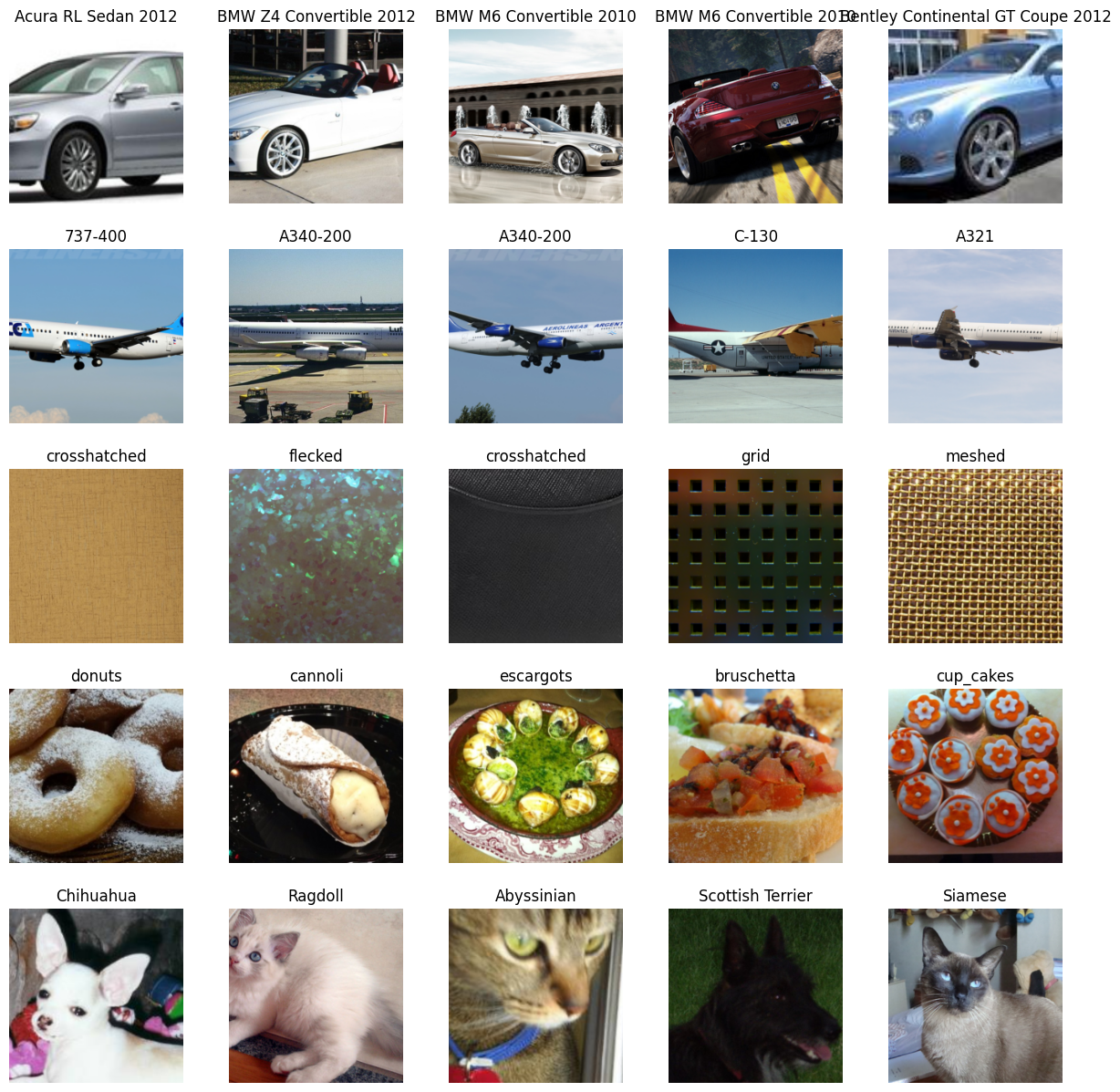}
    \caption{Samples from the MD5 benchmark. Each row corresponds to a completely distinct classification task.}
    \label{fig:MD5-benchmark}
\end{figure}

\subsubsection{Meta-World}
Meta-World \citep{yu2020meta} is a suite of simulated robotic manipulation tasks built on the MuJoCo physics engine. Our experiments use the \texttt{MT10} benchmark, a fixed collection of $10$ distinct manipulation tasks (\textit{reach}, \textit{push}, \textit{pick-place}, \textit{door-open}, \textit{drawer-open}, \textit{drawer-close}, \textit{button-press-topdown}, \textit{peg-insert-side}, \textit{window-open}, \textit{window-close}) on the shared Sawyer robot. We use the \texttt{v3} environment versions provided by the official Meta-World release. Unlike the meta-learning split (ML10), MT10 is designed for multi-task and continual evaluation on a common task set, with no held-out tasks.

At each step, the environment yields a $39$-dimensional observation vector that includes the robot's joint positions, the gripper state, and the spatial coordinates of relevant objects, with the goal position zero-masked. For the JTL agent, this observation is concatenated with a $10$-dimensional one-hot task identifier, yielding a $49$-dimensional input. The agent emits a $4$-dimensional continuous action (end-effector delta plus gripper command), with components squashed to $[-1, 1]$ via a $\tanh$. Each episode lasts at most $500$ time steps and yields a binary \texttt{success} flag in addition to the shaped reward.

\subsection{Synthetic Benchmarks}
\label{apsec:synthetic}

\subsubsection{Permuted CIFAR-10}
Permuted CIFAR-10 is a benchmark constructed from the standard CIFAR-10 dataset \citep{krizhevsky2009learning} by applying fixed random spatial permutations to the images. %
\rev{We use \major{three} permutation severities, defined by the spatial dimensions of the perturbed area: 16$\times$16, \major{24$\times$24,} and 32$\times$32. These give the side length of the centered square patch of pixels that is randomly permuted} (see \cref{fig:C10-16-benchmark} and \cref{fig:C10-32-benchmark} for the smallest and largest patch). We refer to these respective benchmarks as `PC-16'\major{, `PC-24',} and `PC-32'. All input images are normalized by subtracting the mean $\mu = [0.507, 0.486, 0.441]$ and dividing by $\sigma = [0.267, 0.256, 0.276]$.

\begin{figure}[h!]
    \centering
    \includegraphics[width=0.5\linewidth]{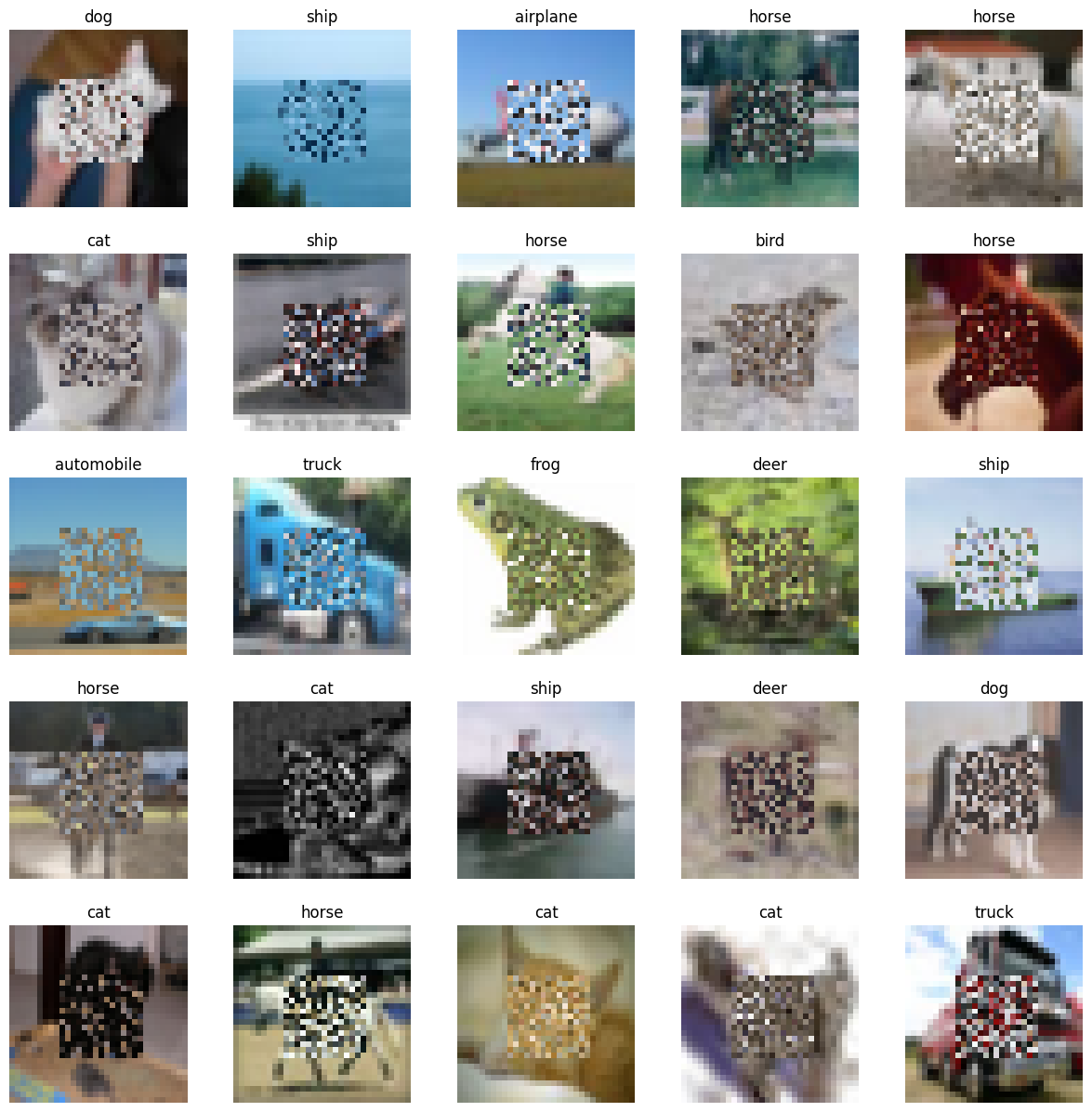}
    \caption{Samples from a Permuted CIFAR-10 task with a 16$\times$16 permutation box.}
    \label{fig:C10-16-benchmark}
    \vspace{0.5cm}
    \includegraphics[width=0.5\linewidth]{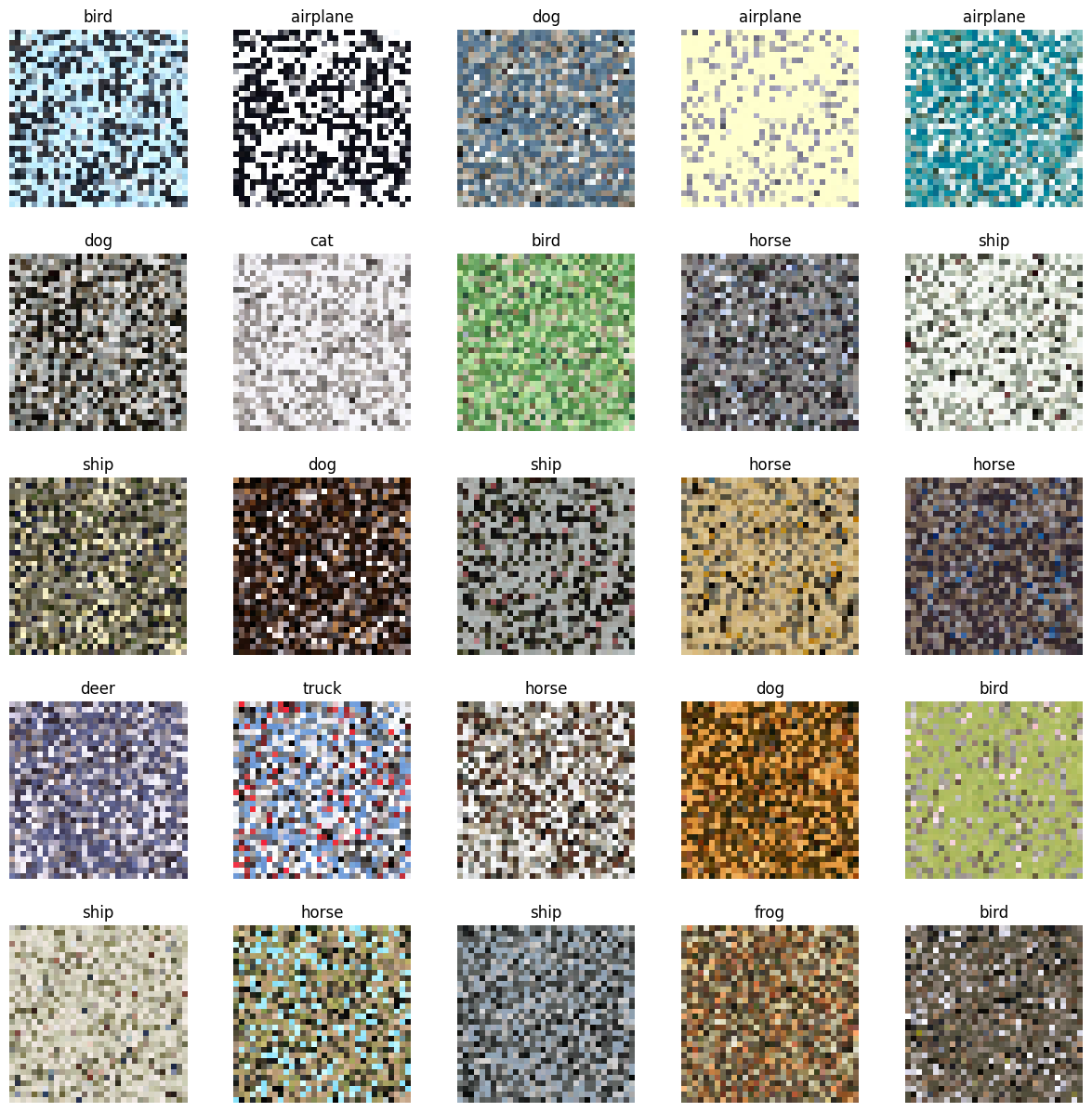}
    \caption{Samples from a Permuted CIFAR-10 task with a 32$\times$32 permutation box.}
    \label{fig:C10-32-benchmark}
\end{figure}

\subsubsection{Shuffled CIFAR-10}
\rev{Shuffled CIFAR-10 is another benchmark built from CIFAR-10 \citep{krizhevsky2009learning}. It isolates the effect of concept drift through targeted label resampling on a set of samples \emph{fixed at the beginning of training}.} We sweep the proportion of \emph{shuffled labels} across three severities: 10\%, 30\%, and 50\%. \rev{To make this label noise strictly adversarial, the true label is never drawn during reshuffling.} We refer to these benchmarks as `SC-10', `SC-30', and `SC-50'. Input images are normalized using the same statistics as Permuted CIFAR-10. \rev{Because the training labels are irreversibly corrupted, we evaluate the resulting generalization penalty on the uncorrupted test set.}

\subsection{Training Procedures \& Networks}
\label{Apdx:training}

\paragraph{Supervised Learning}

All supervised learning agents consist of a deep neural network backbone, an optimizer, and a learning rate scheduler. \rev{The backbone is a ResNet-18 (\emph{RN18}) in all supervised experiments. The final linear classification head is shared across all tasks, and its output dimension equals the total number of classes in the benchmark (100 for CLEAR, 30 for MD5, and 10 for Permuted/Shuffled CIFAR-10). The ITL and JTL agents are trained for the same number of gradient steps $h$ per task. \cref{tab:default_sl_config} lists the general optimization settings and \cref{tab:tuned_sl_config} the tuned hyperparameters.}

\emph{ITL Agent.~~}
\rev{Given a sequence of $K$ tasks, the ITL agent is trained to minimize the cross-entropy loss only on the training data of the current task. To obtain a strict baseline with no forward transfer, we randomly re-initialize the network weights and reset the optimizer state of the ITL agent at the start of every task.}

\emph{JTL Agent.~~}
\rev{The JTL agent minimizes the expected loss over the union of the current and all previously observed tasks. We implement this with an exact experience replay buffer: each gradient step uses a concatenated batch of data from the current task stream and data sampled uniformly from the buffer of previous tasks.} Unlike the ITL agent, the JTL agent's network parameters and optimizer state are preserved across task boundaries.

\paragraph{Reinforcement Learning}

In our reinforcement learning experiments, we compare ITL and JTL agents using Soft Actor-Critic (SAC) \citep{haarnoja2018soft} on the MT10 subset of the Meta-World benchmark \citep{yu2020meta}, with automatic entropy temperature tuning \citep{haarnoja2018softapps}. \rev{SAC is the standard off-policy continuous-control algorithm in the multi-task Meta-World literature \citep{yu2020meta}. Because it uses a single replay buffer, the ITL and JTL agents differ only in (i) what data populates the buffer and (ii) whether the policy is conditioned on a task identifier. The optimizer, network architecture, and per-epoch compute budget are otherwise identical.}

\emph{Architecture (shared).~~} For both agents, the actor and the twin critics are MLPs with two hidden layers of $400$ units and ReLU activations. The actor outputs the mean and (clipped) log-standard-deviation of a diagonal Gaussian, squashed by $\tanh$; the critics output a scalar $Q$-value. The log-std is clipped to $[-20, 2]$.

\emph{ITL Agent.~~} For each task in MT10 we train a separate SAC agent from a random initialization, with its own replay buffer and its own learned entropy temperature $\alpha$. The agent observes only the 39-dimensional Meta-World state of its task; no task identifier is appended. Each ITL run lasts $500$ epochs; per epoch we collect $40$ rollouts of length $500$ ($2 \times 10^{4}$ environment steps) into the buffer and then perform $500$ SAC gradient updates with mini-batch size \rebuttalB{$5000$}.

\emph{JTL Agent.~~} We train a single SAC agent across all MT10 tasks under a \emph{curriculum}: at epoch $e$ the active set is the first $1 + \lfloor e / N \rfloor$ tasks in a fixed ordering, with task duration $N = 500$ epochs. The actor and critics are conditioned on a $10$-dimensional one-hot task identifier concatenated to the observation. A single replay buffer of capacity $10^{6}$ \rev{is shared across all tasks and is \emph{never} reset. Whenever a new task becomes active, we re-initialize} the Adam states of the actor and critic optimizers (the entropy temperature is kept) so that the optimizer can adapt to the new objective. Per epoch we collect $200$ rollouts of length $500$ ($10^{5}$ environment steps) by uniformly cycling over the currently active environments, then perform $500$ SAC gradient updates with mini-batch size \rebuttalB{$5000$} sampled uniformly from the shared buffer. \rev{Each epoch, we evaluate on the current task.}

\emph{Task ordering.~~} For the curriculum experiments shown in the main paper we use the \texttt{easy\_to\_hard} ordering, obtained by sorting MT10 tasks in decreasing order of the final single-task success rate of the ITL SAC agents under the same configuration. The resulting sequence is \texttt{drawer-close}, \texttt{window-close}, \texttt{button-press-topdown}, \texttt{drawer-open}, \texttt{reach}, \texttt{window-open}, \texttt{door-open}, \texttt{push}, \texttt{peg-insert-side}, \texttt{pick-place}. \rev{We run each configuration (ITL and JTL) with $10$ seeds. Per-task figures aggregate the seeds for which the entire $500$-epoch stage is logged, with at least $7$ seeds per relative epoch.}

\subsubsection{Hyperparameters}

\paragraph{Supervised Learning}

\rev{We tune the learning rate and weight decay separately for each combination of agent and benchmark with a grid search.} To ensure a fair comparison, the optimizer and scheduler dynamics are fixed across most supervised learning experiments. \rev{We use SGD with a task-wise Cosine Annealing scheduler. The learning rate is annealed toward zero over each task's duration ($h$ steps) and reset to its maximum initial value at the start of the next task. This reset provides the plasticity needed to minimize the new objective quickly.}

\begin{table}[ht]
    \centering
    \renewcommand{\arraystretch}{1.2}
    \caption{Fixed Hyperparameters for Supervised Learning Experiments. Note that while the batch size was evaluated during tuning, the optimal batch size remained consistent across the primary benchmarks.}
    \begin{tabular}{@{}lc@{}}
        \toprule
        \textbf{Hyperparameter} & \textbf{Value}  \\
        \midrule
        \textbf{Optimizer} & SGD \\
        \textbf{Momentum} & 0.9 \\
        \textbf{Scheduler} & Cosine Annealing (Task-wise) \\
        \textbf{Stream Batch Size} & 256 \\
        \bottomrule
    \end{tabular}
    \label{tab:default_sl_config}
\end{table}
\vspace{-0.3cm}

\begin{table}[ht]
    \centering
    \renewcommand{\arraystretch}{1.2}
    \caption{Tuned Hyperparameters for Supervised Learning Experiments.}
    \begin{tabular}{@{}llcccl@{}}
        \toprule
        \textbf{Dataset} & \textbf{Agent} & \textbf{Network} & \textbf{LR} & \textbf{Weight Decay} \\
        \midrule
        CLEAR & ITL & RN18 & 0.100 & $3 \times 10^{-4}$ \\
        CLEAR & JTL & RN18 & 0.100 & $1 \times 10^{-4}$ \\
        \midrule
        PC-16 & ITL  & RN18 & 0.055 & $8 \times 10^{-4}$  \\
        PC-16 & JTL  & RN18 & 0.075 & $7 \times 10^{-4}$  \\
        \midrule
        PC-32 & ITL  & RN18 & 0.060 & $1 \times 10^{-3}$  \\
        PC-32 & JTL  & RN18 & 0.074 & $1 \times 10^{-3}$ \\
        \midrule
        MD5 & ITL  & RN18 & 0.089 & $9 \times 10^{-4}$ \\
        MD5 & JTL  & RN18 & 0.080 & $6 \times 10^{-4}$  \\
        \bottomrule
    \end{tabular}
    \label{tab:tuned_sl_config}
\end{table}

\paragraph{Reinforcement Learning}
For a fair comparison, ITL and JTL share the same SAC core hyperparameters (network width, learning rates, $\gamma$, $\tau$, entropy auto-tuning, mini-batch size, and the number of SAC gradient updates per epoch). \rev{The two agents differ only in aspects intrinsic to the ITL/JTL distinction:} (i) the size of the rollout collected per epoch ($2 \times 10^{4}$ environment steps for ITL on its single active task vs.\ $10^{5}$ for JTL split uniformly over the currently active tasks), (ii) the presence of a one-hot task conditioning for JTL, (iii) replay buffer scope (per-task vs.\ shared across the curriculum), and (iv) optimizer state at task boundaries (full re-initialization at every task for ITL vs.\ Adam-state reset of actor/critic only when a new task becomes active for JTL). All optimization is performed with Adam. \rev{Here \emph{compute-matched} refers to the gradient-update budget: both agents use the same number of SAC updates ($500$) and mini-batch size (\rebuttalB{$5000$}) per epoch. The environment-step count differs by construction ($2\times10^{4}$ steps per epoch for ITL on its single task vs.\ $10^{5}$ for JTL split over the active tasks). If anything, JTL is given \emph{more} environment data per epoch, yet still underperforms.}

We follow the MT-SAC configuration of \citet{yu2020meta} for the shared core, with the ITL-side cycle/rollout schedule taken from the single-task SAC appendix of the same work. The resulting hyperparameters are reported in \cref{tab:optimal_config_rl}; per-task and per-curriculum runtime budgets follow directly from the rollout schedule and the number of gradient updates.

\begin{table}[htb]
    \centering
    \renewcommand{\arraystretch}{1.2}
    \caption{SAC hyperparameters used for MT10 ITL and JTL agents. The two agents share the SAC core; they differ only in rollout schedule (environment steps collected per epoch), task conditioning, replay-buffer scope, and the optimizer-reset policy at task boundaries.}
    \label{tab:optimal_config_rl}
    \begin{tabular}{@{}lcc@{}}
        \toprule
        \textbf{Parameter} & \textbf{ITL} & \textbf{JTL} \\
        \midrule
        \multicolumn{3}{l}{\emph{Shared SAC core} (identical for ITL and JTL)} \\
        Optimizer & \multicolumn{2}{c}{Adam} \\
        Hidden layers (actor / critic) & \multicolumn{2}{c}{$2 \times 400$} \\
        Activation & \multicolumn{2}{c}{ReLU} \\
        Actor / critic / $\alpha$ learning rate & \multicolumn{2}{c}{$3 \times 10^{-4}$} \\
        Discount $\gamma$ & \multicolumn{2}{c}{$0.99$} \\
        Target smoothing $\tau$ & \multicolumn{2}{c}{$5 \times 10^{-3}$} \\
        Entropy temperature $\alpha$ & \multicolumn{2}{c}{auto-tuned} \\
        $\log \sigma$ clip & \multicolumn{2}{c}{$[-20,\ 2]$} \\
        SAC mini-batch size & \multicolumn{2}{c}{\rebuttalB{$5000$ (\cref{tab:results-wild}); $500$ (ablation)}} \\
        Gradient updates per epoch & \multicolumn{2}{c}{$500$} \\
        Epochs per task ($N$) & \multicolumn{2}{c}{$500$} \\
        \midrule
        \multicolumn{3}{l}{\emph{ITL / JTL-specific}} \\
        Task conditioning & none & one-hot, $\mathbb{R}^{10}$ \\
        Episodes per epoch (length $500$) & $40$ & $200$ \\
        Env.\ steps per epoch & $2 \times 10^{4}$ & $10^{5}$ \\
        Replay buffer & per task & shared, capacity $10^{6}$ \\
        Min.\ buffer size before updates & $1{,}500$ & $1{,}500$ \\
        Buffer reset at task switch & yes & no \\
        Optimizer reset & full reset every task & Adam state of actor/critic re-init ($\alpha$ kept) \\
        Curriculum / ordering & N/A & \texttt{easy\_to\_hard} \\
        Eval target each epoch & current task & current task \\
        Number of seeds & $10$ & $10$ \\
        \bottomrule
    \end{tabular}
\end{table}

\emph{Effect of the SAC mini-batch size on ITL agents.}
We fix the SAC mini-batch size to \rebuttalB{$5000$} for both ITL and JTL so that the two agents share an identical per-update compute budget (\cref{tab:optimal_config_rl}). \rev{The quadratic analysis predicts that the optimization-induced (\emph{Variance}) penalty scales as $1/B$ in the batch size $B$ (\cref{sec:quadratic-losses}), so we ask whether} the absolute performance of the agents, and hence the task-difficulty ordering that drives the aggregate Transfer Efficiency, is sensitive to this choice. We conducted a batch-size ablation for the single-task (ITL) SAC agents, comparing \rebuttalB{$B=5000$ against a $10$ times smaller $B=500$} under otherwise identical hyperparameters. Each of the $10$ ITL agents is trained for $500$ epochs ($2\times10^{4}$ environment steps per epoch, i.e.\ $10^{7}$ steps per task) \rev{in the \texttt{easy\_to\_hard} ordering. We compare the mean success rate over $10$ seeds with a $\pm$SEM band, smoothed with a centered rolling average. We make two observations. First, the larger batch size yields a uniformly higher and lower-variance success rate on nearly every task, consistent with the $1/B$ gradient-noise scaling: less noise in each SAC update gives faster and more stable convergence} (most visibly on the harder \emph{push}, \emph{door-open}, \emph{button-press-topdown}, and \emph{window-open} blocks). Second, and more importantly for our claims, the \emph{relative} ordering of the tasks by difficulty is essentially unchanged: the tasks that are easy at $B=500$ (e.g.\ \emph{drawer-close}, \emph{reach}, \emph{window-close}) remain easy at $B=5000$, and the tasks that single-task SAC fails to solve within the budget (e.g.\ \emph{pick-place}, \emph{peg-insert-side}) remain unsolved at both batch sizes. \rev{Because the per-task Transfer Efficiency decomposition is driven by this difficulty ordering and by the high-instability tasks, the qualitative conclusion is robust to the mini-batch size: aggregate TE on MT10 is negative and dominated by a handful of high-instability tasks.} \rebuttalB{We therefore report the $B=5000$ setting for all quantitative results in the main text (\cref{tab:results-wild} and \cref{fig:example_performance_traj}); \rev{the $B=500$ runs yield the same task-difficulty ordering and hence the same qualitative conclusion.}}

\newpage

\subsection{Additional Figures and Empirical Substantiation}
\label{Apx:add-results}
This subsection collects supporting diagnostics for the empirical claims in the main text. The figures should be read as mechanism checks for the Transfer Efficiency (TE) decomposition rather than as separate benchmark claims: they expose task heterogeneity, instability proxies, width- and parametrization-dependent optimization bias, and the real-time trajectories that are averaged into the final TE curves.

\begin{figure}[ht]
    \centering
    \begin{minipage}{0.32\textwidth}
        \centering
    \includegraphics[width=\linewidth]{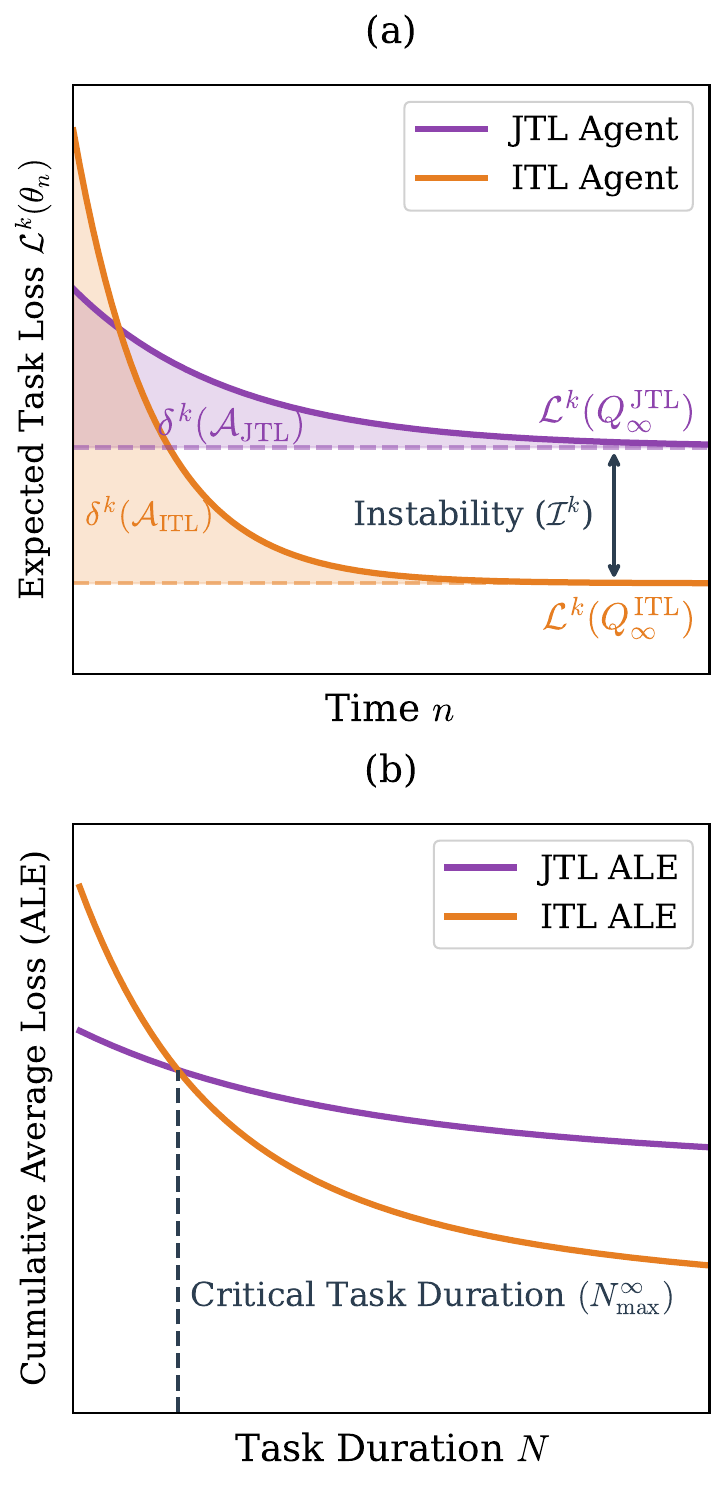}
    \end{minipage}
    \hspace{0.25cm}
    \begin{minipage}{0.6\textwidth}
    \caption{\textbf{Decomposition of Transfer Efficiency and the Critical Task Duration.}
    \emph{(a) Decomposition of Learning Curves:} Expected loss $\mathcal{L}^k(\theta_n)$ over time $n$ for the JTL (violet) and ITL (orange) agents. The shaded areas represent the Transient Errors ($\delta^k$) accumulated as the agents converge toward their respective stationary risks. The gap between these asymptotic limits represents the task Instability ($\mathcal{I}^k$). While JTL benefits from prior knowledge initially, it converges to a higher stationary risk due to interference.
    \emph{(b) Average Task Error and Critical Duration:} Cumulative Average Loss (ALE) as a function of total task duration $N$.
\rev{Because the ITL agent eventually converges to a better local minimum, its ALE keeps dropping and crosses the JTL agent's curve.}
This intersection defines the Critical Task Duration ($N^\infty_{\max}$), beyond which transfer becomes detrimental ($\mathrm{TE}^k < 0$).\vspace{-5mm}}
    \label{fig:transfer_efficiency_final}
    \end{minipage}
\end{figure}

\subsubsection{Pairwise Transfer is Asymmetric, and Why Ordering Matters}
\label{apdx:rl-pairwise-asymmetry}
To isolate the \emph{directional} structure of transfer, we measure pairwise interference directly: for every ordered pair $(t_i, t_j)$ of MT10 tasks we train a two-task joint (JTL) agent and
\rev{compare its performance on $t_i$ with that of the single-task (ITL) specialist for $t_i$.}
\Cref{fig:rl_pairwise_asymmetry} reports the resulting normalized gap
\[
    \Delta(t_i \mid t_j) \;=\; \frac{\text{ITL}(t_i) - \text{JTL}(t_i \mid t_j)}{\text{ITL}(t_i)},
\]
where the row indexes the \emph{evaluated} task $t_i$ and the column indexes the \emph{co-trained} task $t_j$.
\rev{In the figure, $\Delta>0$ (red) means joint training (JTL) \emph{hurts} the row task relative to its ITL specialist, $\Delta<0$ (blue) means JTL \emph{helps} it, and the diagonal is $0$ by construction.}

\rev{The main structural fact is that \textbf{the matrix is not symmetric}: in general $\Delta(t_i \mid t_j) \neq \Delta(t_j \mid t_i)$. Transfer between two tasks is directional: a task can be substantially harmed by a partner that it helps.}
\emph{Drawer-close} and \emph{peg-insert-side}, for instance, are improved by joint training with almost every partner (predominantly blue rows), whereas \emph{push}, \emph{pick-place}, and \emph{button-press-topdown} are degraded by most partners (predominantly red rows), and the two directions of a given pair frequently carry opposite signs.

This asymmetry has a direct consequence for the lifelong objective. Because the quantity we ultimately optimize is the time-averaged error \emph{across the whole task sequence}, a pair in which one direction is positive and the other negative need not produce net-negative transfer: the gains on the helped task can offset or outweigh the losses on the harmed task.
\rev{Asymmetry also turns \emph{task ordering} into a usable degree of freedom. Suppose}
$\Delta(t_1 \mid t_2) > 0$ (joint training would hurt $t_1$) while $\Delta(t_2 \mid t_1) < 0$ (joint training helps $t_2$). Then presenting the tasks in the order $t_1 \rightarrow t_2$ (training $t_1$ \emph{alone} first, and only introducing $t_2$ afterwards under joint training)
\rev{lets $t_1$ reach its ITL-level performance in the first phase and $t_2$ gain from positive transfer in the joint phase.}
The lifelong error over the pair can then fall below the average of the two ITL specialists, even though the symmetric (simultaneous, from-scratch) joint agent would have shown mixed or negative transfer.
\rev{The off-diagonal sign structure of \cref{fig:rl_pairwise_asymmetry} is the lever that a curriculum exploits, and the structural reason to adopt an explicit ordering.}

\rev{The matrix is measured under \emph{simultaneous} two-task training from a shared initialization, so it is a proxy for the directional interference.}
It motivates ordering as a mechanism and explains the observed sensitivity of aggregate TE to the task sequence.
\rev{We leave to future work a systematic search, guided by this asymmetry, for orderings that maximize lifelong performance.}

\begin{figure}[htb]
    \centering
    \includegraphics[width=0.78\linewidth]{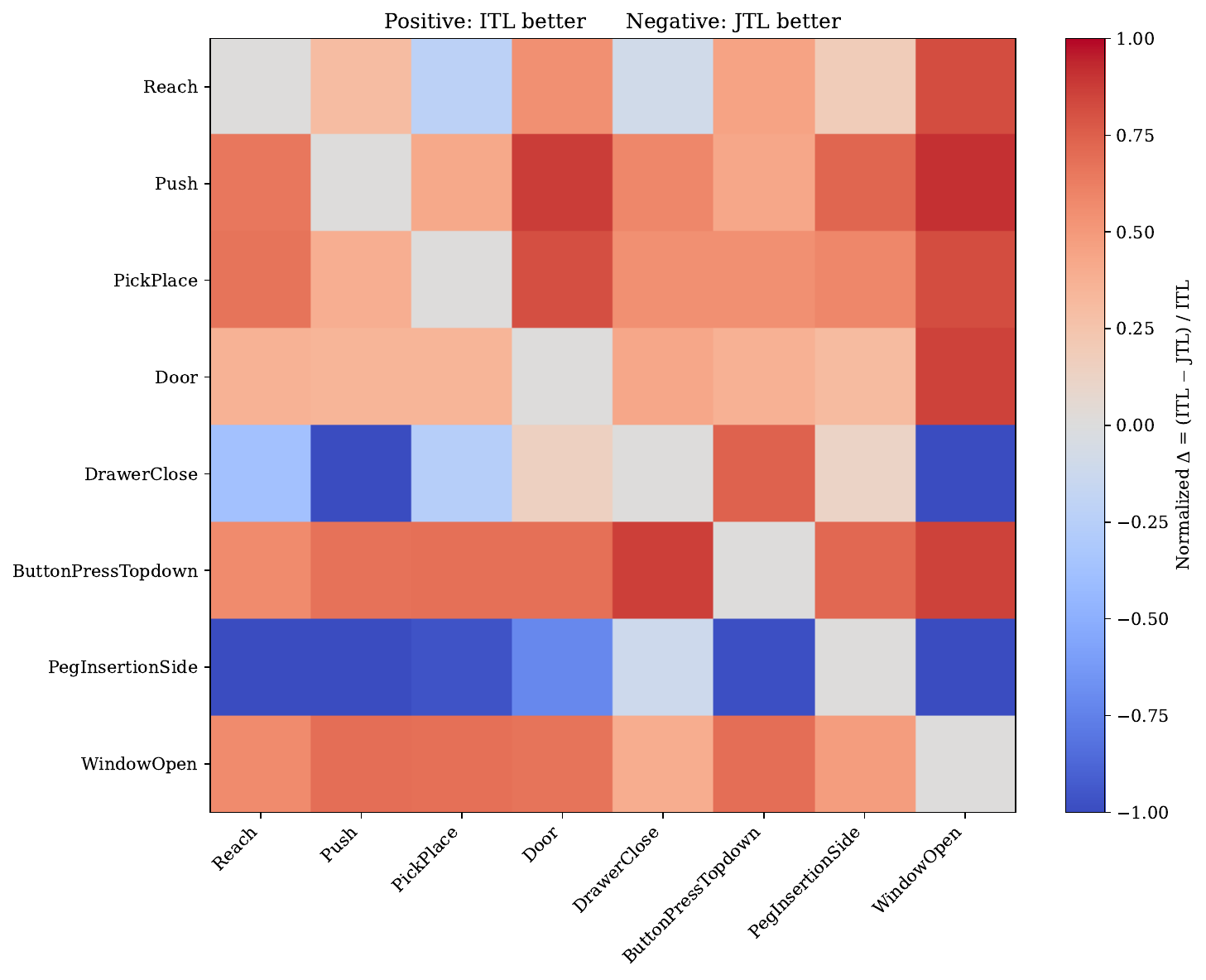}
    \caption{\textbf{Pairwise ITL-vs-JTL transfer is asymmetric.} Each cell is the normalized gap $\Delta(t_i \mid t_j) = (\text{ITL}(t_i) - \text{JTL}(t_i \mid t_j))/\text{ITL}(t_i)$ for the row task $t_i$ when co-trained with the column task $t_j$. \textcolor{red}{Red} ($\Delta>0$) marks pairs where joint training (JTL)
\rev{makes the row task \emph{worse} than its ITL (single-task) specialist; \textcolor{blue}{blue} ($\Delta<0$) marks pairs where JTL makes it \emph{better};}
the diagonal is $0$ by construction. The matrix is markedly \emph{non-symmetric}: $\Delta(t_i\mid t_j)\neq\Delta(t_j\mid t_i)$, so transfer is directional. \vspace{-3mm}
    }
    \label{fig:rl_pairwise_asymmetry}
\end{figure}

\subsubsection{Proxy measures of instability}
\label{apssec:instability-proxies}

We consider two proxies to measure Instability. The first, $\mathcal{I}_{conv}$, approximates the theoretical definition (\cref{def:instability}) by training ITL and JTL models to convergence and measuring the difference in their average task-wise error. The second, $\mathcal{I}_{transfer}$, is a heuristic that measures the average off-diagonal error in a task transfer matrix (i.e., the error of a model trained on task $i$ when evaluated on task $j$).
\rev{$\mathcal{I}_{transfer}$ is cheaper to compute.}

\begin{table}[h]
    \centering
    \begin{minipage}[c]{0.35\textwidth} %
        \caption{Proxies for Instability.
\rev{Higher values mean higher Instability. The range of values differs between supervised and RL benchmarks. Toy benchmarks are in gray.}}
        \label{tab:instability}
    \end{minipage}
    \hspace{0.1cm}
    \begin{minipage}[c]{0.6\textwidth} %
        \centering
        \begin{tabular}{l|c|c}
            \toprule
            Data  & {$\mathcal{I}_{conv}$} & $\mathcal{I}_{transfer}$  \\
            \midrule
            CLEAR & $-0.024_{\pm 0.003}$ & 0.007$_{\pm 0.001}$ \\
            MD5 & $0.017_{\pm 0.008}$  & 0.35$_{\pm 0.01}$ \\
            MT10 & \rebuttalB{$0.252_{\pm 0.044}$}  & $0.139_{\pm 0.009}$  \\
            \midrule
            \rowcolor{gray!10} PC-16 & $-0.0213_{\pm 0.0024}$ & 0.03$_{\pm 0.002}$\\
            \rowcolor{gray!10} PC-32 & $0.0014_{\pm 0.004}$ & 0.30$_{\pm 0.005}$ \\
            \rowcolor{gray!10} SC-10\% & $0.014_{\pm 0.004}$ & 0.007$_{\pm 0.006}$ \\
            \rowcolor{gray!10} SC-30\% & $0.047_{\pm 0.009}$ & 0.008$_{\pm 0.001}$ \\
            \rowcolor{gray!10} SC-50\% & $0.137_{\pm 0.005}$ & 0.0083$_{\pm 0.007}$ \\
            \bottomrule
        \end{tabular}
    \end{minipage}
\end{table}

As shown in \cref{tab:instability}, both proxies confirm our qualitative assessment: CLEAR exhibits low Instability, while MD5 and MT10 are high. The proxies also validate our methodology for the controlled benchmarks, showing that Instability increases with the permutation size
\rev{(PC-16 vs.\ PC-32)}
and the label shuffle percentage (SC-10\% to SC-50\%). The two proxy columns are not expected to match numerically: $\mathcal{I}_{conv}$ compares the converged JTL and ITL solutions directly, while $\mathcal{I}_{transfer}$ measures off-diagonal transfer between separately trained task specialists.
\rev{For our empirical use of Instability, what matters is that they agree on the ordering.}

\rebuttalB{\paragraph{Why can a measured Instability be negative?} \cref{prop:exact_instability,theo:ridgeless_instability_exact} establish $\mathcal{I}^k \ge 0$ under three premises: both agents' moments are at their stationary values when the errors are read off, every update uses fresh samples, and evaluation is on the training distribution. A negative estimate $\hat{\mathcal{I}}_{conv}$ therefore does not contradict these results; it signals an estimation- or optimization-dependent effect outside their premises, and at least one of the following must hold:
\begin{itemize}
    \item[\emph{(i)}] \emph{estimator bias}: the ITL agent is not yet stationary at read-off, so $\hat{e}^{\,ITL}$ overestimates its stationary error and biases $\hat{\mathcal{I}}_{conv}$ downward;
    \item[\emph{(ii)}] \emph{initialization-dependent convergence}: outside the quadratic setting, basin selection can let the warm start genuinely improve the converged solution;
    \item[\emph{(iii)}] \emph{finite-sample pooling}: with finite data, training on the union of tasks can improve \emph{test} error
\rev{when tasks are similar, an effect absent on training error.}
\end{itemize}

To determine which mechanism lies behind CLEAR's negative estimate, we measure three quantities on the dense logged trajectories behind \cref{tab:clear_benchmark}: the slope of the ITL test-error curves over the last two deciles of each task window (a direct test of the stationarity premise); the estimate $\hat{\mathcal{I}}_{conv}(N)$ on both test and training error for $N = 3000$ to $12000$; and the converged training errors of the two agents at the largest $N$.
\rev{The ITL curves are still descending at read-off in $40/40$ task--duration cells, and the estimate collapses monotonically, $-0.066 \to -0.011 \to -0.003 \to -0.0006 \pm 0.0006$, statistically zero at $N = 12000$. A ${\sim}10\times$ smaller warm-start advantage persists on training error ($27/27$ cells favor JTL at $N = 12000$), and the train and test estimates share the same sign at every $N$. We conclude that (i) is the driver, a residue of (ii) is present an order of magnitude below it, and (iii) plays no role.}}

\rebuttalB{\subsubsection{Removing One Assumption at a Time: Simulation Studies}
\label{apdx:assumption-studies}

The closed-form analysis of \cref{sec:quadratic-losses} and the shallow simulation of \cref{fig:sim-shallow} share three modeling ingredients: commuting task Hessians (\cref{ass:alignment}, in the isotropic form $H = \lambda I$), full-rank task covariances (\cref{ass:bounded_curvature}), and a gradient-noise covariance that is constant along the trajectory (\cref{lemma:GNC}).
\rev{The three studies below remove one ingredient at a time,}
holding everything else fixed at the configuration of \cref{apsec:simulations}: $K = 50$ tasks presented sequentially for $N$ steps each; teachers following the random walk $\theta_\star^k = \theta_\star^{k-1} + \Delta_k$ with $\Delta_k \sim \mathcal{N}(0, (\sigma_\Delta^2/d)\, I)$; the ITL agent re-initialized at every task from $\mathcal{N}(0, \sigma_0^2 I)$ with $\sigma_0 = 1$; the JTL agent continuing from $\theta_0 = 0$ and drawing each step's task uniformly from the tasks seen so far; gradients equal to the population gradient of the drawn task plus additive Gaussian noise of covariance $(\sigma_\xi^2/B)\, I$, with the same noise realization fed to both agents; $d = 10$ unless stated otherwise, $\eta = 0.1$, $\sigma_\xi = 1$, $B = 1$.
\rev{Transfer Efficiency is measured as in \cref{fig:sim-shallow} (the average over tasks of the difference between the two agents' within-task average losses) on a}
$40 \times 40$ grid of $(N, \sigma_\Delta^2) \in [1, 10^3] \times [10^{-3}, 10^{-1}]$, averaged over $20$ seeds. Each figure pairs the simulated map with a theoretical map, computed per grid point from the corresponding closed-form expressions averaged over $20$ teacher realizations and rendered with the same plotting routine. As a calibration, re-running the unmodified configuration reproduces \cref{fig:sim-shallow}: the simulated zero-TE boundary lies within one grid cell of the theoretical one (median factor $1.19$ in $N$), and the median pointwise relative deviation between the two maps is $15\%$ ($90$th percentile $27\%$).
\rev{This agreement is the reference standard for the three studies.}

\paragraph{Removing Task Alignment.}
The isotropic covariance is replaced by an anisotropic spectrum $(\lambda_1, \dots, \lambda_{10})$, log-uniform with $\lambda_{\max}/\lambda_{\min} = 5$ and normalized to mean $2$ (the isotropic value);
\rev{even-index tasks use $H = \operatorname{diag}(\lambda)$ and odd-index tasks use}
$H = R_\varphi \operatorname{diag}(\lambda)\, R_\varphi^\top$, where $R_\varphi$ rotates by the angle $\varphi$ simultaneously in the five eigen-planes $(e_i, e_{11-i})$, $i = 1, \dots, 5$,
\rev{so that every direction is misaligned.}
The four panel groups of \cref{fig:assumption-misalignment} show $\varphi \in \{0^\circ, 30^\circ, 60^\circ, 90^\circ\}$; $\varphi = 0^\circ$ is the aligned anchor, differing from \cref{fig:sim-shallow} only through the anisotropy. (At $\varphi = 90^\circ$ each eigen-plane swaps its two eigenvalues and the two frames commute again;
\major{the commutator scales as $\lvert\sin 2\varphi\rvert$, peaking at $45^\circ$ and equal at $30^\circ$ and $60^\circ$.)}
\rev{The theoretical maps evaluate the relaxed expressions of \cref{sec:beyond_alignment} (the Lyapunov Instability of \cref{prop:instability_no_alignment} and the eigenbasis-coupling Transient of \cref{prop:transient_no_alignment}, with the noise covariance of the simulated process), so the aligned formulas are never used.}

\rev{First, at every angle}
the theoretical map matches the simulated one at the reference standard: the median pointwise deviation is $12.5$--$12.9\%$, indistinguishable from the aligned anchor. Second, the zero-TE boundary is insensitive to the misalignment: across the entire sweep it moves by at most one grid cell (at $\varphi = 90^\circ$, the break-even duration at $\sigma_\Delta^2 = 10^{-1}$ shifts from $41$ to $49$ steps). In this drift regime, Task Alignment determines the interpretable form of the expressions, not the location of the memory trade-off.

\begin{figure}[H]
    \centering
    \includegraphics[width=0.49\linewidth]{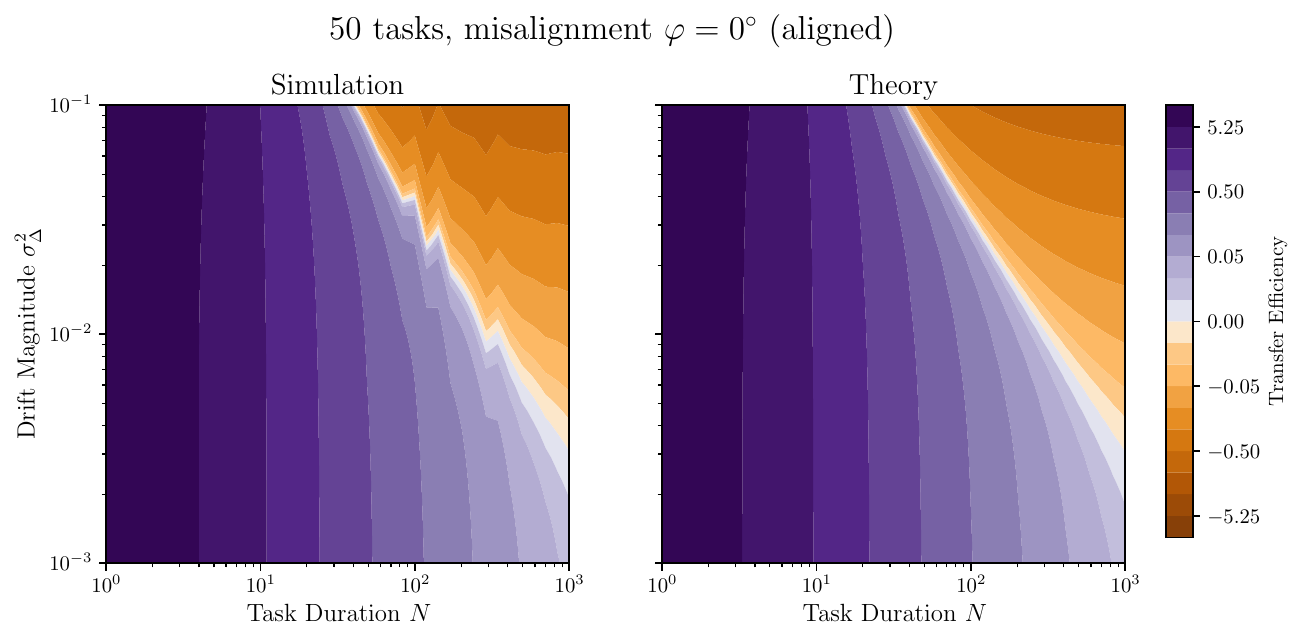}
    \includegraphics[width=0.49\linewidth]{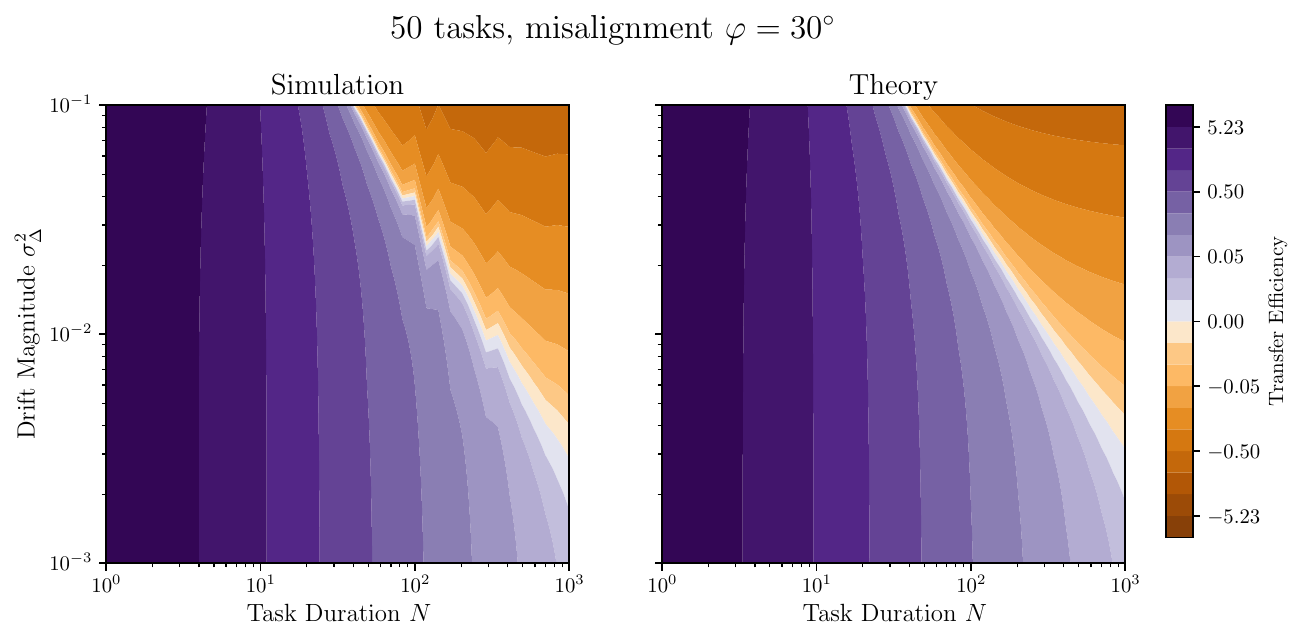}\\
    \includegraphics[width=0.49\linewidth]{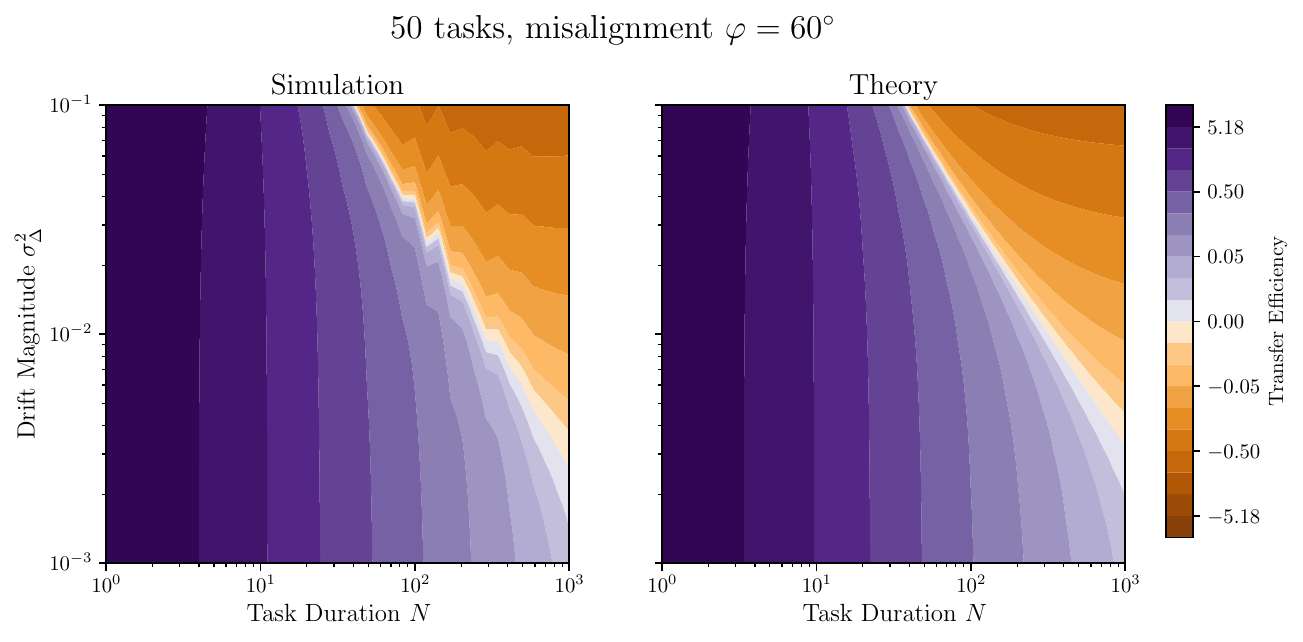}
    \includegraphics[width=0.49\linewidth]{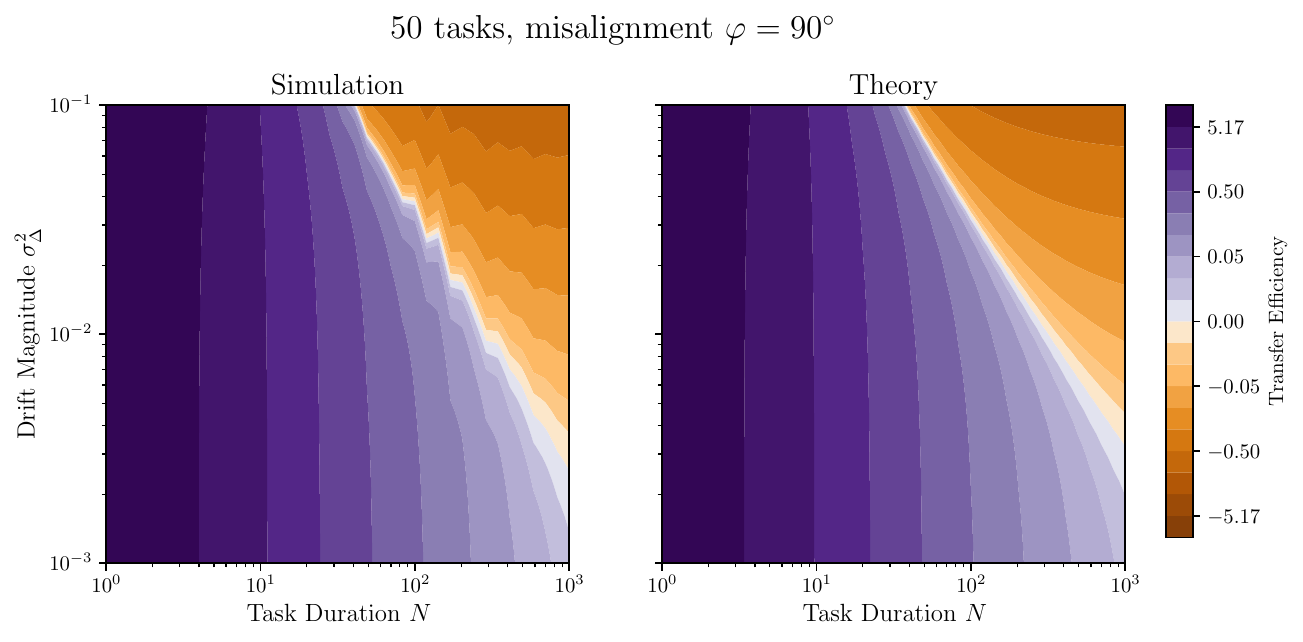}
    \caption{\textbf{Removing Task Alignment.} Simulated (left of each pair) and theoretical (right) Transfer Efficiency across task duration $N$ and drift magnitude $\sigma_\Delta^2$, for misalignment angles $\varphi \in \{0^\circ, 30^\circ, 60^\circ, 90^\circ\}$ between the alternating task covariances. The theoretical maps use the relaxed expressions of \cref{sec:beyond_alignment}. The zero-TE boundary moves by at most one grid cell across the sweep.}
    \label{fig:assumption-misalignment}
\end{figure}

\paragraph{Removing full-rankness.}
This study instantiates the structural low-rank model of \cref{ass:low_rank}. Task $k$ is supported on $r = 10$ coordinates, $S_k = \{ks, \dots, ks + r - 1\}$ with stride $s = r - o$, so that consecutive tasks share $o$ coordinates; the covariance is $H_k = \lambda \operatorname{diag}(m_k)$ with $\lambda = 2$ and $m_k$ the indicator of $S_k$; the additive noise is projected onto the sampled task's support; the drift variance is $\sigma_\Delta^2/r$ per coordinate; and the ambient dimension exceeds the union of the supports by four never-visited coordinates ($d = 14$, $259$, $504$ for $o = 10$, $5$, $0$). At $o = 10$ all tasks share one support and the process coincides with that of \cref{fig:sim-shallow}. The theoretical maps of \cref{fig:assumption-lowrank} evaluate the ridgeless expressions of \cref{sec:ridgeless} (\cref{theo:ridgeless_instability_exact,theo:ridgeless_transients_exact}) coordinate-wise, with the diluted mixture rates $\bar\gamma_i$.

The overlap controls the shape of the phase diagram, not merely its scale. At full overlap the boundary is the familiar diagonal: the crossing is driven by the Instability, which grows with the drift accumulated along the sequence. At $o = 5$ the shared subspace halves, the Instability shrinks accordingly, and the negative region retreats behind a nearly horizontal edge at $\sigma_\Delta^2 \approx 3 \cdot 10^{-2}$: below it, positive transfer persists at every task duration on the grid. At $o = 0$ the supports are disjoint, the Instability vanishes,
\rev{and the remaining negative region at high drift is produced by the dilution penalty alone, the mechanism isolated in \cref{theo:ridgeless_transients_exact}: the JTL agent updates each newly unlocked direction at the diluted rate $\bar\gamma_i = \lambda/k$, so over most of the grid it does not converge within the task.}
The theoretical maps reproduce all three shapes at the reference standard (median deviations $12$--$13\%$).

\begin{figure}[H]
    \centering
    \includegraphics[width=0.49\linewidth]{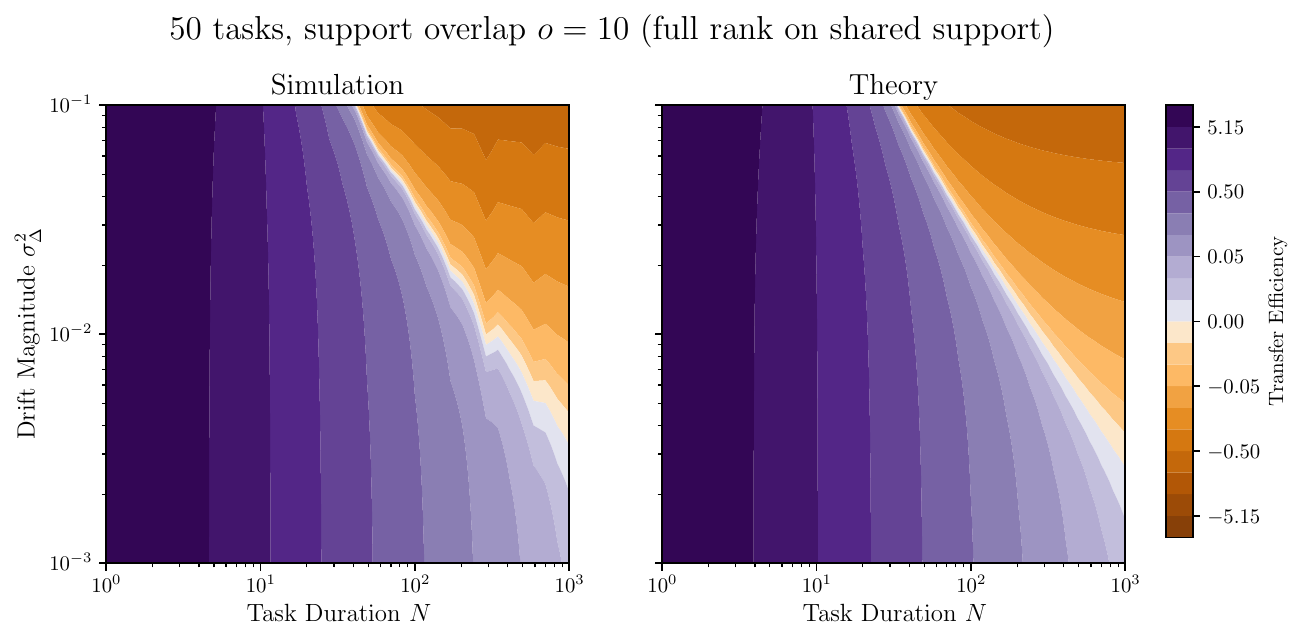}
    \includegraphics[width=0.49\linewidth]{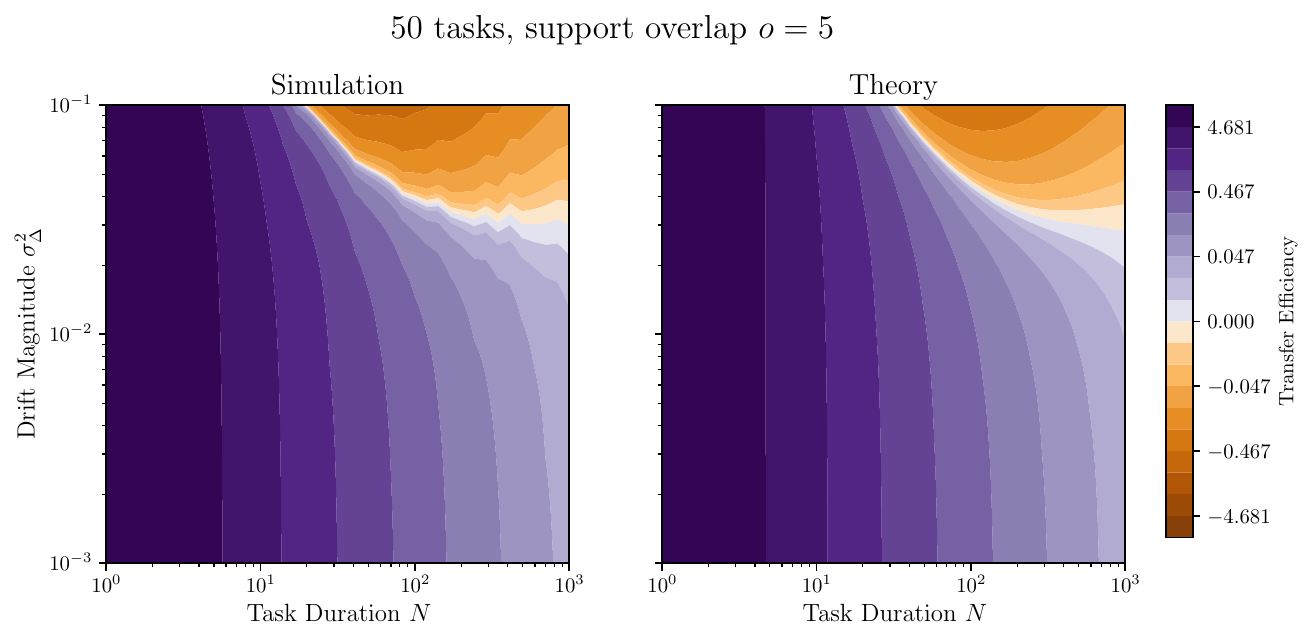}\\
    \includegraphics[width=0.49\linewidth]{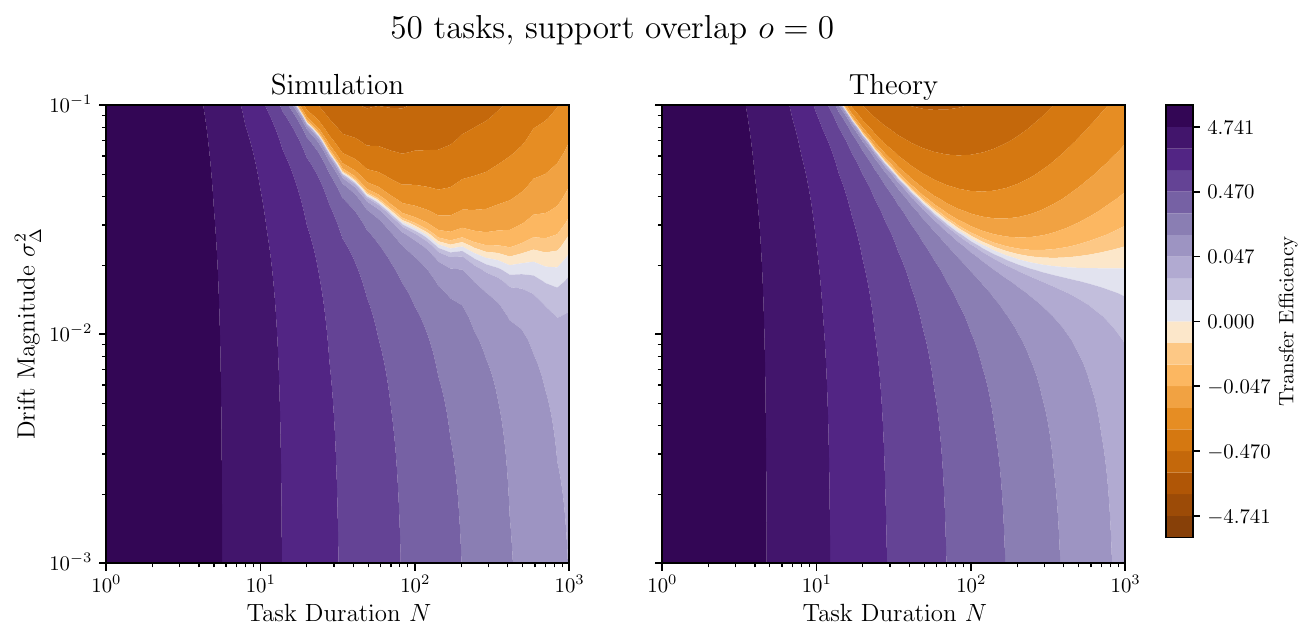}
    \caption{\textbf{Removing full-rankness (structural low rank).} Simulated and theoretical Transfer Efficiency for tasks supported on $r = 10$ of $d$ coordinates with consecutive-support overlap $o = 10$ (full, the configuration of \cref{fig:sim-shallow}), $o = 5$, and $o = 0$ (disjoint). The theoretical maps use the ridgeless expressions of \cref{sec:ridgeless}. Shrinking the shared subspace deforms the transfer boundary from duration-limited (diagonal) to drift-limited (horizontal); at $o = 0$ the negative region is the dilution penalty of \cref{theo:ridgeless_transients_exact} alone.}
    \label{fig:assumption-lowrank}
\end{figure}

\paragraph{Removing the constant-noise idealization.}
This study is run on the original isotropic full-rank configuration.
\rev{The additive noise is replaced by actual data sampling (inputs $x \sim \mathcal{N}(0, \lambda I)$ and labels $y = x^\top \theta_\star + \xi$), so}
the gradient-noise covariance carries the state-dependent term computed, and then discarded, in the proof of \cref{lemma:GNC}. Both noise models are run at $B = 10$, because at $B = 1$ sampled-input SGD is mean-square unstable at $\eta = 0.1$: the configuration of \cref{fig:sim-shallow} is stable only because its noise is additive. The step size is swept over $\eta \in \{0.1, 0.05, 0.02\}$ and both variants are compared to the same closed-form maps.

The deviations of both variants shrink proportionally to $\eta$ (\cref{tab:assumption-noise}),
\rev{confirming the $\mathcal{O}(\eta)$ accuracy of the closed forms for their own noise model. The excess deviation}
of the sampled-input variant is consistent with the size $\eta\, \operatorname{Tr}(H)/2B$ of the discarded term ($10\%$, $5\%$, $2\%$ at the three step sizes). The zero-TE boundary is unaffected even in the worst case ($\eta = 0.1$ with sampled inputs, left group of \cref{fig:assumption-noise}).

\begin{table}[H]
    \centering
    \caption{\textbf{Removing the constant-noise idealization.} Median ($90$th percentile) pointwise relative deviation between simulated and theoretical Transfer Efficiency over grid cells with $|\mathrm{TE}| > 0.05$, for the paper's state-independent noise model and for sampled inputs, at three step sizes ($B = 10$).}
    \label{tab:assumption-noise}
    \begin{tabular}{lcc}
        \toprule
        & state-independent noise & sampled inputs \\
        \midrule
        $\eta = 0.1$ & $13.0\%$ ($25.1\%$) & $24.3\%$ ($41.4\%$) \\
        $\eta = 0.05$ & $6.5\%$ ($16.2\%$) & $10.2\%$ ($21.1\%$) \\
        $\eta = 0.02$ & $2.9\%$ ($8.9\%$) & $4.0\%$ ($10.4\%$) \\
        \bottomrule
    \end{tabular}
\end{table}

\begin{figure}[h]
    \centering
    \includegraphics[width=0.49\linewidth]{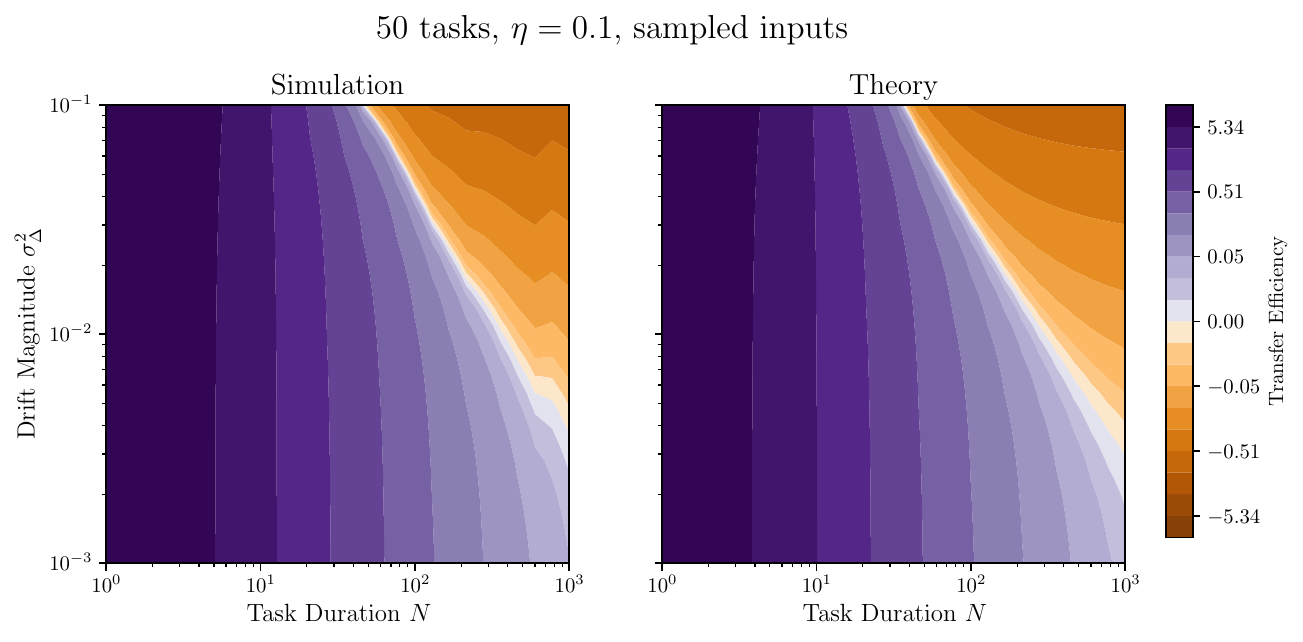}
    \includegraphics[width=0.49\linewidth]{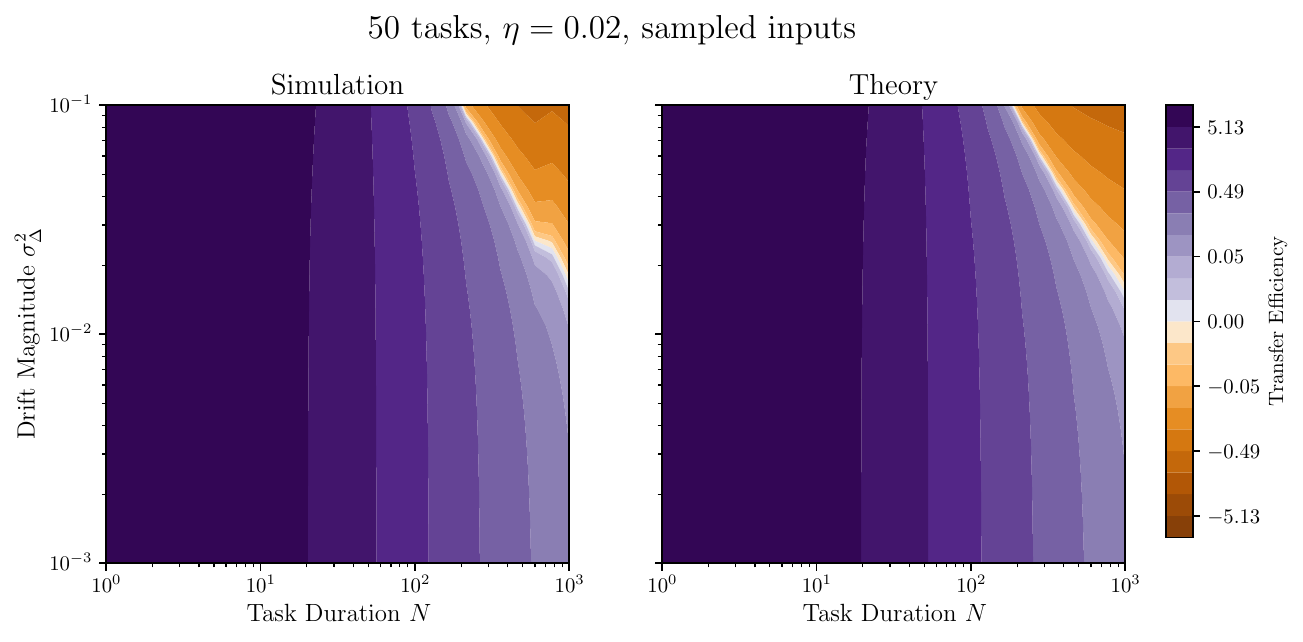}
    \caption{\textbf{Sampled-input simulation vs.\ the closed-form maps.} Even at the largest step size (left, $\eta = 0.1$: the worst case of \cref{tab:assumption-noise}) the sign structure and the zero-TE boundary of the simulated map match the theory; at $\eta = 0.02$ (right) the pointwise agreement is restored as well.}
    \label{fig:assumption-noise}
\end{figure}

In summary, none of the three ingredients is load-bearing for the phase diagram of \cref{fig:sim-shallow}. Misalignment leaves the transfer boundary in place, and costs nothing in accuracy once the relaxed expressions of \cref{sec:beyond_alignment} are used;
\rev{structural rank-deficiency deforms the boundary from duration-limited to drift-limited, in exact agreement with}
\cref{sec:ridgeless}; and the state-dependence of the gradient noise moves pointwise values by $\mathcal{O}(\eta \operatorname{Tr}(H)/B)$ and the boundary not at all, vanishing in the small-step limit in which the theory is derived.
}
\subsubsection{Scaling Deep Linear Networks Experiments}

Figures~\ref{fig:sim-deep-lazy} and~\ref{fig:sim-deep-rich} expand the deep-linear study by varying hidden width while holding the data-generating process fixed. These experiments isolate a key message from the optimization-bias analysis: even when the realizable input-output map is linear, the training parametrization changes the TE landscape. Increasing width does not simply make JTL and ITL equivalent. Instead, the lazy and rich regimes trace different families of TE frontiers
\rev{because they induce different optimization paths, effective conditioning, and amounts of Instability after joint training.}

The lazy-regime simulations in Figure~\ref{fig:sim-deep-lazy} behave most like the fixed-feature or neural-tangent picture. Width increases capacity and smooths some finite-width effects, but the qualitative transfer boundary remains governed by the same tradeoff as in the shallow theory: JTL reduces transient optimization cost by reusing past data, while it pays an Instability cost when the current task has drifted too far from the historical average.
\rev{Positive TE therefore persists only where the transient savings dominate}
the mismatch induced by jointly fitting previous tasks.

The rich-regime simulations in Figure~\ref{fig:sim-deep-rich} show that feature learning changes this balance. Because the representation can move during training, the model is not constrained to interpolate between tasks in the same way as in the lazy regime. This can improve transfer when the tasks share useful structure, but it can also amplify path dependence when the representation learned for earlier tasks is poorly aligned with later ones.
\rev{The comparison of the two grids therefore supports the claim that TE depends not only on task drift and sample count but also on the optimizer-induced}
solution selected within an overparameterized model class.

\begin{figure}[h]
    \centering
    \includegraphics[width=0.35\linewidth]{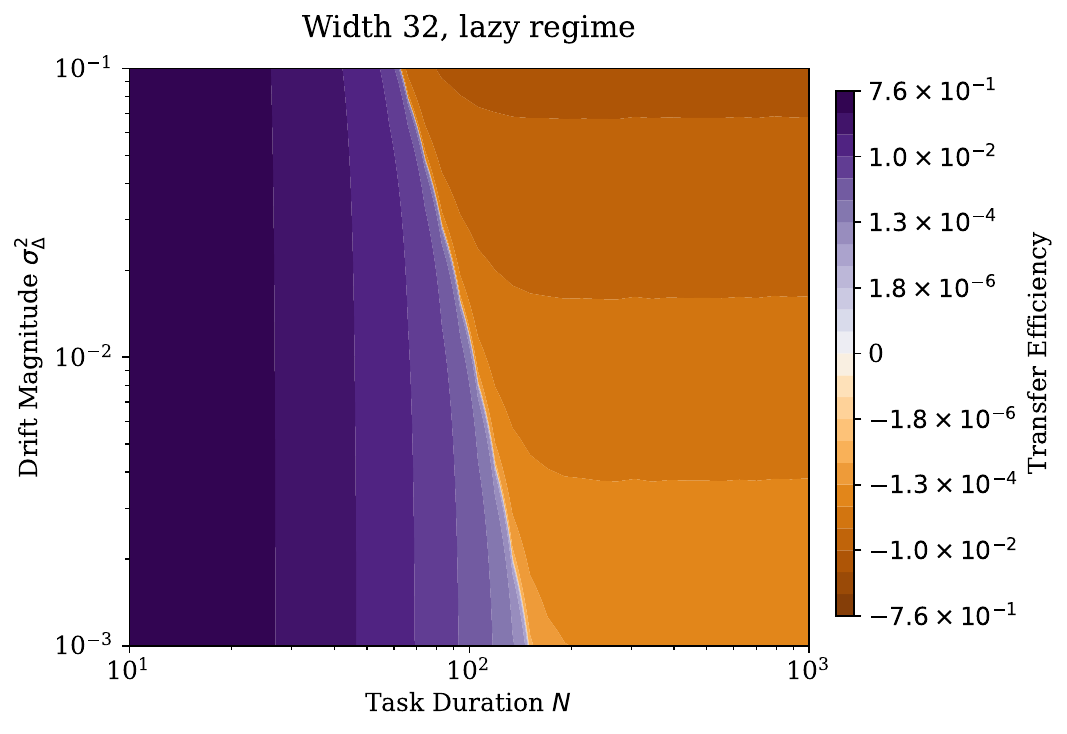}
    \includegraphics[width=0.35\linewidth]{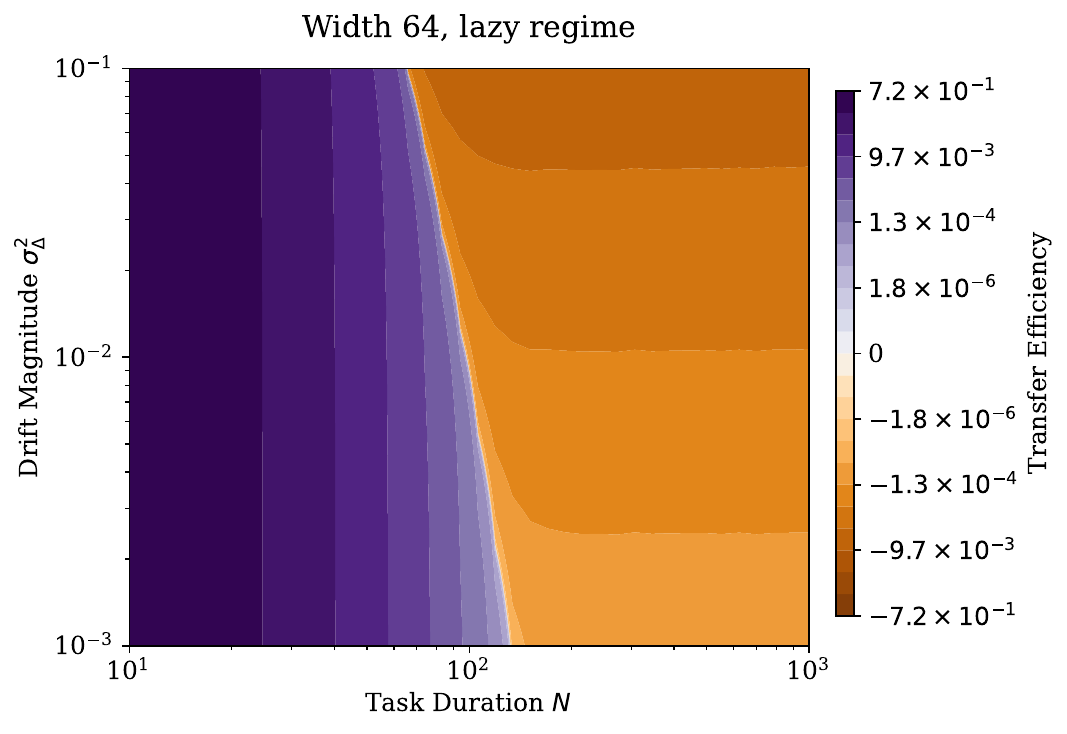}
    \includegraphics[width=0.35\linewidth]{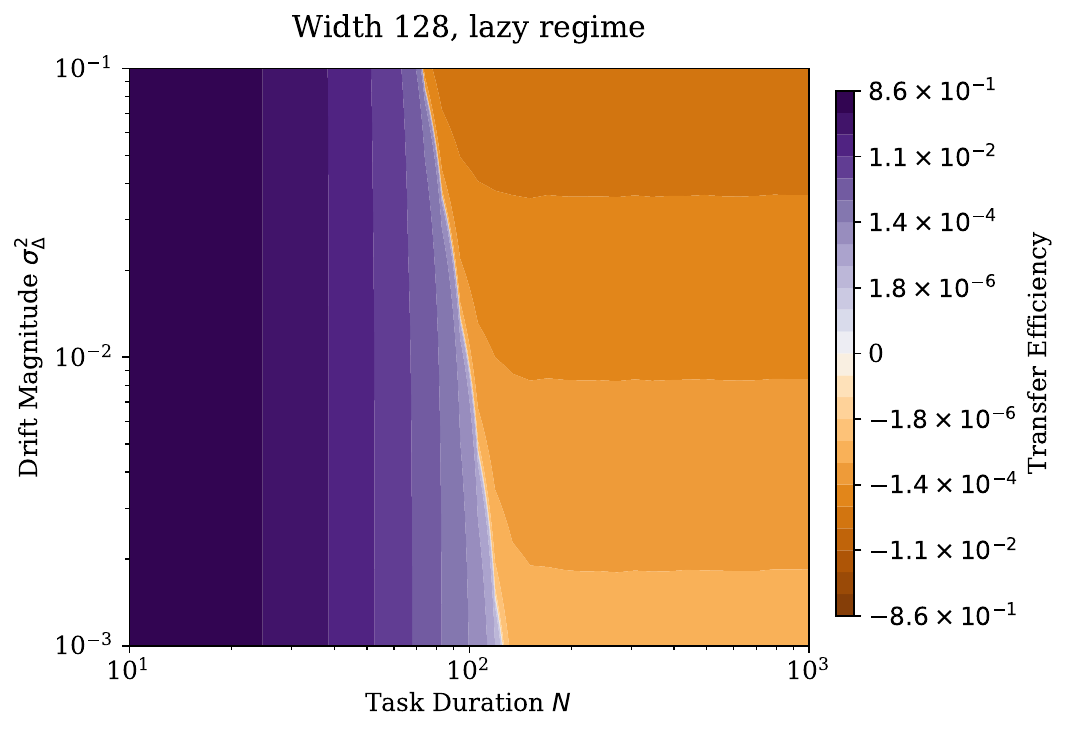}
    \includegraphics[width=0.35\linewidth]{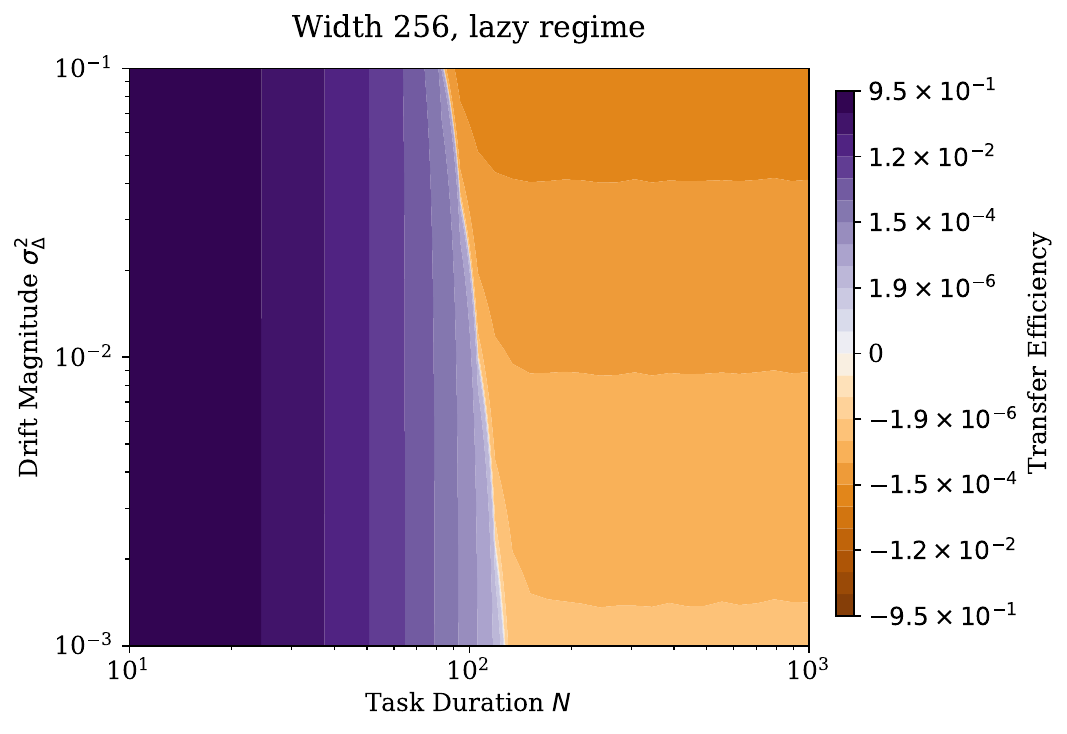}
    \includegraphics[width=0.35\linewidth]{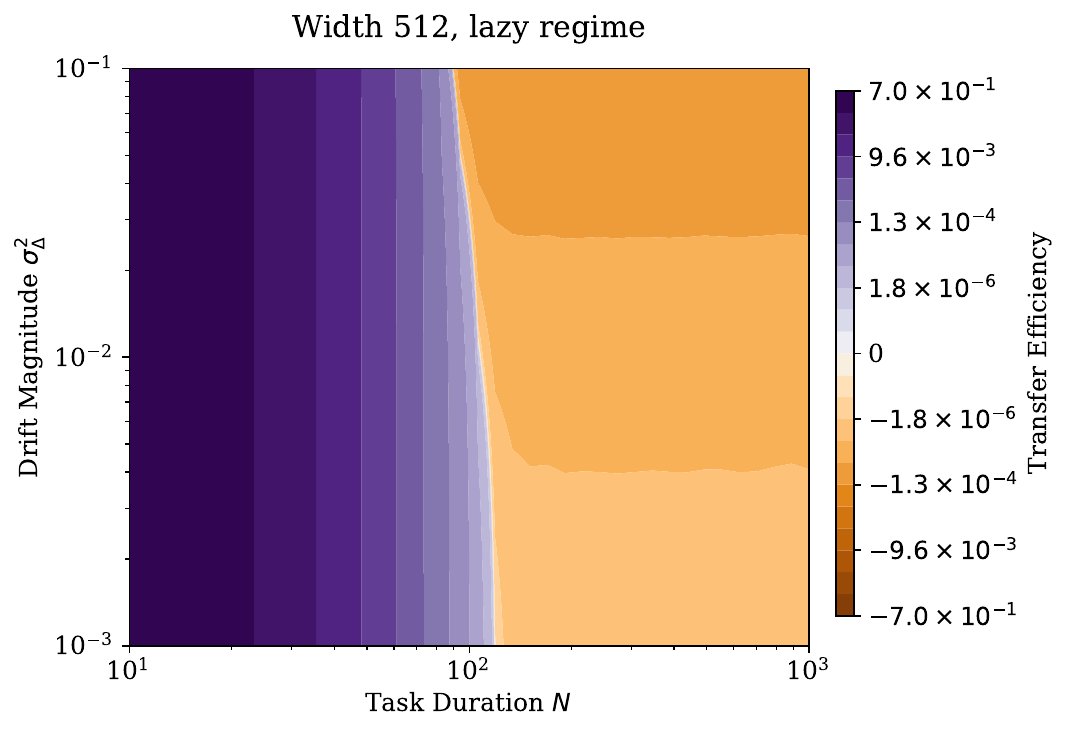}
    \includegraphics[width=0.35\linewidth]{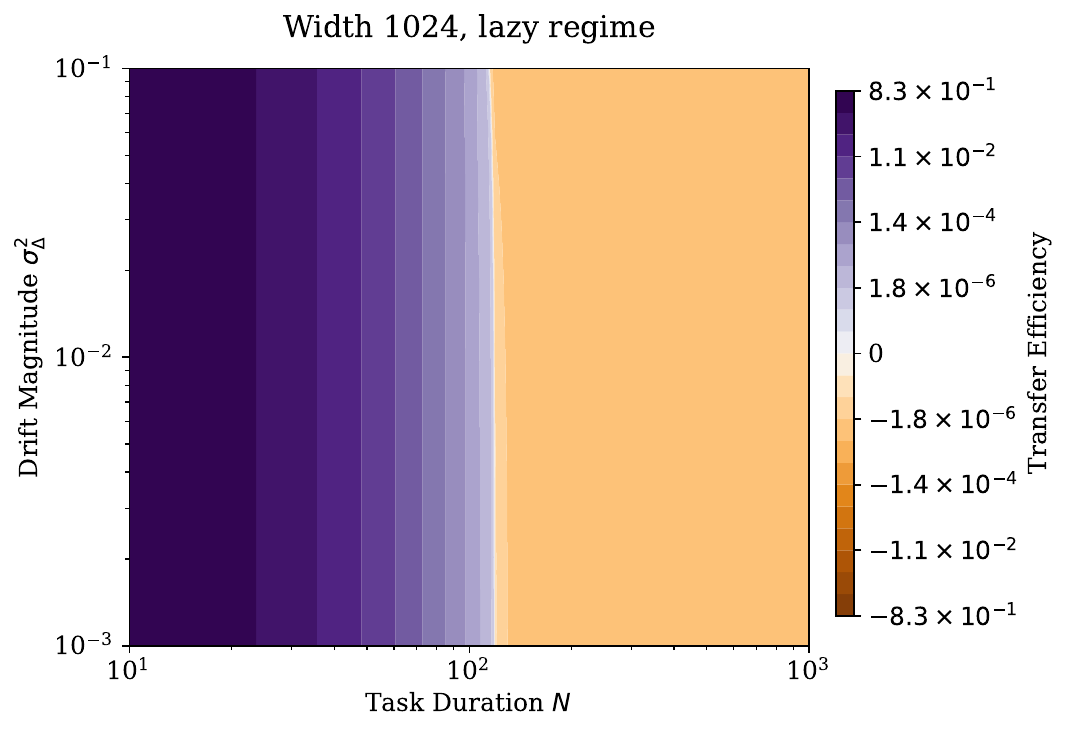}
    \caption{\textbf{Lazy regime simulations.} From left to right the \textbf{model width} is increased from $32$ to $1024$
\rev{($2^5$-fold).}}
    \label{fig:sim-deep-lazy}
\end{figure}

\begin{figure}[h]
    \centering
    \includegraphics[width=0.35\linewidth]{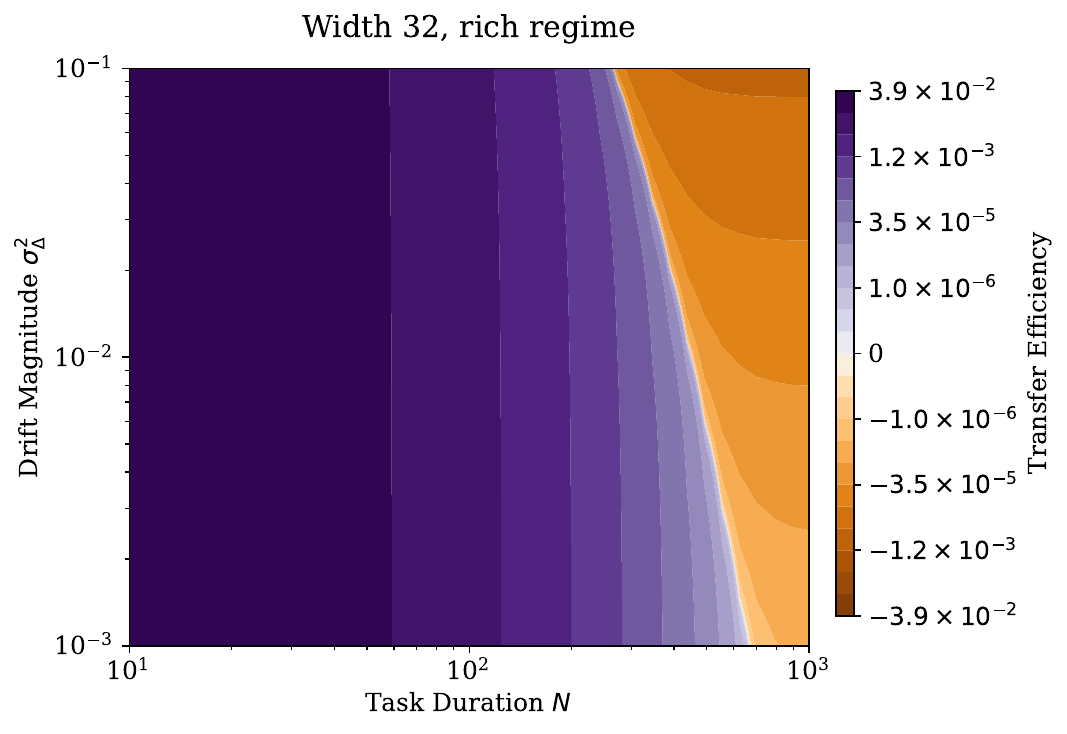}
    \includegraphics[width=0.35\linewidth]{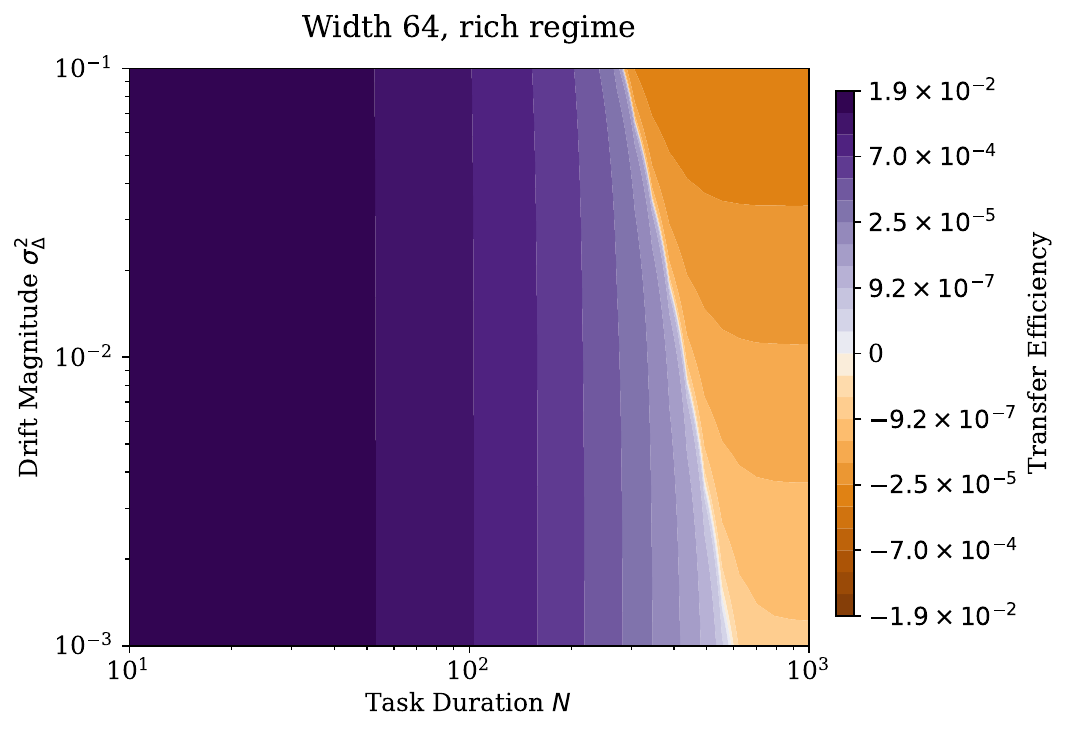}
    \includegraphics[width=0.35\linewidth]{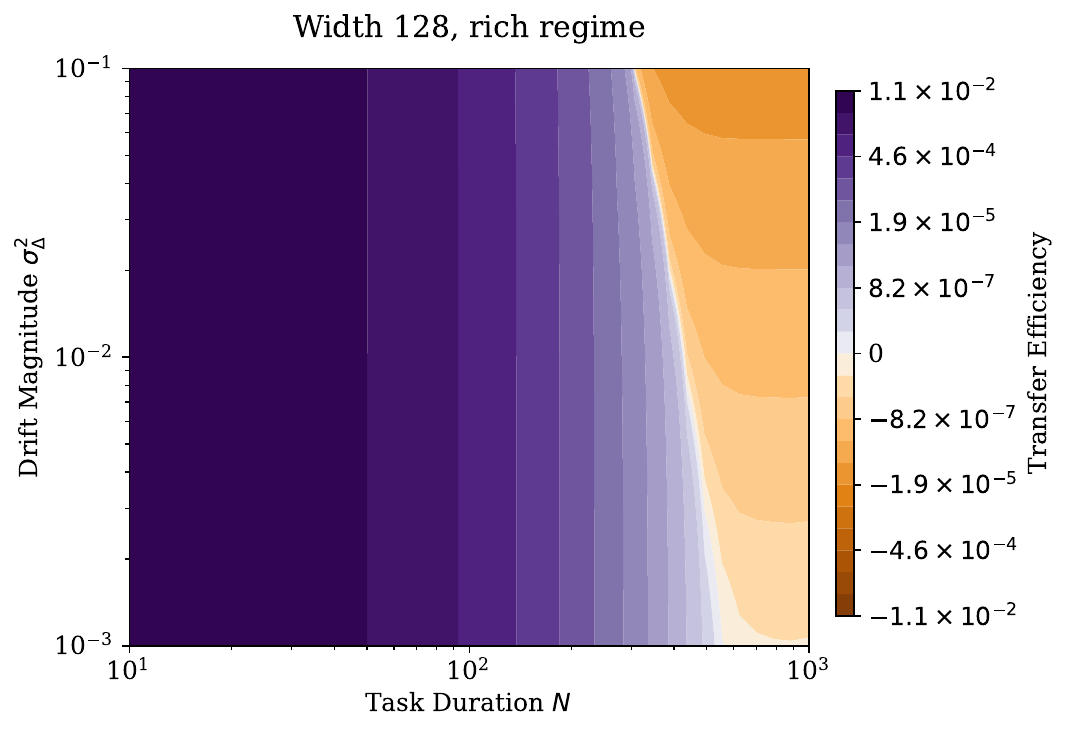}
    \includegraphics[width=0.35\linewidth]{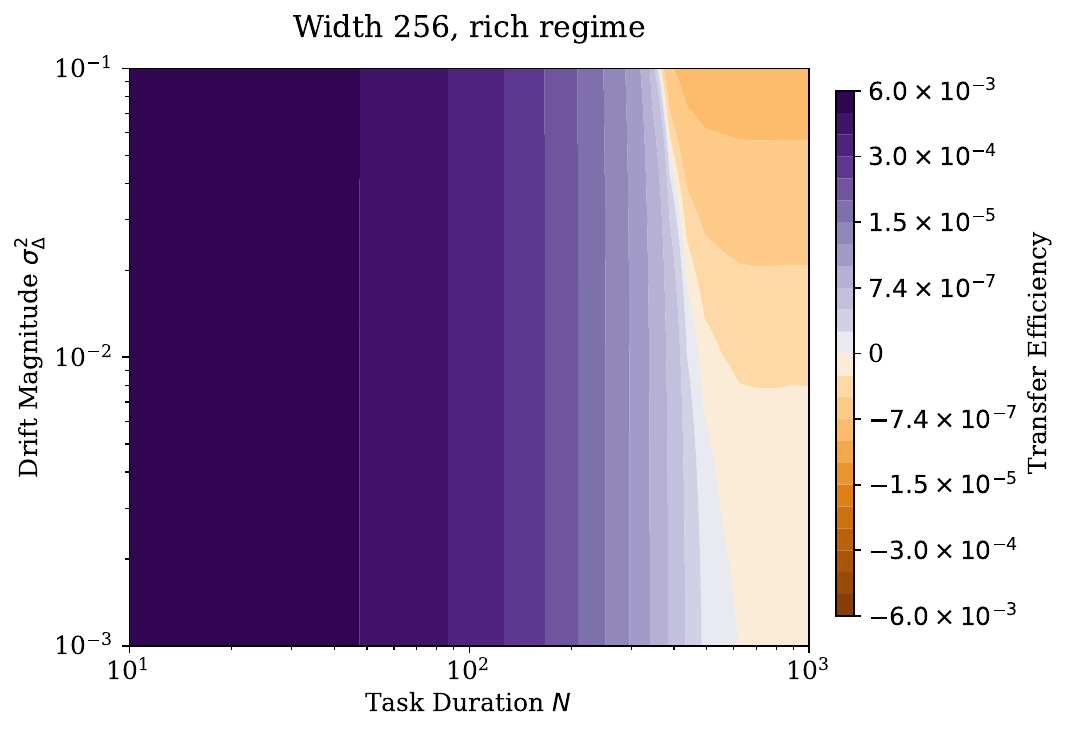}
    \includegraphics[width=0.35\linewidth]{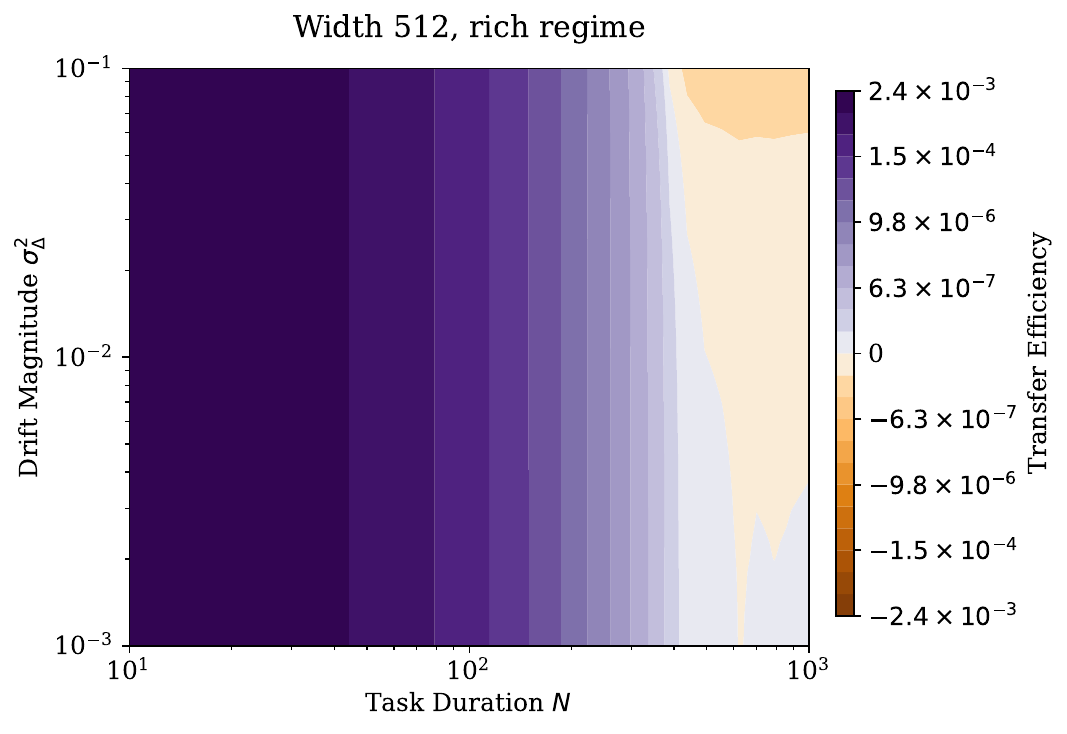}
    \includegraphics[width=0.35\linewidth]{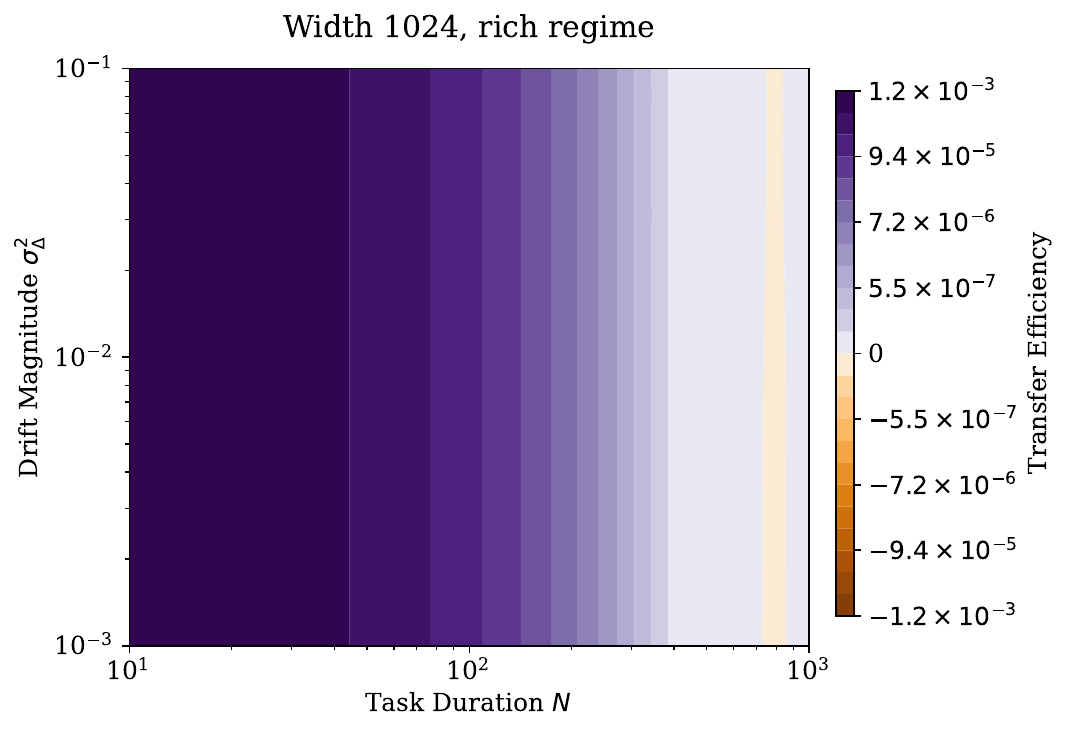}
    \caption{\textbf{Rich regime simulations.}  From left to right the \textbf{model width} is increased from $32$ to $1024$
\rev{($2^5$-fold).}}
    \label{fig:sim-deep-rich}
\end{figure}

\begin{figure}[h]
    \centering
    \begin{minipage}{0.4\textwidth}
        \centering
        \includegraphics[width=\linewidth]{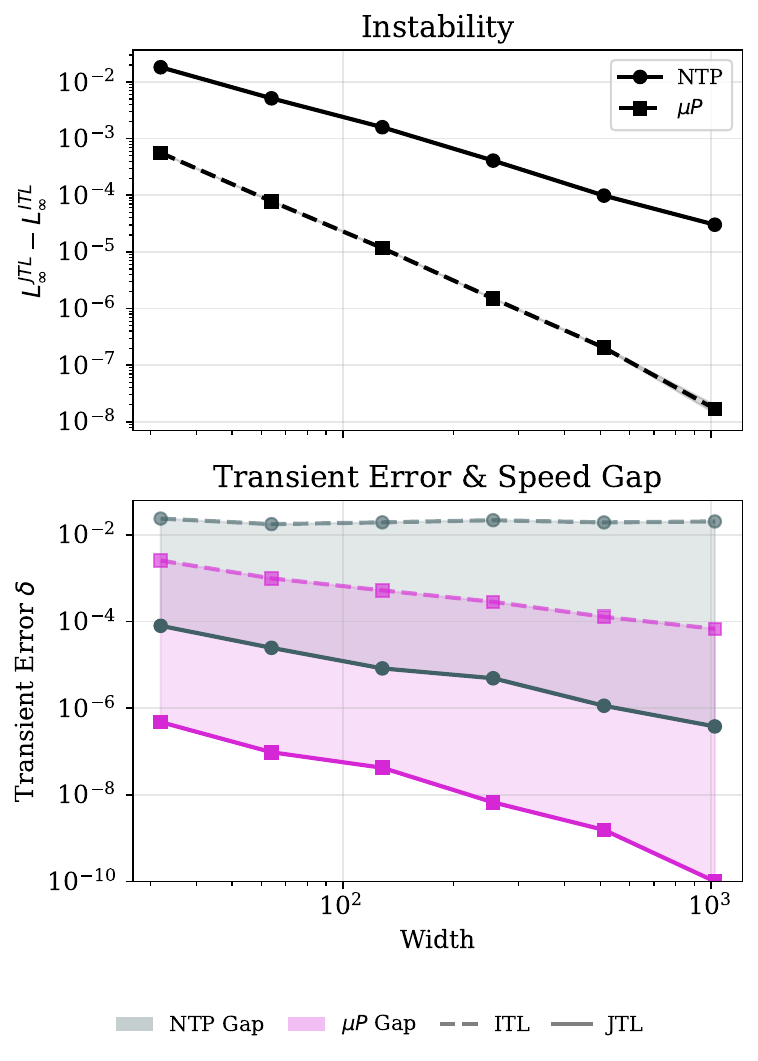}
    \end{minipage}
    \hspace{0.15cm}
    \begin{minipage}{0.5\textwidth}
    \caption{\textbf{Scaling of Instability and Transient Error with Network Width.} Empirical measurements taken in the deep linear network setup at a fixed task duration $N=1000$ and drift magnitude $\sigma_\Delta^2=0.05$. \emph{Left:} Instability (approximated as the average loss at convergence) decreases as network width scales up,
\rev{but remains several orders of magnitude higher under the NTP parametrization (black circles) than under $\mu$P (black squares).}
\emph{Right:} Decomposition of the Transient Error, with NTP shown in gray and $\mu$P in pink.
\rev{For both parametrizations}
and across all widths, the JTL Transient Error (solid lines) is orders of magnitude lower than the ITL Transient Error (dashed lines).
\rev{The shaded areas mark the Transient Error gap, confirming that Transfer Efficiency is bottlenecked by Instability and ITL Transient Error.}}
    \label{fig:insta-transient-scaling}
    \end{minipage}
\end{figure}

\subsubsection{Real-time training trajectories}

The main text reports TE after averaging over tasks and seeds. The trajectories below show where those averages come from in time.
\rev{They separate transient effects}
at the beginning of each task from persistent bias or interference accumulated across the sequence.

\begin{figure}[hb]
    \centering
    \includegraphics[width=0.7\linewidth]{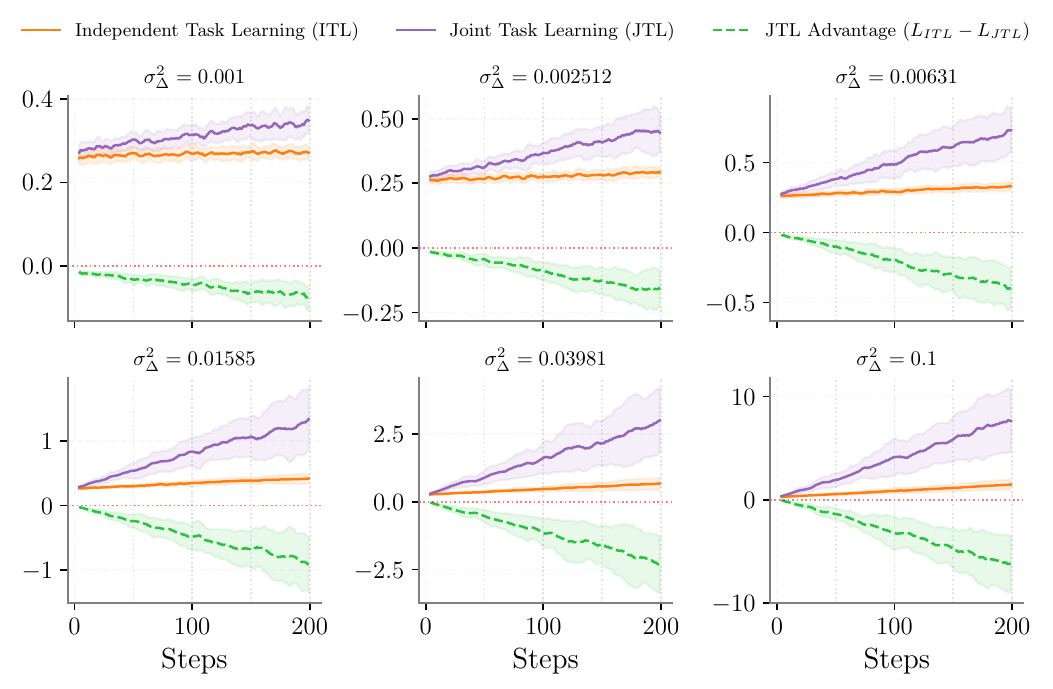}\\[-0.2cm]
    \caption{\textbf{Continuous-drift shallow-linear trajectories.} ITL and JTL losses are shown across task boundaries for several drift magnitudes, together with the instantaneous JTL advantage. Small drift produces positive transfer because JTL reduces repeated transient learning costs; larger drift erodes this advantage as the joint solution becomes increasingly biased toward obsolete tasks.}
    \label{fig:MT_Drift_Analysis}
\end{figure}
\vspace{-2mm}
Figure~\ref{fig:MT_Drift_Analysis} is the trajectory-level analogue of the continuous-drift phase diagrams. At small drift, the history remains predictive of the current optimum, so JTL starts each task closer to a useful solution and the advantage is mostly positive. As drift increases, the same history becomes stale: the transient benefit remains visible near some task changes, but it is eventually outweighed by the stationary error of fitting too many past tasks jointly.

\begin{figure}[h]
    \centering
    \includegraphics[width=0.85\linewidth]{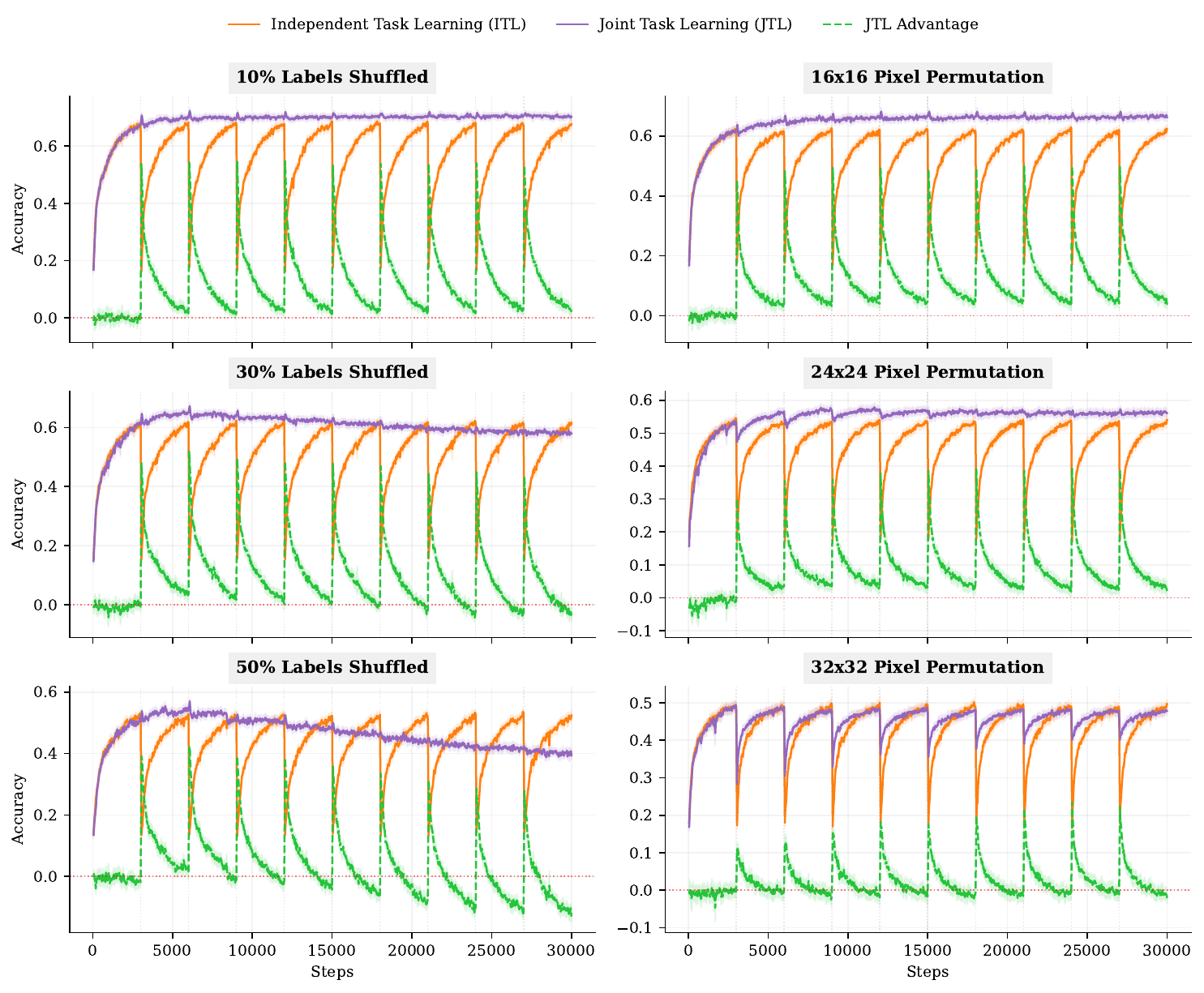}\\[-0.3cm]
    \caption{\textbf{Cross-benchmark comparison at task duration $N=3000$.} Real-time performance trajectories of JTL and ITL agents across the supervised image-classification benchmarks at a fixed task duration. See \cref{Apdx:datasets-networks-conf} for hyperparameter details.}
    \label{fig:benchmark_comparison_n3000}
\end{figure}

The supervised trajectories in Figure~\ref{fig:benchmark_comparison_n3000} make the same decomposition visible on real benchmarks. CLEAR is the closest empirical analogue of low drift: adjacent tasks share visual and semantic structure, so JTL can reuse previous training. MD5 and the synthetic CIFAR variants create sharper input or label shifts, making the historical data less aligned with the current task and increasing the Instability term.

\begin{figure}[h]
    \centering
    \includegraphics[width=0.49\linewidth]{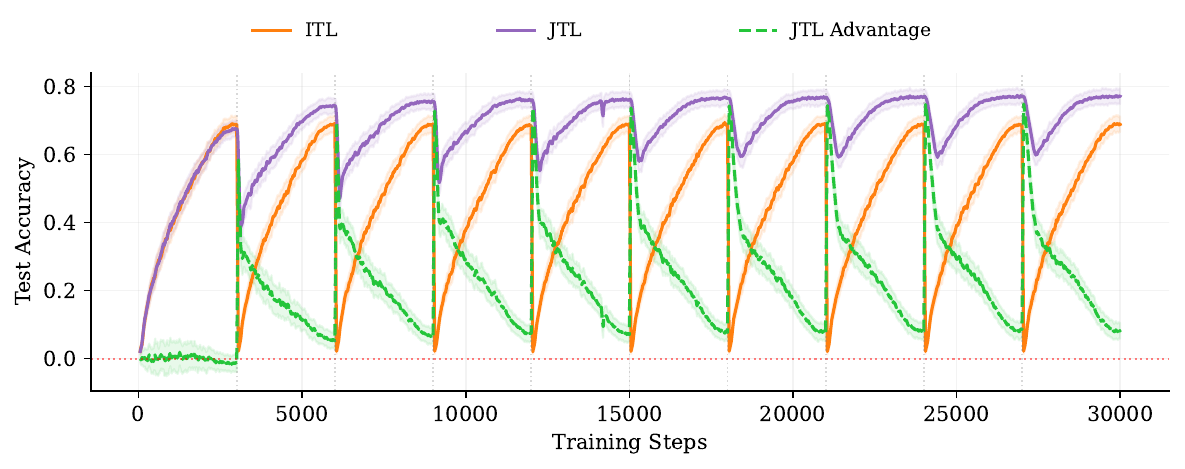}\hfill
    \includegraphics[width=0.49\linewidth]{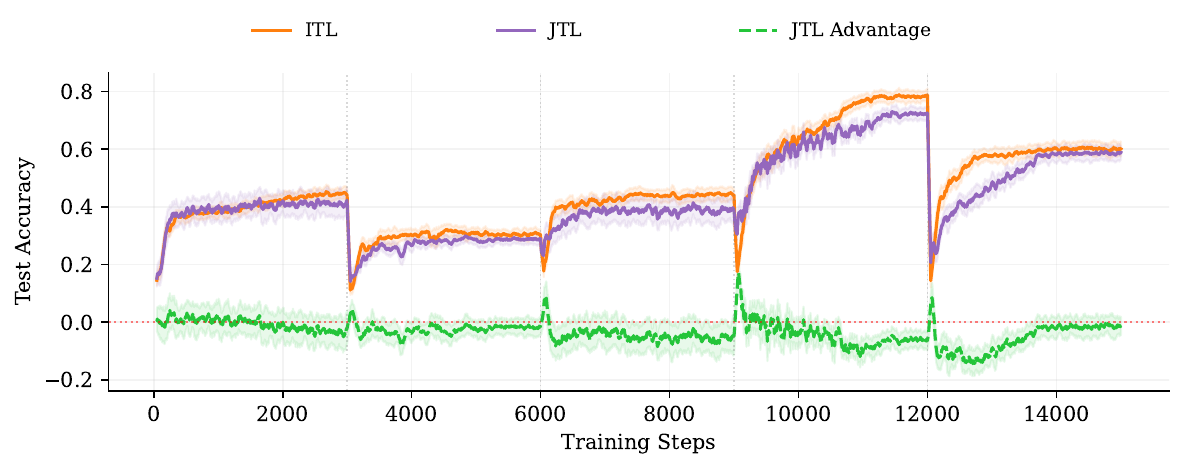}
    \caption{\textbf{Real-time training dynamics on CLEAR (left) and MD5 (right) at task duration $N=3000$.} Per-task evaluation accuracy of JTL and ITL agents across the full task sequence. The two benchmarks exhibit qualitatively different transfer regimes: smooth temporal drift in CLEAR (where JTL transfer is positive) versus sharp semantic shifts in MD5 (where JTL transfer becomes detrimental).\vspace{-3mm}}
    \label{fig:clear-md5-dynamics-n3000}
\end{figure}

Figure~\ref{fig:clear-md5-dynamics-n3000} highlights the contrast most directly. On CLEAR, the JTL trajectory tends to remain competitive or ahead because the stream changes gradually. On MD5, where each task corresponds to a more abrupt semantic change, JTL is frequently pulled toward earlier classes and ITL benefits from specializing to the current task. This is the empirical counterpart of the sign change in TE as Instability grows.

\begin{figure}[hb]
    \centering
    \includegraphics[width=0.65\linewidth]{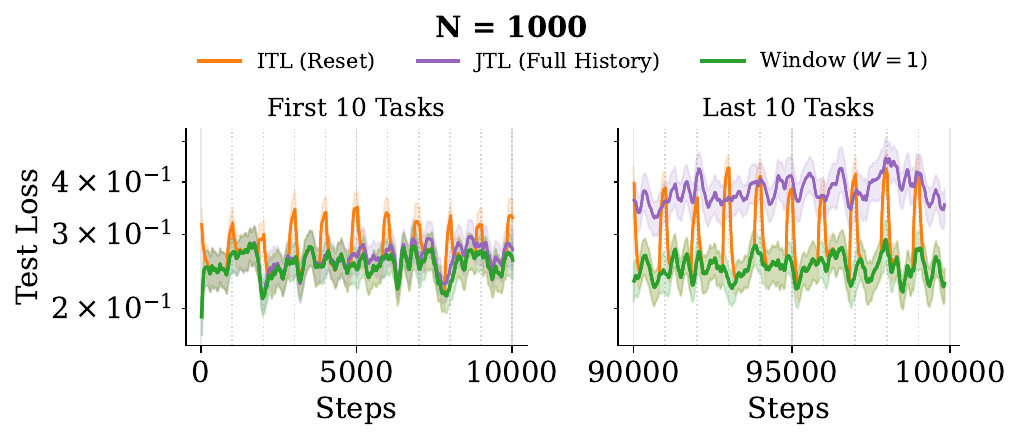}
    \includegraphics[width=0.6\linewidth]{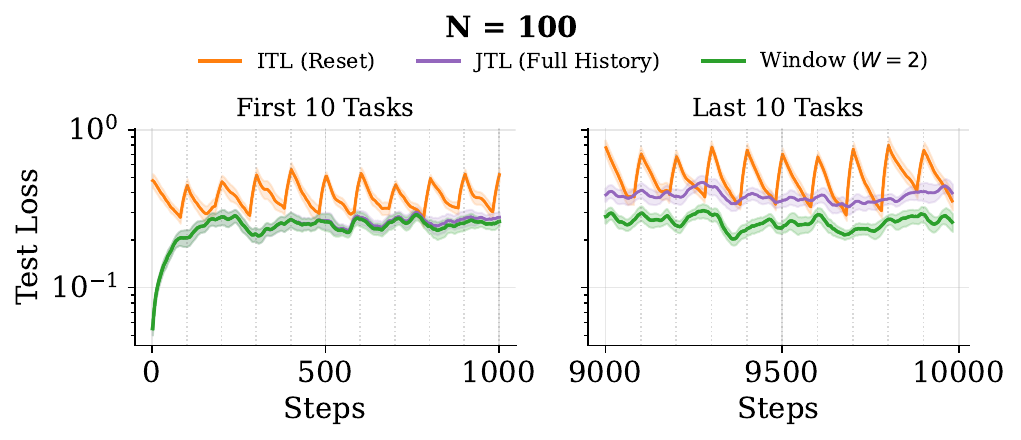}

    \includegraphics[width=0.6\linewidth]{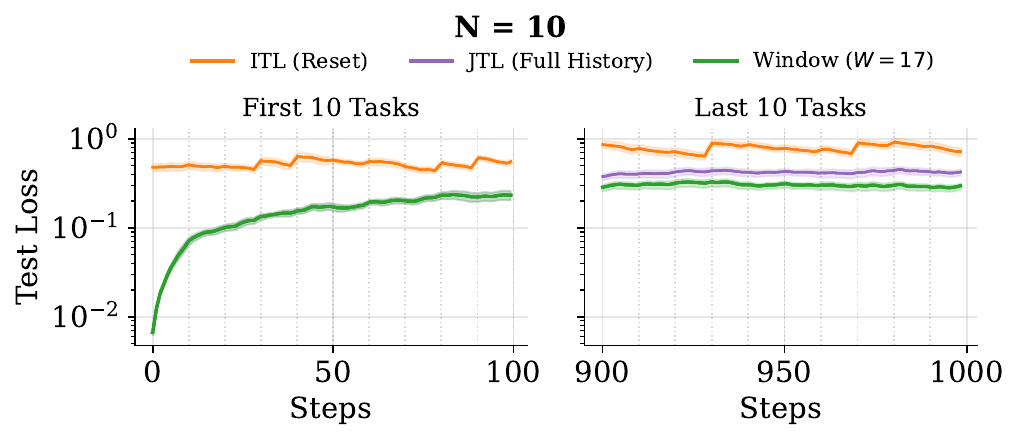}

    \caption{\textbf{Window-agent trajectory zooms.} The first and last ten tasks are shown for several task durations.
\rev{The Window agent keeps only a finite recent history,}
interpolating between ITL, which uses only the current task, and JTL, which retains the entire stream.\vspace{-3mm}}
    \label{fig:window_dynamics_zoom}
\end{figure}

\rev{The Window-agent zooms in Figure~\ref{fig:window_dynamics_zoom}}
illustrate why a finite memory can outperform both extremes. Early in training,
\rev{the Window and JTL agents are similar}
because little history exists. Later, the finite window discards obsolete tasks, retaining the transient benefit of recent replay while reducing the bias from distant tasks. This behavior matches the theoretical picture: the best history length grows when tasks are short or drift is slow, and shrinks when old tasks become poor predictors of the current optimum.

\end{document}